\documentclass[10pt,twoside,journal,compsoc]{IEEEtran}
\ifCLASSOPTIONcompsoc
  \usepackage[nocompress]{cite}
\else
  \usepackage{cite}
\fi
\ifCLASSINFOpdf
  \usepackage[pdftex]{graphicx}
\else
  \usepackage[dvips]{graphicx}
\fi
\usepackage[utf8]{inputenc} 
\usepackage[T1]{fontenc}    
\usepackage{hyperref}       
\usepackage{url}            
\usepackage{booktabs}       
\usepackage{amsfonts}       
\usepackage{nicefrac}       
\usepackage{microtype}      
\usepackage{algorithm}
\usepackage{subfigure}
\usepackage{listings}
\usepackage{wrapfig}
\usepackage{booktabs}
\usepackage{multirow}
\usepackage{amsmath}
\usepackage{mathrsfs}
\usepackage{cases}
\usepackage{threeparttable}
\usepackage{amsthm}
\usepackage{color}
\usepackage{pythonhighlight}
\newtheorem{Theorem}{Theorem}
\newtheorem{Lemma}{Lemma}

\usepackage{amssymb}
\usepackage{xcolor}
\usepackage{bm}

\usepackage{algorithmic}
\newtheorem{theorem}{Theorem}

\begin{document}
	
%
\title{CMW-Net: Learning a Class-Aware Sample Weighting Mapping
	for Robust Deep Learning }
%
%
%
%

\author{Jun~Shu,
        Xiang Yuan,
       Deyu Meng,
        and Zongben Xu
\IEEEcompsocitemizethanks{\IEEEcompsocthanksitem Jun Shu, Xiang Yuan, Deyu Meng and Zongben Xu are with School of Mathematics and Statistics and Ministry of Education Key Lab of Intelligent Networks and Network Security, Xi'an Jiaotong University, Shaanxi, P.R.China.  
	Email: {xjtushujun,relojeffrey}@gmail.com, {dymeng,zbxu}@mail.xjtu.edu.cn.
	\IEEEcompsocthanksitem Deyu Meng and Zongben Xu are also with Peng Cheng Laboratory, Shenzhen, Guangdong, China. Deyu Meng is also with the Macau Institute of Systems Engineering, Macau University of Science and Technology, Taipa, Macau, China.
	\IEEEcompsocthanksitem  Corresponding author: Deyu Meng.
}}



%
%

\markboth{IEEE Transactions on Pattern Analysis and Machine Intelligence}
{IEEE Transactions on Pattern Analysis and Machine Intelligence}
%



\IEEEtitleabstractindextext{%
\begin{abstract}
Modern deep neural networks (DNNs) can easily overfit to biased training data containing corrupted labels or class imbalance. Sample re-weighting methods are popularly used to alleviate this data bias issue. Most current methods, however, require manually pre-specifying the weighting schemes as well as their additional hyper-parameters relying on the characteristics of the investigated problem and training data. This makes them fairly hard to be generally applied in practical scenarios, due to their significant complexities and inter-class variations of data bias situations. To address this issue, we propose a meta-model capable of adaptively learning an explicit weighting scheme directly from data. Specifically, by seeing each training class as a separate learning task, our method aims to extract an explicit weighting function with sample loss and task/class feature as input, and sample weight as output, expecting to impose adaptively varying weighting schemes to different sample classes based on their own intrinsic bias characteristics. Synthetic and real data experiments substantiate the capability of our method on achieving proper weighting schemes in various data bias cases, like the class imbalance, feature-independent and dependent label noise scenarios, and more complicated bias scenarios beyond conventional cases. Besides, the task-transferability of the learned weighting scheme is also substantiated, by readily deploying the weighting function learned on relatively smaller-scale CIFAR-10 dataset on much larger-scale full WebVision dataset. A performance gain can be readily achieved compared with previous state-of-the-art ones without additional hyper-parameter tuning and meta gradient descent step. The general availability of our method for multiple robust deep learning issues, including partial-label learning, semi-supervised learning and selective classification, has also been validated. Code for reproducing our experiments is available at \url{https://github.com/xjtushujun/CMW-Net}.
\end{abstract}

}

\maketitle
\IEEEdisplaynontitleabstractindextext

%
\IEEEpeerreviewmaketitle

\IEEEraisesectionheading{\section{Introduction}\label{introduction}}

\IEEEPARstart{D}{eep} neural networks (DNNs), equipped with highly parameterized structures for modeling complex input patterns, have recently obtained impressive performance on various applications, e.g., computer vision \cite{he2016deep}, natural language processing \cite{devlin2019bert}, speech processing \cite{abdel2014convolutional}, etc. These successes largely attribute to many large-scale paired sample-label datasets expected to properly and sufficiently simulate the testing/evaluating environments. However, in most real applications, collecting such large-scale supervised datasets is notoriously costly, and always highly dependent on a rough crowdsourcing system or search engine. This often makes the training datasets error-prone, with unexpected data bias from the real testing distributions.

This distribution mismatch issue could have many different forms. For example, the collected training sets are often class imbalanced \cite{cui2019class,liu2019large}. Actually, real-world datasets are usually depicted as skewed distributions. Specifically, the frequency distribution of visual categories in our daily life is generally long-tailed, with a few common classes and many more rare ones. This often leads to a mismatch between collected datasets with long-tailed class distributions for training a machine learning model and our expectation on the model to perform well on all classes.
Another popular data bias is the noisy label case \cite{zhang2018generalized,arazo2019unsupervised,li2019dividemix,shu2019meta}.
Even the most celebrated datasets collected from a crowdsourcing system with expert knowledge \cite{bi2014learning}, like ImageNet, have been demonstrated to contain harmful examples with unreliable labels \cite{recht2019imagenet,northcutt2021pervasive,yun2021re}.
To mitigate the high labeling cost, it has received increasing attention to collect web images by search engines \cite{li2017webvision}. Though cheaper and easier to obtain training data, it often yields inevitable noisy labels due to the error-prone automatic tagging system \cite{shu2018small,guo2018curriculumnet}.

The overparameterized DNNs tend to suffer significantly from overfitting on these biased training data, then conducting their poor performance in generalization. This robust deep learning issue has been theoretically illustrated in multiple literatures \cite{arpit2017closer,zhang2016understanding} and gradually attracted more attention in the field.
Recently, various methods have been proposed to deal with such biased training data. Readers can refer to \cite{zhang2021deep,frenay2013classification,algan2021image,karimi2020deep,song2020learning,han2020survey} for an overall review. In this paper, we focus on the sample re-weighting approach, which is a commonly used strategy against such data bias issue and has been widely investigated started at 1950s \cite{kahn1953methods}.

\vspace{-4mm}

\subsection{Deficiencies of Sample Re-weighting Approach}

The sample re-weighting approach \cite{ren2018learning,shu2019meta} attempts to assign a weight to each example and minimize the corresponding weighted training loss to learn a classifier model. The example weights are typically calculated based on the training loss.
More specifically, the learning methodology of sample re-weighting is to design a weighting function mapping from training loss to sample weight, and then iterates between calculating weights from current sample loss values and minimizing weighted training loss for classifier updating (that's why the method is called ``re-weighting"). However, there exist two entirely contrary ideas for constructing such a loss-weight mapping. In class imbalanced problems, the function is generally set as monotonically increasing, aiming to enforce the learning to more emphasize samples with larger loss values since they are more like to be the minority class. Typical methods include Boosting and AdaBoost \cite{freund1997decision,friedman2000additive}, hard negative mining \cite{malisiewicz2011ensemble} and focal loss \cite{lin2018focal}.
But in noisy label problems, the function is more commonly set as monotonically decreasing, i.e., taking samples with smaller loss values as more important ones, since they are more likely to be high-confident ones with clean labels.
Typical methods include self-paced learning (SPL) \cite{kumar2010self}, iterative reweighting \cite{fernando2003reweight} and multiple variants~\cite{jiang2014easy,jiang2014self,wang2017robust}.

Although these pre-defined weighting schemes have substantiated to help improve the robustness of a learning algorithm on certain data bias scenarios, they still have evident deficiencies in practice. On the one hand, they need to manually preset a specific form of weighting function based on certain assumptions of training data. This, however, tends to be infeasible when we know insufficient knowledge underlying data or the label conditions are too complicated, like the case that the training set is both imbalanced and label-noisy. On the other hand, even when we properly specify certain weighting schemes, like focal loss~\cite{lin2018focal} or SPL \cite{kumar2010self}, they inevitably involve hyper-parameters, like focusing parameter in the former and age parameter in the latter, to be manually preset or tuned by cross-validation. This tends to further raise their application difficulty in real problems.

\vspace{-4mm}
\subsection{Limitations and Meta-Essence Insight for MW-Net} \label{section12}

To alleviate the above issues, our earlier work attempts to parameterize the weighting function as an MLP (multilayer perceptron) network with one hidden layer called MW-Net \cite{shu2019meta}, as depicted in Fig. \ref{fig33a}, which is theoretically capable of dealing with such weighting function approximation problem \cite{csaji2001approximation}. Instead of assuming a pre-defined weighting scheme, MW-Net can automatically learn a suitable weighting strategy from data for the training dataset at hand. Experiments on datasets with class imbalance or noisy labels show that the automatically learned weighting schemes are consistent with the properly defined ones as traditional.

\begin{figure}[t]\vspace{-2mm}
	\centering
	\vspace{-2mm}
	\subfigcapskip=-2mm
	\raisebox{0.5\height}{\subfigure[MW-Net]{\label{fig33a}
		\includegraphics[width=0.20\textwidth]{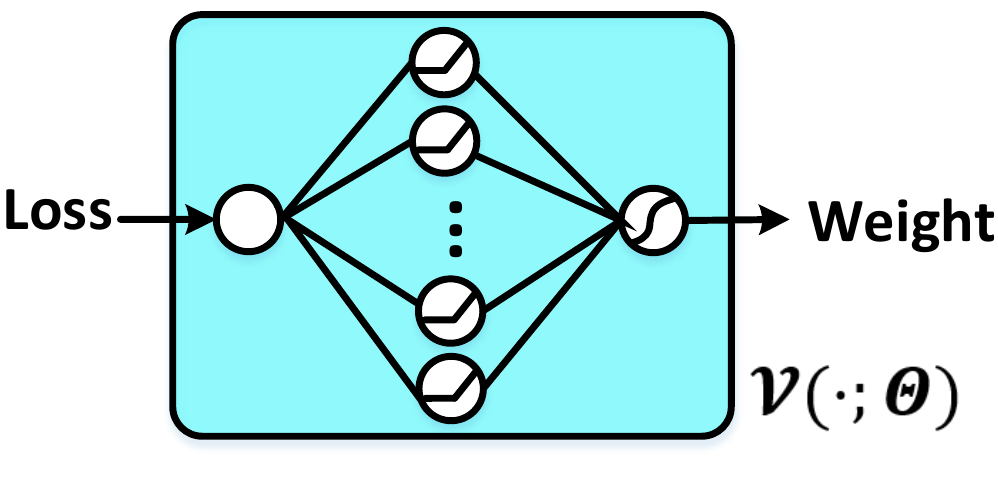}} } \hspace{-4mm}
	\subfigure[CMW-Net]{\label{fig33b}
		\includegraphics[width=0.28\textwidth]{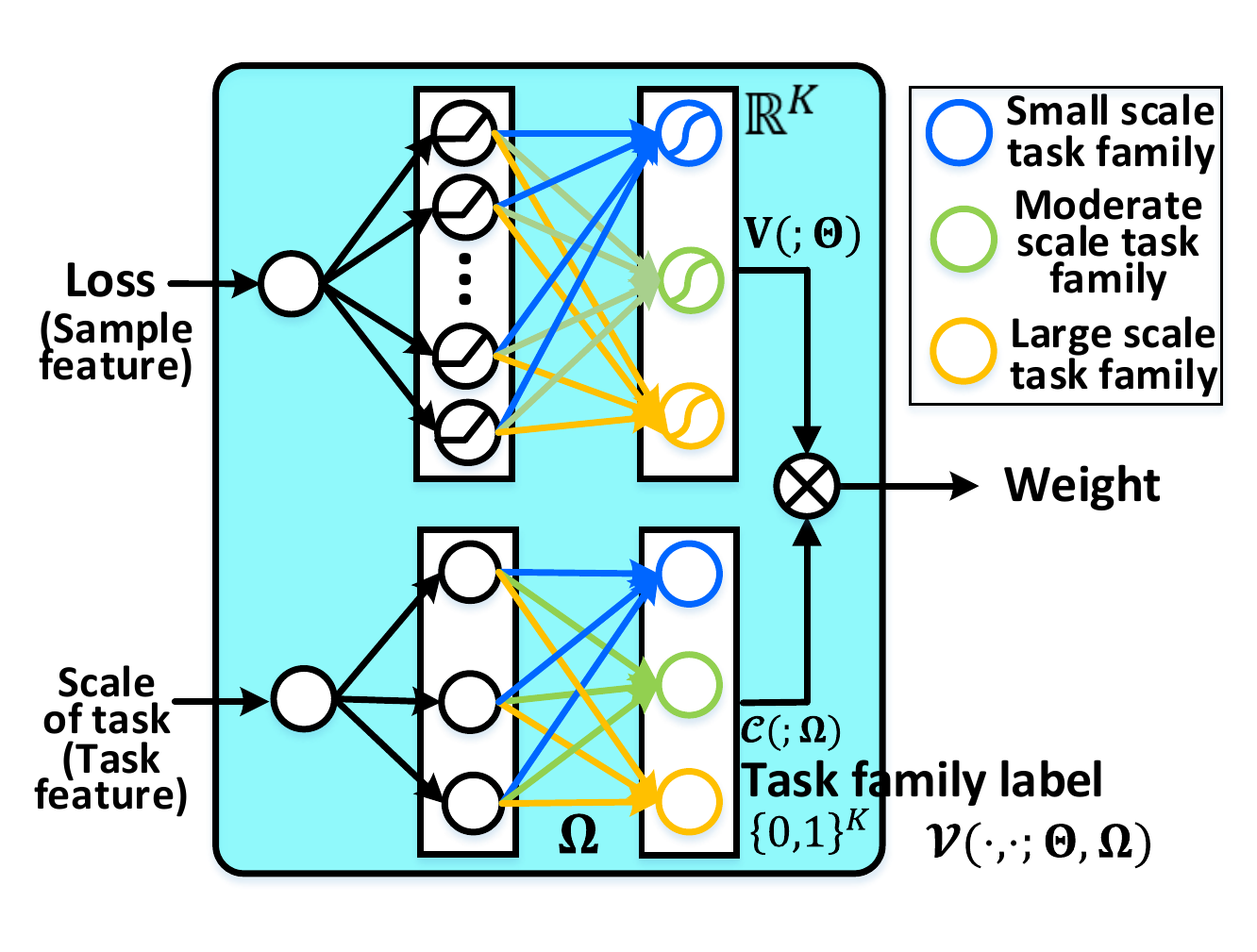}}\vspace{-3mm}
	\caption{The architectures of (a) MW-Net and (b) CMW-Net.}\label{fignet}  
\end{figure}

Nevertheless, MW-Net uses one unique weighting function shared by all classes of training dataset to deal with data bias,
implying that different classes should possess consistent bias. For example, we plot the empirical probability density function (pdf) of training loss\footnote{Since the loss values of all samples have been considered and validated to be important and beneficial information for exploring proper sample weight assignment principle, its distribution largely delivers the underlying bias configurations underlying data \cite{jiang2018mentornet,arazo2019unsupervised,li2019dividemix}.} under the symmetric noise assumption (i.e., biases for each class are with almostly equal possibility), as shown in Fig. \ref{fig11a}. It can be seen that each class shares to an approximately homoscedastic training loss distribution.
However, such homoscedastic bias assumption can not perfectly reflect the real complicated data bias scenarios. In fact, real-world biased datasets (like WebVision \cite{li2017webvision}) are often heterogeneous \cite{collier2021correlated}, i.e., biases are input-dependent, e.g., class-dependent or instance-dependent. Fig. \ref{fig11b} shows the training loss distribution under asymmetric (class-dependent) noise assumption.
We can observe that losses of clean and noisy samples nearly overlap, and thus it is difficult to differentiate noisy samples from clean samples based on loss information. For such class-dependent bias case, MW-Net learns a monotonically increasing weighting function as shown in Fig.\ref{figac}, which implies that MW-Net inclines to significantly lose efficacy. This naturally leads to significant performance degradation for MW-Net (see Table \ref{noisy}). Considering real-world biased datasets always possess even more inter-class heterogenous bias configurations than these simulated biased ones, it is thus fairly insufficient and improper to employ only single weighting function to deal with such complicated real-world biased datasets.

\begin{figure*}[htb]\vspace{-2mm}
	\centering
	\vspace{-2mm}
	\subfigcapskip=-4mm
	\subfigure[Symmetric noise case with noise rate 40\%]{\label{fig11a}
		\includegraphics[width=0.91\textwidth]{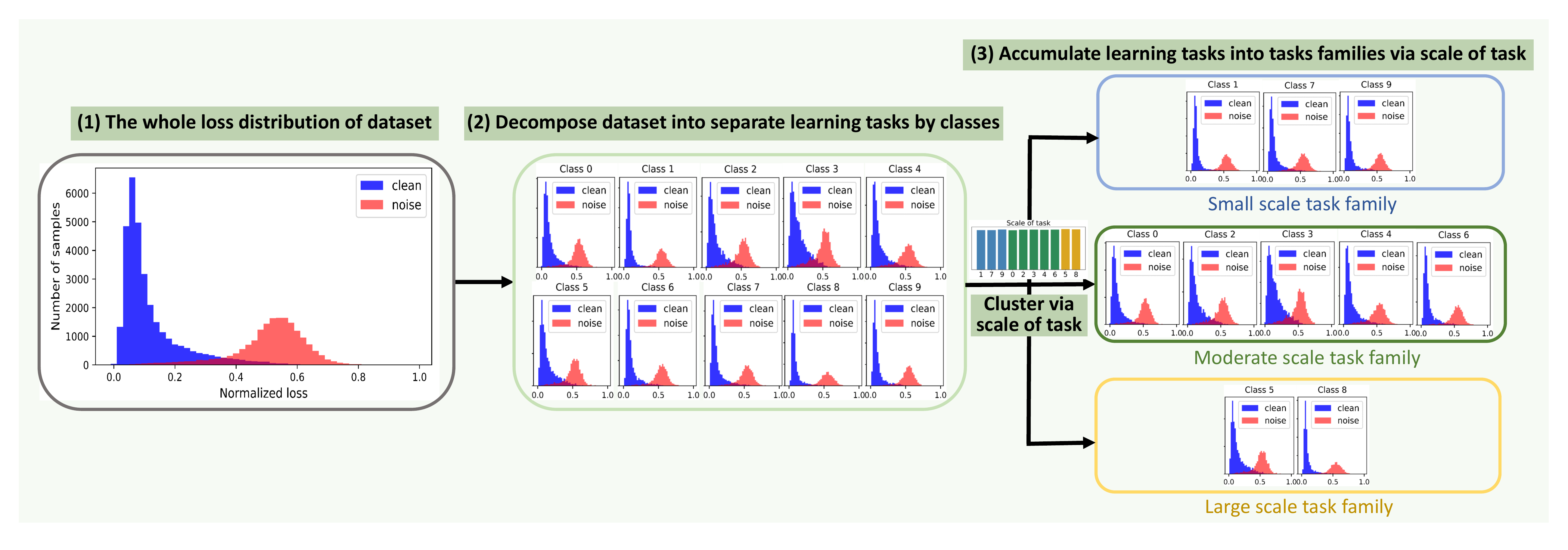}} \\\vspace{-3mm}
	\subfigure[Asymmetric noise case with noise rate 40\%]{\label{fig11b}
		\includegraphics[width=0.91\textwidth]{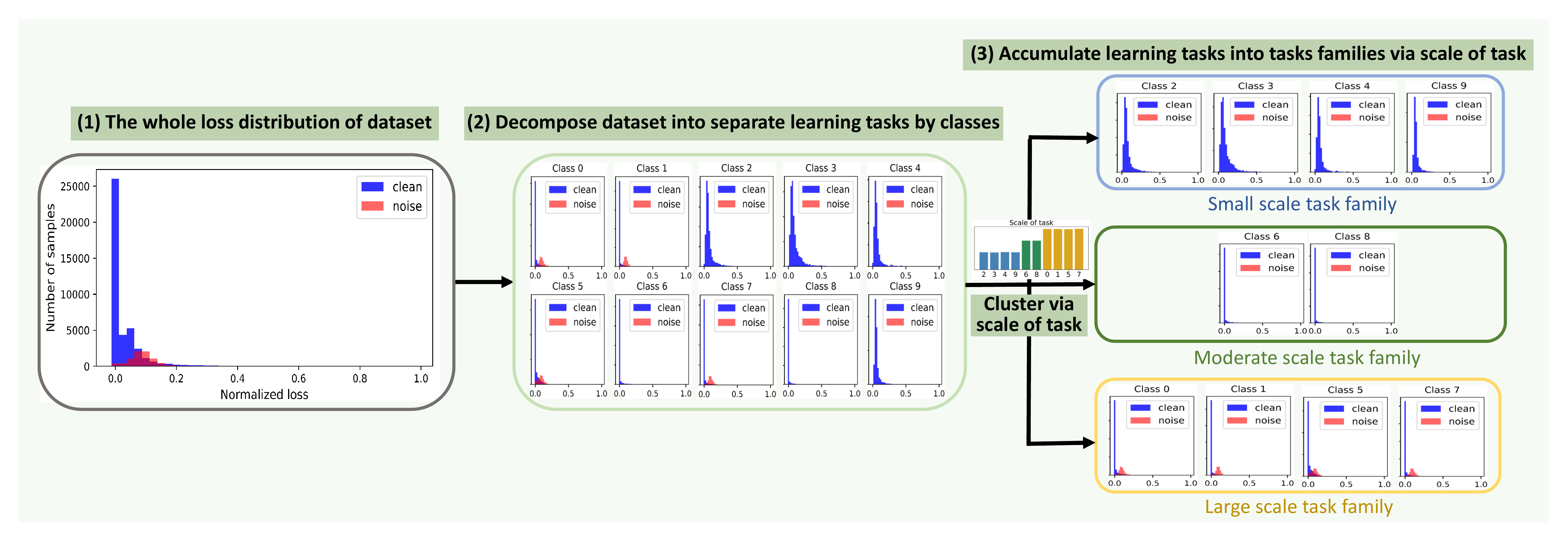}}\vspace{-3mm}
	\caption{Illustration of the limitation and meta-essence understanding for MW-Net. The success of MW-Net is built upon homoscedastic bias assumption (e.g., in (a.1)(a.2), each class has similar loss distributions of clean and noise samples). While MW-Net fails under the heterogeneous bias (e.g., in (b.1)(b.2), each class has their specific loss distributions). The rationality can be revealed from the perspective of meta-learning (see Sec. 1.2). The limitation of MW-Net demonstrates that only sample-level loss information can not sufficiently characterize the heterogeneous bias. This motivates us to introduce task-level information (i.e., scale of task) to reform MW-Net, making it able to distinguish individual bias properties of different tasks, and accumulate tasks with approximately homoscedastic data bias as a task family (see (a.3) and (b.3)). Please see more details in Sec. 1.3 \& 3.3.}   \label{fig11} \vspace{-4mm}
\end{figure*}

This issue can be more intrinsically analyzed under the framework of meta-learning. From the task-distribution view, a meta-learning approach attempts to learn a task-agnostic learning algorithm from a family of training tasks, that is hopeful to be generalizable across tasks and enable new tasks to be learned better and more easily \cite{hospedales2020meta}. By taking every class of training samples as a separate learning task, MW-Net can also be seen as a meta-learning strategy, aiming to learn how to properly impose an explicit weighting function from a set of training classes/tasks. The properness of using such meta-learning regime, however, is built on the premise that all training tasks approximately follow a similar task distribution \cite{baxter2000model,finn2017model,denevi2020advantage,shu2021learning}. In complicated data bias scenarios, however, such premise is evidently hampered by the heterogeneous bias situations across different classes, making MW-Net hardly fit a concise weighting rule generally suitable for all training classes/tasks.

The issue will be more prominent for practical large-scale datasets (e.g., WebVision \cite{li2017webvision}), especially for those containing a large number of training classes but also possessing many rare ones. Then the inter-class heterogeneity will be more significant and the data bias situation more complicated. An easy amelioration is to separately learn a weighting function for each training class to obtain a better flexibility. This easy learning manner, however, is not only impractical due to its required large computation burden, but also easily leads to overfitting and thus hardly extracts available weight schemes from highly insufficient training task information for each class. More importantly, such learning manner is deviated from the original motivation of meta-learning, i.e., learning a general weighting function imposing methodology generalizable and transferable to new biased datasets. The learned weighting scheme is even infeasible to be utilized in new problems with different class numbers and features due to their mismatched input information to the meta-model.

\begin{figure*}[t] \vspace{-0mm}
\centering 
\subfigcapskip=-2mm
\hspace{-0.25cm}
\subfigure[]{
\label{figaa} 
\includegraphics[width=0.24\textwidth]{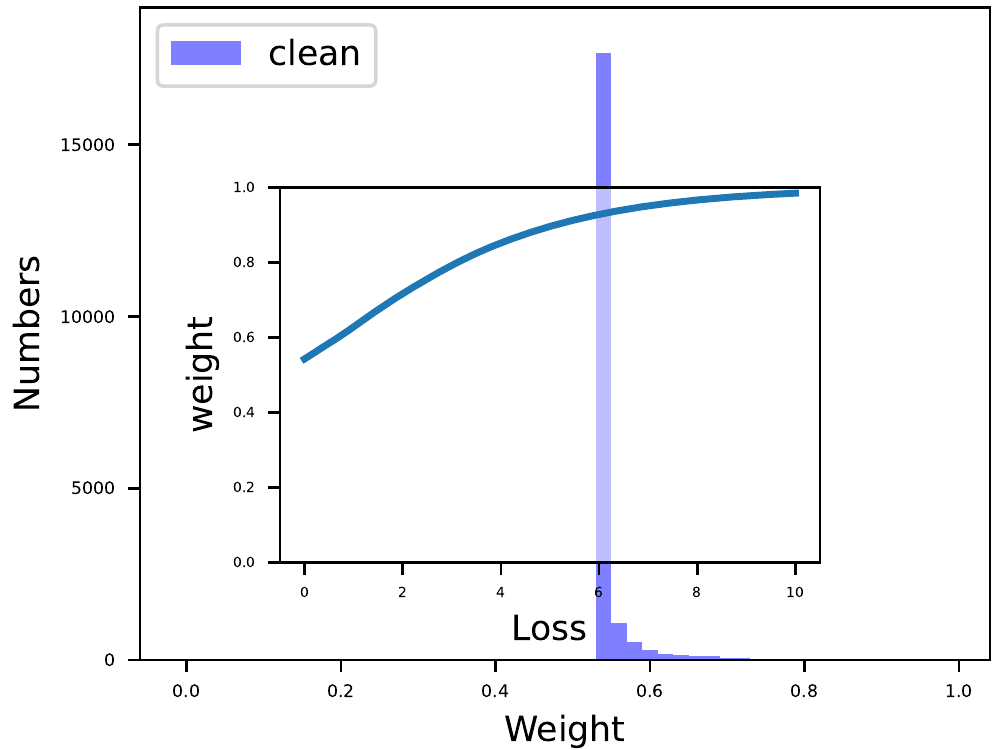}}
\subfigure[]{\label{figab}
\includegraphics[width=0.24\textwidth]{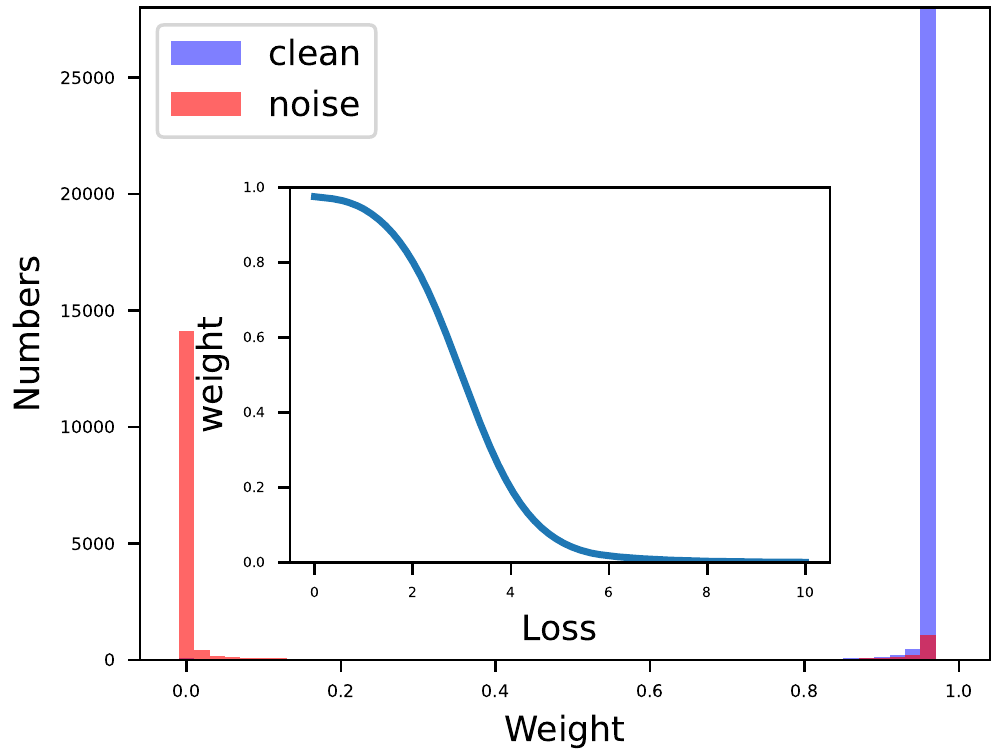}}
\subfigure[]{\label{figac}
\includegraphics[width=0.24\textwidth]{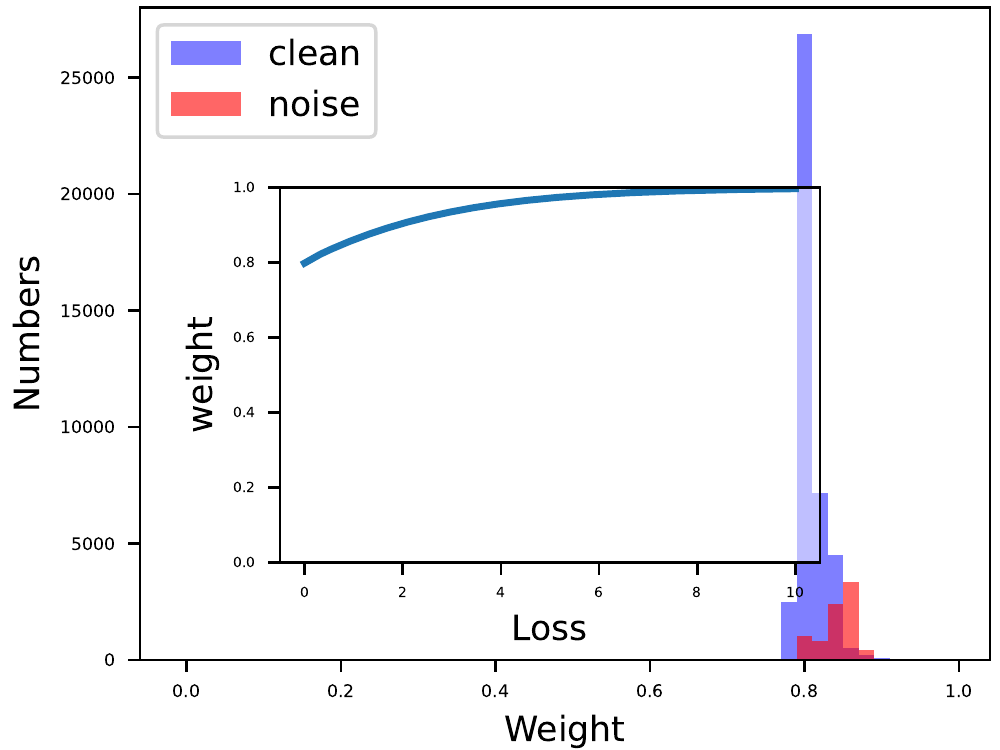}}
\subfigure[]{\label{figad}
\includegraphics[width=0.24\textwidth]{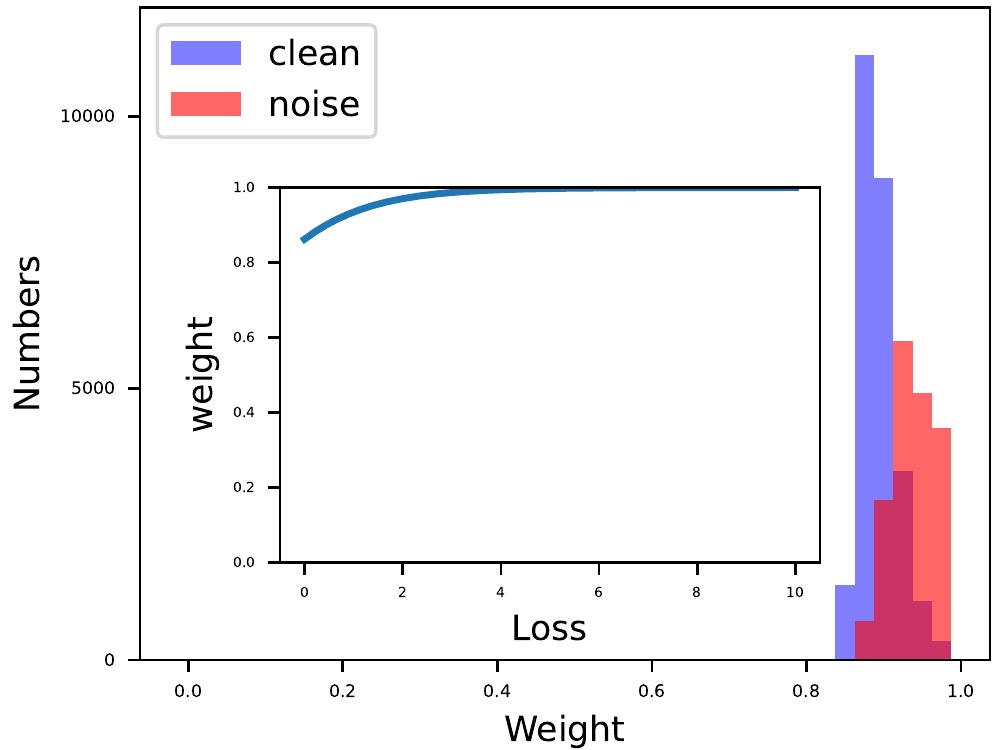}}
\\ \vspace{-0.25cm}
\subfigure[]{\label{figae}
\includegraphics[width=0.24\textwidth]{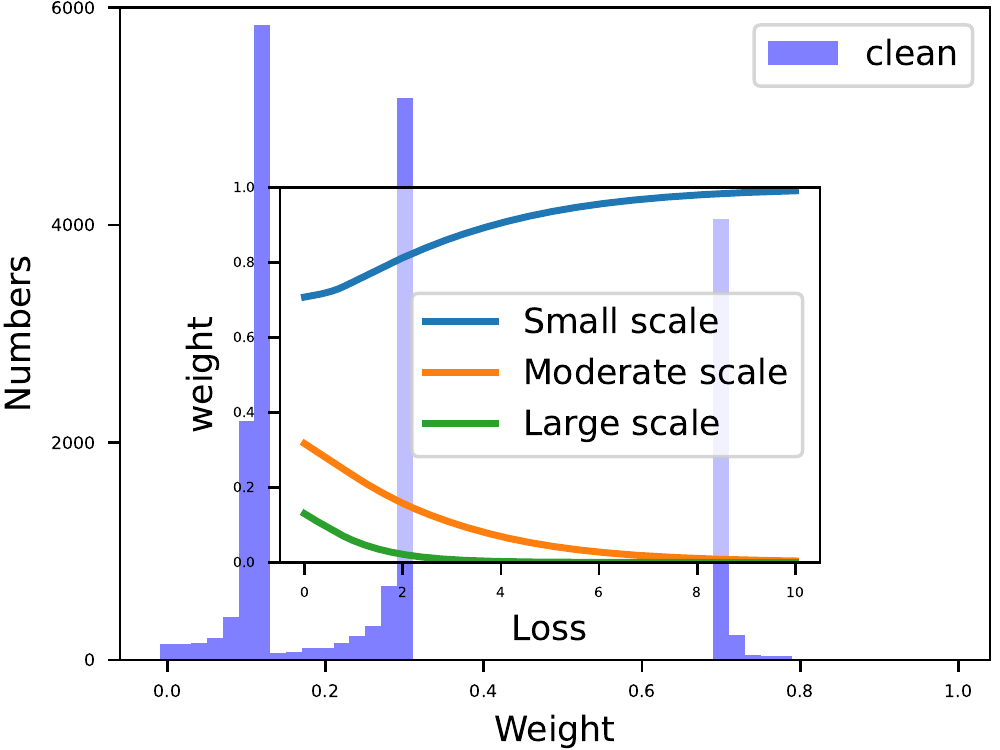}}
\subfigure[]{\label{figaf}
\includegraphics[width=0.24\textwidth]{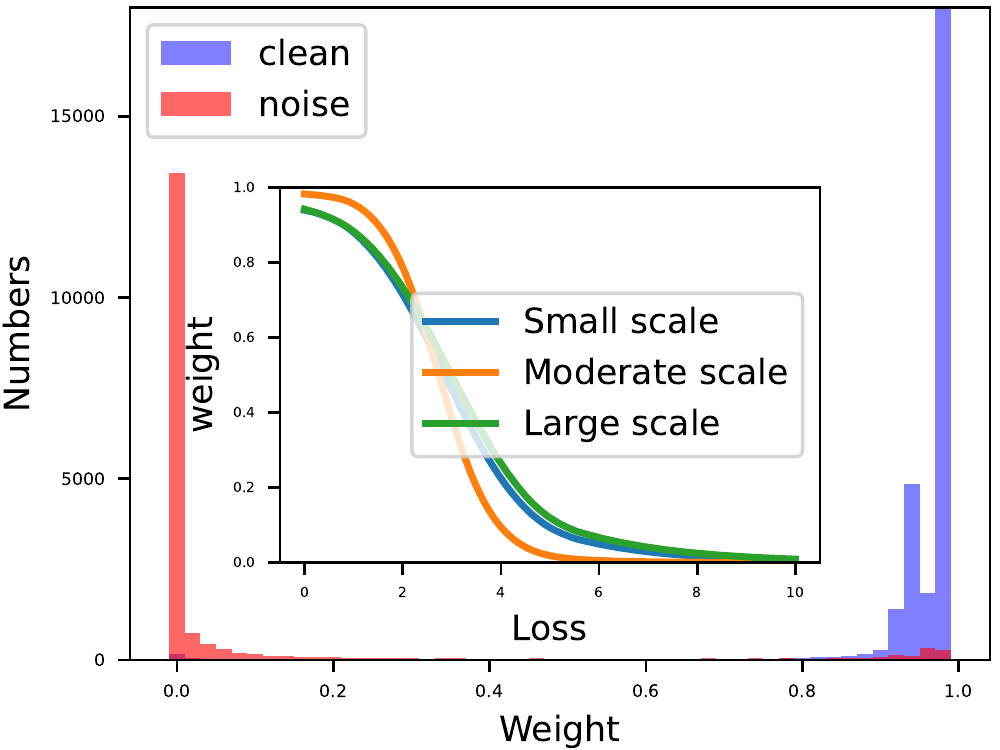}}
	\subfigure[]{\label{figag}
	\includegraphics[width=0.24\textwidth]{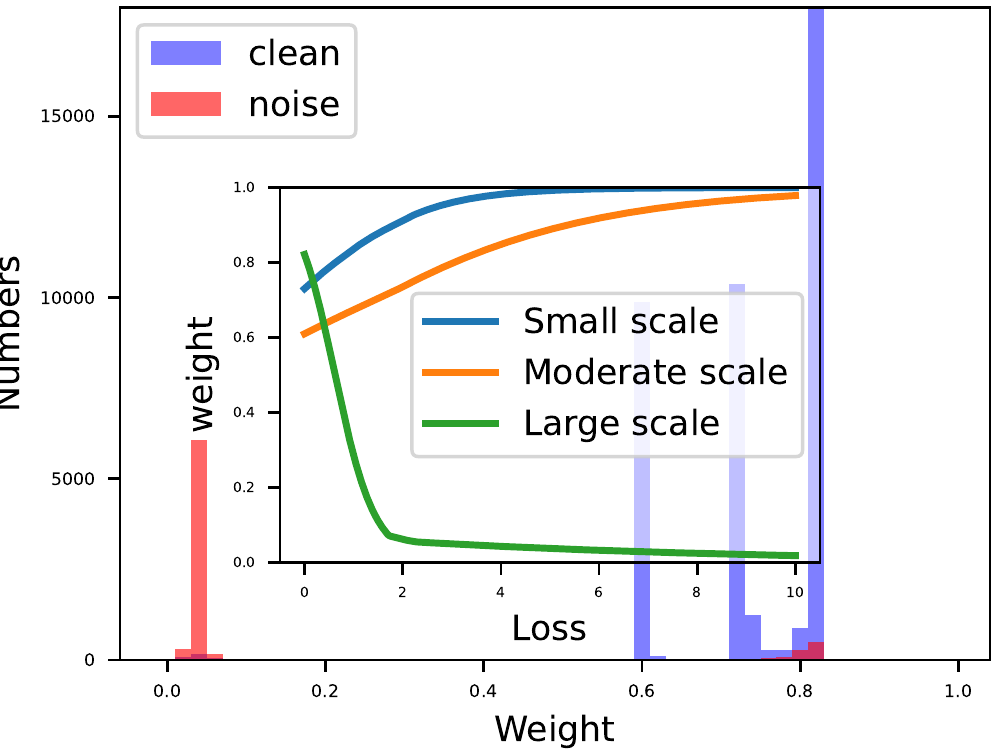}}
			\subfigure[]{\label{figah}
		\includegraphics[width=0.24\textwidth]{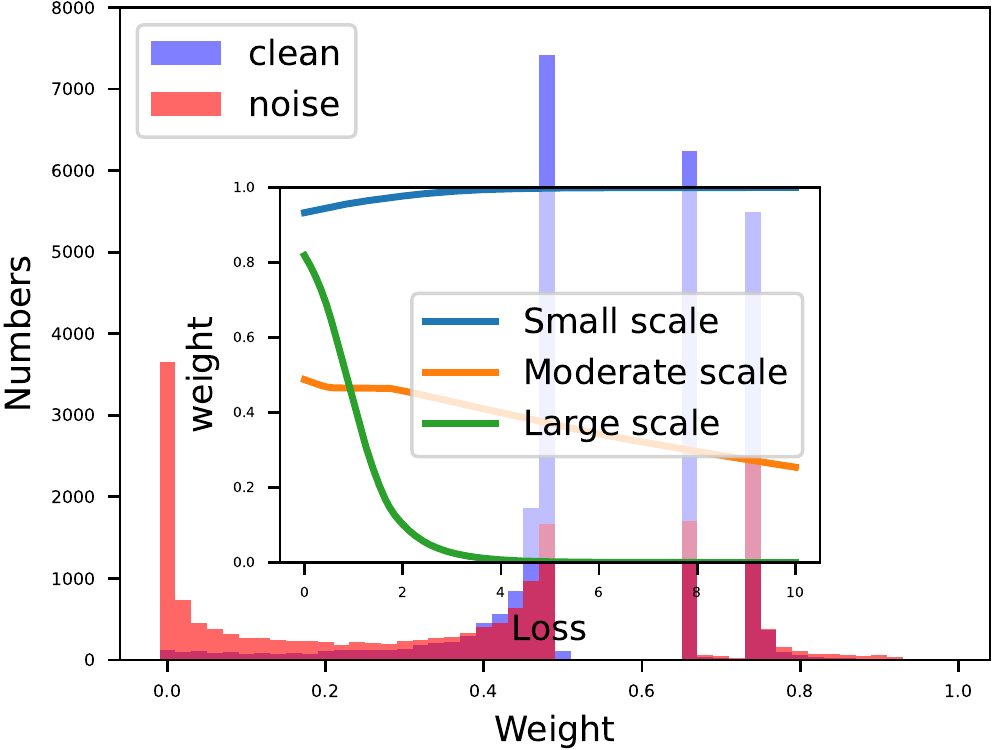}} \vspace{-0.4cm}
	\caption{ (a-d) The weighting function extracted by MW-Net \cite{shu2019meta}, and (e-h) three weighting functions extracted by CMW-Net (corresponding to three task families with small, moderate, large data scales), alongside the histogram of all sample weights calculated by them, for four types of simulated biased datasets. From left to right: Class imbalance (imbalanced factor 10), Symmetric noise (noise rate 40\%), Asymmetric noise (noise rate 40\%), Feature-dependent noise (Type-\uppercase\expandafter{\romannumeral1} + 30\% Asymmetric). The details of simulated biased datasets please refer to Section \ref{section4}. }\label{ss}
	\vspace{-4mm}
\end{figure*}

\vspace{-4mm}
\subsection{CMW-Net and Our Contributions }

Against the aforementioned issues, in this study, we substantially reform MW-Net to make it performable in practical scenarios with complicated data biases. Compared to that we use sample-level information (e.g., loss) to distinguish individual bias properties of different samples,
the core idea is to extract certain higher task-level feature representation from all training classes/tasks to deliver their specific heterogeneous bias characteristics for discriminating training classes/tasks with similar data bias. And then we can accumulate tasks with approximately homoscedastic data bias (e.g., using a clustering algorithm according to the task feature) as a task family. Thus the training dataset can be divided into several task families, where intra-task-families own similar data bias, while inter-task-families own different data biases. To this aim, we simply take the scale level of each task (i.e., the number of samples for each class/task in our implementation) as the task feature, which can be validated to be effective and capable of assembling training classes with approximately homoscedastic loss distributions, as shown in Fig.\ref{fig11}.
Then it is hopeful to deal with heterogeneous data bias by distinguishing individual bias properties of different classes/tasks, and adaptively ameliorating their imposed weighting function forms.
Therefore, we can reform MW-Net by taking such task feature as the supplementary input information besides the sample loss into the weighting function, as shown in Fig. \ref{fig33b}. We call this approach the Class-aware Meta-Weight-Net, or CMW-Net for brevity.

In a nutshell, the main contribution of this paper can be summarized as follows.

1) We propose a CMW-Net model, as shown in Fig.\ref{fig33b}, to automatically learn a proper weighting strategy for real-world heterogeneous data bias in a meta-learning manner.

2) The proposed CMW-Net is model-agnostic, and is substantiated to be performable in different complicated data bias cases, and obtain competitive results with state-of-the-art (SOTA) methods on real-world biased datasets, like ANIMAL-10N \cite{xiao2015learning}, Webvision \cite{li2017webvision} and WebFG-496 \cite{sun2021webly}

3) We further make soft-label amelioration for the CMW-Net model by integrating sample pseudo-label knowledge estimated by model prediction, aiming to correct and reuse the suspected noisy samples into the model training.

4) We study the transferability of CMW-Net. The learned weighting scheme can be used in a plug-and-play manner, and can be directly deployed on unseen datasets, without need to specifically extra tune hyperparameters of CMW-Net.

5) We also evaluate easy generality of CMW-Net to other robust learning tasks, including partial-label learning \cite{jin2002learning}, semi-supervised learning \cite{sohn2020fixmatch} and selective classification \cite{geifman2017selective}.

The paper is organized as follows. Sec. \ref{section3} discusses related work. Sec. \ref{section2} presents the proposed CMW-Net method as well as its learning algorithm and convergence analysis. Simulated and real-world experiments are demonstrated in Sec. \ref{section4} and Sec. \ref{real}, respectively. Sec. \ref{trans} evaluates the transferability of CMW-Net. Sec. \ref{application} introduces the evaluation of CMW-Net to several related applications. The conclusion is finally made.

\section{Related Work} \label{section3}
\textbf{Conventional Sample Weighting Methods.}
The idea of re-weighting examples can be dated back to importance sampling \cite{kahn1953methods}, aiming to assign weights to samples in order to match one distribution to another. Besides, the early attempts of dataset resampling \cite{chawla2002smote,dong2017class} or instance re-weight \cite{zadrozny2004learning} pre-evaluate the sample weights using certain prior knowledge on the task or data. To make sample weights fit data more flexibly, more recent researchers focus more on pre-designing an explicit weighting function mapping from training loss to sample weight, and dynamically ameliorate weights during training process. There are mainly two manners to design such weighting function. One is to make it monotonically increasing, which is specifically effective in class imbalance case. Typical methods along this line include the boosting algorithm \cite{freund1997decision, friedman2000additive,johnson2019survey}, hard example mining \cite{malisiewicz2011ensemble} and focal loss \cite{lin2018focal}, which impose larger weights to ones with larger loss values. On the contrary, another series of methods specify the weighting function as monotonically decreasing, more popularly used in noisy label cases. Typical examples include SPL \cite{kumar2010self} and its extensions \cite{jiang2014easy,shu2020meta}, iterative reweighting ~\cite{zhang2018generalized,wang2017robust}, paying more emphasis on easy samples with smaller losses. The evident limitation of these methods is that they all need to manually pre-specify the form of weighting function as well as its hyper-parameters based on users' prior expert knowledge on the investigated data and learning problem, raising their difficulty to be readily used in real applications. Meanwhile, presetting a certain form of weighting function suffers from the limited flexibility to make the model adaptable to the complicated training data biases, like those with inter-class bias-heterogenous distributions.

\textbf{Meta Learning Methods for Sample Weighting.} Inspired by meta-learning developments \cite{shu2018small,finn2017model,hospedales2020meta,shu2021learning}, recently some methods have been proposed to adaptively learn sample weights from data to make the learning more automatic and reliable.
Typical methods along this line include FWL \cite{dehghani2017fidelity}, learning to teach \cite{fan2018learning}, MentorNet \cite{jiang2018mentornet}, L2RW \cite{ren2018learning}, and MW-Net \cite{shu2019meta}. Especially, MW-Net \cite{shu2019meta} adopts an MLP net to learn an explicit weighting scheme instead of conventional pre-defined weighting scheme. It has been substantiated that weighting function automatically extracted from data comply with those proposed in the hand-designed studies for class-imbalance or noisy labels \cite{shu2019meta}. As analyzed in Sec. 1, the effectiveness of the method, however, is built on the premise assumption that all training classes are with approximately homogeneous biases. However, real-world biased datasets are always inter-class heteroscedastic, and thus it tends to lose efficacy in more practical applications.

\textbf{Other Methods for Class Imbalance.}
Except for sample re-weighting methods, there exist other learning paradigms for handling class imbalance. Typically, \cite{wang2017learning,cui2018large} try to transfer the knowledge learned from major classes to minor ones. \cite{wang2020meta} uses meta feature modulator to balance the contribution per class during the training phase.
The metric learning based methods, e.g., triple-header loss \cite{huang2016learning} and range loss \cite{zhang2017range}, have also been developed to effectively exploit the tailed data to improve the generalization. Furthermore, \cite{jamal2020rethinking} applies domain
adaptation on learning tail class representation.

\textbf{Other Methods for Corrupted Labels.} For handling noisy label issues, many methods have also been designed by making endeavors to correct noisy labels to their true ones to more sufficiently discover and reuse the beneficial knowledge underlying these corrupted data. The typical strategies include supplementing an extra label correction step \cite{arazo2019unsupervised,huang2020self,zheng2020error,wu2021learning}, designing a robust loss function \cite{ghosh2017robust,zhang2018generalized,wang2019symmetric,amid2019robust,ma2020normalized,shu2020learning}, revising the loss function via loss correction \cite{patrini2017making,hendrycks2018using,shu2020meta,lukasik2020does}, and so on. Please refer to references \cite{frenay2013classification,algan2021image,karimi2020deep,song2020learning,han2020survey} for a more overall review.

\section{Class-Aware Meta-Weight-Net}\label{section2}
\subsection{Sample Re-weighting Methodology}
Consider a classification problem with biased training set $\mathcal{D}^{tr} = \{x_i,y_i\}_{i=1}^N$,  where $x_i$ denotes the $i$-th training sample, $y_i \in \{0,1\}^C$ is the one-hot encoding label corresponding to $x_i$, and $N$ is the number of the entire training data. $f(x;\mathbf{w})$ denotes the classifier with $\mathbf{w}$ representing its model parameters. In current applications, $f(x,\mathbf{w})$ is always set with a DNN architecture. We thus also adopt DNN as our prediction model, and call it a classifier network for convenience in the following. Generally, the optimal model parameter $\mathbf{w^*}$ can be extracted by minimizing the following training loss calculated on the training set:
\begin{align}
\mathbf{w}^* = \mathop{\arg\min}_{\mathbf{w} } \frac{1}{N}\sum_{i=1}^N {\ell}(f(x_i;\mathbf{w}),y_i),
\end{align}
where $\ell(f(x;\mathbf{w}),y)$ denotes the training loss on training sample $(x,y)$. In this study, we adopt the commonly adopted cross-entropy (CE) loss $\ell(f(x;\mathbf{w}),y)=-y^T \log(f(x;\mathbf{w})) $, where $f(x;\mathbf{w})$ denotes the network output (especially, $f(x;\mathbf{w})\in \Delta^c$ is a simplex when using Softmax function in the end layer of the network). For notation convenience, we denote $L_i^{tr}(\mathbf{w})= {\ell}(f(x_i;\mathbf{w}),y_i)$ in the following.

In the presence of biased training data, sample re-weighting methods aim to enhance the robustness of network training by imposing a weight $v_i \in [0,1]$ on the $i$-th training sample loss. Then the optimal parameter $\mathbf{w^*}$ is
calculated by minimizing the following weighted loss function:
\begin{align}
	\mathbf{w}^* = \mathop{\arg\min}_{{\mathbf{w} }} \frac{1}{N}\sum_{i=1}^N v_i {\ell}(f(x_i;\mathbf{w}),y_i).
\end{align}
To make sample weights fit data more flexibly, researchers mostly focused on pre-defining a weighting function mapping from training loss to sample weight, and dynamically ameliorate weights during training process \cite{friedman2000additive,lin2018focal,kumar2010self}. More details can refer to literatures provided in related work.

\begin{figure*}[t]
	\centering
	\subfigcapskip=-5mm
	\vspace{-2mm} \includegraphics[width=0.92\textwidth]{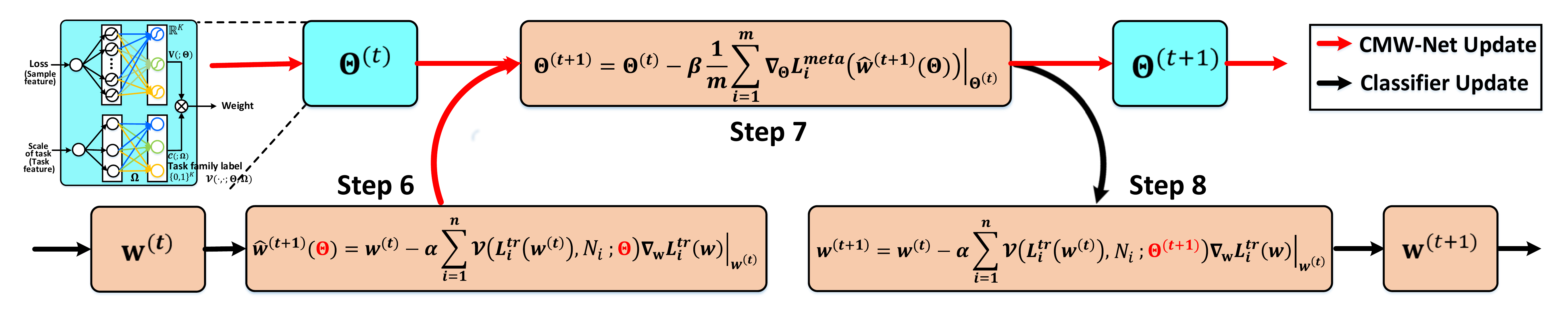}\vspace{-6mm}
	\caption{Main flowchart of the proposed CMW-Net meta-training  algorithm (steps 6-8 in Algorithm \ref{alg:example}).}\label{fig1}  \vspace{-2mm}
\end{figure*}

\subsection{Meta-Weight-Net}\label{section32}
As aforementioned, most conventional sample re-weighting studies need to manually pre-specify the form of weighting function as well as their hyper-parameters based on certain expert knowledge for the investigated problem. This naturally raises their difficulty in readily using them in real applications.
Meanwhile, such weighting function pre-setting manner suffers from the limited flexibility to adapt complicated data bias cases, like applications simultaneously containing class imbalance and noisy label abnormalities in their certain classes.
To address above issues, MW-Net \cite{shu2019meta} is proposed to use an MLP to deliver a suitable weighting function from data.
The architecture of the MW-Net (see Fig. \ref{fig33a}), denoted as $V(\ell;\theta)$, naturally succeededs from the previous sample re-weighting approaches, by setting its input as training loss and output as sample weight, with $\theta$ as its network parameter. Just following standard MLP net, each hidden node is with ReLU activation function, and the output is with the Sigmoid activation function, to guarantee the output located in the interval of $[0,1]$. This weight net is known with a strong fitting capability to represent a wide range of weighting function forms, like those monotonically increasing or decreasing ones as conventional manually specified ones \cite{csaji2001approximation}. The MW-Net thus ideally includes many conventional sample weighting schemes as its special cases.

The parameters contained in MW-Net can be optimized in a meta learning manner \cite{finn2017model,hospedales2020meta,shu2021learning}.
Specifically, with a small amount of unbiased meta-data set $\mathcal{D}^{meta} = \{x_i^{(meta)},y_i^{(meta)}\}_{i=1}^M$ (i.e., with clean labels and balanced data class distribution), representing the meta-knowledge of ground-truth sample-label distribution, where $M$ is the number of meta-samples, the optimal parameter $\theta^*$ of MW-Net can be obtained by minimizing the following bi-level optimization problem:
\begin{align}\label{eq4}
	\begin{split}
	{\theta}^* &= \mathop{\arg\min}_{{\theta }}  \frac{1}{M} \sum_{i=1}^M L_i^{meta}(\mathbf{w}^*(\theta)),\\
	s.t. &\ \mathbf{w}^*(\theta) =\mathop{\arg\min}_{{\mathbf{w} }} \sum_{i=1}^N V(L_i^{tr}(\mathbf{w});\theta) L_i^{tr}(\mathbf{w}),
	\end{split}
\end{align}
where $L_i^{meta}(\mathbf{w}^*(\theta)) \!=\! \ell\!\left(\!f(x_i^{(meta)};\mathbf{w}^*(\theta)),y_i^{(meta)}\!\right)$.
Experimental results on datasets with inter-class homogeneous bias situations, like all classes with similar imbalance rate for class imbalance or similar noise rate for noisy labels case, have shown that the learned weighting schemes are consistent with empirical pre-defined ones as conventional methods \cite{shu2019meta}.

\vspace{-4mm}
\subsection{Class-aware Meta-Weight-Net} \label{section33}
In this section, we first show our motivation of constructing CMW-Net beyond MW-Net, and then we introduce the fundamental consideration and principle of constructing the two branches contained in the CMW-Net architecture. Next we introduce how the two branches are practically formulated in our algorithm, and finally we summarize the overall formulation of CMW-Net and the bi-level optimization objective for calculating its final result.

According to the analysis in Section \ref{section12}, the main limitation of MW-Net is built on the premise that all training classes/tasks approximately follow a similar task distribution. Then MW-Net can use only one unique weighting scheme to handle homogeneous data biases over all training classes/tasks. However, real-world biased datasets are often heterogeneous with obvious inter-class variations of bias, especially for those with a large number of training classes. Since the premise is evidently hampered by heteroscedastic bias situations across different classes/tasks, MW-Net tends to largely lose its efficacy when encountering complicated biased datasets.
This motivates us to reform MW-Net to make it possess adaptability of specifying proper weighting schemes to different classes/tasks based on their own internal bias characteristics.

To this aim, we propose a new weighting model as shown in Fig.\ref{fig33b}, called Class-aware Meta-Weight-Net (CMW-Net for brevity). The architecture of CMW-Net is composed of two branches. The below branch integrates task-level feature knowledge into the input of the original MW-Net as a beneficial compensation besides the original sample-level loss input. The function of this branch is to distinguish individual bias properties of different classes/tasks, and accumulate training classes/tasks with approximately homogeneous bias types (e.g., using a clustering algorithm according to the task feature) as a task family. Based on the identified task families, we can extract a possible faithful sample-weighting scheme shared by a task family, while suppressing the unexpected interference by other task families with heterogeneous ones. Such amelioration is expected to enable the output weight of a sample to correlate with its included training classes/tasks, so as to make it possibly adaptable to class bias variations.

In this study, we attempt to take the scale level (i.e., the number of samples) of each training class/task to represent its task feature. Albeit simple, such feature does be able to deliver helpful class/task pattern underlying its bias types. For instance, from Fig.\ref{fig11}, we can validate the effect for symmetric and asymmetric noise cases (please see more cases in supplementary material).
Denote $N_i$ ($i=1,\!\cdots\!,N$) as the number of samples contained in the training class to which the $i$-th sample $x_i$ belongs. Then below branch of CMW-Net can be expressed as $\mathcal{C}(N_i;\Omega)\in \{0,1\}^K$, by taking $N_i$ as its input to represent the task feature, and including a hidden layer containing $K$ nodes, attached with $K$-levels of scales $\Omega=\{\mu_k\}_{k=1}^{K}$ sorted in ascending order (i.e., $\mu_1<\mu_2<\cdots<\mu_K$). The output of this branch is a $K$-dimensional one-hot vector (i.e., task family label),
whose $1$-element is located at its $k$-th dimension corresponding to the nearest $\mu_k$ to the input $N_i$.

The other branch can be represented as ${\mathbf{V}} (L_i^{tr}(\mathbf{w});\Theta)\in [0,1]^K$, built as an MLP architecture with the loss value of the $i$-th sample as its input, containing one hidden layer and a $K$-dimensional output\footnote{In all our experiments, we just simply set the hidden layer containing 100 nodes with ReLU activation function, and specify the output node with Sigmoid activation function, to guarantee the output of each task family located in the interval of $[0,1]$.}.
Different from $1$-dimensional weight output of MW-Net, this network contains $K$ output weights, corresponding to its $K$ different weighting schemes imposed on samples located in different task families. The sharing hidden layer among these task families extracts the correlation among weighting principles of different task families, which helps reduce the risk of overfitting.

Then the CMW-Net weighting function is formulated as:
\begin{align}
	\mathcal{V}(L_i^{tr}(\mathbf{w}),N_i;\Theta,\Omega) = {\mathbf{V}}(L_i^{tr}(\mathbf{w});\Theta)\otimes \mathcal{C}(N_i;\Omega),
\end{align}
where $\!\otimes\!$ denotes the dot product between two vectors. Through the modulation of the higher-level task feature information, CMW-Net is expected to learn a class-aware weighting function by accumulating training classes/tasks with homogeneous bias situations, and allow different training classes/tasks possessing different weighting schemes complying with their own internal bias characteristics.

\begin{algorithm}[t]
	\vspace{0mm}
	\setlength{\abovecaptionskip}{0.cm}
	\setlength{\belowcaptionskip}{-2cm}
	\renewcommand{\algorithmicrequire}{\textbf{Input:}}
	\renewcommand{\algorithmicensure}{\textbf{Output:}}
	\caption{The CMW-Net Meta-training Algorithm}
	\label{alg:example}
	\begin{algorithmic}[1]  \small
		\REQUIRE  Training dataset $\mathcal{D}^{tr}$, meta-data set $\mathcal{D}^{meta}$, batch size $n,m$, max iterations $T$.
		\ENSURE  Classifier parameter $\mathbf{w}^{(*)}$, CMW-Net parameter $\Theta^{(*)}$
		\STATE
		Apply $K$-means on the sample numbers of all training classes to obtain $\Omega=\{\mu_k\}_{k=1}^{K}$ sorted in ascending order.
		\STATE Initialize classifier network parameter $\mathbf{w}^{(0)}$ and CMW-Net parameter $\Theta^{(0)}$.
		\FOR{$t=0$ {\bfseries to} $T-1$}
		\STATE $\{x,y\} \leftarrow$ SampleMiniBatch($\tilde{\mathcal{D}}^{tr},n$).
		\STATE $\{x^{meta},y^{meta}\} \leftarrow$ SampleMiniBatch($\mathcal{D}^{meta},m$).
		\STATE Formulate the learning manner of classifier network $\hat{\mathbf{w}}^{(t+1)}(\Theta)$ by Eq. (\ref{eq6}).
		\STATE Update parameter $\Theta^{(t+1)}$ of CMW-Net by Eq. (\ref{eqtheta}).
		\STATE Update parameter $\mathbf{w}^{(t+1)}$ of classifier by Eq. (\ref{eq8}).
		\ENDFOR
	\end{algorithmic}
\end{algorithm}

Now, the objective function of CWM-Net can be written as the following bi-level optimization problem:
\begin{align}
	\{{\Theta}^*,\Omega^*\} &= \mathop{\arg\min}_{{\Theta,\Omega }}  \frac{1}{M} \sum_{i=1}^M L_i^{meta}(\mathbf{w}^*(\Theta,\Omega)),          \label{eqout}\\
	\mathbf{w}^*(\Theta,\Omega) &= \mathop{\arg\min}_{{\mathbf{w} }} \sum_{i=1}^N \mathcal{V}(L_i^{tr}(\mathbf{w}),N_i;\Theta,\Omega) L_i^{tr}(\mathbf{w}). \label{eqin}
\end{align}
Note that CMW-Net is degenerated to MW-Net if we take $K=1$, i.e., all training classes own one weighting scheme.

\subsection{Learning Algorithm of CMW-Net}
\subsubsection{Meta-training: learning CMW-Net from training data}
We firstly discuss how to train the CMW-Net from the given training data. There are two groups of parameters, including $\Theta$ and $\Omega$, required to be optimized to attain the CMW-Net model. Therein, the optimization of the scale parameters $\Omega$ corresponds to an integer programming problem and thus hard to design an efficient algorithm for getting its global optimum. We thus adopt a two-stage process to first pre-determine a rational specification of $\Omega^*$, and then focus the computation on optimizing other parameters in the problem. In specific, the standard $K$-means algorithm \cite{bishop2006pattern} is employed on the sample numbers within all training classes (including $C$ positive integers) to obtain cluster centers $\Omega=\{\mu_k\}_{k=1}^{K}$ sorted in ascending order. Throughout all our experiments, we simply set $K\!=\!3$. The small, moderate, and large-scale task families for different datasets can then be distinguished based on the ascending $\{\mu_k\}_{k=1}^{K}$.
All our experiments show consistently and stably fine performance under such simple setting. This also implies that there remains a large room for further performance enhancement of our model by utilizing more elegant optimization techniques and designing more comprehensive task-level feature representations, which will be further investigated in our future research.

Then our aim is to solve the bi-level optimization of Eqs. (\ref{eqout}) and (\ref{eqin}) to obtain optimal $\Theta^*$ and $\mathbf{w}^*$. To make notation concise, we directly neglect $\Omega$ in Eqs. (\ref{eqout}) and (\ref{eqin}) in the following. Note that exact solutions to Eqs. (\ref{eqout}) and (\ref{eqin}) require solving the optimal $\mathbf{w}^*$ whenever $\Theta$ gets updated. This is both analytically infeasible and computationally expensive.
Following previous works \cite{shu2019meta,ren2018learning}, we adopt one step of stochastic gradient descent (SGD) update for $\mathbf{w}$ to online approximate the optimal classifier for a given $\Theta$, which guarantees the efficiency of the algorithm.

\textbf{Formulating learning manner of classifier network.}
To optimize Eq. (\ref{eqin}), in each iteration a mini-batch of training samples $\{(x_i,y_i)\}_{i=1}^n$ is sampled, where $n$ is the mini-batch size. Then the classifier parameter can be updated by moving the current $\mathbf{w}^{(t)}$ along the descent direction of Eq. (\ref{eqin}) on the mini-batch training data as the following expression:
\begin{align}\label{eq6}
	\begin{split}
		\hat{\mathbf{w}}^{(t+1)}(\Theta)=  &
		\mathbf{w}^{(t)} - \\
		&\alpha
		\sum_{i=1}^n    \mathcal{V}(L_i^{tr}(\mathbf{w}^{(t)}),N_i;\Theta)\nabla_{\mathbf{w}} L_i^{tr}(\mathbf{w})\Big|_{\mathbf{w}^{(t)}},
	\end{split}
\end{align} 
where $\alpha$ is the learning rate for the classifier network $f$.

\textbf{Updating parameters of CMW-Net:} Based on the classifier updating formulation $\hat{\mathbf{w}}^{(t+1)}(\Theta)$ from Eq.(\ref{eq6}), the parameter $\Theta$ of the CMW-Net can then be readily updated guided by Eq.(\ref{eqout}), i.e., moving the current parameter $\Theta^{(t)}$ along the objective gradient of Eq.(\ref{eqout}). Similar to the updating step for $\mathbf{w}$, the stochastic gradient descent (SGD) is also adopted. That is, the update is calculated on a sampled mini-batch of meta-data $\{(x^{meta}_i,y^{meta}_i)\}_{i=1}^m$, expressed as
\begin{align}\label{eqtheta}
	\Theta^{(t+1)} =  \Theta^{(t)} -\beta \frac{1}{m}\sum_{i=1}^{m} \nabla_{ \Theta} L_i^{meta}(\hat{\mathbf{w}} ^{(t+1)}(\Theta))\Big|_{\Theta^{(t)}},
\end{align}
where $\beta$ is the learning rate for CMW-Net. Notice that $\Theta$ in $\hat{\mathbf{w}}^{t+1}(\Theta)$ here is a variable instead of a quantity, which makes the gradient in Eq. (\ref{eqtheta}) able to be computed.

\textbf{Updating parameters of classifier network:} Then, the updated $\Theta^{(t+1)}$ is employed to ameliorate the parameter $\mathbf{w}$ of the classifier network, i.e.,
\begin{align}\label{eq8}
	\begin{split}
		\mathbf{w}^{(t+1)} = &
		\mathbf{w}^{(t)} - \\
		&\alpha
		\!\sum_{i=1}^n \!  \mathcal{V}(L_i^{tr}(\mathbf{w}^{(t)}),N_i;\Theta^{(t+1)}) \nabla_{\mathbf{w}} L_i^{tr}(\mathbf{w})\Big|_{\mathbf{w}^{(t)}}.
	\end{split}
\end{align}

Note that we derive with plain SGD here. This, however, also holds for most variants of SGD, like Adam \cite{kingma2015adam}. The CMW-Net learning algorithm can then be summarized in Algorithm \ref{alg:example}, and Fig.\ref{fig1} illustrates its main implementation process (steps 6-8).
All computations of gradients can be efficiently implemented by automatic differentiation techniques and generalized to any deep learning architectures of the classifier. The algorithm can be easily implemented using popular deep learning frameworks like PyTorch \cite{paszke2019pytorch}. It is easy to see that both the classifier and CMW-Net gradually ameliorate their parameters during the learning process based on their values calculated in the last step, and the weights thus tend to be updated in a stable manner.

\vspace{-2mm}
\subsubsection{Analysis on intrinsic learning mechanism of CMW-Net}\label{analysis}
We then present some insightful analysis for revealing some intrinsic learning mechanisms underlying CMW-Net. The updating step of Eq. (\ref{eqtheta}) can be equivalently rewritten as (derivations are presented in supplementary material):
\begin{align}\label{eq10}
	\begin{split}
		& \Theta^{(t+1)}	= \Theta^{(t)} +\\
		&  {\alpha\beta}\sum_{j=1}^n \left(\frac{1}{m} \sum_{i=1}^{m}G_{ij} \right)  \frac{\partial \mathcal{V}(L_j^{tr}(\mathbf{w}^{(t)}),N_j;\Theta)}{\partial \Theta}\Big|_{ \Theta^{(t)}},
	\end{split}
\end{align}
where $G_{ij}=\frac{\partial L_i^{meta} (\hat{\mathbf{w}})}{\partial \hat{\mathbf{w}}}\Big|_{\hat{\mathbf{w}} ^{(t+1)}(\Theta)}^T \frac{\partial L_j^{tr} (\mathbf{w})}{\partial \mathbf{w}} \Big|_{\mathbf{w}^{(t)}}$.
Neglecting the coefficient $\frac{1}{m} \sum_{i=1}^{m} G_{ij} $, it is easy to see that each term in the above sum orients to the ascending gradient of the weight function $\mathcal{V}(L_j^{tr}(\mathbf{w}^{(t)}),N_j;\Theta)$. The coefficient imposed on the $j$-th gradient term, $\frac{1}{m} \sum_{i=1}^{m} G_{ij}$, represents the similarity between the gradient of the $j$-th training sample computed on the training loss and the average gradient of the mini-batch meta data calculated on meta loss. This means that if the learning gradient of a training sample is similar to that of the meta samples, then it inclines to be considered as in-distribution and CMW-Net tends to produce a higher
sample weight for it. Conversely, samples with gradient different from that of the meta set incline to be suppressed. This understanding is consistent with the intrinsic working mechanism underlying the well-known MAML \cite{finn2017model,nichol2018reptile}.

\begin{algorithm}[t]
	\renewcommand{\algorithmicrequire}{\textbf{Input:}}
	\renewcommand{\algorithmicensure}{\textbf{Output:}}
	\caption{The CMW-Net Meta-test Algorithm}
	\label{alg:example1}
	\begin{algorithmic}[1]  \small
		\REQUIRE  Training dataset $\mathcal{D}^{q}$, batch size $n'$, max iterations $T'$ and meta-learned CMW-Net with parameter $\Theta^{*}$.
		\ENSURE  Classifier parameter $\mathbf{u}^{*}$.
		\STATE Apply $K$-means on sample numbers of all training classes to obtain $\Omega^q=\{\mu_k^q\}_{k=1}^{K}$ sorted in ascending order.
		\STATE Initialize classifier network parameter $\mathbf{u}^{(0)}$.
		\FOR{$t=0$ {\bfseries to} $T'-1$}
		\STATE Update classifier $\mathbf{u}^{(t+1)}$ by solving Eq. (\ref{eqinw}).
		\ENDFOR 
	\end{algorithmic}
\end{algorithm}

\vspace{-2mm}
\subsubsection{Meta-test: transferring CMW-Net to unseen tasks}
After the meta-training stage, the learned CMW-Net with parameter $\Theta^{(*)}$ can then be transferred to readily assign proper sample weights on unseen biased datasets. Specifically, for a query dataset $\mathcal{D}^{q} = \{x_i^q,y^q_i\}_{i=1}^{N^q}$, we first need to implement $K$-means on sample numbers of all classes to obtain its cluster centers $\Omega^q=\{\mu_k^q\}_{k=1}^{K}$ as new scale parameters of CMW-Net. Then the learned CMW-Net can be directly used to impose sample weights to the classifier learning of the problem by solving the following objective of query task:
\begin{align}
	\mathbf{u}^* &= \mathop{\arg\min}_{{\mathbf{u} }} \sum_{i=1}^{N^q} \mathcal{V}(L_i^{q}(\mathbf{u}),N^q_i;\Theta^*,\Omega^q) L_i^{q}(\mathbf{u}), \label{eqinw}
\end{align}
where
$L_i^{q}(\mathbf{u}) = \ell\left(f(x_i^{q};\mathbf{u}),y_i^{q}\right)$, and $N_i^q$ denotes the number of samples contained in the class to which $x_i^q$ belongs. Then we can solve Eq.(\ref{eqinw}) with the learned $\Theta^{*}$ to obtain classifier $\mathbf{u}^*$. The overall algorithm is summarized in Algorithm \ref{alg:example1}.

\vspace{-2mm}
\subsection{Convergence of the CMW-Net Learning Algorithm}
Next we attempt to establish a convergence result of our method for calculating Eqs. (\ref{eqout}) and (\ref{eqin}) in a bi-level optimization manner. In particular, we theoretically show that our method converges to critical points of both the meta loss (Eq.(\ref{eqout})) and training loss (Eq.(\ref{eqin})) under some mild conditions in Theorem \ref{th11} and \ref{th22}, respectively. The proofs are presented in the supplementary material (SM for brevity).
\begin{theorem} \label{th11}
	Suppose the loss function $\ell$ is Lipschitz smooth with constant $L$, and CMW-Net $\mathcal{V}(\cdot,\cdot;\Theta)$ is differential with a $\delta$-bounded gradient and twice differential with its Hessian bounded by $\mathcal{B}$, and the loss function $\ell$ have $\rho$-bounded gradients with respect to training/meta data. Let the learning rate $\alpha_t, \beta_t, 1\leq t\leq T$ be monotonically decreasing sequences, and satisfy $\alpha_t=\min\{\frac{1}{L},\frac{c_1}{\sqrt{T}}\}, \beta_t=\min\{\frac{1}{L},\frac{c_2}{\sqrt{T}}\}$, for some $c_1,c_2>0$, such that $\frac{\sqrt{T}}{c_1}\geq L, \frac{\sqrt{T}}{c_2}\geq L$. Meanwhile, they satisfy $\sum_{t=1}^\infty \alpha_t = \infty,\sum_{t=1}^\infty \alpha_t^2 < \infty ,\sum_{t=1}^\infty \beta_t = \infty,\sum_{t=1}^\infty \beta_t^2 < \infty $. Then CMW-Net can then achieve $\mathbb{E}[ \|\nabla \mathcal{L}^{meta}(\hat{\mathbf{w}}^{(t)}(\Theta^{(t)}))\|_2^2] \leq \epsilon$ in $\mathcal{O}(1/\epsilon^2)$ steps. More specifically,
	\begin{align*}
		\min_{0\leq t \leq T} \mathbb{E}\left[ \left\|\nabla \mathcal{L}^{meta}(\hat{\mathbf{w}}^{(t)}(\Theta^{(t)}))\right\|_2^2\right] \leq \mathcal{O}(\frac{C}{\sqrt{T}}),
	\end{align*}
	where $C$ is some constant independent of the convergence process.
\end{theorem}

\begin{theorem}\label{th22}
	Under the conditions of Theorem \ref{th11}, CMW-Net can achieve $\mathbb{E}[ \|\nabla \mathcal{L}^{tr}(\mathbf{w}^{(t)};\Theta^{(t)})\|_2^2] \leq \epsilon$ in $\mathcal{O}(1/\epsilon^2)$ steps, where $\mathcal{L}^{tr}(\mathbf{w};\Theta) \!=\! \sum_{i=1}^N \mathcal{V}(L_i^{tr}(\mathbf{w}),N_i;\Theta) L_i^{tr}(\mathbf{w})$. More specifically,
	\begin{align*}
		\min_{0\leq t \leq T} \mathbb{E}\left[ \left\|\nabla \mathcal{L}^{tr}(\mathbf{w}^{(t)};\Theta^{(t)})\right\|_2^2 \right] \leq \mathcal{O}(\frac{C}{\sqrt{T}}),
	\end{align*}
	where $C$ is some constant independent of the convergence process.
\end{theorem}

\vspace{-6mm}
\subsection{Enhancing CMW-Net with Soft Label Supervision} \label{SL}
In the typical bias case that some training samples are with corrupted labels, the sample weighting strategy tends to largely neglect the function of these samples by imposing small or even zero weights on them. This manner, however, inclines to regrettably waste the beneficial information essentially contained in these samples. Some recent researches have thus been presented to possibly correct the noisy labels and reuse them in training. One popular option is to extract a pseudo soft label $z$ on a sample $x$ through the clue of the classifier's estimation during the training iterations, and then set the training loss as a convex weighting combination of loss terms computed with the suspected noisy label $y$ and the pseudo-label $z$ \cite{reed2014training,song2019selfie,arazo2019unsupervised,li2019dividemix}, i.e.,
\begin{align}\label{eqsoftv2}
	\ell_{S}(f(x;\mathbf{w}),y) \!=\!v {\ell}(f(x;\mathbf{w}),y) \!+\! (1-v){\ell}(f(x;\mathbf{w}),z) ,
\end{align}
where $v \in [0,1]$ denotes the sample weight. By setting the loss as the cross-entropy, the loss (\ref{eqsoftv2}) can be rewritten as:
\begin{align}
	\ell_{S}(f(x;\mathbf{w}),y) = -\left(vy +(1-v)z\right)^T \log(f(x;\mathbf{w})). \label{vcombin}
\end{align}
It can then be understood as setting a corrected soft label $vy +(1-v)z$ to ameliorate the original label $y$ to make it more reliably reused and avoid roughly suppressing or throwing off the sample from training as conventional.

We then shortly introduce the current research on how to set the sample weight $v$ in the above (\ref{eqsoftv2}) or (\ref{vcombin}). The early attempts often adopted a manual manner for setting this hyper-parameter, e.g., the $v$ is empirically set as $v=0.8$ for all samples in \cite{reed2014training}. Evidently, such a fixed and constant weight specification could not sufficiently convey the variant knowledge of training samples with different contents of corruption and reliability. Afterwards, some methods try to dynamically assign individual weights for different samples. Typically, SELFIE \cite{song2019selfie} iteratively selects clean samples by assigning weights $v=1$ on them, and neglects doubtful noisy samples by setting their weights as $v=0$ in (\ref{vcombin}). M-correction \cite{arazo2019unsupervised} ameliorates this hard weighting manner as soft, by fitting a two-component Beta mixture model per epoch to estimate the probability of a sample being clean or noisy, and then use this probability to assign a soft weight for the corresponding sample. Recently, DivideMix \cite{li2019dividemix} improves \cite{arazo2019unsupervised} by adopting a two-component Gaussian mixture model to assign a soft weight $v$ for the corresponding sample.

However, all above methods require exploiting a separate early-learning stage \cite{liu2020early} to heuristically pre-determine the sample weights $v$, while certainly ignore the beneficial feedback from the classifier during the learning process. We thus can naturally introduce our CMW-Net method to automatically explore a weighting scheme by making it trained together with the classifier in a meta-learning manner. Specifically, we just need to easily  revise the training objective of CMW-Net in Eq.(\ref{eqin}) as (called CMW-Net-SL):
\begin{align}
	\mathbf{w}^*(\Theta)=& \mathop{\arg\min}_{{\mathbf{w} }} \ \sum_{i=1}^N [\mathcal{V}(L_i^{tr}(\mathbf{w}),N_i;\Theta) L_i^{tr}(\mathbf{w}) \nonumber \\
&\ \ \ \ \ \ \ \ \ \ \ \ +(1-\mathcal{V}(L_i^{tr}(\mathbf{w}),N_i;\Theta))L_i^{Pse}(\mathbf{w})], \label{eqin1}
\end{align}
where
$L_i^{Pse}(\mathbf{w})={\ell}(f(x_i;\mathbf{w}),z_i)$.
Taking a similar process as Sec. \ref{analysis}, we have
\begin{align}  \label{eq15}
	\begin{split}
		&\Theta^{(t+1)}	= \Theta^{(t)} +{\alpha\beta}\times \\
		& \sum_{j=1}^n \!\left[\frac{1}{m} \sum_{i=1}^{m}(G_{ij}-G'_{ij}) \right]\!  \frac{\partial \mathcal{V}(L_j^{tr}(\mathbf{w}^{(t)}),N_j;\Theta)}{\partial \Theta}\Big|_{ \Theta^{(t)}},
	\end{split}
\end{align}
where $G'_{ij}=\frac{\partial L_i^{meta} (\hat{\mathbf{w}})}{\partial \hat{\mathbf{w}}}\Big|_{\hat{\mathbf{w}} ^{(t+1)}(\Theta)}^T \frac{\partial L_j^{Pse} (\mathbf{w})}{\partial \mathbf{w}} \Big|_{\mathbf{w}^{(t)}}$. Compared with CMW-Net, it is seen that CMW-Net-SL produces another term $G'_{ij}$ to control the learning of the meta-learner. Specifically, if $\frac{1}{m} \sum_{i=1}^{m} (G_{ij}-G'_{ij})>0$, it means that the similarity between learning gradient of a training sample with original label and the meta samples is larger than that of a training sample with pseudo-label, and then it will be considered as a relatively clean label and CMW-Net tends to produce a higher
sample weight to it. Otherwise, it inclines to be considered as a relatively noisy label and CMW-Net will suppress the influence of original labeled sample while produce more confidence on pseudo-labeled one.

In our experiments, we apply EMA \cite{tarvainen2017mean} and temporal ensembling \cite{laine2016temporal} techniques to produce pseudo-labels in our CMW-Net-SL algorithm, which has been verified to be effective in tasks like semi-supervised learning \cite{laine2016temporal,laine2016temporal} and robust learning \cite{arazo2019unsupervised,liu2020early}. Note that the meta-train and meta-test algorithms of CMW-Net-SL are similar to Algorithms \ref{alg:example} and \ref{alg:example1} except that the training loss is revised from (\ref{eqout}) to (\ref{eqin1}). More detailed algorithm description is provided in the SM.

\begin{table*}
	\setlength{\abovecaptionskip}{0.cm}
	\setlength{\belowcaptionskip}{-2cm}
	\caption{Test top-1 error (\%) comparison of different competing methods with ResNet-32 classifier on CIFAR-10-LT and CIFAR-100-LT under different imbalance settings. * indicates results reported in \cite{jamal2020rethinking}.}\label{classimc} \vspace{-1mm}
	\centering
	\begin{tabular}{l|c|c|c|c|c|c|c|c|c|c|c|c}
		\toprule
		Dataset Name & \multicolumn{6}{c|}{CIFAR-10-LT} & \multicolumn{6}{c}{CIFAR-100-LT}\\
		\hline
		Imbalance factor & 200 & 100 & 50 & 20 & 10 &1& 200 & 100 & 50 & 20 & 10 &1 \\
		\hline
		ERM & 34.32 & 29.64 & 25.19 & 17.77 & 13.61 &7.53& 65.16 & 61.68 &56.15 & 48.86 & 44.29 &29.50 \\
		Focal loss \cite{lin2018focal} & 34.71& 29.62 &23.29 &17.24 &13.34& 6.97& 64.38 &61.59& 55.68 &48.05& 44.22& 28.85 \\
		CB loss \cite{cui2019class} & 31.11 &27.63& 21.95& 15.64 &13.23 &7.53&64.44 &61.23& 55.21 &48.06& 42.43&29.37\\
		LDAM loss \cite{cao2019learning}* & - &26.65& -& - &13.04 &-&60.40 &-& - &-& 43.09&-\\
		L2RW \cite{ren2018learning} &  33.49& 25.84& 21.07& 16.90&14.81 &10.75& 66.62 &59.77& 55.56& 48.36&46.27&35.89\\
		MW-Net \cite{shu2019meta} &32.80& 26.43& 20.90& 15.55& 12.45& 7.19&63.38 &58.39& 54.34 &46.96& 41.09& 29.90\\ \hline
		MCW \cite{jamal2020rethinking} with CE loss* & 29.34 &23.59& 19.49& 13.54& 11.15& \textbf{7.21}&\textbf{60.69}& 56.65 &51.47& 44.38 &40.42& -\\
		CMW-Net with CE loss& \textbf{27.80} & \textbf{21.15} & \textbf{17.26 }&\textbf{ 12.45} & \textbf{10.97} & 8.30 & 60.85 & \textbf{55.25} & \textbf{49.73} &\textbf{43.06}& \textbf{39.41} & 30.81 \\ \hline
		MCW \cite{jamal2020rethinking} with LDAM loss* & \textbf{25.10}& 20.00 &17.77& 15.63& 12.60& 10.29
		&60.47& 55.92& \textbf{50.84}& 47.62 &42.00&-\\
		CMW-Net with LDAM loss& 25.57 &\textbf{19.95} & \textbf{17.66} & \textbf{13.08} & \textbf{11.42} &\textbf{7.04} & \textbf{59.81} &\textbf{55.87} & 51.14 & \textbf{45.26} & \textbf{40.32} & \textbf{29.19}\\
		\hline
		SADE \cite{zhang2022self} & 19.37 & 16.78 & 14.81 & 11.78 & 9.88 & 7.72 & 54.78 & 50.20 & 46.12 & 40.06 &  36.40 & 28.08  \\
		CMW-Net with SADE & \textbf{19.11} & \textbf{16.04} &\textbf{ 13.54} & \textbf{10.25} &\textbf{9.39} & \textbf{5.39 }& \textbf{54.59} & \textbf{49.50} & \textbf{46.01} & \textbf{39.42} &  \textbf{34.78} & \textbf{27.50} \\
		\bottomrule
	\end{tabular} \vspace{-3mm}
\end{table*}

\begin{table*}
	\setlength{\abovecaptionskip}{0.cm}
	\setlength{\belowcaptionskip}{-2cm}
	\caption{Performance comparison of different competing methods in test accuracy (\%) on CIFAR-10 and CIFAR-100 with symmetric and asymmetric noise. The average accuracy and standard deviation over 3 trials are reported.}\label{noisy} \vspace{-1mm}
	\centering
	\resizebox{\textwidth}{28mm}{
		\begin{tabular}{c|l|c|c|c|c|c|c|c|c}
			\toprule
			\multirow{2}{*}{Datasets} &\multirow{2}{*}{Noise}  & \multicolumn{4}{c|}{Symmetric Noise} & \multicolumn{4}{c}{Asymmetric Noise}\\\cline{3-10}
			& & 0.2 & 0.4 & 0.6 & 0.8  & 0.2 & 0.4 & 0.6 & 0.8   \\ \cline{3-10}
			\hline
			\multirow{12}{*}{CIFAR-10}&ERM & 86.98 $\pm$ 0.12 & 77.52 $\pm$ 0.41 & 73.63 $\pm$ 0.85 & 53.82 $\pm$ 1.04   & 83.60 $\pm$ 0.24 & 77.85 $\pm$ 0.98 & 69.69 $\pm$ 0.72 & 55.20 $\pm$ 0.28      \\
			&Forward \cite{patrini2017making} & 87.99 $\pm$ 0.36 &  83.25 $\pm$ 0.38 & 74.96 $\pm$ 0.65 & 54.64 $\pm$ 0.44   &  91.34 $\pm$ 0.28 & 89.87 $\pm$ 0.61 & 87.24 $\pm$ 0.96  & 81.07 $\pm$ 1.92        \\
			&GCE \cite{zhang2018generalized} & 89.99 $\pm$ 0.16 & 87.31 $\pm$ 0.53 & 82.15 $\pm$ 0.47 & 57.36 $\pm$ 2.08 &   89.75 $\pm$ 1.53 & 87.75 $\pm$ 0.36 & 67.21 $\pm$ 3.64 & 57.46 $\pm$ 0.31     \\
			&M-correction \cite{arazo2019unsupervised} & 93.80 $\pm$ 0.23 & 92.53 $\pm$ 0.11  & 90.30 $\pm$ 0.34 & 86.80 $\pm$ 0.11 & 92.15 $\pm$ 0.18 & 91.76 $\pm$ 0.57 & 87.59 $\pm$ 0.33 & 67.78 $\pm$ 1.22 \\
			&DivideMix \cite{li2019dividemix} & 95.70 $\pm$ 0.31 &  95.00 $\pm$ 0.17  &   94.23 $\pm$ 0.23 & \textbf{92.90 $\pm$ 0.31} & 93.96 $\pm$ 0.21 & 91.80 $\pm$ 0.78 & 80.14 $\pm$ 0.45 & 59.23 $\pm$ 0.38  \\
			&L2RW \cite{ren2018learning}& 89.45 $\pm$ 0.62 & 87.18 $\pm$ 0.84 & 81.57 $\pm$ 0.66 & 58.59 $\pm$ 1.84 &   90.46 $\pm$ 0.56 & 89.76 $\pm$ 0.53 & 88.22 $\pm$ 0.71 & 85.17 $\pm$ 0.31    \\
			&MW-Net \cite{shu2019meta} & 90.46 $\pm$ 0.52 & 86.53 $\pm$ 0.57 & 82.98 $\pm$ 0.34 & 64.41 $\pm$ 0.92 &  92.69 $\pm$ 0.24 & 90.17 $\pm$ 0.11 & 68.55 $\pm$ 0.76 & 58.29 $\pm$ 1.33     \\
			&CMW-Net & 			91.09 $\pm$ 0.54 & 86.91 $\pm$ 0.37 & 83.33 $\pm$ 0.55 & 64.80 $\pm$ 0.72 &   93.02 $\pm$ 0.25 & 92.70 $\pm$ 0.32 & 91.28 $\pm$ 0.40 & 87.50 $\pm$ 0.26    \\
			&CMW-Net-SL & \textbf{96.20 $\pm$ 0.33} & \textbf{95.29 $\pm$ 0.14} & \textbf{94.51 $\pm$ 0.32} & {92.10 $\pm$ 0.76}
			&  \textbf{95.48 $\pm$ 0.29} & \textbf{94.51 $\pm$ 0.52} & \textbf{94.18 $\pm$ 0.21} & \textbf{93.07 $\pm$ 0.24}     \\
			\hline
			\hline
			\multirow{12}{*}{CIFAR-100}&ERM & 60.38 $\pm$ 0.75 & 46.92 $\pm$ 0.51 & 31.82 $\pm$ 1.16 & 8.29 $\pm$ 3.24   & 61.05 $\pm$ 0.11 & 50.30 $\pm$ 1.11 & 37.34 $\pm$ 1.80 & 12.46 $\pm$ 0.43     \\
			& Forward \cite{patrini2017making} &  63.71 $\pm$ 0.49 & 49.34 $\pm$ 0.60 & 37.90 $\pm$ 0.76   & 9.57 $\pm$ 1.01 & 64.97 $\pm$ 0.47  & 52.37 $\pm$ 0.71 &  44.58 $\pm$ 0.60 &   15.84 $\pm$ 0.62      \\
			&GCE \cite{zhang2018generalized} & 68.02 $\pm$ 1.05 & 64.18 $\pm$ 0.30 & 54.46 $\pm$ 0.31 & 15.61 $\pm$ 0.97   & 66.15 $\pm$ 0.44 & 56.85 $\pm$ 0.72 & 40.58 $\pm$ 0.47 & 15.82 $\pm$ 0.63   \\
			&M-correction \cite{arazo2019unsupervised} & 73.90 $\pm$ 0.14 & 70.10 $\pm$ 0.14 & 59.50 $\pm$ 0.35 & 48.20 $\pm$ 0.23 & 71.85 $\pm$ 0.19  & 70.83 $\pm$ 0.48 & 60.51 $\pm$ 0.52 & 16.06 $\pm$ 0.33        \\
			&DivideMix \cite{li2019dividemix} & 76.90 $\pm$ 0.21 & 75.20 $\pm$ 0.12   & 72.00 $\pm$ 0.33  & \textbf{59.60 $\pm$ 0.21} &  76.12 $\pm$ 0.44 & 73.47 $\pm$ 0.63 & 45.83 $\pm$ 0.83 & 16.98 $\pm$ 0.40   \\
			&L2RW \cite{ren2018learning} & 65.32 $\pm$ 0.42 & 55.75 $\pm$ 0.81 & 41.16 $\pm$ 0.85 & 16.80 $\pm$ 0.22 &   65.93 $\pm$ 0.17 & 62.48 $\pm$ 0.56 & 51.66 $\pm$ 0.49 & 12.40 $\pm$ 0.61   \\
			&MW-Net \cite{shu2019meta} & 69.93 $\pm$ 0.40 & 65.29 $\pm$ 0.43 & 55.59 $\pm$ 1.07 & 27.63 $\pm$ 0.56 &  69.80 $\pm$ 0.34 & 64.88 $\pm$ 0.63 & 56.89 $\pm$ 0.95 & 17.05 $\pm$ 0.52  \\
			&CMW-Net & 70.11 $\pm$ 0.19 & 65.84 $\pm$ 0.50 & 56.93 $\pm$ 0.38 & 28.36 $\pm$ 0.67 &   71.07 $\pm$ 0.56 & 66.15 $\pm$ 0.51 & 58.21 $\pm$ 0.78 & 17.41 $\pm$ 0.16   \\
			&CMW-Net-SL & \textbf{77.84 $\pm$ 0.12} & \textbf{76.25 $\pm$ 0.67} & \textbf{72.61 $\pm$ 0.92} & 55.21 $\pm$ 0.31 &   \textbf{77.73 $\pm$ 0.37} & \textbf{75.69 $\pm$ 0.68} & \textbf{61.54 $\pm$ 0.72} & \textbf{18.34 $\pm$ 0.21}       \\
			\bottomrule
	\end{tabular}}\vspace{-3mm}
\end{table*}

\vspace{-4mm}
\section{Learning with  Synthetic Biased Data}  \label{section4}\vspace{-1mm}
\subsection{Class Imbalance Experiments}\label{imbalance}\vspace{-1mm}

\textbf{Datasets.}\ We use long-tailed versions of CIFAR-10 and CIFAR-100 datasets (CIFAR-10-LT and CIFAR-100-LT) as in \cite{cui2019class}. They contain the same categories as the original CIFAR dataset \cite{krizhevsky2009learning}, but are created by reducing the number of training samples per class according to an exponential function $n = n_i \mu^i$, where $i$ denotes the class index, $n_i$ is the original number of training images and $\mu \in (0,1)$. The imbalance factor of a dataset is defined as the number of training samples in the largest class divided by the smallest.

\textbf{Baselines.}\ The comparison methods include: 1) Empirical risk minimization (ERM): all examples have the same weights. By default, we use standard cross-entropy loss; 2) {Focal loss} \cite{lin2018focal} and 3) {CB loss} \cite{cui2019class}: represent SOTA pre-defined sample re-weighting techniques; 4) LDAM loss \cite{cao2019learning}: dynamically tune the margins between classes according to their degrees of dominance in the training set;
5) L2RW \cite{ren2018learning}: adaptively assign sample weights by meta-learning; 6) MW-Net \cite{shu2019meta}: learn an explicit weighting function by meta-learning;
7) MCW \cite{jamal2020rethinking}: also use a meta-learning framework, while consider an elegantly designed class-wise weighting scheme, validated to be specifically effective for class imbalance bias.
8) SADE \cite{zhang2022self}: leverage self-supervision to aggregate the learned multiple experts for achieving SOTA performance.
More implementation details are specified in SM.

\begin{figure}[t]\vspace{-2mm}
	\centering
	\subfigcapskip=-2mm
	\includegraphics[width=0.23\textwidth]{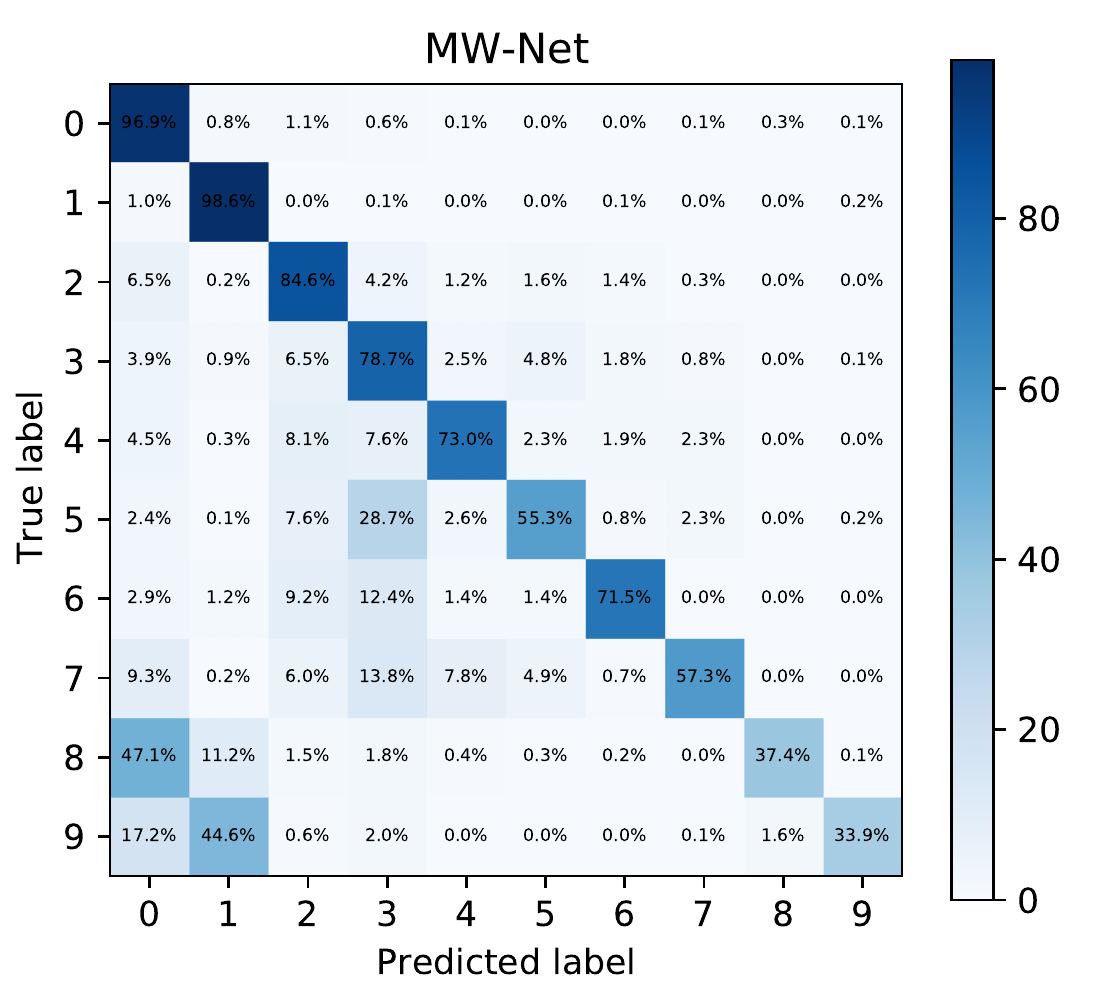}
	\includegraphics[width=0.23\textwidth]{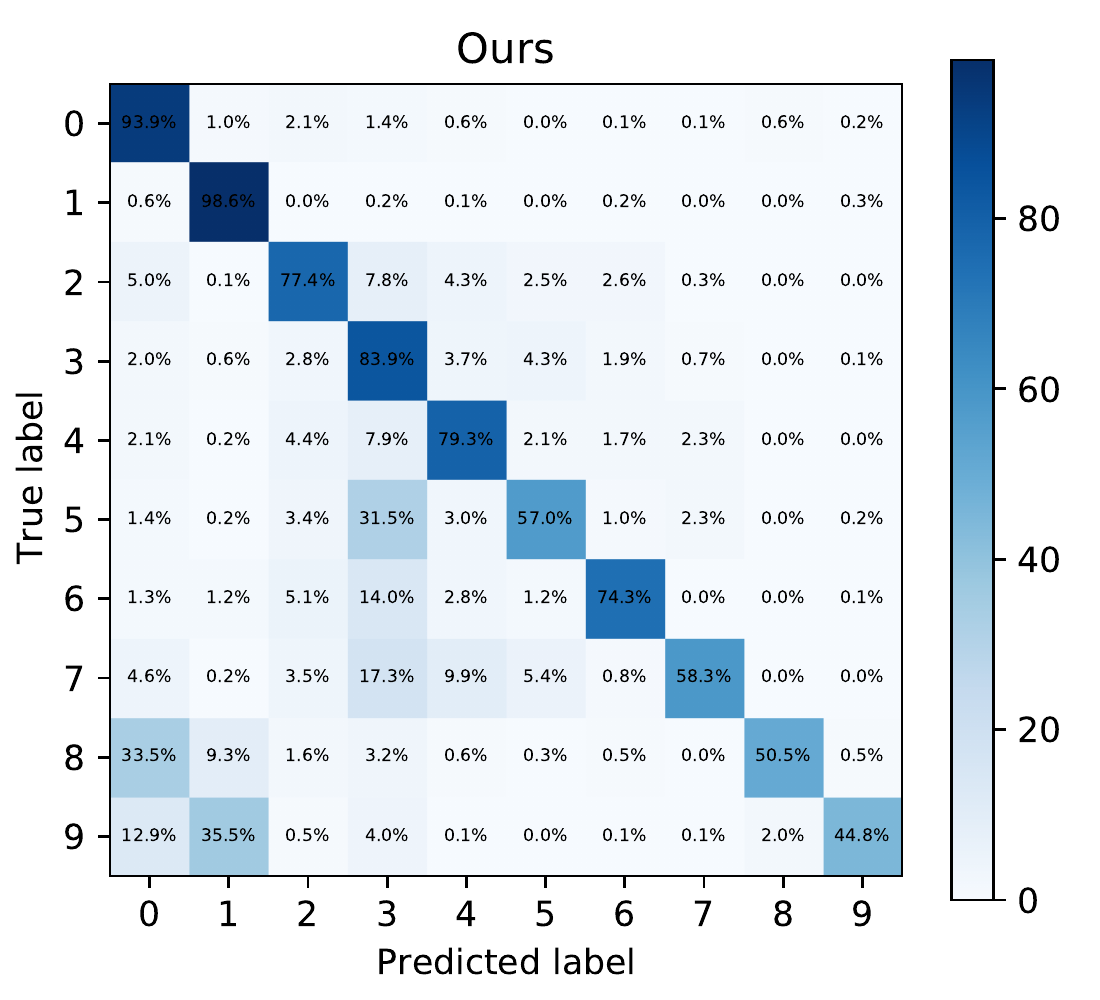}\vspace{-4mm}
	\caption{Confusion matrices obtained by (left) MW-Net and (right) CMW-Net on CIFAR-10-LT (with imbalance factor 200).}\label{figimbalance}  \vspace{-4mm}
\end{figure}

\textbf{Results.}\
Table \ref{classimc} shows the test errors of all competing methods by taking ResNet-32 as the classifier model on CIFAR-10-LT and CIFAR-100-LT with different imbalance factors. It can be observed that: 1) Our algorithm outperforms other competing methods on the datasets, showing its robustness in such biased data; 2) CMW-Net evidently outperforms MW-Net in each experiment. Especially, the performance gain tends to be more evident under larger imbalanced factors. Fig. \ref{figimbalance} shows confusion matrices produced by the results of MW-Net and CMW-Net on CIFAR-10-LT with imbalance factor 200\footnote{The confusion matrix is calculated by applying the trained classifier to the corresponding testing set included with the CIFAR-10 dataset.}. Compared with MW-Net, it is seen that CMW-Net improves the accuracies on tail classes and meanwhile maintains good performance on head classes.
3) Although LDAM loss already has the capacity of mitigating the long-tailed issue by penalizing hard examples, our method can further boost its performances. 4) Owing to its class-wise weighting scheme, MCW also attains good performance in these experiments. Yet CMW-Net still performs better in most cases. Considering its adaptive weighting-scheme-setting capability and general usability in a wider range of biased issues, it should be rational to say that CMW-Net is effective. 5) SADE uses multiple expertise-guided losses to produce competitive results, and the performance can be further boosted via introducing CMW-Net, demonstrating the effectiveness of our general weighting strategy.

To understand the weighing scheme learned by CMW-Net, we also depict the weighting functions learned by the CMW-Net in Fig.\ref{figae}. It is seen that compared with MW-Net shown in Fig.\ref{figaa}, CMW-Net produces three weighting functions corresponding to small, moderate and large-scale task families. The overall tendency complies with conventional empirical setting for such  class-wise weight functions, like CB loss \cite{cui2019class} and MCW \cite{jamal2020rethinking}, i.e., assigning weights inversely related to the class sizes. Specifically, the learned weights of the tail classes are more prominent than those of the head ones, implying that samples in tail classes should be more emphasized in training to alleviate the class imbalanced bias issue. This also explains the consistently better performance of CMW-Net as compared with MW-Net.

\begin{table}
	\setlength{\abovecaptionskip}{0.cm}
	\setlength{\belowcaptionskip}{-2cm}
	\caption{Comparison with SOTA methods on CIFAR-10 and CIFAR-100 with symmetric and asymmetric noise. The compared results are directly taken from original literatures. We report test accuracy at the last epoch.}\label{sota} 
	\centering
	\resizebox{0.48\textwidth}{22mm}{
		\begin{tabular}{c|l|c|c|c|c|c}
			\toprule
			\multirow{2}{*}{Datasets} &\multirow{2}{*}{Noise}  & \multicolumn{4}{c|}{Symmetric Noise} & \multicolumn{1}{c}{Asy. Noise}\\\cline{3-7}
			& & 0.2 & 0.5 & 0.8 & 0.9  & 0.4 \\ \cline{3-7}
			\hline
			\multirow{8}{*}{CIFAR-10}&DivideMix \cite{li2019dividemix} & 95.7 & 94.4 & 92.9 & 75.4 & 92.1 \\
			&ELR+ \cite{liu2020early} & 94.6 & 93.8 & 93.1 & 75.2 & 92.7 \\
			& REED\cite{zhang2020decoupling} &  95.7 &  95.4 & 94.1 & 93.5 & - \\
			&AugDesc \cite{nishi2021augmentation} & 96.2 & 95.1 & 93.6 & 91.8 & 94.3 \\
			& C2D \cite{zheltonozhskii2022contrast}& 96.2 & 95.1 & 94.3 & 93.4 & 90.8 \\
&Two-step \cite{He2022twostep} & 96.2 & 95.3 & 93.7 & 92.7 & 92.4\\ \cline{2-7}
			&CMW-Net-SL &  96.2 & 95.1 & 92.1 & 48.0 & 94.5 \\
			&CMW-Net-SL+ &  \textbf{96.6} & \textbf{96.2} & \textbf{95.4} & \textbf{93.7} & \textbf{96.0} \\
			\hline
			\hline
			\multirow{8}{*}{CIFAR-100}&DivideMix \cite{li2019dividemix} & 77.3 & 74.6 & 60.2& 31.5 & 72.1 \\
			&ELR+ \cite{liu2020early} & 77.5 & 72.4 & 58.2 & 30.8 & 76.5 \\
			& REED\cite{zhang2020decoupling} &  76.5& 72.2& 66.5& 59.4 &  -\\
			&AugDesc \cite{nishi2021augmentation} & 79.2 & 77.0 & 66.1 & 40.9 & 76.8 \\
			& C2D \cite{zheltonozhskii2022contrast}& 78.3 & 76.1 & 67.4 & 58.5 & 75.1 \\
&Two-step \cite{He2022twostep} & 79.1 & 78.2 & 70.1 & 53.2 & 65.5\\ \cline{2-7}
				&CMW-Net-SL &  77.84 & 76.2 & 55.2 & 21.2 & 75.7 \\
			&CMW-Net-SL+ &  \textbf{80.2} & \textbf{78.2} & \textbf{71.1} & \textbf{64.6} & \textbf{77.2} \\
			\bottomrule
		\end{tabular}
	}  \vspace{-4mm}
\end{table}

\begin{table*}[t]
	\setlength{\abovecaptionskip}{0.cm}
	\setlength{\belowcaptionskip}{-2cm}
	\caption{Test accuracy (\%) of all competing methods on CIFAR-10 and CIFAR-100 under different feature-dependent noise types and levels. The average accuracy and standard deviation over 3 trials are reported.}\label{classfea1} \vspace{-1mm}
	\centering
	\resizebox{\textwidth}{20mm}{
		\begin{tabular}{l|c|c|c|c|c|c|c|c}
			\toprule
			Datasets  & Noise &  ERM & LRT \cite{zheng2020error} &GCE \cite{zhang2018generalized}& MW-Net \cite{shu2019meta}& PLC \cite{zhang2020learning} & CMW-Net  & CMW-Net-SL    \\  \hline
			\multirow{6}{*}{CIFAR-10} &  Type-\uppercase\expandafter{\romannumeral1} (35\%) & 78.11 $\pm$ 0.74 & 80.98 $\pm$ 0.80 & 80.65 $\pm$ 0.39 &  82.20 $\pm$ 0.40 & 82.80 $\pm$ 0.27 & 82.27 $\pm$ 0.33  & \textbf{84.23 $\pm$ 0.17} \\
			&  Type-\uppercase\expandafter{\romannumeral1} (70\%) & 41.98 $\pm$ 1.96 & 41.52 $\pm$ 4.53 & 36.52 $\pm$ 1.62 &  38.85 $\pm$ 0.67 & 42.74 $\pm$ 2.14 &  42.23 $\pm$ 0.69 &\textbf{44.19 $\pm$ 0.69} \\
			&  Type-\uppercase\expandafter{\romannumeral2} (35\%) & 76.65 $\pm$ 0.57 &80.74 $\pm$ 0.25  & 77.60 $\pm$ 0.88 &  81.28 $\pm$ 0.56 & 81.54 $\pm$ 0.47 & 81.69 $\pm$ 0.57 &  \textbf{83.12 $\pm$ 0.40} \\
			&  Type-\uppercase\expandafter{\romannumeral2} (70\%) & 45.57 $\pm$ 1.12 & 81.08 $\pm$ 0.35 & 40.30 $\pm$ 1.46 & 42.15 $\pm$ 1.07 & 46.04 $\pm$ 2.20 & 46.30 $\pm$ 0.77 & \textbf{48.26 $\pm$ 0.88} \\
			&  Type-\uppercase\expandafter{\romannumeral3} (35\%) & 76.89 $\pm$ 0.79 & 76.89 $\pm$ 0.79  & 79.18 $\pm$ 0.61 & 81.57 $\pm$ 0.73 & 81.50 $\pm$ 0.50 & 81.52 $\pm$ 0.38 &  \textbf{83.10 $\pm$ 0.34} \\
			&  Type-\uppercase\expandafter{\romannumeral3} (70\%) & 43.32 $\pm$ 1.00 & 44.47 $\pm$ 1.23 & 37.10 $\pm$ 0.59 & 42.43 $\pm$ 1.27 & 45.05 $\pm$ 1.13 & 43.76 $\pm$ 0.96 &  \textbf{45.15 $\pm$ 0.91} \\
			\hline
			\hline
			\multirow{6}{*}{CIFAR-100} &  Type-\uppercase\expandafter{\romannumeral1} (35\%) & 57.68 $\pm$ 0.29 & 56.74 $\pm$ 0.34 & 58.37 $\pm$ 0.18 & 62.10 $\pm$ 0.50 & 60.01 $\pm$ 0.43 & 62.43 $\pm$ 0.38 & \textbf{64.01 $\pm$ 0.11} \\
			&  Type-\uppercase\expandafter{\romannumeral1} (70\%) & 39.32 $\pm$ 0.43 & 45.29 $\pm$ 0.43 & 40.01 $\pm$ 0.71 &  44.71 $\pm$ 0.49 & 45.92 $\pm$ 0.61 & 46.68 $\pm$ 0.64 &  \textbf{47.62 $\pm$ 0.44} \\
			&  Type-\uppercase\expandafter{\romannumeral2} (35\%) & 57.83 $\pm$ 0.25 & 57.25 $\pm$ 0.68 & 58.11 $\pm$ 1.05 &  63.78 $\pm$ 0.24 & 63.68 $\pm$ 0.29 & 64.08 $\pm$ 0.26 &  \textbf{64.13 $\pm$ 0.19} \\
			&  Type-\uppercase\expandafter{\romannumeral2} (70\%) & 39.30 $\pm$ 0.32 &43.71 $\pm$ 0.51  & 37.75 $\pm$ 0.46 &  44.61 $\pm$ 0.41 & 45.03 $\pm$ 0.50 & 50.01 $\pm$ 0.51 &  \textbf{51.99 $\pm$ 0.35} \\
			&  Type-\uppercase\expandafter{\romannumeral3} (35\%) & 56.07 $\pm$ 0.79 & 56.57 $\pm$ 0.30  & 57.51 $\pm$ 1.16 & 62.53 $\pm$ 0.33 & 63.68 $\pm$ 0.29 & 63.21 $\pm$ 0.23 &  \textbf{64.47 $\pm$ 0.15} \\
			&  Type-\uppercase\expandafter{\romannumeral3} (70\%) & 40.01 $\pm$ 0.18 & 44.41 $\pm$ 0.19 & 40.53 $\pm$ 0.60 &  45.17 $\pm$ 0.77 & 44.45 $\pm$ 0.62 & 47.38 $\pm$ 0.65 &  \textbf{48.78 $\pm$ 0.62} \\
			\bottomrule
	\end{tabular}}  \vspace{-2mm}
\end{table*}

\vspace{-2mm}
\subsection{Feature-independent Label Noise Experiment}\label{noise}
\textbf{Datasets.}\ We study two types of label noise following previous works \cite{patrini2017making,hendrycks2018using}: 1) Symmetric noise: randomly replace sample labels for a percentage of the training data with all possible labels.
2) Asymmetric noise: try to mimic the structure of real-life label noise, where labels are only replaced by similar classes. Two benchmark datasets are employed: CIFAR-10 and CIFAR-100 \cite{krizhevsky2009learning}.

\textbf{Baselines.}\ The comparison methods include: 1) ERM; 2) Forward \cite{patrini2017making}: corrects the prediction by the label transition matrix;
3) GCE \cite{zhang2018generalized}: behaves as a robust loss to handle the noisy labels; 4) M-correction \cite{arazo2019unsupervised}; 5) DivideMix \cite{li2019dividemix}: represents the SOTA method for handling noisy label bias, by dividing training data into clean and noisy ones through a loss threshold and designing different label amelioration strategies on them through two diverged networks to co-train the classifier; 6) L2RW \cite{ren2018learning} and 7) MW-Net \cite{shu2019meta}: represents the sample re-weighting methods by meta-learning.
More experimental details are listed in SM.

\textbf{Establishing Meta dataset.}\ Motivated by curriculum learning \cite{bengio2009curriculum,kumar2010self,jun2020meta}, we select the most confident training samples as meta data. Specifically, we explore to create the meta dataset based on the high-quality clean samples as well as its high-quality pseudo labels from the training set (with lowest losses) as an unbiased estimator of the clean data-label distribution. To make the meta dataset balanced, we selected 10 images per class in each epoch iteration. In this case, the performance of meta dataset can be served as an indicator to measure how much extent CMW-Net is trained to filter noisy samples and generalized to clean test distribution.

Such meta dataset may lack of diversity pattern to characterize the latent clean data-label distribution. To alleviate this issue, we further explore to utilize mixup technique \cite{zhang2018mixup} to enrich the variety of the meta data distribution while possibly maintain its unbiasedness. The hyperparameter of convex combination in the technique is randomly sampled from a Beta distribution $Beta(1,1)$. Our extensive experiments have verified the effectiveness of using such generated meta dataset from training data. Such property makes such meta-learning strategy applicable to real-world biased dataset, since it is always not easy to collect an additional clean meta dataset in practice. We also use such meta-data-generation strategy in the following noisy labels experiments as well as the real-world biased dataset, where an expected clean meta-dataset is always unavailable.

\begin{figure}[t] \vspace{-2mm}
	\centering
		\includegraphics[width=0.23\textwidth]{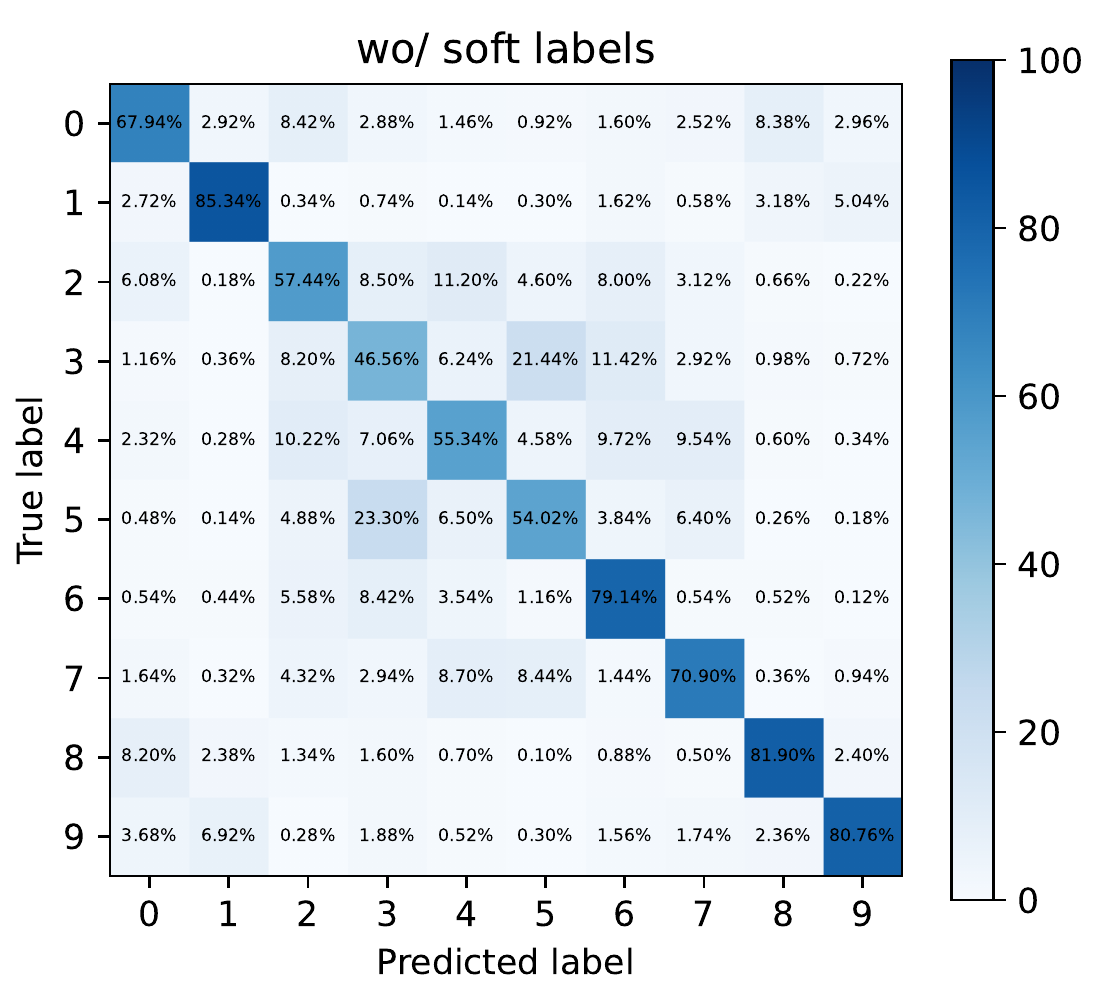}
		\includegraphics[width=0.23\textwidth]{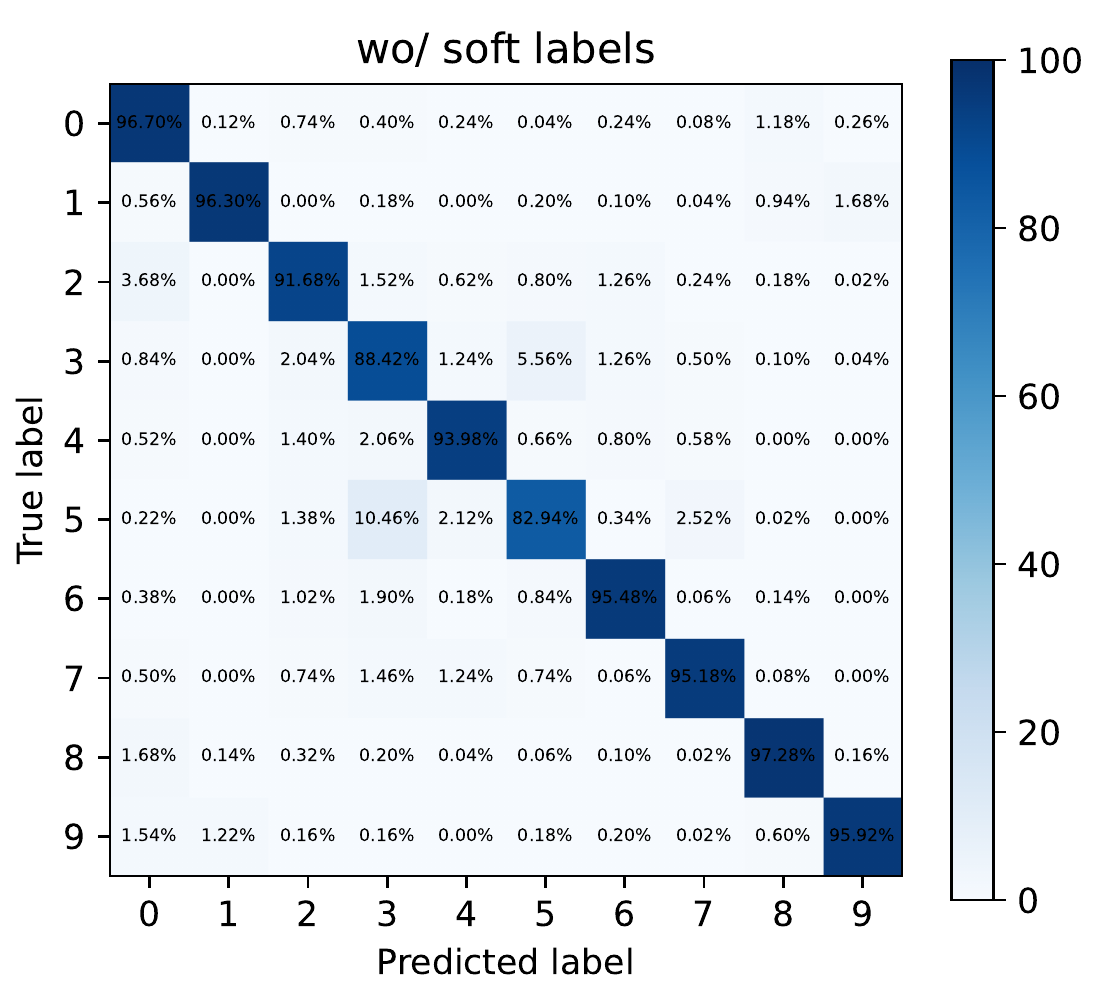}\vspace{-4mm}
		\caption{Confusion matrices obtained by CMW-Net (Left) and CMW-Net-SL (Right) on CIFAR-10 with Symmetry Noise 80\%.}\label{figoft}  \vspace{-6mm}
	\end{figure}

\begin{table*}
	\setlength{\abovecaptionskip}{0.cm}
	\setlength{\belowcaptionskip}{-2cm}
	\caption{Test accuracy (\%) of all competing methods on CIFAR-10 and CIFAR-100 under different feature dependent (35\%) and independent (30\%) noise types and levels. The average accuracy and standard deviation over 3 trials are reported.}\label{classfea2} \vspace{-1mm}
	\centering
	\resizebox{\textwidth}{20mm}{
		\begin{tabular}{l|l|c|c|c|c|c|c|c}
			\toprule
			Datasets  & Noise &  ERM & LRT \cite{zheng2020error} &GCE \cite{zhang2018generalized}& MW-Net \cite{shu2019meta}& PLC \cite{zhang2020learning} & CMW-Net & CMW-Net-SL    \\  \hline
			\multirow{6}{*}{CIFAR-10}
			&  Type-\uppercase\expandafter{\romannumeral1} + Symmetric & 75.26 $\pm$ 0.32 & 75.97 $\pm$ 0.27 & 78.08 $\pm$ 0.66 &  76.39 $\pm$ 0.42 & 79.04 $\pm$ 0.50 & 78.42 $\pm$ 0.47 & \textbf{82.00 $\pm$ 0.36} \\
			&  Type-\uppercase\expandafter{\romannumeral1} + Asymmetric & 75.21 $\pm$ 0.64 &76.96 $\pm$ 0.45  &  76.91 $\pm$ 0.56 & 76.54 $\pm$ 0.56 & 78.31 $\pm$ 0.41 & 77.14 $\pm$ 0.38 & \textbf{80.69 $\pm$ 0.47} \\
			\cline{2-9}
			&  Type-\uppercase\expandafter{\romannumeral2} + Symmetric & 74.92 $\pm$ 0.63 & 75.94 $\pm$ 0.58 & 75.69 $\pm$ 0.21 & 76.57 $\pm$ 0.81 & 80.08 $\pm$ 0.37 & 76.77 $\pm$ 0.63 & \textbf{80.96 $\pm$ 0.23} \\
			&  Type-\uppercase\expandafter{\romannumeral2} + Asymmetric & 74.28 $\pm$ 0.39 & 77.03 $\pm$ 0.62  & 75.30 $\pm$ 0.81 & 75.35 $\pm$ 0.40 & 77.63 $\pm$ 0.30 & 77.08 $\pm$ 0.52 & \textbf{80.94 $\pm$ 0.14} \\
			\cline{2-9}
			&  Type-\uppercase\expandafter{\romannumeral3} + Symmetric & 74.00 $\pm$ 0.38 & 75.66 $\pm$ 0.57 & 77.00 $\pm$ 0.12 &  76.28 $\pm$ 0.82 & 80.06 $\pm$ 0.47 & 77.16 $\pm$ 0.30 & \textbf{81.58 $\pm$ 0.55} \\
			&  Type-\uppercase\expandafter{\romannumeral3} + Asymmetric & 75.31 $\pm$ 0.34 &  77.19 $\pm$ 0.74 & 75.70 $\pm$ 0.91 &  75.82 $\pm$ 0.77 & 77.54 $\pm$ 0.70 & 76.49 $\pm$ 0.88 & \textbf{80.48 $\pm$ 0.48} \\
			\hline
			\hline
			\multirow{6}{*}{CIFAR-100}
			&  Type-\uppercase\expandafter{\romannumeral1} + Symmetric & 48.86 $\pm$ 0.56 & 45.66 $\pm$ 1.60 & 52.90 $\pm$ 0.53 &  57.70 $\pm$ 0.32 & 60.09 $\pm$ 0.15 & 59.17 $\pm$ 0.42 & \textbf{60.87 $\pm$ 0.56} \\
			&  Type-\uppercase\expandafter{\romannumeral1} + Asymmetric & 45.85 $\pm$ 0.93 &52.04 $\pm$ 0.15  & 52.69 $\pm$ 1.14 & 56.61 $\pm$ 0.71 & 56.40 $\pm$ 0.34 & 57.42 $\pm$ 0.81 & \textbf{61.35 $\pm$ 0.52} \\
			\cline{2-9}
			&  Type-\uppercase\expandafter{\romannumeral2} + Symmetric & 49.32 $\pm$ 0.36 & 43.86 $\pm$ 1.31 & 53.61 $\pm$ 0.46 &  54.08 $\pm$ 0.18 & 60.01 $\pm$ 0.63 & 59.16 $\pm$ 0.18 & \textbf{61.00 $\pm$ 0.41} \\
			&  Type-\uppercase\expandafter{\romannumeral2} + Asymmetric & 46.50 $\pm$ 0.95 &52.11 $\pm$ 0.46  & 51.98 $\pm$ 0.37 &  58.53 $\pm$ 0.45 & 61.43 $\pm$ 0.33 & 58.99 $\pm$ 0.91 & \textbf{61.35 $\pm$ 0.57} \\
			\cline{2-9}
			&  Type-\uppercase\expandafter{\romannumeral3} + Symmetric & 48.94 $\pm$ 0.61 &  42.79 $\pm$ 1.78 & 52.07 $\pm$ 0.35 &  55.29 $\pm$ 0.57 & 60.14 $\pm$ 0.97 & 58.48 $\pm$ 0.79 & \textbf{60.21 $\pm$ 0.48} \\
			&  Type-\uppercase\expandafter{\romannumeral3} + Asymmetric & 45.70 $\pm$ 0.12 &  50.31 $\pm$ 0.39& 50.87 $\pm$ 1.12 &  58.43 $\pm$ 0.60 & 54.56 $\pm$ 1.11 & 58.83 $\pm$ 0.57 & \textbf{60.52 $\pm$ 0.53} \\
			\bottomrule
	\end{tabular}}\vspace{-2mm}
\end{table*}

\textbf{Results.}\
Table \ref{noisy} evaluates the performance of our method on CIFAR-10 and CIFAR-100 with different levels of symmetric and asymmetric label noise. We report the averaged test accuracy over the last 10 epochs. It is seen that CMW-Net evidently outperforms MW-Net in all cases. For symmetric noise, the training loss distribution is usually homoscedastic, as depicted in Fig.\ref{fig11a} and thus CMW-Net learns three similar weighting schemes as that extracted by MW-Net, as shown in Fig.\ref{figaf}. For asymmetric noise, while MW-Net hardly adapts to the heterogeneous loss distribution across different classes, CMW-Net finely produces weighting schemes conditioned on different task families. Specifically, as shown in Fig.\ref{figag}, weighting functions of small and moderate-scale task families tend to more emphasize those informative marginal samples since they barely contain replaced labels from other classes. While for large-scale task family, with many corrupted labels, CMW-Net tends to impose smaller weights on samples with relatively large losses to suppress the effect of these noisy labels. This shows that learned weighting schemes by CMW-Net can adapt to the internal bias patterns of different classes, and thus naturally leads to its superiority over MW-Net.

By introducing soft label amelioration, CMW-Net-SL can further enhance the performance of CMW-Net, as clearly shown in Table \ref{noisy}. This is natural since CMW-Net-SL is able to adaptively refurbish noisy labeled samples rather than roughly trash them from training, and thus a more sufficient exploration on beneficial knowledge from training data could be obtained. Fig. \ref{figoft} shows the confusion matrices obtained by CMW-Net and CMW-Net-SL in 60\% label noise rate case. It is seen that CMW-Net-SL evidently improves the prediction accuracy, especially for classes with heavily corrupted labels.

Note that by involving label amelioration through using pseudo-prediction information as our method, the DivideMix method also performs well in most symmetric noise experiments, especially, slightly better than CMW-Net-SL in 80\% noise rate. However, the superiority of our method is still significant in all asymmetric label noise cases. Different from prior works only reported results with a ratio of 40\% for asymmetric noise, we consider more noise ratio settings to evidently show this phenomenon. It can be seen that DivideMix has a substantial degradation for higher noise ratio.
This can be rationally explained by that DivideMix uses a consistent loss threshold for distinguishing clean and noisy samples, which, however, is certainly deviated from the insight of inter-class heteroscedastic loss distributions underlying this type of data bias. As can be observed in Fig.\ref{fig11b}, since clean training classes are simultaneously tail classes, the loss values of some training samples in these classes are possibly larger than those of head classes, especially for their contained noisy samples. Thus DivideMix tends to mistakenly recognize a certain amount of clean/noisy samples, which then results in its performance degradation.
Comparatively, the class-aware capability possessed by CMW-Net-SL enables the method more properly treat heteroscedastic loss distributions across different classes, and thus obtain more accurate weighting functions specifically suitable for them, which then naturally leads to its relatively superior performance.

\textbf{Compared with SOTA methods.}\
As shown in Table \ref{sota}, our method underperforms the SOTA DivideMix method in the extreme large label noise cases (80\% and 90\% noise ratios), which possibly attributes to that DivideMix treats most noisy label samples as unlabeled samples, and uses strong semi-supervised MixMatch algorithms against noisy labels. Even though, benefiting from current SOTA methods like REED \cite{zhang2020decoupling}, C2D \cite{zheltonozhskii2022contrast}, AugDesc \cite{nishi2021augmentation}, Two-step \cite{He2022twostep}, which additionally use two general tricks of self-supervised learning for performance improvement, i.e., adding a warm-up self-supervised pre-training step and imposing a data augmentation based consistency regularization as commonly used in semi-supervised algorithms, we also easily borrow these common tricks to boost our method (denoted by CMW-Net-SL+). It can be easily seen that the ameliorated CMW-Net-SL+ consistently outperforms the compared SOTA methods, and beats them on CIFAR-100 with 80\% and 90\% symmetric noise by an evident margin.

\begin{table}[t]
	\setlength{\abovecaptionskip}{0.cm}
	\setlength{\belowcaptionskip}{-2cm}
	\caption{Comparison of different competing methods on Animal-10N dataset. Results for baseline methods are copied from \cite{zhang2020learning}} \label{animal}
	\vspace{-1mm}
	\centering
	\begin{tabular}{c|c|c|c}
		\toprule
		Method &  Test Accuracy & Method &  Test Accuracy \\ \hline
		ERM  & 79.4 $\pm$ 0.14 & ActiveBias \cite{chang2017active} &  80.5 $\pm$ 0.26 \\
		Co-teaching \cite{han2018co} & 80.2 $\pm$ 0.13   &  SELFIE \cite{song2019selfie} &   81.8 $\pm$ 0.09    \\
		PLC \cite{zhang2020learning} &   83.4 $\pm$ 0.43  & MW-Net \cite{shu2019meta}  &  80.7 $\pm$ 0.52 \\ \hline
		CMW-Net  &  80.9 $\pm$ 0.48  &  CMW-Net-SL  &  \textbf{84.7 $\pm$ 0.28} \\
		\bottomrule
	\end{tabular} \vspace{-2mm}
\end{table}	

\vspace{-3mm}
\subsection{Feature-dependent Label Noise Experiment}\label{dependnoise}
We then evaluate the capability of our method against the feature-dependent label noise, which is more approximate to the real-world bias scenarios \cite{chen2021beyond,zhang2020learning}.

\textbf{Datasets.}\ We follow the PMD noise generation scheme proposed in \cite{zhang2020learning}.
Let $\eta_{y_1}(x)=P(y=y_1|x)$ be the true posterior label distribution for the sample $x$. The noise label is generated by replacing the most confident label $u_x=\arg \max_{y} \eta_{y}(x)$ of each training sample $x$ to its second confident category $s_x$ with conditional probability
$\tau_{u_x,s_x} = P(\tilde{y}=u_x|y=s_x,x)$. We use three types of $\tau_{u_x,s_x}$ designed in \cite{zhang2020learning} as follows:
\begin{align*}
	\text{Type-\uppercase\expandafter{\romannumeral1}}: \tau_{u_x,s_x} =& -\frac{1}{2}\left[\eta_{u_x}(x)-\eta_{s_x}(x)\right]^2+\frac{1}{2}, \\
	\text{Type-\uppercase\expandafter{\romannumeral2}}: \tau_{u_x,s_x} =&1-\left[\eta_{u_x}(x)-\eta_{s_x}(x)\right]^3, \\
	\text{Type-\uppercase\expandafter{\romannumeral3}}: \tau_{u_x,s_x} =&1- \frac{1}{3}\left[\left[\eta_{u_x}(x)-\eta_{s_x}(x)\right]^3 +  \right.\\
	&\left.\left[\eta_{u_x}(x)-\eta_{s_x}(x)\right]^2 +\left[\eta_{u_x}(x)-\eta_{s_x}(x)\right] \right].
\end{align*}
Besides, we also consider the hybrid noise consisting of both feature-dependent noise and symmetric as well as asymmetric noise as in Sec. \ref{noise}. We use CIFAR-10 and CIFAR-100 benchmarks with such simulated label noise.

\textbf{Baselines.}\
Following the benchmark in \cite{zhang2020learning}, we compare the following baselines: 1) ERM; 2) LRT \cite{zheng2020error}; 3) GCE \cite{zhang2018generalized}; 4) MW-Net \cite{shu2019meta} and 5) PLC \cite{zhang2020learning}, which represents the SOTA method specifically designed for addressing heterogeneous feature-dependent label noise. All these methods are generic and handle label noise without assuming the noise structures.

\begin{table}[t]
	\setlength{\abovecaptionskip}{0.cm}
	\setlength{\belowcaptionskip}{-2cm}
	\caption{Comparison of different competing methods on mini WebVision dataset. Results for baseline methods are copied from \cite{li2019dividemix}. * denotes results trained with Inception-ResNet-v2.} \label{webvision}\vspace{-1mm}
	\centering
	\begin{tabular}{c|c|c|c|c}
		\toprule
		\multirow{2}{*}{Methods} &  \multicolumn{2}{c|}{WebVision} & \multicolumn{2}{c}{ILSVRC12}\\
		&   top1 & top5 & top1 & top5 \\
		\hline
		Forward* \cite{patrini2017making}& 61.12 &82.68 &  57.36 & 82.36  \\
		MentorNet* \cite{jiang2018mentornet}&  63.00& 81.40 &57.80& 79.92  \\
		Co-teaching* \cite{han2018co}  & 63.58 &85.20&61.48 &84.70 \\
		Interative-CV* \cite{chen2019understanding} &   65.24 &85.34& 61.60& 84.98 \\
		MW-Net \cite{shu2019meta} & 69.34 & 87.44 & 65.80 & 87.52  \\
		CMW-Net  & 70.56 & 88.76 & 66.44 & 87.68 \\
		DivideMix*  \cite{li2019dividemix} & 77.32 & 91.64& 75.20 & 90.84  \\
		ELR* \cite{liu2020early} & 77.78 & 91.68& 70.29 & 89.76  \\
		DivideMix \cite{li2019dividemix}&   76.32 & 90.65& 74.42 & 91.21  \\	
		CMW-Net-SL &   78.08 & 92.96& 75.72 & 92.52  \\
		\hline
		DivideMix with C2D \cite{zheltonozhskii2022contrast}& 79.42 & 92.32 &  \textbf{78.57} & 93.04  \\
		CMW-Net-SL+C2D & \textbf{80.44} & \textbf{93.36} &  {77.36} & \textbf{93.48}  \\
		\bottomrule
	\end{tabular} \vspace{-4mm}
\end{table}
\begin{table}\vspace{-0mm}
	\setlength{\abovecaptionskip}{0.cm}
	\setlength{\belowcaptionskip}{-2cm}
	\caption{performance comparison of classification accuracy (\%) on WebFG-496.}\label{FG} \vspace{0mm}
	\centering
	\setlength{\tabcolsep}{1mm}{
		\begin{tabular}{c|c|c|c|c}
			\toprule
			Methods & Web-Bird & Web-Aircraft & Web-Car &Average   \\ \hline
			ERM & 66.56 & 64.33 & 67.42 &66.10  \\
			Decoupling \cite{malach2017decoupling}& 70.56 &75.97 &75.00& 73.84  \\
			Co-teaching \cite{han2018co}& 73.85& 72.76 &73.10 &73.24    \\
			Peer-learning \cite{sun2021webly} & 76.48 &74.38& 78.52 &76.46  \\
			MW-Net      & 75.60 & 72.93 & 77.33 &  75.29 \\
			CMW-Net     & 75.72 & 73.72 & 77.42 & 75.62  \\
			CMW-Net-SL & \textbf{77.41} & \textbf{76.48} & \textbf{79.70 }& \textbf{77.86 } \\
			\bottomrule
	\end{tabular}}\vspace{-6mm}
\end{table}
\begin{figure*}
	\centering 
	\subfigure[Typical noisy labeled samples corrected by our method from Animal-10N \cite{song2019selfie}. The original training label is {\color{red}{cat}}.]{\label{figsm1}
	\begin{minipage}[b]{0.12\textwidth} 
		\centering 
		\centerline{\includegraphics[width=0.9\textwidth, height=0.9\textwidth]{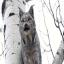}} 
		\centerline{\color{blue}{lynx}}
	\end{minipage}
	\begin{minipage}[b]{0.12\textwidth} 
		\centering 
		\centerline{\includegraphics[width=0.9\textwidth, height=0.9\textwidth]{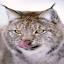}} 
		\centerline{\color{blue}{lynx}}
	\end{minipage}
	\begin{minipage}[b]{0.12\textwidth} 
		\centering 
		\centerline{\includegraphics[width=0.9\textwidth, height=0.9\textwidth]{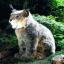}} 
		\centerline{\color{blue}{lynx}}
	\end{minipage}
	\begin{minipage}[b]{0.12\textwidth} 
		\centering 
		\centerline{\includegraphics[width=0.9\textwidth, height=0.9\textwidth]{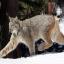}} 
		\centerline{\color{blue}{lynx}}
	\end{minipage}
	\begin{minipage}[b]{0.12\textwidth} 
		\centering 
		\centerline{\includegraphics[width=0.9\textwidth, height=0.9\textwidth]{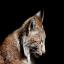}} 
		\centerline{\color{blue}{lynx}}
	\end{minipage}
	\begin{minipage}[b]{0.12\textwidth} 
		\centering 
		\centerline{\includegraphics[width=0.9\textwidth, height=0.9\textwidth]{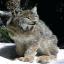}} 
		\centerline{\color{blue}{lynx}}
	\end{minipage}
	\begin{minipage}[b]{0.12\textwidth} 
		\centering 
		\centerline{\includegraphics[width=0.9\textwidth, height=0.9\textwidth]{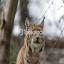}} 
		\centerline{\color{blue}{lynx}}
	\end{minipage}
	\begin{minipage}[b]{0.12\textwidth} 
		\centering 
		\centerline{\includegraphics[width=0.9\textwidth, height=0.9\textwidth]{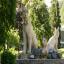}} 
		\centerline{\color{blue}{lynx}}
\end{minipage}}  \\ \vspace{-2mm}
\subfigure[Typical noisy labeled samples corrected by our method from Animal-10N \cite{song2019selfie}. The original training label is {\color{red}{lynx}}.]{\label{figsm2}
	\begin{minipage}[b]{0.12\textwidth} 
		\centering 
		\centerline{\includegraphics[width=0.9\textwidth, height=0.9\textwidth]{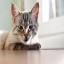}} 
		\centerline{\color{blue}{cat}}
	\end{minipage}
	\begin{minipage}[b]{0.12\textwidth} 
		\centering 
		\centerline{\includegraphics[width=0.9\textwidth, height=0.9\textwidth]{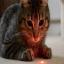}} 
		\centerline{\color{blue}{cat}}
	\end{minipage}
	\begin{minipage}[b]{0.12\textwidth} 
		\centering 
		\centerline{\includegraphics[width=0.9\textwidth, height=0.9\textwidth]{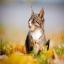}} 
		\centerline{\color{blue}{cat}}
	\end{minipage}
	\begin{minipage}[b]{0.12\textwidth} 
		\centering 
		\centerline{\includegraphics[width=0.9\textwidth, height=0.9\textwidth]{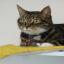}} 
		\centerline{\color{blue}{cat}}
	\end{minipage}
	\begin{minipage}[b]{0.12\textwidth} 
		\centering 
		\centerline{\includegraphics[width=0.9\textwidth, height=0.9\textwidth]{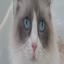}} 
		\centerline{\color{blue}{cat}}
	\end{minipage}
	\begin{minipage}[b]{0.12\textwidth} 
		\centering 
		\centerline{\includegraphics[width=0.9\textwidth, height=0.9\textwidth]{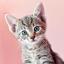}} 
		\centerline{\color{blue}{cat}}
	\end{minipage}
	\begin{minipage}[b]{0.12\textwidth} 
		\centering 
		\centerline{\includegraphics[width=0.9\textwidth, height=0.9\textwidth]{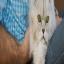}} 
		\centerline{\color{blue}{cat}}
	\end{minipage}
	\begin{minipage}[b]{0.12\textwidth} 
		\centering 
		\centerline{\includegraphics[width=0.9\textwidth, height=0.9\textwidth]{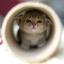}} 
		\centerline{\color{blue}{cat}}
\end{minipage}} \\ \vspace{-2mm}
\subfigure[Typical noisy labeled samples corrected by our method from mini-WebVision \cite{li2017webvision}. The original training label is {\color{red}{tailed frog}}.]{ \label{figsm3}
	\begin{minipage}[b]{0.12\textwidth} 
		\centering 
		\centerline{\includegraphics[width=0.9\textwidth, height=0.9\textwidth]{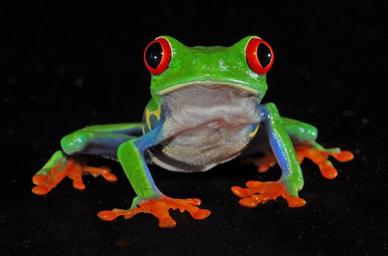}} 
		\centerline{\color{blue}{tree frog}}
	\end{minipage}
	\begin{minipage}[b]{0.12\textwidth} 
		\centering 
		\centerline{\includegraphics[width=0.9\textwidth, height=0.9\textwidth]{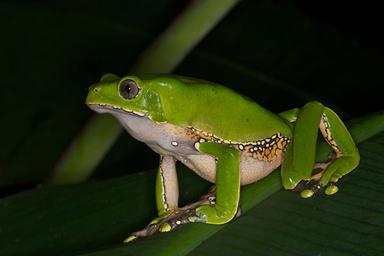}} 
		\centerline{\color{blue}{tree frog}}
	\end{minipage}
	\begin{minipage}[b]{0.12\textwidth} 
		\centering 
		\centerline{\includegraphics[width=0.9\textwidth, height=0.9\textwidth]{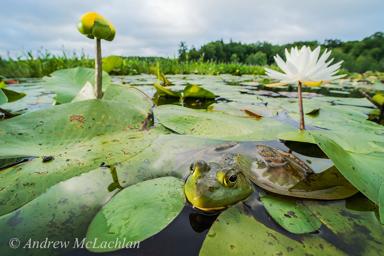}} 
		\centerline{\color{blue}{bullfrog}}
	\end{minipage}
	\begin{minipage}[b]{0.12\textwidth} 
		\centering 
		\centerline{\includegraphics[width=0.9\textwidth, height=0.9\textwidth]{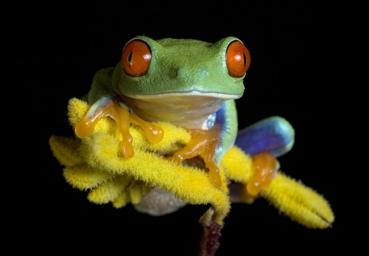}} 
		\centerline{\color{blue}{tree frog}}
	\end{minipage}
	\begin{minipage}[b]{0.12\textwidth} 
		\centering 
		\centerline{\includegraphics[width=0.9\textwidth, height=0.9\textwidth]{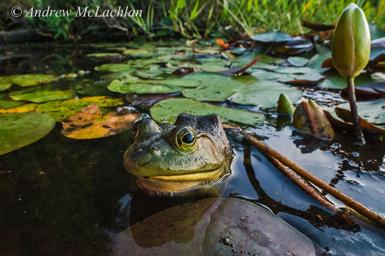}} 
		\centerline{\color{blue}{bullfrog}}
	\end{minipage}
	\begin{minipage}[b]{0.12\textwidth} 
		\centering 
		\centerline{\includegraphics[width=0.9\textwidth, height=0.9\textwidth]{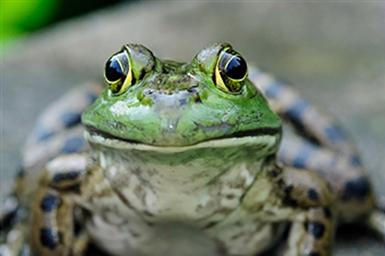}} 
		\centerline{\color{blue}{bullfrog}}
	\end{minipage}
	\begin{minipage}[b]{0.12\textwidth} 
		\centering 
		\centerline{\includegraphics[width=0.9\textwidth, height=0.9\textwidth]{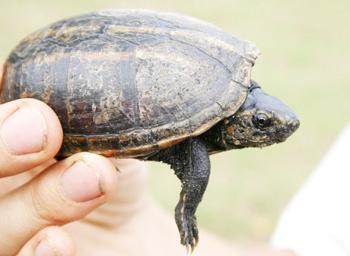}} 
		\centerline{\color{blue}{mud turtle}}
	\end{minipage}
	\begin{minipage}[b]{0.12\textwidth} 
		\centering 
		\centerline{\includegraphics[width=0.9\textwidth, height=0.9\textwidth]{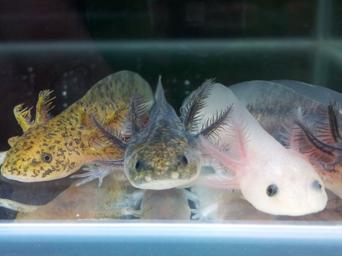}} 
		\centerline{\color{blue}{axolotl}}
	\end{minipage}}
	\vspace{-6mm}
	\caption{Examples of randomly selected samples with noisy labels corrected by our method. The original training labels and generated pseudo-labels by model are shown in {\color{red}{red}} and {\color{blue}{blue}}, respectively. More comprehensive examples are depicted in the SM.}	\label{figwebfig}   \vspace{-4mm}
\end{figure*}

\textbf{Results.}\
Table \ref{classfea1} lists the performance of different competing methods under three types of feature-dependent noise at noise levels 35\% and 70\%. It is seen that our method achieves the best performance
on all cases. Table \ref{classfea2} further shows the results on datasets corrupted with a combination of feature dependent and independent noises, where feature-independent noise is overlayed on the feature-dependent one and thus bias patterns are more complicated. The superiority of the proposed method can still be easily observed.

\begin{figure}[t] \vspace{-2mm}
	\centering
	\subfigcapskip=-1mm
	\subfigure[ANIMAL-10N]{\label{figanimal}
		\includegraphics[width=0.21\textwidth]{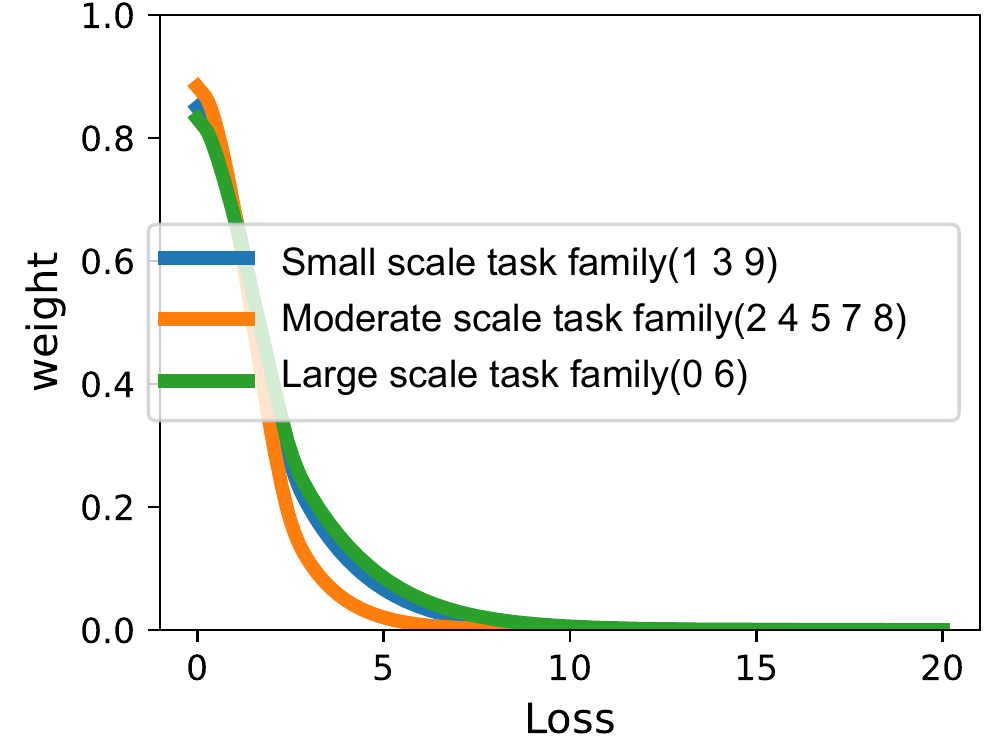}}  \ \
	\subfigure[mini WebVision]{\label{figwebvision}
		\includegraphics[width=0.21\textwidth]{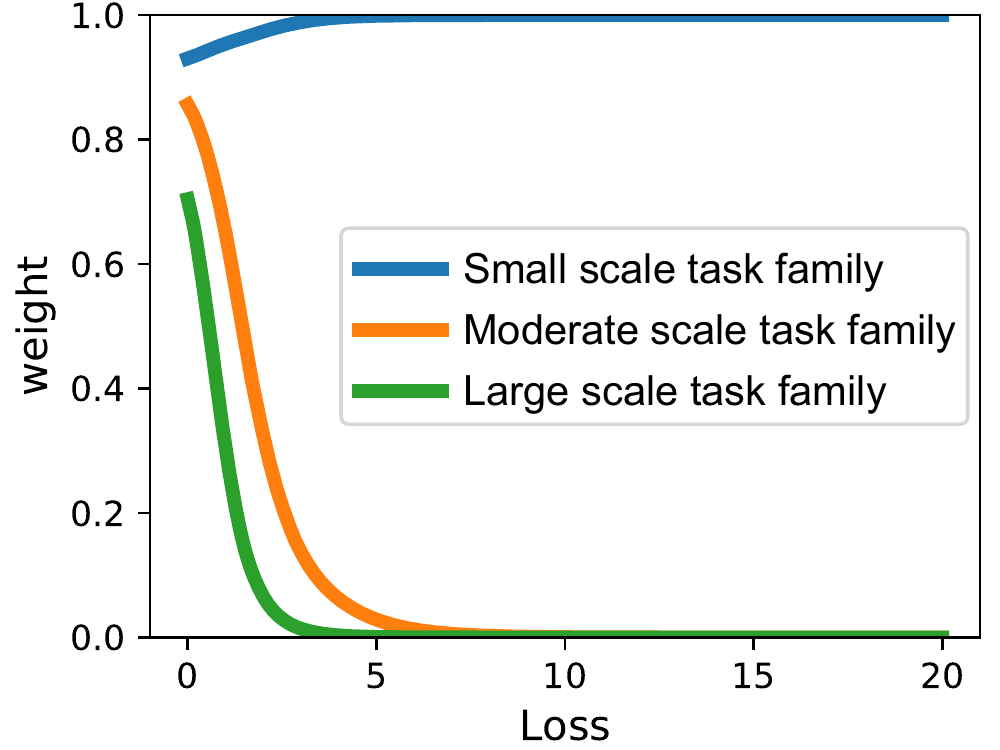}}\vspace{-4mm}
	\caption{Weighting schemes learned by CMW-Net on real biased datasets.}\vspace{-6mm}
\end{figure}

The generation mechanism of such feature-dependent noise results in noisy samples near the decision boundary \cite{zhang2020learning}, which are harder to distinguish and more likely to be mislabeled. From Fig.\ref{figah}, it can be seen that the proposed method can still finely distinguish most clean and noisy samples (some noisy samples are wrongly assigned to high weights due to they are samples near the decision boundary). As compared, from Fig.\ref{figad}, it can be observed that MW-Net totally fails to distinguish clean and noisy samples, and assigns high weights to all samples, naturally leading to its performance degeneration. Note that the PLC method \cite{zhang2020learning}, which is specifically designed for feature-dependent label noise data, also achieves fine results. The main idea of this method is to progressively correct noisy labels and refine the model for those relatively reliable samples with high confidence, measured by a dynamically specified threshold gradually deceased in iterations. Considering its general availability to a wider range of data bias cases and relatively more concise meta-learning framework, it should be rational to say that the proposed method is effective.

\vspace{-4mm}
\section{Learning with  Real Biased Data}	\label{real}
\subsection{Learning with Real-world Noisy Datasets}
\textbf{Datasets.}\ We adopt two real-world datasets, ANIMAL-10N \cite{xiao2015learning} and WebVision \cite{song2019selfie}. Animal-10N contains 55,000 human-labeled online images for 10 confusing animal classes, all with approximately similar noisy label distributions (8\% noisy samples). Following previous works \cite{song2019selfie}, 50,000 images are exploited for training while the left for testing. For ease of comparison to previous works \cite{jiang2018mentornet,chen2019understanding}, we consider the mini WebVision dataset which contains the top 50 classes from the Google image subset of WebVision. The performance evaluation is implemented on both the validation sets of mini WebVision \cite{li2017webvision} and the corresponding class samples of ImageNet \cite{deng2009imagenet}. ResNet-50 is adopted as the classifier network. More implementation details are specified in SM.

\textbf{Results.}\
Tables \ref{animal} and \ref{webvision} compare the test performance of all competing methods trained on the Animal-10N and mini WebVision datasets, respectively. For the Animal-10N dataset, we compare $4$ methods that have reported performance on this dataset. Compared with sample selection methods ActiveBias \cite{chang2017active} and Co-teaching \cite{han2018co}, our CMW-Net attains better performance, showing its better screening capability for useful samples. Under soft-label amelioration, our method achieves a further performance gain over recent label correction methods SELFIE \cite{song2019selfie} and PLC \cite{zhang2020learning}. Figs. \ref{figsm1} and \ref{figsm2}
visualize typical noisy examples selected by CMW-Net as well as its generated pseudo-labels. Though they are a pair of easily confused categories (cat, lynx), our method can still extract their wrong labels and correct them as the true ones.
Fig. \ref{figanimal} further shows the weighting functions learned by CMW-Net, complying with the class balance and inter-class noise homogeneity property of this dataset.

\begin{table}[t]
	\setlength{\abovecaptionskip}{0.cm}
	\setlength{\belowcaptionskip}{-2cm}
	\caption{Long-tail recognition accuracy of different competing methods by using ResNet-10 as the classifier on ImageNet-LT \cite{liu2019large} . Results for baselines are copied from \cite{liu2019large}.} \label{imagenetlt}\vspace{-1mm}
	\centering
	\begin{tabular}{c|c|c|c|c}
		\toprule
		\multirow{2}{*}{Methods} &  \multicolumn{4}{c}{Accuracy} \\ \cline{2-5}
		&   Many & Medium &Few & Overall \\
		\hline
		ERM & 40.9 &10.7 &0.4 &20.9 \\
		Lifted Loss \cite{oh2016deep} & 35.8 & 30.4 & 17.9& 30.8 \\
		Focal loss \cite{lin2018focal} & 36.4 &29.9 &16 &30.5 \\
		Range Loss \cite{zhang2017range}  & 35.8& 30.3 &17.6& 30.7\\
		OLTR \cite{liu2019large}  & 43.2 & 35.1 & 18.5 & 35.6 \\ \hline
		OLTR \cite{liu2019large} + CMW-Net  & \textbf{47.2} & \textbf{39.2} & \textbf{19.7} & \textbf{39.5} \\
		\bottomrule
	\end{tabular}  \vspace{-4mm}
\end{table}

Table \ref{webvision} also shows the superiority of CMW-Net compared to other competing methods without involving soft labels. By introducing soft labels, our CMW-Net-SL achieves superior performance to recent SOTA methods, DivideMix and ELR. Furthermore, by combining the self-supervised pretraining technique proposed in the C2D method \cite{zheltonozhskii2022contrast}, we further boost the performance. In Fig. \ref{figwebfig}, we show some typical noisy examples corrected by the proposed method, showing its capability of recovering these easily confused samples.

Fig. \ref{figwebvision} plots the learned weighting functions by CMW-Net, revealing certain helpful data bias insight. For small-scale task family, the corresponding weighting function is with larger weights and shows increasing tendency. It is beneficial to more emphasize their contained rare samples especially those marginal informative ones for alleviating their possibly encountered class-imbalance bias issue. Yet for moderate and larger-scale task family containing relatively abundant training data, the weighting functions are with monotonically decreasing shapes to suppress the negative effect brought by their contained noisy samples. Such more comprehensive and faithful exploration and encoding for data bias situations naturally leads to the better performance of CMW-Net than conventional sample weighting strategies.

\subsection{Webly Supervised Fine-Grained Recognition}	
We further run our method on a benchmark WebFG-496 dataset proposed in \cite{sun2021webly}, consisting of three sub-datasets: Web-aircraft, Web-bird, Web-car, which contain 13,503 images with 100 types of airplanes, 18,388 images with 200 species of birds, and 21,448 images with 196 categories of cars, respectively. The aim is to use web images to train a fine-grained recognition model. The data bias of this dataset is validated to be complicated, with both label noise and class imbalance patterns, as well as certain inter-class variance \cite{sun2021webly}. Experimental results are shown in Table \ref{FG}. It is seen that CMW-Net-SL evidently improves other reported SOTA performance \cite{sun2021webly}. This further validates the effectiveness of our method for such real dataset with complex data biases. More implementation details are given in SM.

\begin{table}[t]
	\setlength{\abovecaptionskip}{0.cm}
	\setlength{\belowcaptionskip}{-2cm}
	\caption{ Validation accuracy of InceptionResNet-v2 with transferable CMW-Net and different competing methods on full WebVision and ImageNet validation sets. Results for baselines are copied from original papers.} \label{fullwebvision}
	\centering
	\begin{tabular}{c|c|c|c|c}
		\toprule
		\multirow{2}{*}{Methods} &  \multicolumn{2}{c|}{WebVision} & \multicolumn{2}{c}{ILSVRC12} \\ \cline{2-5}
		&    top1 & top5 & top1 & top5   \\
		\hline
		ERM & 69.7 &87.0 & 62.9 & 83.6  \\
		MentorNet \cite{jiang2018mentornet}& 70.8& 88.0 &62.5& 83.0  \\
		MentorMix \cite{jiang2020beyond}   & {74.3} &90.5&67.5 &{87.2}  \\
		HAR \cite{cao2020heteroskedastic} & 75.0 & 90.6 & 67.1 & 86.7 \\
		 MILE \cite{rajeswar2022multi} & 76.5 & 90.9 & 68.7 &  86.4  \\
		Heteroscedastic \cite{collier2021correlated} & 76.6  &  92.1 & 68.6 &  87.1  \\	
	CurriculumNet \cite{guo2018curriculumnet}  & \textbf{79.3}	 & \textbf{93.6} & - & - \\  \hline
		ERM + CMW-Net-SL   & 77.9 & 92.6 & \textbf{69.6} & \textbf{88.5} \\
		\bottomrule
	\end{tabular}  \vspace{-4mm}
\end{table}

\section{Transferability of CMW-Net} \label{trans}
As aforementioned, a potential usefulness of the meta-learned weighing scheme by CMW-Net is that it is model-agnostic and hopefully equipped into other learning algorithms in a plug-and-play manner. To validate such transferable capability of CMW-Net, we attempt to transfer meta-learned CMW-Net on relatively smaller dataset to significantly larger-scale ones. In specific, we use CMW-Net trained on CIFAR-10 with feature-dependent label noise (i.e.,35\% Type-\uppercase\expandafter{\romannumeral1} + 30\% Asymmetric) as introduced in Sec. 4.3 since it finely simulates the real-world noise configuration. The extracted weighting function is depicted in Fig.\ref{figad}. We deploy it on two large-scale real-world biased datasets, ImageNet-LT \cite{liu2019large} and full WebVision \cite{li2017webvision}.

Table \ref{imagenetlt} shows the performance on ImageNet-LT. By readily equipping our learned CMW-Net upon the SOTA OLTR algorithm \cite{liu2019large} on this dataset, it can be seen that around 4\% higher overall accuracy can be readily obtained.
Besides, the performance on full WebVision is compared in Table \ref{fullwebvision}.

It is interesting to see that by directly integrating the learned CMW-Net into the simple ERM algorithm with more training epochs, the performance can be further improved, outperforming most of these SOTA methods, only slightly inferior to the CurriculumNet method \cite{guo2018curriculumnet}, whose results were obtained with ensemble of six models. Even with a relatively concise form, our method still outperforms the second-best Heteroscedastic method by an evident margin. This further validates the potential usefulness of CMW-Net to practical large-scale problems with complicated data bias situations, with an intrinsic reduction of the labor and computation costs by readily specifying proper weighting scheme for a learning algorithm. More experimental details are presented in SM.

\vspace{-3mm}
\section{Extensional Applications}		\label{application}
We then evaluate the generality of our proposed adaptive sample weighting strategy in more robust learning tasks, including partial-label learning and semi-supervised learning. The experiments on the selective classification task are introduced in SM due to page limitation.
\vspace{-3mm}
\subsection{Partial-Label Learning}
\subsubsection{Problem Formulation}	
Partial-label learning (PLL) \cite{jin2002learning} aims to deal with the problem where each instance is provided with a set of candidate labels, only one of which is the correct label.
Denote $\mathcal{X} \subset \mathbb{R}^d$ as the input space, $\mathcal{Y} :=\{1,\cdots,C\}$ as the label space, where $C$ is the number of all training classes. Denote the partially labeled dataset as $\mathcal{D}_{PLL} = \{(x_i,Y_i)\}_{i=1}^N$, where $Y_i \in \mathcal{Y}$ is the candidate label set of $x_i$. 
The goal of PLL is to find latent ground-truth label $y$ for each of $x_i$s through observing their partial label sets. The basic definition of PLL is that true label $y$ of an instance $x$ must be in its candidate label set ${Y}$.
The PLL risk estimator is then defined as:
\begin{equation} \label{pc}
	\mathcal{R}_{PLL}(f) = \mathbb{E}_{P(x, Y)} [\ell_{PLL}(f(x), {Y})],
\end{equation} 		
where $\ell_{PLL}(\cdot, \cdot)$ is the loss function and $f(\cdot)$ is the classifier.

To estimate Eq.(\ref{pc}), it usually treats all the candidate labels equally \cite{jin2002learning}, i.e., $\ell_{PLL}(f(x), Y) = \frac{1}{|Y|} \sum_{y\in Y} \ell(f(x), y)$. Considering that only the true label contributes to retrieving the classifier, PRODEN \cite{lv2020progressive} defines the PLL loss as the minimal loss over
the candidate label set:
\begin{align*}
	\ell_{PLL}(f(x), {Y}) = \min_{y\in Y} \ell(f(x), y).
\end{align*}
They further relax the $\min$ operator of the above equality by the dynamic weights as follows:
\begin{align*}
	\ell_{PLL}(f(x), {Y}) =\sum_{y\in Y}  w_{y} \ell(f(x), y),
\end{align*}
where all $w_y$s consist of a one-hot vector, expected to reflect the confidence of the label $y\in Y$ being the true label.

\begin{figure}[t]
	\centering
	\subfigcapskip=-1mm
	\includegraphics[width=0.24\textwidth]{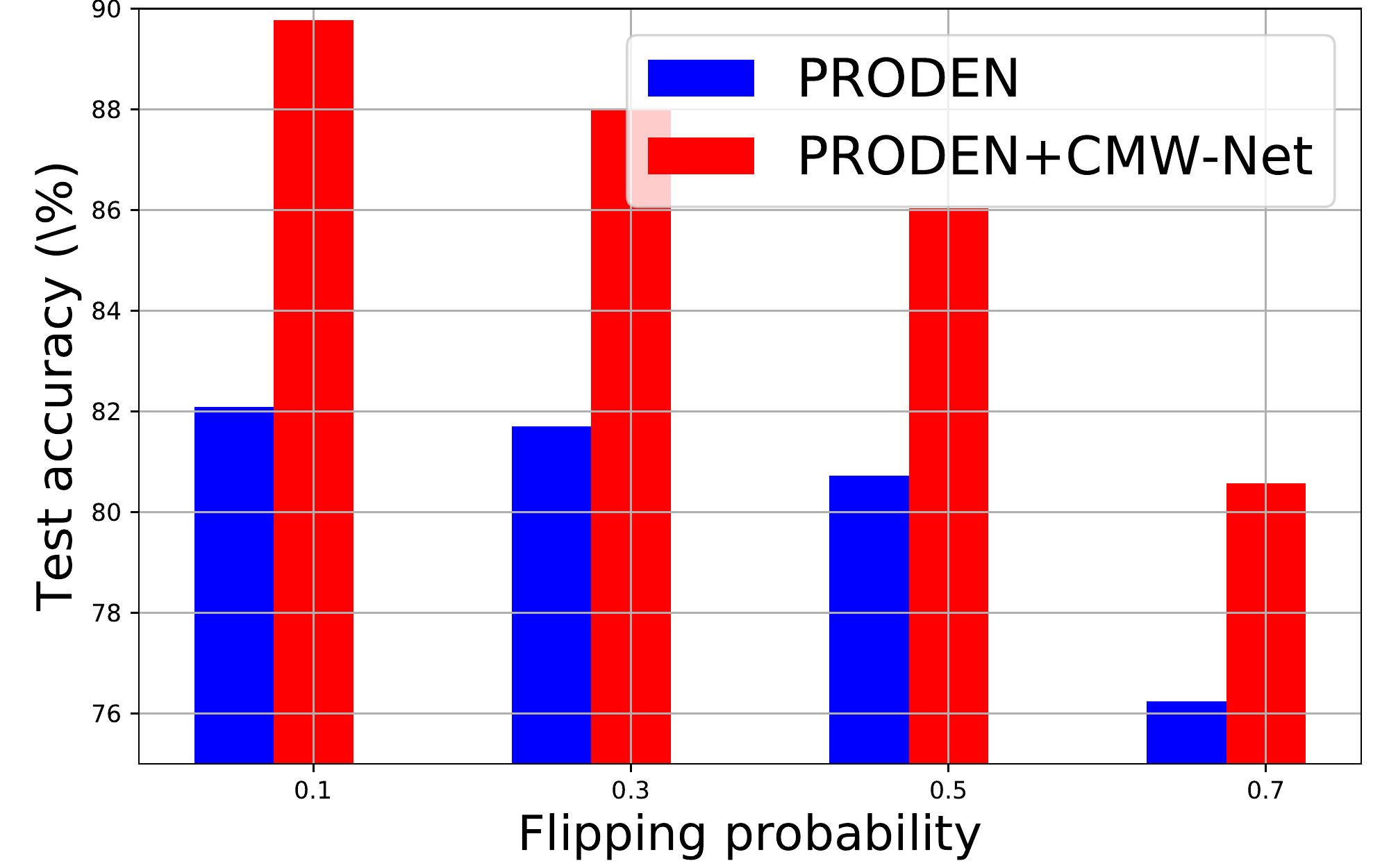}  \hspace{-1mm}
	\includegraphics[width=0.24\textwidth]{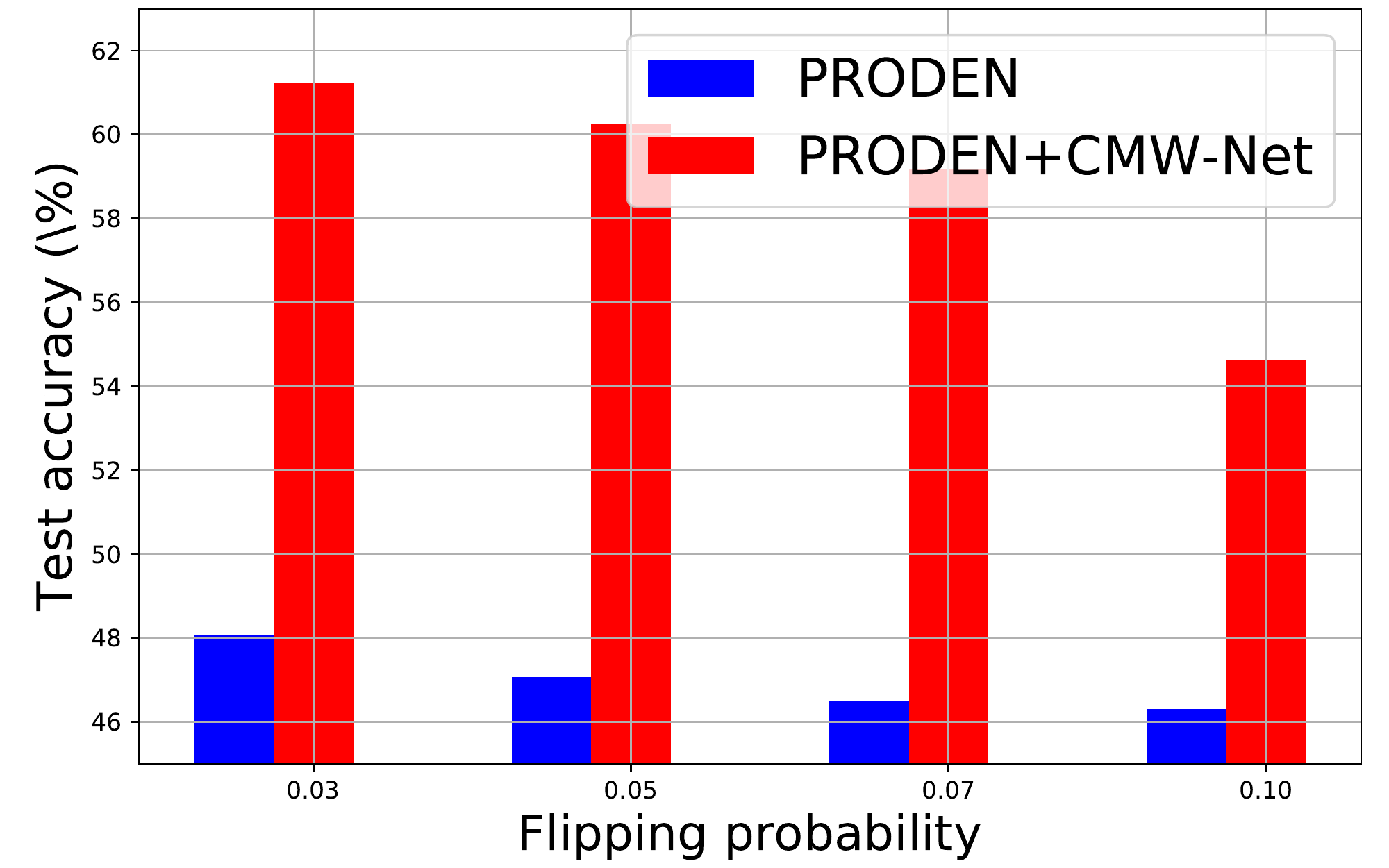}\vspace{-8mm}
	\caption{Accuracy comparisons on PRODEN w/o CMW-Net strategy over (left) CIFAR-10 and (right) CIFAR100 under partial label learning setting.}\label{figpatialnet}  \vspace{-4mm}
\end{figure}

\vspace{-1mm}
\subsubsection{CMW-Net Amelioration and Experiments}
PRODEN \cite{lv2020progressive} represents the recent SOTA method against such PLL task, through progressively identifying the true labels from the partial label sets, and refining them in turn to ameliorate the classifier. Then by taking the samples with predicted labels as training data, which still contain many wrong annotations, the CMW-Net method (i.e., Algorithm \ref{alg:example} by using meta-data establishing technique introduced in Sec. 4.2) can be readily employed to further improve the classifier.

Following PRODEN \cite{lv2020progressive}, two sets of partial label datasets are generated from CIFAR-10 and CIFAR-100, respectively, under different flipping probabilities. As shown in Fig. \ref{figpatialnet}, it is seen that CMW-Net can significantly enhance the performance of the baseline method in both test cases, showing its potential usability in this PLL task. More experimental settings and results are presented in SM.

\begin{table*}
\setlength{\abovecaptionskip}{0.cm}
\setlength{\belowcaptionskip}{-2cm}
\caption{Performance comparison of Fixmatch w/o CMW-Net on CIFAR-10, CIFAR-100 and ImageNet datasets in test error over 3 trials. 
The baselines results of CIFAR are copied from \cite{sohn2020fixmatch}, and those of ImageNet are copied from \cite{cai2021exponential}. } \label{semi}\vspace{-1mm}
\centering
\begin{tabular}{l|c|c|c|c|c|c|c|c}
\toprule
&   \multicolumn{3}{c|}{CIFAR-10}   &    \multicolumn{3}{c|}{CIFAR-100} & \multicolumn{2}{c}{ImageNet (10\% labels)}\\
\hline
Method & 40 labels & 250 labels & 4000 labels &400 labels & 2500 labels & 10000 labels & top-1 & top5\\ \hline
FixMatch (RA) \cite{sohn2020fixmatch}& 13.81 $\pm$ 3.37 & 5.07 $\pm$ 0.65 & 4.26 $\pm$ 0.05 & 48.85 $\pm$ 1.75 & 28.29 $\pm$ 0.11 & 22.60 $\pm$ 0.12 &32.9  & 13.3\\ \hline
FixMatch (RA) + CMW-Net &  \textbf{9.60 $\pm$ 0.62} & \textbf{4.73 $\pm$ 0.15}  & \textbf{4.25 $\pm$ 0.03}  &\textbf{ 47.70 $\pm$ 1.14}  & \textbf{27.43 $\pm$ 0.12} & \textbf{22.55 $\pm$ 0.09}   & \textbf{30.8} & \textbf{ 11.3} \\
\bottomrule
\end{tabular}\vspace{-4mm}
\end{table*}

\vspace{-3mm}
\subsection{Semi-Supervised Learning}
\subsubsection{Problem Formulation}
To reduce the annotation cost for supervised learning, an alternative strategy is to train the classifier with small labeled set as well as a large amount of unlabeled samples. This constitutes the main aim of semi-supervised learning (SSL). Let $D = \{D_L,D_U\}$ denote the entire dataset, including a small labeled dataset $D_L=\{(x_i,y_i)\}_{i=1}^L$ and a large-scale unlabeled dataset $D_U=\{(x_i)\}_{i=1}^U$, and $L\ll U$.
Formally, SSL aims to solve the following optimization problem \cite{yang2021survey}:
\begin{align*}
	\min_{\mathbf{w}} \sum_{(x_l,y_l)\in D_L} \mathcal{L}_S (x_l,y_l;\mathbf{w}) + \alpha \sum_{x_u\in D_U}  \mathcal{L}_U (x_u;\mathbf{w}),
\end{align*}
where $\mathcal{L}_S$ denotes the supervised loss, e.g., cross-entropy for classification, and $\mathcal{L}_U$ denotes the unsupervised loss, e.g., consistency loss \cite{xie2020unsupervised} or a regularization term \cite{miyato2018virtual}. $\mathbf{w}$ denotes the model parameters and $\alpha > 0$ denotes the compromise parameter balancing two terms.

Generally, different specifications of the unsupervised loss $\mathcal{L}_U$ lead to different SSL algorithms.
One commonly used strategy is the Pseudo Labeling approach \cite{lee2013pseudo}, which aims to sufficiently use labelled data to predict the labels of the unlabeled data, and take these pseudo-labeled data as labeled ones in training (reflected in the term $\mathcal{L}_U$). Recently, the SOTA SSL methods, like VAT \cite{miyato2018virtual}, MixMatch \cite{berthelot2019mixmatch}, UDA \cite{xie2020unsupervised}, and FixMatch \cite{sohn2020fixmatch} makes good progress to enhance the pseudo labeling capability by using sample augmentation techniques, through encouraging consistency under different augmented data \cite{xie2020unsupervised}. We take the recent SOTA Fixmatch \cite{sohn2020fixmatch} as a typical example.
Denote $\alpha(x_l)$ and $\mathcal{A}(x_u)$ as augmentation operators imposed on labeled and unlabeled samples, respectively. Then the FixMatch model can be written as:
\begin{align}
	\min_{\mathbf{w}} &\sum_{(x_l,y_l)\in D_L} \ell(f(\alpha(x_l);\mathbf{w}),y_l) \nonumber \\
	&  +\alpha \sum_{x_u \in D_U}  \mathbf{1}(\max(z_{u}\geq \tau) \ell(f(\mathcal{A}(x_u);\mathbf{w}),{y}_u), \label{SSL1}
\end{align}
where ${y}_u$ is the pseudo label on $x_u$, calculated by $y_u=\arg\max_{j} z_{uj}, z_u = f(\alpha(x_u);\mathbf{w})$ in iteration. $\tau$ is a scalar hyperparameter denoting the threshold above which we retain a pseudo-label. Note that $ \mathbf{1}(\max(z_{u}\geq \tau) $ corresponds to a hard weighting scheme with manually specified hyperparameter $\tau$. Albeit attaining good performance, the above FixMatch model is still with limitation that its hard-thresholding weighting scheme treats all unlabeled (augmented) samples equally, and its involved hyper-parameter $\tau$ is often not easily and adaptably specified against different tasks. The method thus still has room for further performance enhancement.
\vspace{-1mm}
\subsubsection{CMW-Net Amelioration and Experiments}\vspace{-1mm}
To better distinguish clean and noisy pseudo-labels, we can easily substitute the original hard weighting scheme as CMW-Net, to make sample weights capable of more sufficiently reflecting noise extents and  adaptable to training data/tasks. Then the problem (\ref{SSL1}) can be ameliorated as:
\begin{align*}
	\min_{\mathbf{w}} &\sum_{(x_l,y_l)\in D_L} \ell(f(\alpha(x_l);\mathbf{w}),y_l) \\
	&  +\alpha\sum_{(x_u)\in D_U} \mathcal{V}(L_u^{tr}(\mathbf{w}),N_i;\Theta)  \ell(f(\mathcal{A}(x_u);\mathbf{w}),y_u).
\end{align*}
The algorithm can also be readily designed by integrating the updating step for the meta-parameter $\Theta$ into the original algorithm of Fixmatch (the labeled data are naturally used as meta data), so as to make the weighting scheme iteratively extracted together with the classifier parameter $\mathbf{w}$ in an automatic and more likely intelligent manner.

We conduct experiments on several standard SSL image classification benchmarks, including CIFAR-10, CIFAR-100 \cite{krizhevsky2009learning} and ImageNet dataset \cite{deng2009imagenet}. Results are shown in Table \ref{semi}. It is evident that our CMW-Net consistently helps improve the performance of FixMatch, showing its potential application prospects on this task.
Especially, when FixMatch is trained with smaller labeled data resources, pseudo labels generated by FixMatch tend to be relatively unreliable, naturally resulting in performance degradation. CMW-Net is capable of adaptively reducing the negative effect of unreliable pseudo labels, and thus improves FixMatch significantly in this case.
More experimental settings and results are presented in SM.

\vspace{-3mm}
\section{Conclusion and Discussion}\label{conclusion}
In this study, we have proposed a novel meta-model, called CMW-Net, for adaptively extracting an explicit sample weighing scheme directly from training data. Compared with current sample weighing approaches, CMW-Net is validated to possess better flexibility against complicated data bias situations with inter-class heterogeneity. Assisted by additional soft pseudo-label information, the proposed method achieves competitive (mostly superior) performance under various data bias cases, including class imbalance, feature independent or dependent label noise, and more practical real-world data bias scenarios, beyond those SOTA methods specifically designed on these robust learning tasks. The extracted weighting schemes can always help faithfully reveal bias insights underlying training data, making the good effect of the method rational and interpretable. Two potential application prospects of CMW-Net are specifically illustrated and substantiated. One is its fine task-transferability of the learned weighting scheme, implying a possible efficiency-speedup methodology for handling robust learning tasks under big data, through avoiding its time-consuming and laborsome weighting function tuning process. The other is its wide range of possible extensional applications for other robust learning tasks,  e.g., partial-label learning, semi-supervised learning and selective classification.

In our future investigation, we'll apply the proposed adaptive sample weighting strategies to more robust learning tasks to further validate its generality. Attributed to its relatively concise modeling manner, it is also hopeful to develop deeper and more comprehensive statistical learning understanding for revealing its intrinsic generalization capability across different tasks \cite{shu2021learning}. Besides, we'll try to build more wider range of connections with our method to previous techniques on exploring data insights, like importance weighting  \cite{liu2015classification}. More sufficient and comprehensive task-level feature representation will also be further investigated in our future research. Further algorithm efficiency enhancement of our model will also be investigated in our future research.

\ifCLASSOPTIONcompsoc
\section*{Acknowledgments}
\else
\section*{Acknowledgment}
\fi

This work was supported by the National Key Research and Development Program of China under Grant 2021ZD0112900; in part by the National Natural Science
Foundation of China (NSFC) Project under Contract 61721002; and in part by the Macao Science and Technology Development Fund under Grant 061/2020/A2 and The Major Key Project of PCL  under contract PCL2021A12.

\newpage
\begin{IEEEbiography}[{\includegraphics[width=1in,height=1.25in,clip,keepaspectratio]{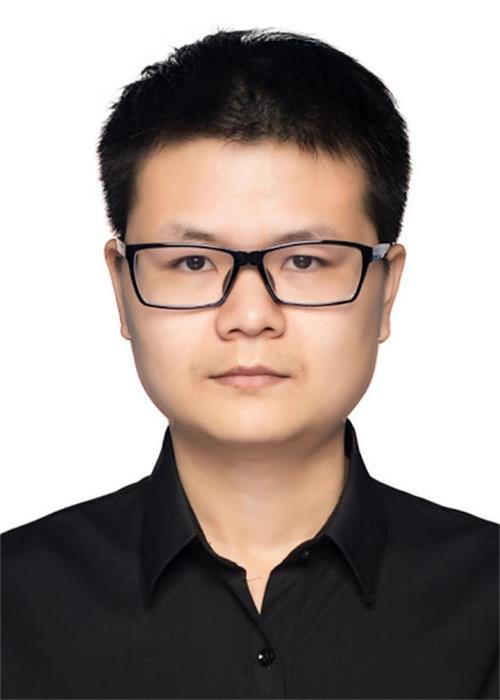}}]{Jun Shu}
	received the B.E. degree from Xi'an Jiaotong University, Xi'an, China, in 2016, where he is currently pursuing the Ph.D degree, under the tuition of Prof. Deyu Meng and Prof. Zongben Xu. His current research interests include machine learning and computer vision, especially on meta learning, robust deep learning and AutoML.
\end{IEEEbiography}
\begin{IEEEbiography}[{\includegraphics[width=1in,height=1.25in,clip,keepaspectratio]{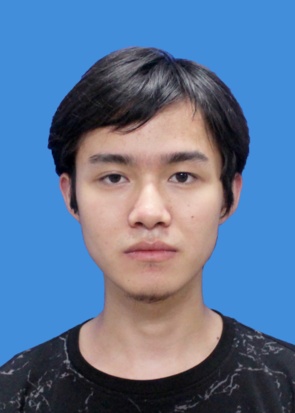}}]{Xiang Yuan}
	received the B.E. degree from Xi'an Jiaotong University, Xi'an, China, in 2020, where he is currently pursuing the Ph.D degree, under the tuition of Prof. Deyu Meng and Prof. Zongben Xu. His current research interests include meta learning and robust deep learning.
\end{IEEEbiography}

\begin{IEEEbiography}[{\includegraphics[width=1in,height=1.25in,clip,keepaspectratio]{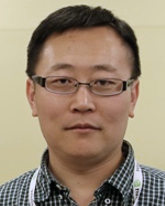}}]{Deyu Meng}
	received the B.Sc., M.Sc., and Ph.D. degrees from Xi'an Jiaotong University, Xi'an, China, in 2001, 2004, and 2008,
	respectively. He was a Visiting Scholar with Carnegie Mellon University, Pittsburgh, PA, USA, from 2012 to 2014. He is currently a Professor with the School of Mathematics and Statistics, Xi'an Jiaotong University, and an Adjunct Professor with the Faculty	of Information Technology, Macau University of	Science and Technology, Taipa, Macau, China. His research interests include model-based deep learning, variational networks, and meta learning.
\end{IEEEbiography}

\begin{IEEEbiography}[{\includegraphics[width=1in,height=1.25in,clip,keepaspectratio]{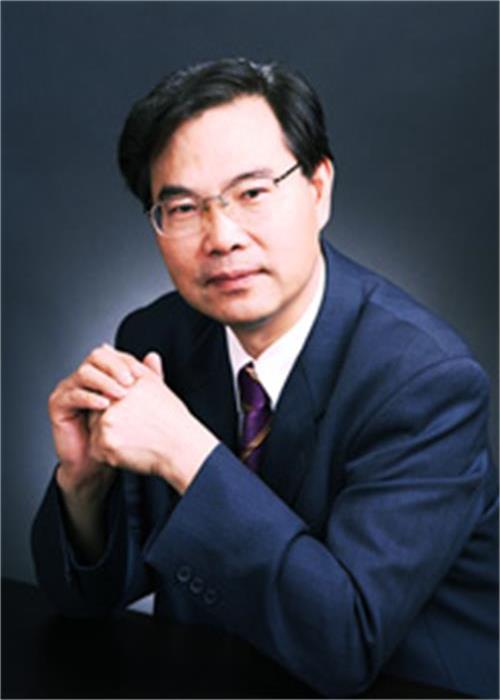}}]{Zongben Xu}
	received the PhD degree in mathematics from Xi'an Jiaotong University, Xi'an, China, in 1987. He currently serves as the Academician of the Chinese Academy of Sciences, the chief scientist of the National Basic Research	Program of China (973 Project), and the director
	of the Institute for Information and System Sciences with Xi’an Jiaotong University. His current research interests include nonlinear functional analysis and intelligent information processing.	He was a recipient of the National Natural Science Award of China, in 2007, and the winner of the CSIAM Su Buchin Applied Mathematics Prize, in 2008.
\end{IEEEbiography}

\clearpage
\onecolumn
\appendices

\section{Technical Details in Section 3} \label{sec1}
\subsection{Derivation of the Weighting Scheme in CMW-Net }
We first derive the equivalent forms of the updating steps for CMW-Net and CMW-Net-SL parameters $\Theta$, as expressed in Eqs. (10) and (15), in the main text, respectively.

Recall the update equation of the CWM-Net parameters as follows:
\begin{align}\label{eqcwm}
\Theta^{(t+1)} =  \Theta^{(t)} -\beta \frac{1}{m}\sum_{i=1}^{m} \nabla_{ \Theta} L_i^{meta}(\hat{\mathbf{w}} ^{(t+1)}(\Theta))\Big|_{\Theta^{(t)}}.
\end{align}
The gradient can be calculated by the following derivation:
\begin{align}\label{eq9}
\begin{split}
& \frac{1}{m}\sum_{i=1}^{m} \nabla_{ \Theta} L_i^{meta}(\hat{\mathbf{w}} ^{(t+1)}(\Theta))\Big|_{\Theta^{(t)}}\\
=& \frac{1}{m}\sum_{i=1}^{m} \frac{\partial L_i^{meta} (\hat{\mathbf{w}} ^{(t+1)}(\Theta))}{\partial \hat{\mathbf{w}} ^{(t+1)}(\Theta)} \frac{\partial \hat{\mathbf{w}} ^{(t+1)}(\Theta)}{\partial \Theta}\Big|_{ \Theta^{(t)}} \\
=& \frac{1}{m}\sum_{i=1}^{m} \frac{\partial L_i^{meta} (\hat{\mathbf{w}} ^{(t+1)}(\Theta))}{\partial \hat{\mathbf{w}} ^{(t+1)}(\Theta)} \sum_{j=1}^n \frac{\partial \hat{\mathbf{w}} ^{(t+1)}(\Theta)}{\partial \mathcal{V}(L_j^{tr}(\mathbf{w}^{(t)}),N_j;\Theta)} \frac{\partial \mathcal{V}(L_j^{tr}(\mathbf{w}^{(t)}),N_j;\Theta)}{\partial \Theta}\Big|_{ \Theta^{(t)}}.
\end{split}
\end{align}
Let
\begin{align}\label{eqg}
G_{ij}=\frac{\partial L_i^{meta} (\hat{\mathbf{w}})}{\partial \hat{\mathbf{w}}}\Big|_{\hat{\mathbf{w}} ^{(t+1)}(\Theta)}^T \frac{\partial L_j^{tr} (\mathbf{w})}{\partial \mathbf{w}} \Big|_{\mathbf{w}^{(t)}},
\end{align}
and by substituting Eqs. (\ref{eqg}) and (\ref{eq9}) into Eq. (\ref{eqcwm}), we can get:
\begin{align}\label{eq10}
\begin{split}
\Theta^{(t+1)}	= \Theta^{(t)} + {\alpha\beta}\sum_{j=1}^n \left(\frac{1}{m} \sum_{i=1}^{m}G_{ij} \right)  \frac{\partial \mathcal{V}(L_j^{tr}(\mathbf{w}^{(t)}),N_j;\Theta)}{\partial \Theta}\Big|_{ \Theta^{(t)}}.
\end{split}
\end{align}
This corresponds to Eq. (10) in the main text.

1) For the CMW-Net, since
\begin{align}
\hat{\mathbf{w}}^{(t+1)}(\Theta)=
\mathbf{w}^{(t)} - \alpha
\sum_{i=1}^n    \mathcal{V}(L_i^{tr}(\mathbf{w}^{(t)}),N_i;\Theta)\nabla_{\mathbf{w}} L_i^{tr}(\mathbf{w})\Big|_{\mathbf{w}^{(t)}},
\end{align}
thus we have
\begin{align*}
& \frac{1}{m}\sum_{i=1}^{m} \nabla_{ \Theta} L_i^{meta}(\hat{\mathbf{w}} ^{(t+1)}(\Theta))\Big|_{\Theta^{(t)}}\\
=& \frac{-\alpha}{m}\sum_{i=1}^{m} \frac{\partial L_i^{meta} (\hat{\mathbf{w}})}{\partial \hat{\mathbf{w}}}\Big|_{\hat{\mathbf{w}} ^{(t+1)}(\Theta)} \sum_{j=1}^n \frac{\partial L_j^{tr} (\mathbf{w})}{\partial \mathbf{w}}\Big|_{\mathbf{w}^{(t)}} \frac{\partial \mathcal{V}(L_j^{tr}(\mathbf{w}^{(t)}),N_j;\Theta)}{\partial \Theta}\Big|_{ \Theta^{(t)}}\\
=& -\alpha\sum_{j=1}^n \left(\frac{1}{m} \sum_{i=1}^{m} \frac{\partial L_i^{meta} (\hat{\mathbf{w}})}{\partial \hat{\mathbf{w}}}\Big|_{\hat{\mathbf{w}} ^{(t+1)}(\Theta)}^T \frac{\partial L_j^{tr} (\mathbf{w})}{\partial \mathbf{w}} \Big|_{\mathbf{w}^{(t)}}\right)  \frac{\partial \mathcal{V}(L_j^{tr}(\mathbf{w}^{(t)}),N_j;\Theta)}{\partial \Theta}\Big|_{ \Theta^{(t)}}.
\end{align*}

2) For the CMW-Net-SL, since
\begin{align}
\hat{\mathbf{w}}^{(t+1)}(\Theta)=
\mathbf{w}^{(t)} - \alpha
\sum_{i=1}^n   \left\{\mathcal{V}(L_i^{tr}(\mathbf{w}^{(t)}),N_i;\Theta)\nabla_{\mathbf{w}} L_i^{tr}(\mathbf{w})\Big|_{\mathbf{w}^{(t)}} + \left(1-\mathcal{V}(L_i^{tr}(\mathbf{w}^{(t)}),N_i;\Theta)\right)\nabla_{\mathbf{w}} L_i^{P_{se}}(\mathbf{w})\Big|_{\mathbf{w}^{(t)}} \right\},
\end{align}
where $L_i^{tr}(\mathbf{w}) = {\ell}(f(x_i;\mathbf{w}),y_i), L_i^{P_{se}}(\mathbf{w}) = {\ell}(f(x_i;\mathbf{w}),z_i) $, $z_i$ is the pseudo-label for example $x_i$,
we thus have
\begin{align}\label{wMsoft}
& \frac{1}{m}\sum_{i=1}^{m} \nabla_{ \Theta} L_i^{meta}(\hat{\mathbf{w}} ^{(t+1)}(\Theta))\Big|_{\Theta^{(t)}}\\
=& \frac{-\alpha}{m}\sum_{i=1}^{m} \frac{\partial L_i^{meta} (\hat{\mathbf{w}})}{\partial \hat{\mathbf{w}}}\Big|_{\hat{\mathbf{w}} ^{(t+1)}(\Theta)} \sum_{j=1}^n \left[\frac{\partial L_j^{tr} (\mathbf{w})}{\partial \mathbf{w}}\Big|_{\mathbf{w}^{(t)}} -\frac{\partial L_j^{P_{se}} (\mathbf{w})}{\partial \mathbf{w}}\Big|_{\mathbf{w}^{(t)}}\right ]\frac{\partial \mathcal{V}(L_j^{tr}(\mathbf{w}^{(t)}),N_j;\Theta)}{\partial \Theta}\Big|_{ \Theta^{(t)}}\\
=& -\alpha\sum_{j=1}^n \left(\frac{1}{m} \sum_{i=1}^{m} \frac{\partial L_i^{meta} (\hat{\mathbf{w}})}{\partial \hat{\mathbf{w}}}\Big|_{\hat{\mathbf{w}} ^{(t+1)}(\Theta)}^T \left[\frac{\partial L_j^{tr} (\mathbf{w})}{\partial \mathbf{w}}\Big|_{\mathbf{w}^{(t)}} -\frac{\partial L_j^{P_{se}} (\mathbf{w})}{\partial \mathbf{w}}\Big|_{\mathbf{w}^{(t)}}\right ]\right)  \frac{\partial \mathcal{V}(L_j^{tr}(\mathbf{w}^{(t)}),N_j;\Theta)}{\partial \Theta}\Big|_{ \Theta^{(t)}}.
\end{align}
Let
\begin{align}\label{GG}
G_{ij}=\frac{\partial L_i^{meta} (\hat{\mathbf{w}})}{\partial \hat{\mathbf{w}}}\Big|_{\hat{\mathbf{w}} ^{(t+1)}(\Theta)}^T \frac{\partial L_j^{tr} (\mathbf{w})}{\partial \mathbf{w}} \Big|_{\mathbf{w}^{(t)}},	G'_{ij}=\frac{\partial L_i^{meta} (\hat{\mathbf{w}})}{\partial \hat{\mathbf{w}}}\Big|_{\hat{\mathbf{w}} ^{(t+1)}(\Theta)}^T \frac{\partial L_j^{P_{se}} (\mathbf{w})}{\partial \mathbf{w}} \Big|_{\mathbf{w}^{(t)}},
\end{align}
by substituting Eqs. (\ref{wMsoft}) and (\ref{GG}) into Eq. (\ref{eqcwm}), we can get:
\begin{align}
\begin{split}
\Theta^{(t+1)}	= \Theta^{(t)} + {\alpha\beta}\sum_{j=1}^n \left[\frac{1}{m} \sum_{i=1}^{m}(G_{ij}-G'_{ij}) \right]  \frac{\partial \mathcal{V}(L_j^{tr}(\mathbf{w}^{(t)}),N_j;\Theta)}{\partial \Theta}\Big|_{ \Theta^{(t)}}.
\end{split}
\end{align}
This corresponds to Eq. (15) in the main text.

\subsection{Complete Learning Algorithm of CMW-Net-SL}
Recently, some works are presented to extract pseudo soft labels on samples through the clue of the classifier’s estimation during the training iterations, and then use such beneficial information to improve the robustness of classifier training especially in the presence of noisy labels. The utilized techniques include temporal ensembling \cite{laine2016temporal}, weight averaging, mixup \cite{zhang2018mixup}, and others. In our experiments, we just directly apply the strategy utilized in ELR [79] and DivideMix \cite{li2019dividemix}, which has been verified to be effective in tasks like semi-supervised learning \cite{laine2016temporal,sohn2020fixmatch} and robust learning \cite{arazo2019unsupervised,liu2020early}, to produce pseudo-labels in our CMW-Net-SL algorithm. The complete algorithm is summarized in the above Algorithm \ref{alg:example1}.



\begin{algorithm}[t]
	\vspace{0mm}
	\renewcommand{\algorithmicrequire}{\textbf{Input:}}
	\renewcommand{\algorithmicensure}{\textbf{Output:}}
	\caption{Learning Algorithm for CMW-Net-SL Model}
	\label{alg:example1}
	\begin{algorithmic}[1]  \small
		\REQUIRE  Training data $\mathcal{D}^{tr}$, meta data $\mathcal{D}^{meta}$, batch size $n$, temporal ensembling momentum $\alpha \in [0,1)$, weight averaging momentum $\beta\in [0,1)$, mixup hyperparameter $\gamma >0$, learning rates $\eta_1,\eta_2$.
		\ENSURE  Classifier parameter $\mathbf{w}^{*}$
		\STATE Initialize classifier network parameter $\mathbf{w}^{(0)}$. Initialize averaged predictions with noisy labels $\mathbf{z}^{(0)} = \mathbf{\hat{y}}_{[N\times C]}$, and averaged weights (untrainable) $\mathbf{w}_{WA}^{(0)} =\mathbf{0}$.
		\FOR{$t=0$ {\bfseries to} $T-1$}
		\STATE $\{x,y\} \leftarrow$ SampleMiniBatch($\mathcal{D}^{tr},n$).
		\STATE $\{x^{meta},y^{meta}\} \leftarrow$ SampleMiniBatch($\mathcal{D}^{meta},m$).
		\STATE 	Generate mixing coefficient $\lambda \sim Beta(\gamma,\gamma), \lambda=\max(\lambda,1-\lambda)$.
		\STATE Calculate weight averaging: $\mathbf{w}_{WA}^{(t+1)} = \beta \mathbf{w}_{WA}^{(t)} + (1-\beta) \mathbf{w}^{(t)}$.
		\STATE Calculate temporal ensembling: $\mathbf{z}^{(t+1)} = \alpha \mathbf{z}_i^{(t)} + (1-\alpha) f(x; \mathbf{w}_{WA}^{(t+1)})$.
		\STATE Generate new index sequence $\text{idx = torch.randperm(n)}$.
		\STATE Generate $\tilde{x} = \lambda' x +(1-\lambda')x[\text{idx}]$, and let $\tilde{y} = y[\text{idx}], \tilde{z}^{(t+1)} = {z}^{(t+1)}[\text{idx}]$, $\ell_i = {\ell}(f(\tilde{x}_i;\mathbf{w}),y_i)$, $\tilde{\ell}_i = {\ell}(f(\tilde{x}_i;\mathbf{w}),\tilde{y}_i)$. Calculate $N_i$ and $\tilde{N}_i$, representing the numbers of samples contained in the classes to which $x_i$ and $x[\text{idx}]_i$ belong, respectively.
		\STATE Formulate the learning manner of classifier network:
		\begin{align*}
		\hat{\mathbf{w}}^{(t+1)}(\Theta)=
		\mathbf{w}^{(t)} & - \eta_1
		\sum_{i=1}^n  \{ \lambda \left[ \mathcal{V}(\ell_i,N_i;\Theta)\nabla_{\mathbf{w}} {\ell}(f(\tilde{x}_i;\mathbf{w}),y_i)\Big|_{\mathbf{w}^{(t)}}   +(1-\mathcal{V}(\ell_i,N_i;\Theta)) \nabla_{\mathbf{w}} {\ell}(f(\tilde{x}_i;\mathbf{w}),\mathbf{z}_i^{(t+1)})\Big|_{\mathbf{w}^{(t)}} \right] \\ &+(1-\lambda) \left[ \mathcal{V}(\tilde{\ell}_i,\tilde{N}_i;\Theta)\nabla_{\mathbf{w}} {\ell}(f(\tilde{x}_i;\mathbf{w}),\tilde{y}_i)\Big|_{\mathbf{w}^{(t)}}   +(1-\mathcal{V}(\tilde{\ell}_i,{\tilde{N}_i};\Theta)) \nabla_{\mathbf{w}} {\ell}(f(\tilde{x}_i;\mathbf{w}),\tilde{\mathbf{z}}_i^{(t+1)})\Big|_{\mathbf{w}^{(t)}} \right] \}.
		\end{align*}
		\STATE Update parameters of CMW-Net $\Theta^{(t+1)}$ by
		\begin{align*}\label{eqtheta}
		\Theta^{(t+1)} =  \Theta^{(t)} -\eta_2 \frac{1}{m}\sum_{i=1}^{m} \nabla_{ \Theta} \ell\left(f(x_i^{(meta)};\hat{\mathbf{w}}^{(t+1)}(\Theta)),y_i^{(meta)}\right)\Big|_{\Theta^{(t)}}.
		\end{align*}
		\STATE Update parameters of classifier $\mathbf{w}^{(t+1)}$ by
		\begin{align*}
		\mathbf{w}^{(t+1)}=
		\mathbf{w}^{(t)} & - \eta_2
		\sum_{i=1}^n  \{ \lambda \left[ \mathcal{V}(\ell_i,N_i;\Theta^{(t+1)})\nabla_{\mathbf{w}} {\ell}(f(\tilde{x}_i;\mathbf{w}),y_i)\Big|_{\mathbf{w}^{(t)}}   +(1-\mathcal{V}(\ell_i,N_i;\Theta^{(t+1)})) \nabla_{\mathbf{w}} {\ell}(f(\tilde{x}_i;\mathbf{w}),\mathbf{z}_i^{(t+1)})\Big|_{\mathbf{w}^{(t)}} \right] \\ &+(1-\lambda) \left[ \mathcal{V}(\tilde{\ell}_i,\tilde{N}_i;\Theta^{(t+1)})\nabla_{\mathbf{w}} {\ell}(f(\tilde{x}_i;\mathbf{w}),\tilde{y}_i)\Big|_{\mathbf{w}^{(t)}}   +(1-\mathcal{V}(\tilde{\ell}_i,{\tilde{N}_i};\Theta^{(t+1)})) \nabla_{\mathbf{w}} {\ell}(f(\tilde{x}_i;\mathbf{w}),\tilde{\mathbf{z}}_i^{(t+1)})\Big|_{\mathbf{w}^{(t)}} \right] \}.
		\end{align*}
		\ENDFOR
	\end{algorithmic}
\end{algorithm}

\subsection{Convergence Proof of Proposed CMW-Net Learning Algorithm}
This section provides the proofs of Theorems 1 and 2 in the paper.

Suppose that we have a small amount of meta (validation) dataset with $M$ samples $\{(x_i^{(m)},y_i^{(m)}),1\leq i\leq M\}$ with clean labels, and the overall meta loss is,
\begin{align}
\mathcal{L}^{meta}(\mathbf{w}^*(\Theta))=\frac{1}{M} \sum_{i=1}^M L_i^{meta}(\mathbf{w}^*(\Theta)),
\end{align}
where $\mathbf{w}^*$ is the parameter of the classifier network, and $\Theta$ is the parameter of the CMW-Net. Let's suppose we have another $N$ training data, $\{(x_i,y_i),1\leq i \leq N\}$, where $M\ll N$, and the overall training loss is,

\begin{align}\label{eqob}
\mathcal{L}^{tr}(\mathbf{w};\Theta) = \sum_{i=1}^N \mathcal{V}(L_i^{tr}(\mathbf{w}),N_i;\Theta) L_i^{tr}(\mathbf{w}),
\end{align}
where $\sum_{i=1}^N \mathcal{V}(L_i^{tr}(\mathbf{w}),N_i;\Theta) =1$.

\begin{Lemma}\label{lemma1}
	Suppose the meta loss function is Lipschitz smooth with constant $L$, and $\mathcal{V}(\cdot,\cdot;\Theta)$ is differential with a $\delta$-bounded gradient and twice differential with its Hessian bounded by $\mathcal{B}$, and the loss function $\ell$ has $\rho$-bounded gradient with respect to training/meta data. Then the gradient of $\Theta$ with respect to the meta loss is Lipschitz continuous.
\end{Lemma}

\begin{proof}
	The gradient of $\Theta$ with respect to the meta loss can be written as:
	\begin{align}\label{eq20}
	\begin{split}
	\nabla_{ \Theta} L_i^{meta}(\hat{\mathbf{w}} ^{(t+1)}(\Theta))\Big|_{\Theta^{(t)}}			= -\alpha\sum_{j=1}^n \left(\frac{\partial L_i^{meta} (\hat{\mathbf{w}})}{\partial \hat{\mathbf{w}}}\Big|_{\hat{\mathbf{w}} ^{(t+1)}(\Theta)}^T \frac{\partial L_j^{tr} (\mathbf{w})}{\partial \mathbf{w}} \Big|_{\mathbf{w}^{(t)}}\right)  \frac{\partial \mathcal{V}(L_j^{tr}(\mathbf{w}^{(t)}),N_j;\Theta)}{\partial \Theta}\Big|_{ \Theta^{(t)}}.
	\end{split}
	\end{align}
	Let  $\mathcal{V}_j(\Theta)=\mathcal{V}(L_j^{train}(\mathbf{w}^{(t)});\Theta)$ and $G_{ij}$ be defined in Eq.(\ref{eqg}). Taking gradient of $\Theta$ in both sides of Eq.(\ref{eq20}), we have
	\begin{equation*}
	\begin{split}
	\nabla_{ \Theta^2}^2  L_i^{meta}(\hat{\mathbf{w}} ^{(t+1)}(\Theta))\Big|_{\Theta^{(t)}}			= {-\alpha}\sum_{j=1}^n \left[\frac{\partial}{\partial \Theta}\left( G_{ij}\right)\Big|_{ \Theta^{(t)}}  \frac{\partial \mathcal{V}_j(\Theta)}{\partial \Theta}\Big|_{ \Theta^{(t)}} + \left(G_{ij}\right)\frac{\partial^2 \mathcal{V}_j(\Theta)}{\partial \Theta^2}\Big|_{ \Theta^{(t)}} \right].
	\end{split}
	\end{equation*}
	For the first term in the right hand side, we have that
	\begin{align}\label{eq21}
	\begin{split}
	&\left\|\frac{\partial}{\partial \Theta}\!\left(G_{ij}\right)\!\Big|_{ \Theta^{(t)}} \!\frac{\partial \mathcal{V}_j(\Theta)}{\partial \Theta}\!\Big|_{ \Theta^{(t)}}\right\| \\ \leq & \delta \left\|\frac{\partial}{\partial \Theta}\!\left(G_{ij}\right)\!\Big|_{ \Theta^{(t)}} \right\|
	= \delta \left\|\frac{\partial}{\partial \hat{\mathbf{w}}}\!\left(\!\frac{\partial  L_i^{meta} (\hat{\mathbf{w}})}{\partial \Theta}\Big|_{\Theta^{(t)}}\!\right)\!\Big|_{\hat{\mathbf{w}} ^{(t+1)}(\Theta)}^T\!\frac{\partial L_j^{tr} (\mathbf{w})}{\partial \mathbf{w}} \Big|_{\mathbf{w}^{(t)}}\!\right\| \\
	= & \delta \left\|\frac{\partial}{\partial \hat{\mathbf{w}}}\!\left( \!\frac{\partial L_i^{meta} (\hat{\mathbf{w}})}{\partial \hat{\mathbf{w}}}\!\Big|_{\hat{\mathbf{w}} ^{(t+1)}(\Theta)} \!(-\alpha) \!\sum_{k=1}^n \!\frac{\partial L_k^{tr} (\mathbf{w})}{\partial \mathbf{w}}\!\Big|_{\mathbf{w}^{(t)}} \frac{\partial \mathcal{V}_k(\Theta)}{\partial \Theta}\Big|_{ \Theta^{(t)}}  \!\right)\! \Big|_{\hat{\mathbf{w}} ^{(t+1)}(\Theta)}^T\!\frac{\partial L_j^{tr} (\mathbf{w})}{\partial \mathbf{w}} \Big|_{\mathbf{w}^{(t)}}\!\right\| \\
	= &\delta \left\|\left(\frac{\partial^2 L_i^{meta} (\hat{\mathbf{w}})}{\partial \hat{\mathbf{w}}^2}\Big|_{\hat{\mathbf{w}} ^{(t+1)}(\Theta)} (-\alpha)\sum_{k=1}^n \frac{\partial L_k^{tr} (\mathbf{w})}{\partial \mathbf{w}}\Big|_{\mathbf{w}^{(t)}} \frac{\partial \mathcal{V}_k(\Theta)}{\partial \Theta}\Big|_{ \Theta^{(t)}}\right)\Big|_{\hat{\mathbf{w}}^{(t+1)}(\Theta)}^T\frac{\partial L_j^{tr} (\mathbf{w})}{\partial \mathbf{w}}\Big|_{\mathbf{w}^{(t)}} \right\| \\
	\leq &\alpha n L\rho^2\delta^2,
	\end{split}
	\end{align}
	since $\left\|\frac{\partial^2 L_i^{meta} (\hat{\mathbf{w}})}{\partial \hat{\mathbf{w}}^2}\Big|_{\hat{\mathbf{w}}^{(t+1)}(\Theta)} \right\|\leq L,\left\|\frac{\partial L_j^{tr} (\mathbf{w})}{\partial \mathbf{w}}\Big|_{\mathbf{w}^{(t)}}\right\|\leq \rho, \left\|\frac{\partial \mathcal{V}_j(\Theta)}{\partial \Theta}\Big|_{\Theta^{(t)}}\right\|\leq \delta$.
	And for the second term we have
	\begin{align}\label{eq22}
	\begin{split}
	\left\|  \left(G_{ij}\right)\frac{\partial^2 \mathcal{V}_j(\Theta)}{\partial \Theta^2}\Big|_{ \Theta^{(t)}} \right\| 			= \left\|\frac{\partial L_i^{meta} (\hat{\mathbf{w}})}{\partial \hat{\mathbf{w}}}\Big|_{\hat{\mathbf{w}}^{(t+1)}(\Theta)}^T \frac{\partial L_j^{tr} (\mathbf{w})}{\partial \mathbf{w}} \Big|_{\mathbf{w}^{(t)}}\frac{\partial^2 \mathcal{V}_j(\Theta)}{\partial \Theta^2}\Big|_{ \Theta^{(t)}} \right\|
	\leq \mathcal{B}\rho^2,
	\end{split}
	\end{align}
	since $\left\|\!\frac{\partial L_i^{meta} (\hat{\mathbf{w}})}{\partial \hat{\mathbf{w}}}\!\Big|_{\hat{\mathbf{w}}^{(t+1)}(\Theta)}^T\!\right\|\!\!\leq\!\! \rho, \left\|\!\frac{\partial^2 \mathcal{V}_j(\Theta)}{\partial \Theta^2}\!\Big|_{ \Theta^{(t)}} \!\right\|\!\!\leq\!\! \mathcal{B} $. Combining the above two inequalities Eqs.(\ref{eq21}) and (\ref{eq22}), we then have
	\begin{align}
	\left\|\nabla_{ \Theta^2}^2 L_i^{meta}(\hat{\mathbf{w}}^{(t)}(\Theta))\Big|_{\Theta^{(t)}}\right\| \leq \alpha\rho^2 (n\alpha L\delta^2+\mathcal{B}).
	\end{align}
	Define $L_{V} = \alpha\rho^2 (n\alpha L\delta^2+\mathcal{B})$, and based on Lagrange mean value theorem, we have:
	\begin{align}
	\begin{split}
	\|\nabla \mathcal{L}^{meta}(\hat{\mathbf{w}}^{(t+1)}(\Theta_1))-\nabla \mathcal{L}^{meta}(\hat{\mathbf{w}}^{(t+1)}(\Theta_2))\| \leq L_V \|\Theta_1-\Theta_2\|,  for \ all \ \ \Theta_1, \Theta_2,
	\end{split}
	\end{align}
	where $\nabla \mathcal{L}^{meta}(\hat{\mathbf{w}}^{(t+1)}(\Theta_1) )=  \nabla_{ \Theta}  L_i^{meta}(\hat{\mathbf{w}}^{(t+1 )}(\Theta))\big|_{\Theta_1}$.
\end{proof}

\begin{Theorem} \label{th1}
	Suppose the loss function $\ell$ is Lipschitz smooth with constant $L$, and CMW-Net $\mathcal{V}(\cdot,\cdot;\Theta)$ is differential with a $\delta$-bounded gradient and twice differential with its Hessian bounded by $\mathcal{B}$, and the loss function $\ell$ has $\rho$-bounded gradient with respect to training/meta data. Let the learning rate $\alpha_t, \beta_t, 1\leq t\leq T$ be monotonically descent sequences, and satisfy $\alpha_t=\min\{\frac{1}{L},\frac{c_1}{\sqrt{T}}\}, \beta_t=\min\{\frac{1}{L},\frac{c_2}{\sqrt{T}}\}$, for some $c_1,c_2>0$, such that $\frac{\sqrt{T}}{c_1}\geq L, \frac{\sqrt{T}}{c_2}\geq L$. Meanwhile, they satisfy $\sum_{t=1}^\infty \alpha_t = \infty,\sum_{t=1}^\infty \alpha_t^2 < \infty ,\sum_{t=1}^\infty \beta_t = \infty,\sum_{t=1}^\infty \beta_t^2 < \infty $. Then CMW-Net can achieve $\mathbb{E}[ \|\nabla \mathcal{L}^{meta}(\hat{\mathbf{w}}^{(t)}(\Theta^{(t)}))\|_2^2] \leq \epsilon$ in $\mathcal{O}(1/\epsilon^2)$ steps. More specifically,
	\begin{align}
	\min_{0\leq t \leq T} \mathbb{E}\left[ \left\|\nabla \mathcal{L}^{meta}(\hat{\mathbf{w}}^{(t)}(\Theta^{(t)}))\right\|_2^2\right] \leq \mathcal{O}(\frac{C}{\sqrt{T}}),
	\end{align}
	where $C$ is some constant independent of the convergence process.
\end{Theorem}

\begin{proof}
	The update equation of $\Theta$ in each iteration is as follows:
	\begin{align*}
	\Theta^{(t+1)} =  \Theta^{(t)} -\beta \frac{1}{m}\sum_{i=1}^{m} \nabla_{ \Theta} L_i^{meta}(\hat{\mathbf{w}} ^{(t+1)}(\Theta))\Big|_{\Theta^{(t)}}.
	\end{align*}
	
	Under the sampled mini-batch $\Xi_t$, the updating equation can be rewritten as:
	\begin{align*}
	\Theta^{(t+1)} =  \Theta^{(t)} -\beta_t \nabla \mathcal{L}^{meta}(\hat{\mathbf{w}}^{(t+1)}(\Theta^{(t)}))\big|_{\Xi_t}.
	\end{align*}
	Since the mini-batch is drawn uniformly from the entire data set, the above update equation can be written as:
	\begin{align*}
	\Theta^{(t+1)} =  \Theta^{(t)} -\beta_t[ \nabla \mathcal{L}^{meta}(\hat{\mathbf{w}}^{(t+1)}(\Theta^{(t)}))+\xi^{(t)}],
	\end{align*}
	where $\xi^{(t)} = \nabla \mathcal{L}^{meta}(\hat{\mathbf{w}}^{(t+1)}(\Theta^{(t)}))\big|_{\Xi_t}-\nabla \mathcal{L}^{meta}(\hat{\mathbf{w}}^{(t+1)}(\Theta^{(t)}))$. Note that $\xi^{(t)}$ are i.i.d random variable with finite variance, since $\Xi_t$ are drawn i.i.d with a finite number of samples. Furthermore, $\mathbb{E}[\xi^{(t)}]=0$, since samples are drawn uniformly at random.
	Observe that
	\begin{align}\label{eqthetas}
	\begin{split}
	&\mathcal{L}^{meta}(\hat{\mathbf{w}}^{(t+1)}(\Theta^{(t+1)}))-\mathcal{L}^{meta}(\hat{\mathbf{w}}^{(t)}(\Theta^{(t)})) \\
	=& \left\{\mathcal{L}^{meta}(\hat{\mathbf{w}}^{(t+1)}(\Theta^{(t+1)}))- \mathcal{L}^{meta}(\hat{\mathbf{w}}^{(t+1)}(\Theta^{(t)}))\right\} +\left\{\mathcal{L}^{meta}(\hat{\mathbf{w}}^{(t+1)}(\Theta^{(t)}))-\mathcal{L}^{meta}(\hat{\mathbf{w}}^{(t)}(\Theta^{(t)}))\right\}.
	\end{split}
	\end{align}
	
	For the first term, by Lipschitz continuity of $\nabla_{\Theta} \mathcal{L}^{meta}(\hat{\mathbf{w}}^{(t+1)}(\Theta))$ according to Lemma \ref{lemma1}, we can deduce that:
	\begin{align*}
	&\mathcal{L}^{meta}(\hat{\mathbf{w}}^{(t+1)}(\Theta^{(t+1)}))-\mathcal{L}^{meta}(\hat{\mathbf{w}}^{(t+1)}(\Theta^{(t)}))  \\
	\leq &\left \langle\nabla_{\Theta} \mathcal{L}^{meta}(\hat{\mathbf{w}}^{(t+1)}(\Theta^{(t)})),\Theta^{(t+1)}-\Theta^{(t)} \right\rangle + \frac{L}{2} \left\|\Theta^{(t+1)}-\Theta^{(t)}\right\|_2^2\\
	= & \left\langle\nabla_{\Theta} \mathcal{L}^{meta}(\hat{\mathbf{w}}^{(t+1)}(\Theta^{(t)})), -\beta_t [\nabla_{\Theta} \mathcal{L}^{meta}(\hat{\mathbf{w}}^{(t+1)}(\Theta^{(t)}))+\xi^{(t)} ] \right\rangle + \frac{L\beta_t^2}{2} \left\|\nabla_{\Theta} \mathcal{L}^{meta}(\hat{\mathbf{w}}^{(t+1)}(\Theta^{(t)}))+\xi^{(t)}\right\|_2^2\\
	=& -(\beta_t-\frac{L\beta_t^2}{2}) \left\|\nabla_{\Theta} \mathcal{L}^{meta}(\hat{\mathbf{w}}^{(t+1)}(\Theta^{(t)}))\right\|_2^2
	+ \frac{L\beta_t^2}{2}\|\xi^{(t)}\|_2^2 - (\beta_t-L\beta_t^2)\langle \nabla_{\Theta} \mathcal{L}^{meta}(\hat{\mathbf{w}}^{(t)}(\Theta^{(t)})),\xi^{(t)}\rangle.
	\end{align*}
	
	For the second term, by Lipschitz smoothness of the meta loss function $\mathcal{L}^{meta}(\hat{\mathbf{w}}^{(t+1)}(\Theta^{(t+1)}))$,  we have
	\begin{align*}
	&\mathcal{L}^{meta}(\hat{\mathbf{w}}^{(t+1)}(\Theta^{(t)}))- \mathcal{L}^{meta}(\hat{\mathbf{w}}^{(t)}(\Theta^{(t)})) \\
	\leq & \left\langle \nabla_{\mathbf{w}} \mathcal{L}^{meta}(\hat{\mathbf{w}}^{(t)}(\Theta^{(t)})), \hat{\mathbf{w}}^{(t+1)}(\Theta^{(t)})-\hat{\mathbf{w}}^{(t)}(\Theta^{(t)}) \right\rangle
	+ \frac{L}{2}\|\hat{\mathbf{w}}^{(t+1)}(\Theta^{(t)})-\hat{\mathbf{w}}^{(t)}(\Theta^{(t)})\|_2^2.
	\end{align*}
	Since
	\begin{align*}
	\hat{\mathbf{w}}^{(t+1)}(\Theta^{(t)})-\hat{\mathbf{w}}^{(t)}(\Theta^{(t)})
	= - \alpha_t \nabla \mathcal{L}^{tr}(\mathbf{w}^{(t)};\Theta^{(t)})|_{\Psi_t},
	\end{align*}
	where $\Psi_t$ denotes the mini-batch drawn randomly from the training dataset  in the $t$-th iteration, $\nabla \mathcal{L}^{train}(\mathbf{w}^{(t)};\Theta^{(t)})=
	\sum_{i=1}^n \mathcal{V}(L_i^{tr}(\mathbf{w}^{(t)}),N_i;\Theta^{(t)}) \nabla_{\mathbf{w}^{(t)}} L_i^{tr}(\mathbf{w})\Big|_{\mathbf{w}^{(t)}}$, and $\sum_{i=1}^n \mathcal{V}(L_i^{tr}(\mathbf{w}^{(t)}),N_i;\Theta^{(t+1)}) = 1$. Since the mini-batch $\Psi_t$ is drawn uniformly at random, we can rewrite the update equation as:
	\begin{align*}
	\hat{\mathbf{w}}^{(t+1)}(\Theta^{(t)})=\hat{\mathbf{w}}^{(t)}(\Theta^{(t)})  - \alpha_t [\nabla \mathcal{L}^{tr}(\mathbf{w}^{(t)};\Theta^{(t)})+\psi^{(t)}],
	\end{align*}
	where $\psi^{(t)} = \nabla \mathcal{L}^{tr}(\mathbf{w}^{(t)};\Theta^{(t)})|_{\Psi_t}-\nabla \mathcal{L}^{tr}(\mathbf{w}^{(t)};\Theta^{(t)})$.
	Note that $\psi^{(t)}$ are i.i.d. random variables with finite variance, since $\Psi_t$ are drawn i.i.d. with a finite number of samples, and thus $\mathbb{E}[\psi^{(t)}]=0$,  $\mathbb{E}[\|\psi^{(t)}\|_2^2]\leq \sigma^2$.
	Thus we have
	\begin{align*}
	& \mathcal{L}^{meta}(\hat{\mathbf{w}}^{(t+1)}(\Theta^{(t)}))- \mathcal{L}^{meta}(\hat{\mathbf{w}}^{(t)}(\Theta^{(t)}))  \\
	\leq &  \left\langle \nabla_{\mathbf{w}} \mathcal{L}^{meta}(\hat{\mathbf{w}}^{(t)}(\Theta^{(t)})), - \alpha_t [\nabla \mathcal{L}^{tr}(\mathbf{w}^{(t)};\Theta^{(t)})+\psi^{(t)}]\right\rangle
	+ \frac{L}{2}\left\|\alpha_t [\nabla \mathcal{L}^{tr}(\mathbf{w}^{(t)};\Theta^{(t)})+\psi^{(t)}]\right\|_2^2                     \\
	= &  \frac{L\alpha_t^2}{2} \left\|\nabla \mathcal{L}^{tr}(\mathbf{w}^{(t)};\Theta^{(t)})\right\|_2^2 -\alpha_t \left\langle  \nabla_{\mathbf{w}} \mathcal{L}^{meta}(\hat{\mathbf{w}}^{(t)}(\Theta^{(t)})),   \nabla \mathcal{L}^{tr}(\mathbf{w}^{(t)};\Theta^{(t)})     \right\rangle
	+ \frac{L\alpha_t^2}{2}\left\|\psi^{(t)}\right\|_2^2 \\
	& +L\alpha_t^2\left\langle \nabla \mathcal{L}^{tr}(\mathbf{w}^{(t)};\Theta^{(t)}), \psi^{(t)}\right\rangle -\alpha_t \left\langle  \nabla_{\mathbf{w}} \mathcal{L}^{meta}(\hat{\mathbf{w}}^{(t)}(\Theta^{(t)})), \psi^{(t)}      \right\rangle \\
	\leq & \frac{L\alpha_t^2 \rho^2}{2} + \alpha_t \rho \left\|\nabla \mathcal{L}^{tr}(\mathbf{w}^{(t)};\Theta^{(t)})\right\|_2  + \frac{L\sigma^2 \alpha_t^2}{2}+L\alpha_t^2\left\langle \nabla \mathcal{L}^{tr}(\mathbf{w}^{(t)};\Theta^{(t)}), \psi^{(t)}\right\rangle -\alpha_t \left\langle  \nabla_{\mathbf{w}} \mathcal{L}^{meta}(\hat{\mathbf{w}}^{(t)}(\Theta^{(t)})), \psi^{(t)}\right\rangle.
	\end{align*}
	The last inequality holds since $\left\langle  \nabla_{\mathbf{w}} \mathcal{L}^{meta}(\hat{\mathbf{w}}^{(t)}(\Theta^{(t)})),   \nabla \mathcal{L}^{tr}(\mathbf{w}^{(t)};\Theta^{(t)})     \right\rangle \leq \left\|\nabla \mathcal{L}^{meta}(\hat{\mathbf{w}}^{(t)}(\Theta^{(t)}))\right\|_2\left\|\nabla \mathcal{L}^{tr}(\mathbf{w}^{(t)};\Theta^{(t)})\right\|_2$. Thus Eq.(\ref{eqthetas}) satifies
	\begin{align*}
	\begin{split}
	&\mathcal{L}^{meta}(\hat{\mathbf{w}}^{(t+1)}(\Theta^{(t+1)}))-\mathcal{L}^{meta}(\hat{\mathbf{w}}^{(t)}(\Theta^{(t)})) \\
	\leq & \frac{L\alpha_t^2 \rho^2}{2} + \alpha_t \rho \left\|\nabla \mathcal{L}^{tr}(\mathbf{w}^{(t)};\Theta^{(t)})\right\|_2+ \frac{L\sigma^2 \alpha_t^2}{2}+L\alpha_t^2\left\langle \nabla \mathcal{L}^{tr}(\mathbf{w}^{(t)};\Theta^{(t)}), \psi^{(t)}\right\rangle -\alpha_t \left\langle  \nabla_{\mathbf{w}} \mathcal{L}^{meta}(\hat{\mathbf{w}}^{(t)}(\Theta^{(t)})), \psi^{(t)}\right\rangle \\ &-(\beta_t-\frac{L\beta_t^2}{2}) \left\|\nabla_{\Theta} \mathcal{L}^{meta}(\hat{\mathbf{w}}^{(t+1)}(\Theta^{(t)}))\right\|_2^2
	+ \frac{L\beta_t^2}{2}\|\xi^{(t)}\|_2^2 - (\beta_t-L\beta_t^2)\langle \nabla_{\Theta} \mathcal{L}^{meta}(\hat{\mathbf{w}}^{(t)}(\Theta^{(t)})),\xi^{(t)}\rangle.
	\end{split}
	\end{align*}
	Rearranging the terms, and taking expectations with respect to $\xi^{(t)}$ and $\psi^{(t)}$ on both sides, we can obtain
	\begin{align*}
	\begin{split}
	&(\beta_t-\frac{L\beta_t^2}{2})
	\left\|\nabla_{\Theta} \mathcal{L}^{meta}(\hat{\mathbf{w}}^{(t+1)}(\Theta^{(t)}))\right\|_2^2 \\
	\leq & \frac{L\alpha_t^2 \rho^2}{2} + \alpha_t \rho \left\|\nabla \mathcal{L}^{tr}(\mathbf{w}^{(t)};\Theta^{(t)})\right\|_2 + \frac{L\sigma^2 \alpha_t^2}{2} +\mathcal{L}^{meta}(\hat{\mathbf{w}}^{(t)}(\Theta^{(t)}))-\mathcal{L}^{meta}(\hat{\mathbf{w}}^{(t+1)}(\Theta^{(t+1)}))
	+ \frac{L\beta_t^2}{2}\sigma^2,
	\end{split}
	\end{align*}
	since $\mathbb{E}[\xi^{(t)}]=0,\mathbb{E}[\psi^{(t)}]=0$ and $\mathbb{E} [\|\xi^{(t)}\|_2^2] \leq \sigma^2$. Summing up the above inequalities, we can obtain
	\begin{align*}\label{eqrand}
	\begin{split}
	&\sum\nolimits_{t=1}^T (\beta_t-\frac{L\beta_t^2}{2})
	\left\|\nabla_{\Theta} \mathcal{L}^{meta}(\hat{\mathbf{w}}^{(t+1)}(\Theta^{(t)}))\right\|_2^2 \\
	\leq & \mathcal{L}^{meta}(\hat{\mathbf{w}}^{(1)})(\Theta^{(1)}) -\mathcal{L}^{meta}(\hat{\mathbf{w}}^{(T+1)})(\Theta^{(T+1)}) +
	\frac{L(\sigma^2+\rho^2)}{2} \sum_{t=1}^T\alpha_t^2 +\rho\sum_{t=1}^T\alpha_t \left\|\nabla \mathcal{L}^{tr}(\mathbf{w}^{(t)};\Theta^{(t)})\right\|_2^2  +\frac{L}{2}\sum_{t=1}^T\beta_t^2\sigma^2 \\
	\leq & \mathcal{L}^{meta}(\hat{\mathbf{w}}^{(1)})(\Theta^{(1)})  +
	\frac{L(\sigma^2+\rho^2)}{2} \sum_{t=1}^T\alpha_t^2 +\rho\sum_{t=1}^T\alpha_t \left\|\nabla \mathcal{L}^{tr}(\mathbf{w}^{(t)};\Theta^{(t)})\right\|_2^2  +\frac{L}{2}\sum_{t=1}^T\beta_t^2\sigma^2 .
	\end{split}
	\end{align*}
	Furthermore, we can deduce that
	\begin{align*}
	\begin{split}
	&\min_{t} \mathbb{E} \left[ \left\|\nabla_{\Theta} \mathcal{L}^{meta}(\hat{\mathbf{w}}^{(t+1)}(\Theta^{(t)}))\right\|_2^2 \right] \\
	\leq & \frac{\sum_{t=1}^T (\beta_t-\frac{L\beta_t^2}{2})\mathbb{E} \left\|\nabla_{\Theta} \mathcal{L}^{meta}(\hat{\mathbf{w}}^{(t+1)}(\Theta^{(t)}))\right\|_2^2 }{\sum_{t=1}^{T} \left(\beta_t-\frac{L\beta_t^2}{2}\right)}\\
	\leq 	& \frac{1}{\sum_{t=1}^T (2\beta_t-L\beta_t^2)} \left[\mathcal{L}^{meta}(\hat{\mathbf{w}}^{(1)})(\Theta^{(1)})+\frac{L(\sigma^2+\rho^2)}{2} \sum_{t=1}^T\alpha_t^2 +\rho\sum_{t=1}^T\alpha_t \left\|\nabla \mathcal{L}^{tr}(\mathbf{w}^{(t)};\Theta^{(t)})\right\|_2^2  +\frac{L}{2}\sum_{t=1}^T\beta_t^2\sigma^2  \right]\\
	\leq & \frac{1}{\sum_{t=1}^T \beta_t} \left[2\mathcal{L}^{meta}(\hat{\mathbf{w}}^{(1)})(\Theta^{(1)})+{L(\sigma^2+\rho^2)} \sum_{t=1}^T\alpha_t^2 +\rho\sum_{t=1}^T\alpha_t \left\|\nabla \mathcal{L}^{tr}(\mathbf{w}^{(t)};\Theta^{(t)})\right\|_2^2  +\frac{L}{2}\sigma^2 \sum_{t=1}^T\beta_t^2 \right] \\
	\leq	& \frac{1}{T\beta_T} \left[2\mathcal{L}^{meta}(\hat{\mathbf{w}}^{(1)})(\Theta^{(1)})+{L(\sigma^2+\rho^2)} \sum_{t=1}^T\alpha_t^2 +\rho\sum_{t=1}^T\alpha_t \left\|\nabla \mathcal{L}^{tr}(\mathbf{w}^{(t)};\Theta^{(t)})\right\|_2^2  +\frac{L}{2}\sigma^2 \sum_{t=1}^T\beta_t^2\right]\\
	= &\frac{2\mathcal{L}^{meta}(\hat{\mathbf{w}}^{(1)})(\Theta^{(1)})+{L(\sigma^2+\rho^2)} \sum_{t=1}^T\alpha_t^2 +\rho\sum_{t=1}^T\alpha_t \left\|\nabla \mathcal{L}^{tr}(\mathbf{w}^{(t)};\Theta^{(t)})\right\|_2^2  +\frac{L}{2}\sigma^2 \sum_{t=1}^T\beta_t^2 }{T} \max\{L,\frac{\sqrt{T}}{c}\}   \\
	= & \mathcal{O}(\frac{C}{\sqrt{T}}).
	\end{split}
	\end{align*}
	The third inequality holds since $\sum_{t=1}^T (2\beta_t-L\beta_t^2) \!=\! \sum_{t=1}^T \beta_t (2-L\beta_t) \!\geq\! \sum_{t=1}^T \beta_t$, and the final equality holds since $\lim_{T\to \infty} \sum_{t=1}^T\alpha_t^2 $\  $< \infty, \lim_{T\to \infty} \sum_{t=1}^T\beta_t^2 < \infty, \lim_{T\to \infty} \sum_{t=1}^T\alpha_t \left\|\nabla \mathcal{L}^{tr}(\mathbf{w}^{(t)};\Theta^{(t)})\right\|_2^2 < \infty $.
	Thus we can conclude that our algorithm can always achieve $\min_{0\leq t \leq T} \mathbb{E}[ \|\nabla \mathcal{L}^{meta}(\Theta^{(t)})\|_2^2] \leq \mathcal{O}(\frac{C}{\sqrt{T}})$ in $T$ steps, and this finishes our proof of Theorem \ref{th1}.
\end{proof}

\begin{Theorem}\label{th2}
	Suppose that the loss function $\ell$ is Lipschitz smooth with constant $L$, and CMW-Net $\mathcal{V}(\cdot,\cdot;\Theta)$ is differential with a $\delta$-bounded gradient and twice differential with its Hessian bounded by $\mathcal{B}$, and the loss function $\ell$ has $\rho$-bounded gradient with respect to training/meta data. Let the learning rate $\alpha_t, \beta_t, 1\leq t\leq T$ be monotonically descent sequences, and satisfy $\alpha_t=\min\{\frac{1}{L},\frac{c_1}{\sqrt{T}}\}, \beta_t=\min\{\frac{1}{L},\frac{c_2}{\sqrt{T}}\}$, for some $c_1,c_2>0$, such that $\frac{\sqrt{T}}{c_1}\geq L, \frac{\sqrt{T}}{c_2}\geq L$. Meanwhile, they satisfy $\sum_{t=1}^\infty \alpha_t = \infty,\sum_{t=1}^\infty \alpha_t^2 < \infty ,\sum_{t=1}^\infty \beta_t = \infty,\sum_{t=1}^\infty \beta_t^2 < \infty $. Then CMW-Net can achieve $\mathbb{E}[ \|\nabla \mathcal{L}^{tr}(\mathbf{w}^{(t)};\Theta^{(t)})\|_2^2] \leq \epsilon$ in $\mathcal{O}(1/\epsilon^2)$ steps. More specifically,
	\begin{align}
	\min_{0\leq t \leq T} \mathbb{E}\left[ \left\|\nabla \mathcal{L}^{tr}(\mathbf{w}^{(t)};\Theta^{(t)})\right\|_2^2 \right] \leq \mathcal{O}(\frac{C}{\sqrt{T}})
	\end{align}
	where $C$ is some constant independent of the convergence process.
\end{Theorem}

\begin{proof}
	It is easy to conclude that
	$\alpha_t$ satisfy $\sum_{t=0}^{\infty}\alpha_t = \infty, \sum_{t=0}^{\infty}\alpha_t^2 < \infty$.
	Recall the update equation of $\mathbf{w}$ in each iteration as follows:
	\begin{align*}
	\begin{split}
	\mathbf{w}^{(t+1)}= \mathbf{w}^{(t)} - \alpha
	\sum_{i=1}^n   \mathcal{V}(L_i^{tr}(\mathbf{w}),S_{y_i};\Theta^{(t+1)}) \nabla_{\mathbf{w}} L_i^{tr}(\mathbf{w})\Big|_{\mathbf{w}^{(t)}}.
	\end{split}
	\end{align*}
	Under the sampled mini-batch $\Psi_t$ from the training dataset, the updating equation can be rewritten as:
	\begin{align*}
	\mathbf{w}^{(t+1)}= \mathbf{w}^{(t)} - \alpha_t \nabla \mathcal{L}^{tr}(\mathbf{w}^{(t)};\Theta^{(t+1)})|_{\Psi_t},
	\end{align*}
	where $\nabla \mathcal{L}^{train}(\mathbf{w}^{(t)};\Theta^{(t+1)})=
	\sum_{i=1}^n \mathcal{V}(L_i^{tr}(\mathbf{w}^{(t)}),N_i;\Theta^{(t+1)}) \nabla_{\mathbf{w}^{(t)}} L_i^{tr}(\mathbf{w})\Big|_{\mathbf{w}^{(t)}}$, and $\sum_{i=1}^n \mathcal{V}(L_i^{tr}(\mathbf{w}^{(t)}),N_i;\Theta^{(t+1)}) = 1$. Since the mini-batch $\Psi_t$ is drawn uniformly at random, we can rewrite the update equation as:
	\begin{align*}
	\mathbf{w}^{(t+1)}= \mathbf{w}^{(t)} - \alpha_t [\nabla \mathcal{L}^{tr}(\mathbf{w}^{(t)};\Theta^{(t+1)})+\psi^{(t)}],
	\end{align*}
	where $\psi^{(t)} = \nabla \mathcal{L}^{tr}(\mathbf{w}^{(t)};\Theta^{(t+1)})|_{\Psi_t}-\nabla \mathcal{L}^{tr}(\mathbf{w}^{(t)};\Theta^{(t+1)})$.
	Note that $\psi^{(t)}$ are i.i.d. random variables with finite variance, since $\Psi_t$ are drawn i.i.d. with a finite number of samples, and thus $\mathbb{E}[\psi^{(t)}]=0$,  $\mathbb{E}[\|\psi^{(t)}\|_2^2]\leq \sigma^2$.
	
	The objective function $\mathcal{L}^{tr}(\mathbf{w};\Theta)$ defined in Eq. (\ref{eqob}) can be easily proved to be Lipschitz-smooth with constant $L$, and have $\rho$-bounded gradient with respect to training data.
	Observe that
	\begin{align}\label{eq23}
	\begin{split}
	&\mathcal{L}^{tr}(\mathbf{w}^{(t+1)};\Theta^{(t+1)}) -\mathcal{L}^{tr}(\mathbf{w}^{(t)};\Theta^{(t)}) \\
	= & \left\{\mathcal{L}^{tr}(\mathbf{w}^{(t+1)};\Theta^{(t+1)})-\mathcal{L}^{tr}(\mathbf{w}^{(t+1)};\Theta^{(t)})\right\}
	+\left\{\mathcal{L}^{tr}(\mathbf{w}^{(t+1)};\Theta^{(t)})-\mathcal{L}^{tr}(\mathbf{w}^{(t)};\Theta^{(t)})\right\}.
	\end{split}
	\end{align}
	For the first term, by Lipschitz smoothness of the training loss function $\mathcal{L}^{tr}(\hat{\mathbf{w}}^{(t+1)}(\Theta^{(t+1)}))$,  we have
	\begin{align*}
	\begin{split}
	&\mathcal{L}^{tr}(\mathbf{w}^{(t+1)};\Theta^{(t)})-\mathcal{L}^{tr}(\mathbf{w}^{(t)};\Theta^{(t)}) \\
	\leq & \left\langle\nabla \mathcal{L}^{tr}(\mathbf{w}^{(t)};\Theta^{(t)}),\mathbf{w}^{(t+1)}-\mathbf{w}^{(t)} \right\rangle + \frac{L}{2} \left\|\mathbf{w}^{(t+1)}-\mathbf{w}^{(t)}\right\|_2^2\\
	= & \left\langle \nabla \mathcal{L}^{tr}(\mathbf{w}^{(t)};\Theta^{(t)}), -\alpha_t [\nabla \mathcal{L}^{tr}(\mathbf{w}^{(t)};\Theta^{(t)})+\psi^{(t)}] \right\rangle
	+ \frac{L\alpha_t^2}{2} \left\|\nabla \mathcal{L}^{tr}(\mathbf{w}^{(t)};\Theta^{(t)})+\psi^{(t)}\right\|_2^2\\
	= & -(\alpha_t-\frac{L\alpha_t^2}{2}) \left\|\nabla \mathcal{L}^{tr}(\mathbf{w}^{(t)};\Theta^{(t)})\right\|_2^2
	+ \frac{L\alpha_t^2}{2}\left\|\psi^{(t)}\right\|_2^2 - (\alpha_t-L\alpha_t^2)\left\langle \nabla \mathcal{L}^{tr}(\mathbf{w}^{(t)};\Theta^{(t)}), \psi^{(t)}\right\rangle.
	\end{split}
	\end{align*}
	For the second term, we have
	\begin{small}
		\begin{align*}
		\begin{split}
		&\mathcal{L}^{tr}(\mathbf{w}^{(t+1)};\Theta^{(t+1)})-\mathcal{L}^{tr}(\mathbf{w}^{(t+1)};\Theta^{(t)})\\
		= &\frac{1}{n}\sum_{i=1}^n \left\{\mathcal{V}(\mathcal{L}_i^{tr}(\mathbf{w}^{(t+1)}),N_i;\Theta^{(t+1)})-\mathcal{V}(\mathcal{L}_i^{tr}(\mathbf{w}^{(t+1)}),N_i;\Theta^{(t)})\right\} \mathcal{L}_i^{train}(\mathbf{w}^{(t+1)})\\
		\leq &  \frac{1}{n}\sum_{i=1}^n \left\{\left\langle \frac{\partial \mathcal{V}(\mathcal{L}_i^{tr}(\mathbf{w}^{(t+1)}),N_i;\Theta)}{\partial \Theta}\Big|_{ \Theta^{(t)}}, \Theta^{(t+1)}-\Theta^{(t)}\right\rangle +\frac{\delta}{2} \left\|\Theta^{(t+1)}-\Theta^{(t)}\right\|_2^2\right\}\mathcal{L}_i^{tr}(\mathbf{w}^{(t+1)})\\
		=&  \frac{1}{n}\sum_{i=1}^n \left\{\left\langle \frac{\partial \mathcal{V}(\mathcal{L}_i^{tr}(\mathbf{w}^{(t+1)}),N_i;\Theta)}{\partial \Theta}\Big|_{ \Theta^{(t)}}, -\beta_t[ \nabla_{\Theta} \mathcal{L}^{meta}(\hat{\mathbf{w}}^{(t+1)}(\Theta^{(t)}))+\xi^{(t)}] \right\rangle
		+ \frac{\delta\beta_t^2}{2} \left\| \nabla \mathcal{L}^{meta}(\hat{\mathbf{w}}^{(t+1)}(\Theta^{(t)}))+\xi^{(t)}\right\|_2^2 \right\}\mathcal{L}_i^{tr}(\mathbf{w}^{(t+1)})\\
		=& \frac{1}{n}\sum_{i=1}^n \left\{\left\langle \frac{\partial \mathcal{V}(\mathcal{L}_i^{tr}(\mathbf{w}^{(t+1)}),N_i;\Theta)}{\partial \Theta}\Big|_{ \Theta^{(t)}}, -\beta_t[ \nabla \mathcal{L}^{meta}(\hat{\mathbf{w}}^{(t+1)}(\Theta^{(t)}))+\xi^{(t)}]\right\rangle +\frac{\delta\beta_t^2}{2} \left(\left\| \nabla \mathcal{L}^{meta}(\hat{\mathbf{w}}^{(t+1)}(\Theta^{(t)}))\right\|_2^2+\|\xi^{(t)}\|_2^2 +\right.\right.\\
		&\phantom{=\;\;}\left.\left.2\left\langle\nabla \mathcal{L}^{meta}(\hat{\mathbf{w}}^{(t+1)}(\Theta^{(t)})),\xi^{(t)} \right\rangle \right) \right\}L_i^{tr}(\mathbf{w}^{(t+1)}).
		\end{split}
		\end{align*}
	\end{small}
	Therefore, for Eq.(\ref{eq23}), we have:
	\begin{align*}
	\begin{split}
	&\mathcal{L}^{tr}(\mathbf{w}^{(t+1)};\Theta^{(t+1)}) -\mathcal{L}^{tr}(\mathbf{w}^{(t+1)};\Theta^{(t)})  \\
	\leq &  \frac{1}{n}\sum_{i=1}^n \left\{\left\langle \frac{\partial \mathcal{V}(\mathcal{L}_i^{tr}(\mathbf{w}^{(t+1)}),N_i;\Theta)}{\partial \Theta}\Big|_{ \Theta^{(t)}}, -\beta_t[ \nabla \mathcal{L}^{meta}(\hat{\mathbf{w}}^{(t+1)}(\Theta^{(t)}))+\xi^{(t)}]\right\rangle +\frac{\delta\beta_t^2}{2} \left(\left\| \nabla \mathcal{L}^{meta}(\hat{\mathbf{w}}^{(t+1)}(\Theta^{(t)}))\right\|_2^2+\|\xi^{(t)}\|_2^2 +\right.\right.\\
	&\phantom{=\;\;}\left.\left.2\left\langle\nabla \mathcal{L}^{meta}(\hat{\mathbf{w}}^{(t+1)}(\Theta^{(t)})),\xi^{(t)} \right\rangle \right) \right\}L_i^{tr}(\mathbf{w}^{(t+1)})
	-(\alpha_t-\frac{L\alpha_t^2}{2}) \left\|\nabla \mathcal{L}^{tr}(\mathbf{w}^{(t)};\Theta^{(t)})\right\|_2^2 + \frac{L\alpha_t^2}{2}\left\|\psi^{(t)}\right\|_2^2 \\
	& \phantom{=\;\;}- (\alpha_t-L\alpha_t^2)\left\langle \nabla \mathcal{L}^{tr}(\mathbf{w}^{(t)};\Theta^{(t)}), \psi^{(t)}\right\rangle.
	\end{split}
	\end{align*}
	Taking expectation of the both sides of the above inequality and based on $\mathbb{E}[\xi^{(t)}]=0,\mathbb{E}[\psi^{(t)}]=0$, we have
	\begin{align*}
	&\mathbb{E}[\mathcal{L}^{train}(\mathbf{w}^{(t+1)};\Theta^{(t+1)})]-\mathbb{E}[\mathcal{L}^{train}(\mathbf{w}^{(t)};\Theta^{(t)})] \\
	& \leq \mathbb{E}   \frac{1}{n}\sum_{i=1}^n \left\{\left\langle \frac{\partial \mathcal{V}(\mathcal{L}_i^{tr}(\mathbf{w}^{(t+1)}),N_i;\Theta)}{\partial \Theta}\Big|_{ \Theta^{(t)}}, -\beta_t \nabla \mathcal{L}^{meta}(\hat{\mathbf{w}}^{(t+1)}(\Theta^{(t)}))\right\rangle +\frac{\delta\beta_t^2}{2} \left(\left\| \nabla \mathcal{L}^{meta}(\hat{\mathbf{w}}^{(t+1)}(\Theta^{(t)}))\right\|_2^2+\|\xi^{(t)}\|_2^2 \right) \right\}\\
	&L_i^{tr}(\mathbf{w}^{(t+1)}) -(\alpha_t-\frac{L\alpha_t^2}{2}) \mathbb{E}\left\|\nabla \mathcal{L}^{tr}(\mathbf{w}^{(t)};\Theta^{(t)})\right\|_2^2   + \frac{L\alpha_t^2}{2}  \mathbb{E}[ \|\psi^{(t)}\|_2^2 ].
	\end{align*}
	Summing up the above inequalities over $t=1,...,T$ in both sides and rearranging the terms, we obtain
	\begin{small}
		\begin{align*}
		&\sum_{t=1}^{T} \left(\alpha_t-\frac{L\alpha_t^2}{2}\right) \mathbb{E} \left\|\nabla \mathcal{L}^{tr}(\mathbf{w}^{(t)};\Theta^{(t)})\right\|_2^2  \\
		\leq &
		\sum_{t=1}^{T} \beta_t\mathbb{E} \frac{1}{n}\sum_{i=1}^n \!\left\|  L_i^{tr}(\mathbf{w}^{(t+1)}) \! \right\| \!\left\|\!  \frac{\partial \mathcal{V}(\mathcal{L}_i^{tr}(\mathbf{w}^{(t+1)}),N_i;\Theta)}{\partial \Theta}\Big|_{ \Theta^{(t)}}  \right\| \! \!\left\|\nabla \mathcal{L}^{meta}(\hat{\mathbf{w}}^{(t+1)}(\Theta^{(t)}))\right\|+ \sum_{t=1}^{T} \frac{L\alpha_t^2}{2}\mathbb{E}[\|\psi^{(t)}\|_2^2 ]+\mathbb{E}[\mathcal{L}^{train}(\mathbf{w}^{(1)};\Theta^{(1)})]\\
		&  -	\mathbb{E}[\mathcal{L}^{tr}(\mathbf{w}^{(T+1)};\Theta^{(T+1)})]+\sum_{t=1}^{T} \frac{\delta\beta_t^2}{2}\left\{\frac{1}{n}\sum_{i=1}^n \left\|  L_i^{tr}(\mathbf{w}^{(t+1)})\right\| \left(\mathbb{E}\left\| \nabla  \mathcal{L}^{meta}(\hat{\mathbf{w}}^{(t+1)}(\Theta^{(t)}))\right\|_2^2+\mathbb{E}\left\|\xi^{(t)}\right\|_2^2 \right)\right\} \\
		\leq & \delta M \sum_{t=1}^{T} \beta_t \left\|\nabla \mathcal{L}^{meta}(\hat{\mathbf{w}}^{(t+1)}(\Theta^{(t)}))\right\|+ \sum_{t=1}^{T} \frac{L\alpha_t^2}{2} \sigma^2+\!\mathbb{E}[\mathcal{L}^{tr}(\mathbf{w}^{(1)};\Theta^{(1)})]   + \sum_{t=1}^{T} \frac{\delta\beta_t^2}{2}\left\{M(\rho^2+\sigma^2)\right\} < \infty.
		\end{align*}
	\end{small}
	The last inequality holds since $\sum_{t=0}^{\infty}\alpha_t^2 < \infty, \sum_{t=0}^{\infty}\beta_t^2 < \infty, \sum_{t=1}^{T} \beta_t \left\|\nabla \mathcal{L}^{meta}(\hat{\mathbf{w}}^{(t+1)}(\Theta^{(t)}))\right\| < \infty$, and $\frac{1}{n}\sum_{i=1}^n\|  L_i^{train}(\mathbf{w}^{(t)})\|\leq M$, i.e., the sum of limited number of samples' losses is bounded. Thus we have
	\begin{align*}
	&\min_{t} \mathbb{E} \left[ \left\|\nabla \mathcal{L}^{tr}(\mathbf{w}^{(t)};\Theta^{(t)})\right\|_2^2 \right] \\
	\leq & \frac{\sum_{t=1}^{T} \left(\alpha_t-\frac{L\alpha_t^2}{2}\right) \mathbb{E} \left\|\nabla_{\Theta} \mathcal{L}^{meta}(\hat{\mathbf{w}}^{(t+1)}(\Theta^{(t)}))\right\|_2^2 }{\sum_{t=1}^{T} \left(\alpha_t-\frac{L\alpha_t^2}{2}\right)} \\
	\leq &  \frac{\delta M \sum_{t=1}^{T} \beta_t \left\|\nabla \mathcal{L}^{meta}(\hat{\mathbf{w}}^{(t+1)}(\Theta^{(t)}))\right\|+ \sum_{t=1}^{T} \frac{L\alpha_t^2}{2} \sigma^2+\!\mathbb{E}[\mathcal{L}^{tr}(\mathbf{w}^{(1)};\Theta^{(1)})]   + \sum_{t=1}^{T} \frac{\delta\beta_t^2}{2}\left\{M(\rho^2+\sigma^2)\right\}  }{\sum_{t=1}^{T} \left(\alpha_t-\frac{L\alpha_t^2}{2}\right)} \\
	\leq &  \frac{\delta M \sum_{t=1}^{T} \beta_t \left\|\nabla \mathcal{L}^{meta}(\hat{\mathbf{w}}^{(t+1)}(\Theta^{(t)}))\right\|+ \sum_{t=1}^{T} \frac{L\alpha_t^2}{2} \sigma^2+\!\mathbb{E}[\mathcal{L}^{tr}(\mathbf{w}^{(1)};\Theta^{(1)})]   + \sum_{t=1}^{T} \frac{\delta\beta_t^2}{2}\left\{M(\rho^2+\sigma^2)\right\}  }{\sum_{t=1}^{T} \alpha_t} \\
	\leq &  \frac{\delta M \sum_{t=1}^{T} \beta_t \left\|\nabla \mathcal{L}^{meta}(\hat{\mathbf{w}}^{(t+1)}(\Theta^{(t)}))\right\|+ \sum_{t=1}^{T} \frac{L\alpha_t^2}{2} \sigma^2+\!\mathbb{E}[\mathcal{L}^{tr}(\mathbf{w}^{(1)};\Theta^{(1)})]   + \sum_{t=1}^{T} \frac{\delta\beta_t^2}{2}\left\{M(\rho^2+\sigma^2)\right\}  }{T \alpha_t} \\
	\leq &  \frac{\delta M \sum_{t=1}^{T} \beta_t \left\|\nabla \mathcal{L}^{meta}(\hat{\mathbf{w}}^{(t+1)}(\Theta^{(t)}))\right\|+ \sum_{t=1}^{T} \frac{L\alpha_t^2}{2} \sigma^2+\!\mathbb{E}[\mathcal{L}^{tr}(\mathbf{w}^{(1)};\Theta^{(1)})]   + \sum_{t=1}^{T} \frac{\delta\beta_t^2}{2}\left\{M(\rho^2+\sigma^2)\right\}  }{T} \max\{L,\frac{\sqrt{T}}{c}\}\\
	=& \mathcal{O}(\frac{C}{\sqrt{T}}).
	\end{align*}	
	The third inequality holds since $\sum_{t=1}^T (2\alpha_t-L\alpha_t^2) \!=\! \sum_{t=1}^T \alpha_t (2-L\alpha_t) \!\geq\! \sum_{t=1}^T \alpha_t$.
	Thus we can conclude that our algorithm can always achieve $\min_{0\leq t \leq T} \mathbb{E}\left[ \left\|\nabla \mathcal{L}^{tr}(\mathbf{w}^{(t)};\Theta^{(t)})\right\|_2^2 \right] \leq \mathcal{O}(\frac{C}{\sqrt{T}})$ in $T$ steps, and this completes our proof of Theorem \ref{th2}.
	
\end{proof}

\subsection{Pytorch codes of our algorithm}
The following is the Pytorch codes of our algorithm, which is essily completed based on the code of MW-Net. The main difference from MW-Net is to re-define the structure of meta-model (CMW-Net) and generate the task family labels in advance. The completed training code is avriable at \url{https://github.com/xjtushujun/CMW-Net}.
\begin{python}
	def norm_func(v_lambda):
	norm_c = torch.sum(v_lambda)
	if norm_c != 0:
	v_lambda_norm = v_lambda / norm_c
	else:
	v_lambda_norm = v_lambda
	return  v_lambda_norm

	class share(MetaModule):
	def __init__(self, input, hidden1, hidden2):
	super(share, self).__init__()
	self.layer = nn.Sequential( MetaLinear(input, hidden1), nn.ReLU(inplace=True) )
	
	def forward(self, x):
	output = self.layer(x)
	return output
	
	class task(MetaModule):
	def __init__(self, hidden2, output, num_classes):
	super(task, self).__init__()
	self.layers = nn.ModuleList()
	for i in range(num_classes):
	self.layers.append(nn.Sequential( MetaLinear(hidden2, output), nn.Sigmoid() ))
	
	def forward(self, x, num, c):
	si = x.shape[0]
	output = torch.tensor([]).cuda()
	for i in range(si):
	output = torch.cat(( output, self.layers[c[num[i]]]( x[i].unsqueeze(0) ) ),0)
	
	return output

	# The structure of CMW-Net
	class VNet(MetaModule):
	def __init__(self, input, hidden1, hidden2, output, num_classes):
	super(VNet, self).__init__()
	self.feature = share(input, hidden1, hidden2)
	self.classfier = task(hidden2, output, num_classes)
	
	def forward(self, x, num, c):
	num = torch.argmax(num, -1)
	output = self.classfier( self.feature(x), num, c )
	return output
	
	optimizer_a = torch.optim.SGD(model.params(), args.lr, momentum=args.momentum, nesterov=args.nesterov, weight_decay=args.weight_decay)
	optimizer_c = torch.optim.Adam(vnet.params(), 1e-3, weight_decay=1e-4)
	
	# Generating task family labels
	es = Kmeans(3)
	es.fit(train_loader.dataset.targets)
	c = es.labels_
	
	for iters in range(num_iters):
	adjust_learning_rate(optimizer_a, iters + 1)
	model.train()
	data, target = next(iter(train_loader))
	data, target = data.to(device), target.to(device)
	meta_model.load_state_dict(model.state_dict())
	y_f_hat = meta_model(data)
	cost = F.cross_entropy(y_f_hat, target, reduce=False)
	cost_v = torch.reshape(cost, (len(cost), 1))
	v_lambda = vnet(cost_v.data, target.data, c)
	v_lambda_norm = norm_func(v_lambda)
	l_f_meta = torch.sum(cost_v * v_lambda_norm)
	meta_model.zero_grad()
	grads = torch.autograd.grad(l_f_meta,(meta_model.params()),create_graph=True)
	meta_model.update_params(lr_inner=meta_lr,source_params=grads)
	
	data_meta,target_meta = next(iter(train_meta_loader))
	data_meta,target_meta = data_meta.to(device),target_meta.to(device)
	y_g_hat = meta_model(data_meta)
	l_g_meta = F.cross_entropy(y_g_hat, target_meta)
	optimizer_c.zero_grad()
	l_g_meta.backward()
	optimizer_c.step()
	
	y_f = model(data)
	cost_w = F.cross_entropy(y_f, target, reduce=False)
	cost_v = torch.reshape(cost_w, (len(cost_w), 1))
	with torch.no_grad():
	w_new = vnet(cost_v,target, c)
	w_v = norm_func(w_new)
	l_f = torch.sum(cost_v * w_v)
	optimizer_a.zero_grad()
	l_f.backward()
	optimizer_a.step()
\end{python}

\newpage

\begin{figure}[H]\vspace{-2mm}
	\centering
	\vspace{-2mm}
	\subfigcapskip=-2mm
	\subfigure[Class imbalance with imbalanced factor 10]{\label{fig11c}
		\includegraphics[width=0.85\textwidth]{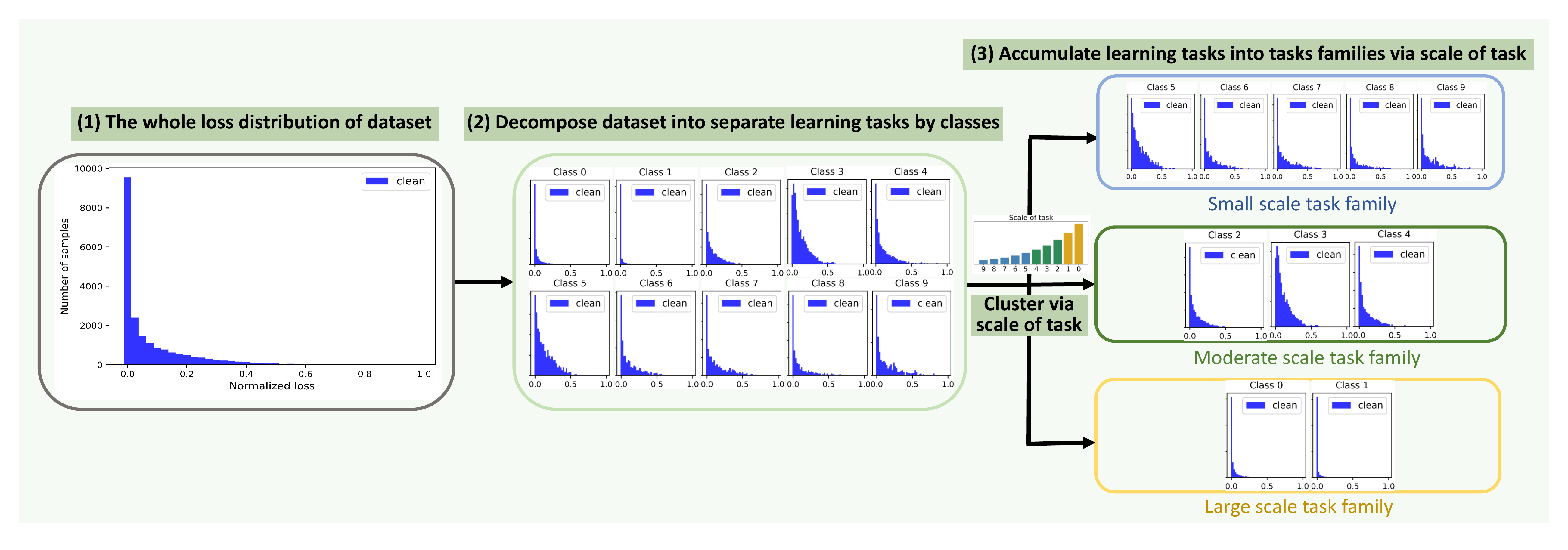}}\\\vspace{-5mm}
	\subfigure[Symmetric noise with noise rate 40\%]{\label{fig11a}
		\includegraphics[width=0.85\textwidth]{./fig/sym_loss.pdf}} \\\vspace{-5mm}
	\subfigure[Asymmetric noise with noise rate 40\%]{\label{fig11b}
		\includegraphics[width=0.85\textwidth]{./fig/asym_loss.pdf}} \\\vspace{-5mm}
	\subfigure[Feature-dependent noise with noise rate 40\%]{\label{fig11d}
		\includegraphics[width=0.85\textwidth]{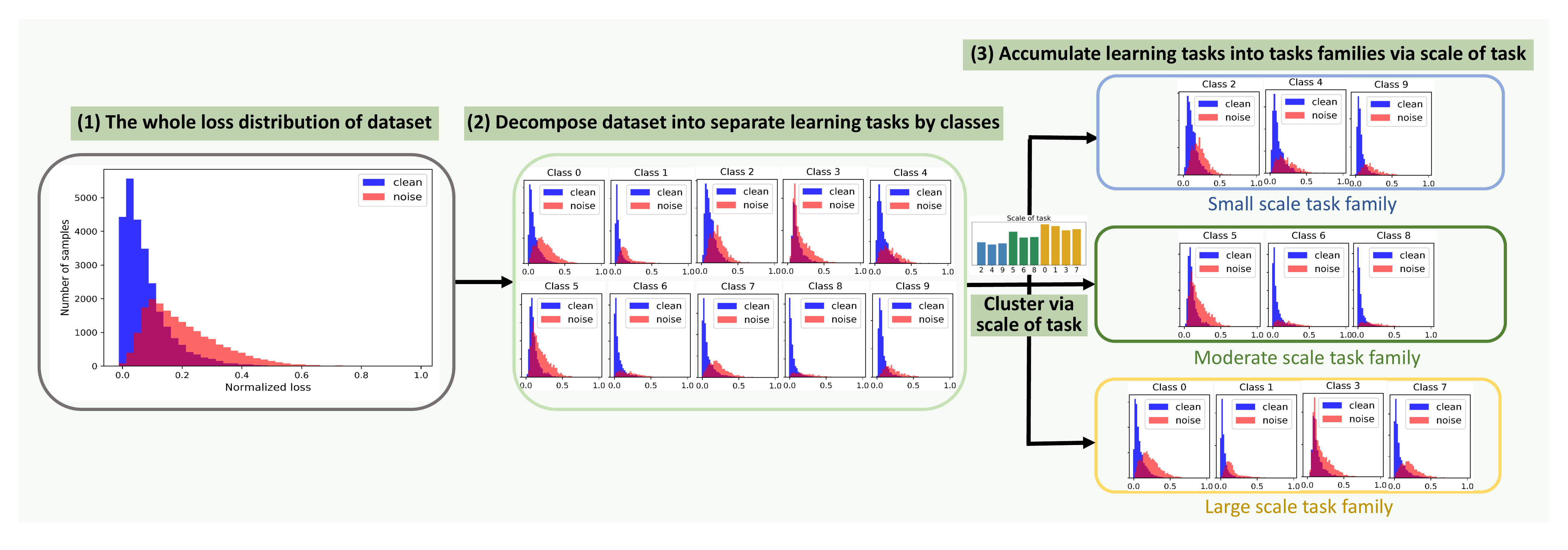}} \vspace{-3mm}
	\caption{ Illustration of the limitation and meta-essence understanding for MW-Net. The success of MW-Net is built upon homoscedastic bias assumption (e.g., in Fig.(b)(1,2), each class has similar loss distributions of clean and noise samples). While MW-Net fails under the heterogeneous bias (e.g., in Fig.(a,c,d)(1,2), each class has their specific loss distributions). The rationality can be revealed from the perspective of meta-learning (see Section 1.2 of main paper). The limitation of MW-Net demostrates that only sample-level loss information can not completely character the heterogeneous bias. This motivates us to introduce task-level information (i.e., scale of task) to reform MW-Net, making it able to distinguish individual bias properties of different tasks, and accumulate tasks with approximately homoscedastic data bias as a task family (e.g., Fig.(a,b,c,d)(3), and details see Section 1.3 \& 3.3 of main paper). }   \label{fig11}
\end{figure}

\section{More Experimental Results and Experimental Settings in Section 4}

\subsection{Additional Illustrations of the Limitation and Meta-Essence Understanding for MW-Net}
Fig. \ref{fig11} illustrates the limitation and meta-essence understanding for MW-Net. Compared with the Fig.2 in the main paper, we further show the class imbalance and feature-dependent bias cases, demonstrating the validity of our claim. 

\begin{algorithm}[t]
	\vspace{0mm}
	\setlength{\abovecaptionskip}{0.cm}
	\setlength{\belowcaptionskip}{-2cm}
	\renewcommand{\algorithmicrequire}{\textbf{Input:}}
	\renewcommand{\algorithmicensure}{\textbf{Output:}}
	\caption{The Efficient CMW-Net Meta-training Algorithm}
	\label{alg:example}
	\begin{algorithmic}[1]  \small
		\REQUIRE  Training dataset $\mathcal{D}^{tr}$, meta-data set $\mathcal{D}^{meta}$, batch size $n,m$, max iterations $T$, meta updating period $T_{Meta}$.
		\ENSURE  Classifier parameter $\mathbf{w}^{(*)}$, CMW-Net parameter $\Theta^{(*)}$
		\STATE
		Apply $K$-means on the sample numbers of all training classes to obtain $\Omega=\{\mu_k\}_{k=1}^{K}$ sorted in an ascending order.
		\STATE Initialize classifier network parameter $\mathbf{w}^{(0)}$ and CMW-Net parameter $\Theta^{(0)}$.
		\FOR  {$t=0$ {\bfseries to} $T-1$}
		\STATE $\{x,y\} \leftarrow$ SampleMiniBatch($\tilde{\mathcal{D}}^{tr},n$).
		\IF{$t$ \% $T_{Meta}=0$}
		\STATE $\{x^{meta},y^{meta}\} \leftarrow$ SampleMiniBatch($\mathcal{D}^{meta},m$).
		\STATE Formulate the learning manner of classifier network $\hat{\mathbf{w}}^{(t+1)}(\Theta)$ by Eq. (7).
		\STATE Update parameter $\Theta^{(t+1)}$ of CMW-Net by Eq. (8).
		\ENDIF 
		\STATE Update parameter $\mathbf{w}^{(t+1)}$ of classifier by Eq. (9).
		\ENDFOR
	\end{algorithmic}
\end{algorithm}

\begin{figure}[t]\vspace{-2mm}
	\centering
	\vspace{-2mm}
	\subfigcapskip=-2mm
	\subfigure[Comparision of computation cost]{\label{figa}
		\includegraphics[width=0.48\textwidth]{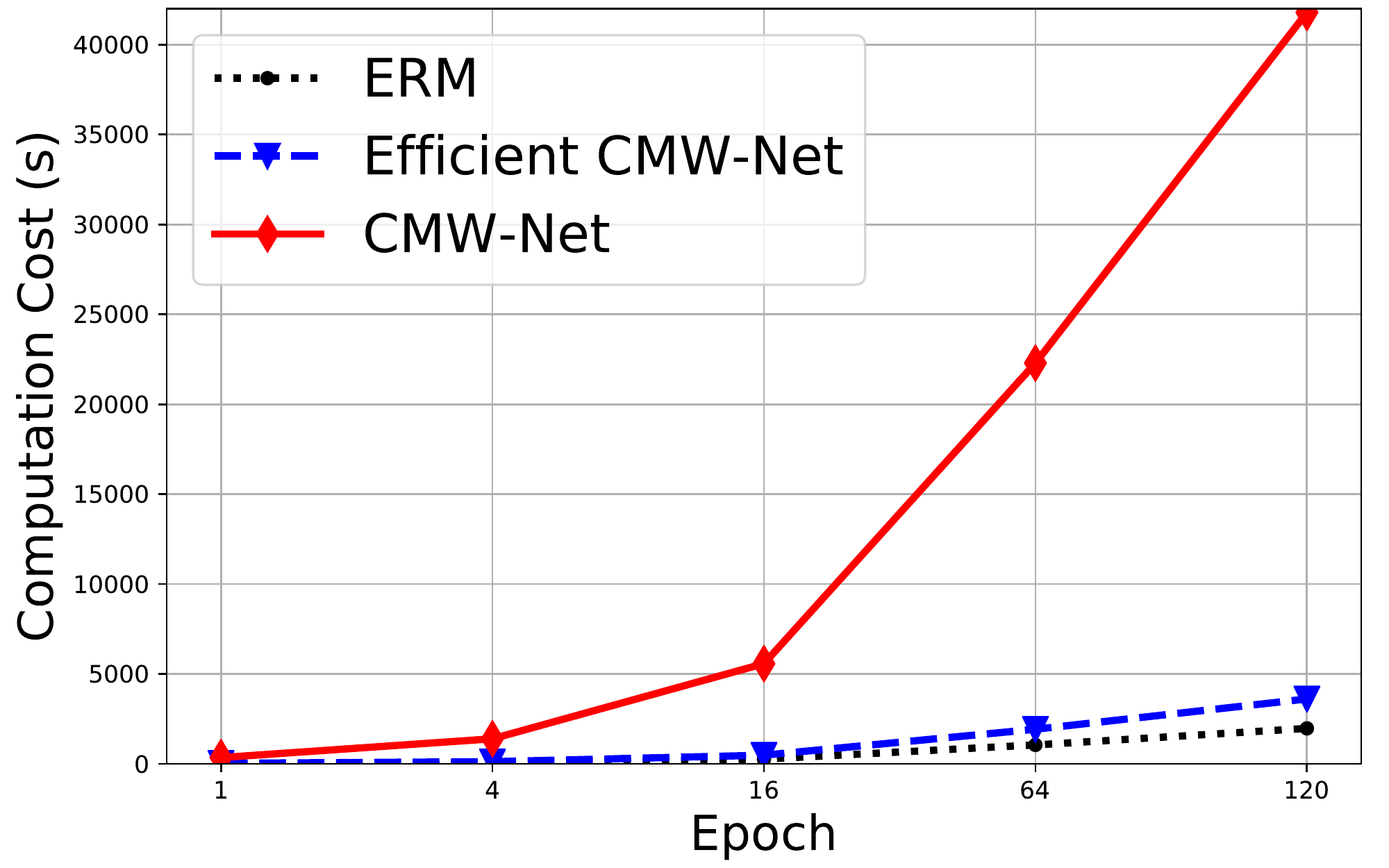}}  \hspace{-4mm}
	\subfigure[Comparision of performance]{\label{figb}
		\includegraphics[width=0.48\textwidth]{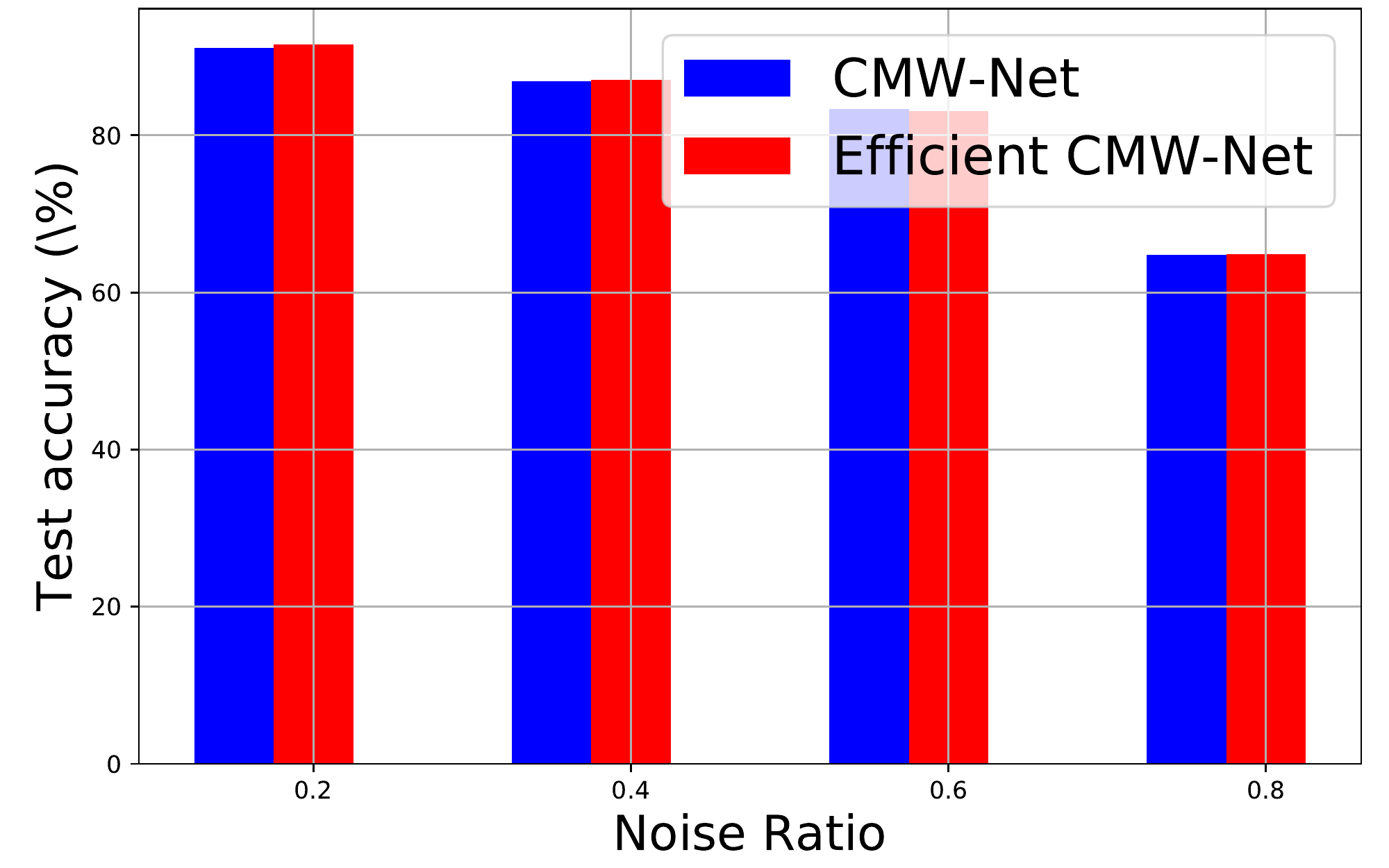}}\vspace{-3mm}
	\caption{Comparison of (a) computation cost and (b) performance between Efficient CMW-Net and CMW-Net. }\label{fig}
\end{figure}

\subsection{Efficient CMW-Net Algorithm}

To reduce the cost of step-wise optimization for CMW-Net, we attempt to update CMW-Net once after updating classifier model several steps ($T_{Meta}$ steps), and the revised algorithm is shown in Algorithm \ref{alg:example}, where the revised steps are highlighted in red. We set $T_{Meta} = 10$, and denote this method as Efficient CMW-Net. Fig.\ref{fig} shows the computation cost and performance of Efficient CMW-Net compared with CMW-Net. We can see that Efficient CMW-Net substantially reduces the computation cost of CMW-Net, while can still reserve the performance. This also implies that there remains a large room for further algorithm efficiency enhancement of our model by reducing the cost of meta-gradient optimization process, which will be further investigated in our future research.

\subsection{Class Imbalance Experiments}
In this series of experiments, we use ResNet-32 \cite{he2016deep} as the classifier network with softmax cross-entropy loss by SGD with a momentum 0.9, a weight decay $5 \times 10^{-4}$, an initial learning rate 0.1. The learning rate of ResNet-32 is divided by 100 after 160 and 180 epoch (for a total 200 epochs). The learning rate of CMW-Net is fixed as $10^{-4}$, and the weight decay of CMW-Net is fixed as $10^{-5}$. The batch size is set as 100 for all experiments.  We randomly selected fixed images per class in every epoch from the training set as the meta-data set and the number of selected images for each class is the same as the number of the least class.

Fig. \ref{SM1} shows the weighting schemes learned by CMW-Net on CIFAR-10-LT and CIFAR-100-LT, under different imbalance settings. It can be seen that our CMW-Net can adaptively learn proper weighting schemes according to different degrees of class imbalance. For example, when the dataset is balanced, CMW-Net tends to learn approximately similar weighting functions for all three  task families. When the degree of class imbalance becomes more significant, the weighting schemes extracted from different task families tend to be more largely varied, showing their different internal bias characteristics.

In Fig. \ref{SM2}, we further plot the confusion matrices produced by the MW-Net and CMW-Net methods, respectively. As can be easily seen, CMW-Net consistently outperforms MW-Net, and CMW-Net more evidently improves the accuracies of MW-Net under larger class imbalance rate. Specifically, CMW-Net gets more prominent performance gain on tail classes, and meanwhile maintains the performance on head classes.

\begin{figure*}[t]
	\centering
	\subfigcapskip=-1mm
	\subfigure[ 1 for CIFAR-10]{
		\includegraphics[width=0.23\textwidth]{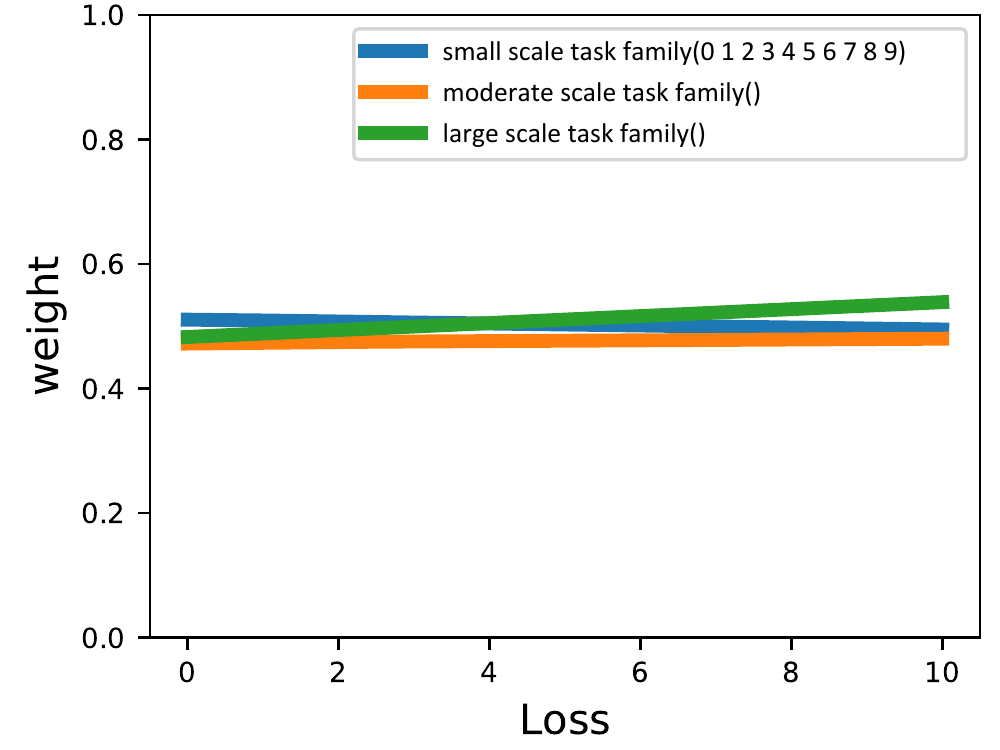}}
	\subfigure[ 10 for CIFAR-10]{
		\includegraphics[width=0.23\textwidth]{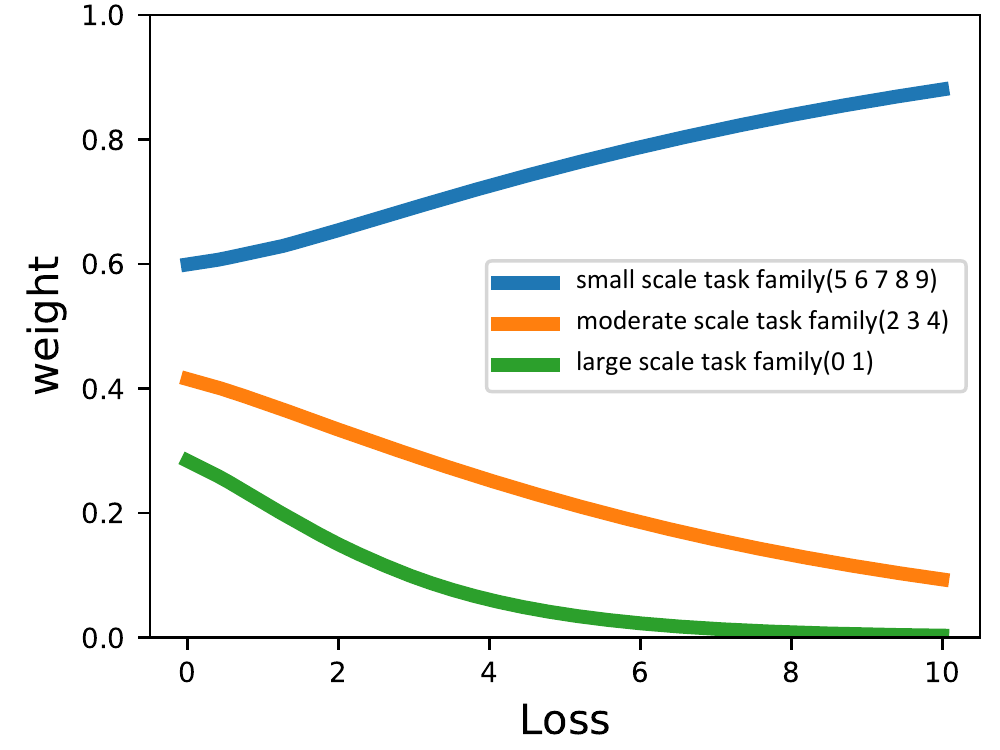}}
	\subfigure[ 20 for CIFAR-10]{
		\includegraphics[width=0.23\textwidth]{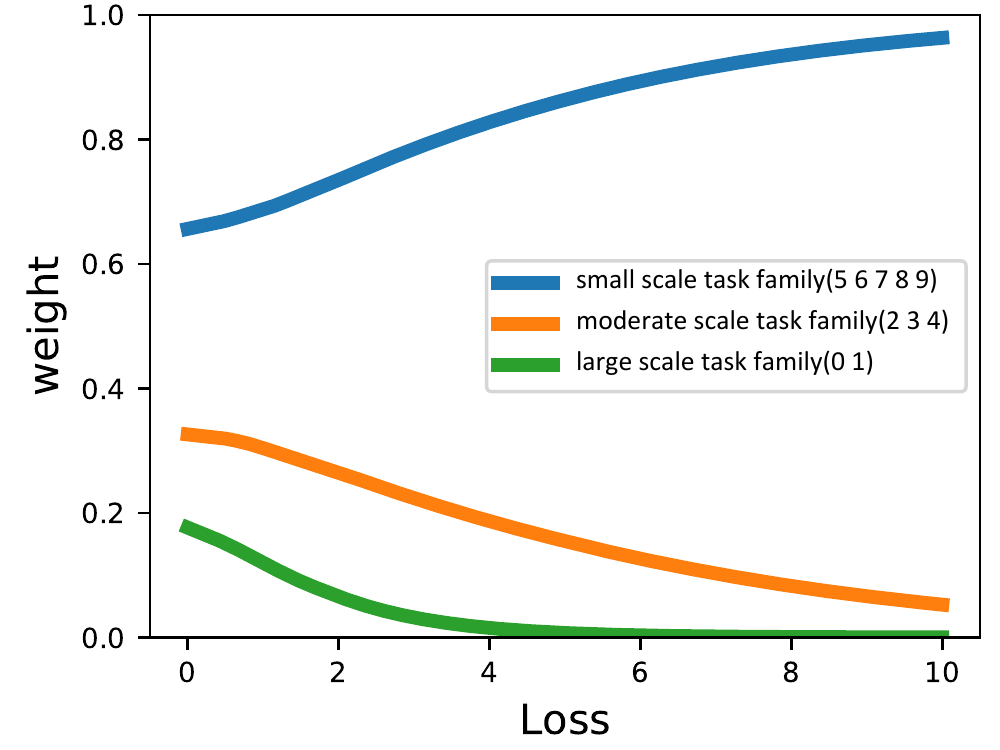}}
	\subfigure[ 50 for CIFAR-10]{
		\includegraphics[width=0.23\textwidth]{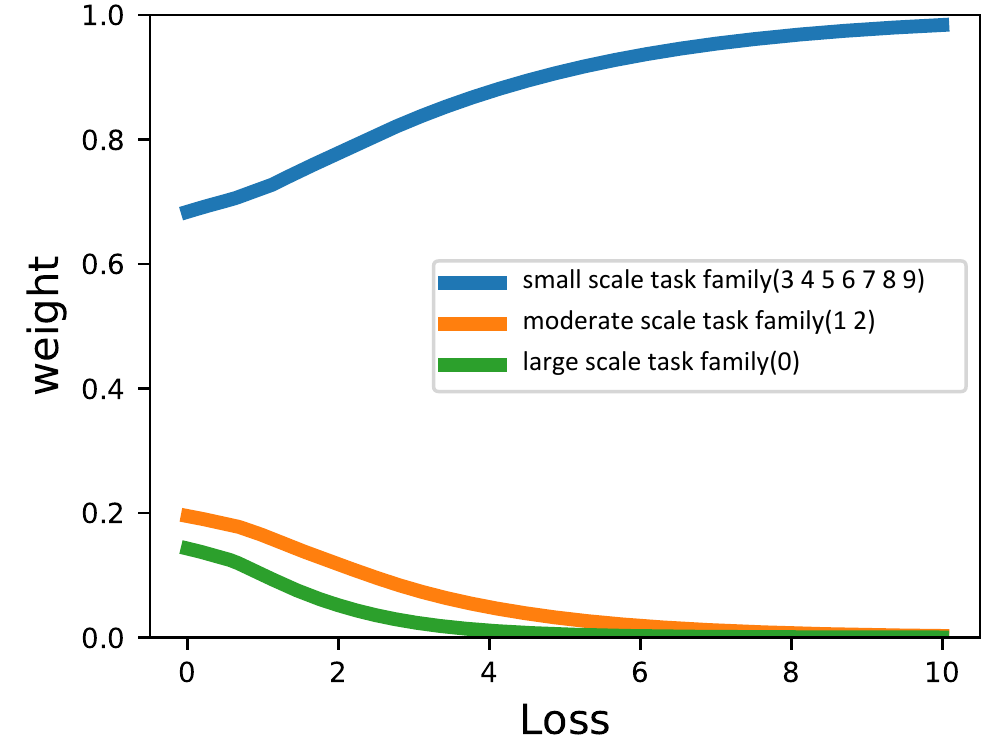}}  \\  \vspace{-3mm}
	\subfigure[ 100 for CIFAR-10]{
		\includegraphics[width=0.23\textwidth]{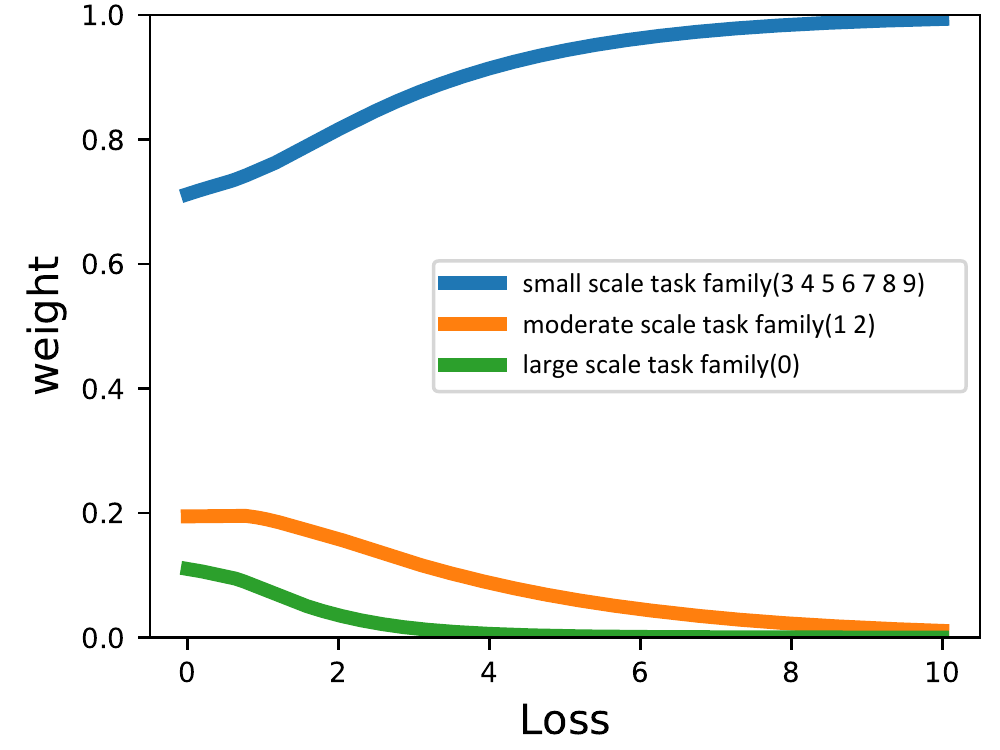}}
	\subfigure[ 200 for CIFAR-10]{
		\includegraphics[width=0.23\textwidth]{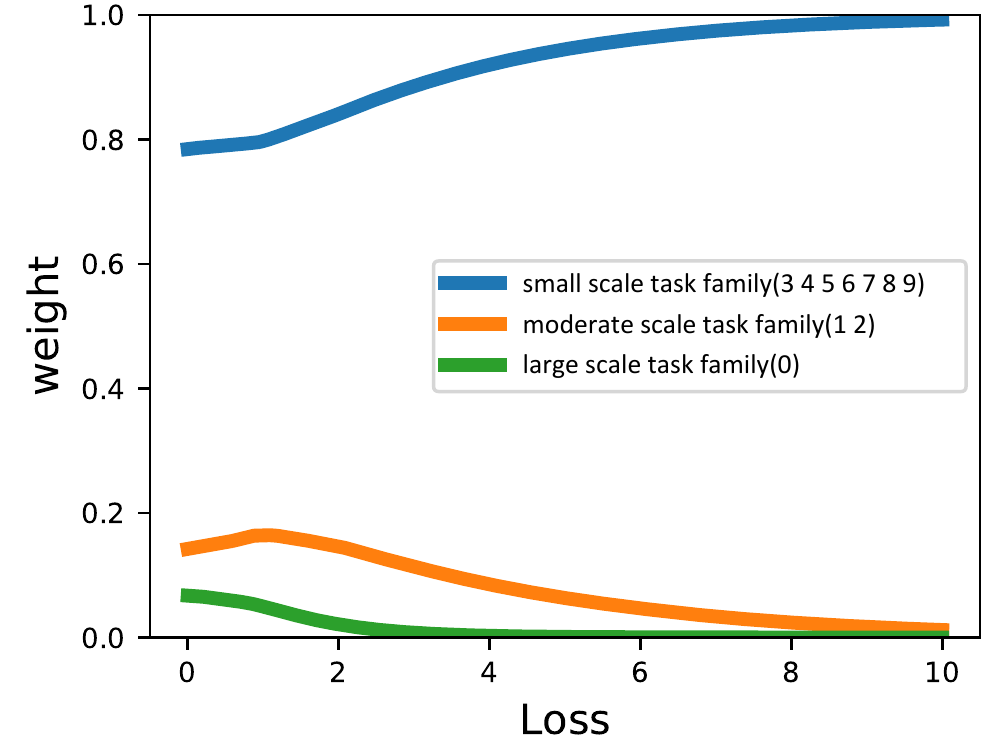}}
	\subfigure[ 1 for CIFAR-100]{
		\includegraphics[width=0.23\textwidth]{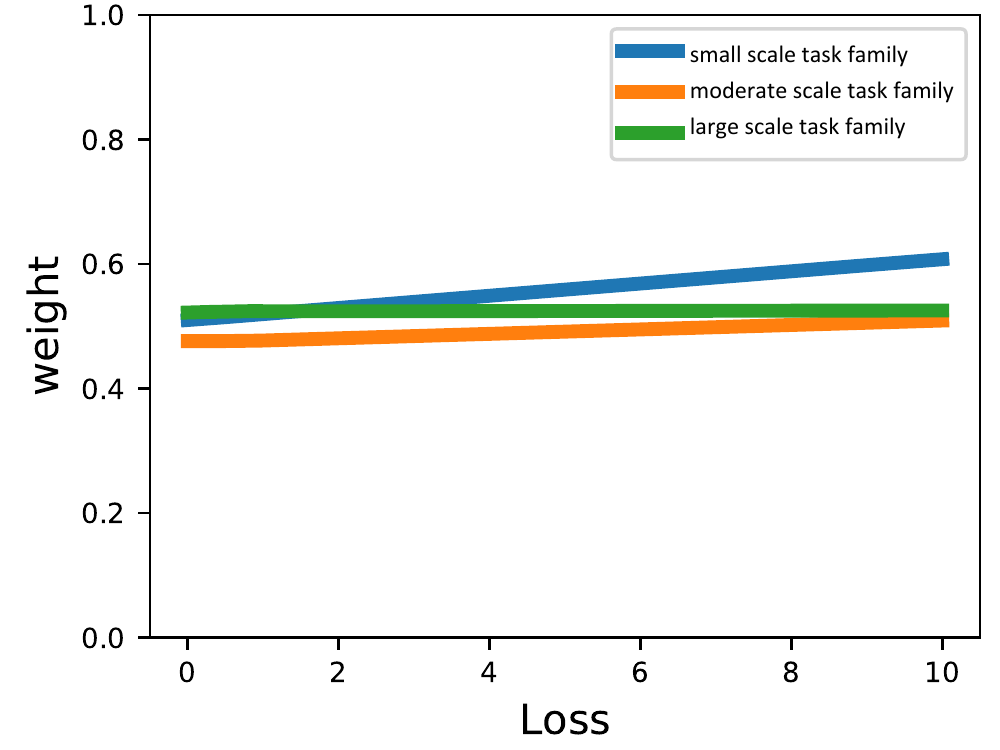}}
	\subfigure[ 10 for CIFAR-100]{
		\includegraphics[width=0.23\textwidth]{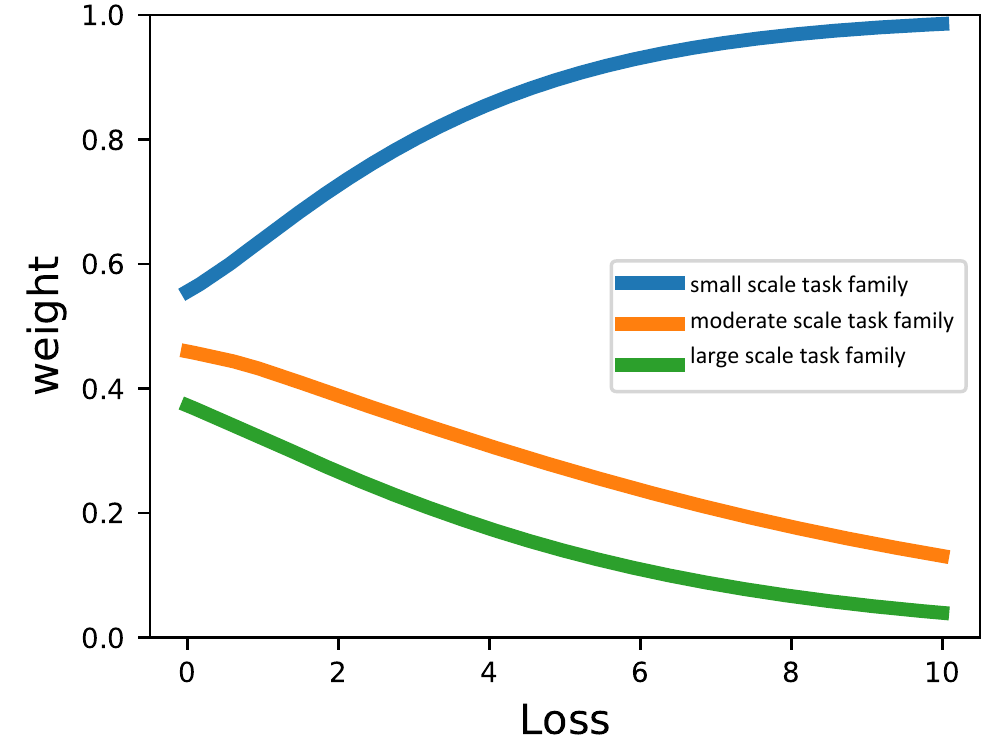}}\\ \vspace{-3mm}
	\subfigure[ 20 for CIFAR-10]{
		\includegraphics[width=0.23\textwidth]{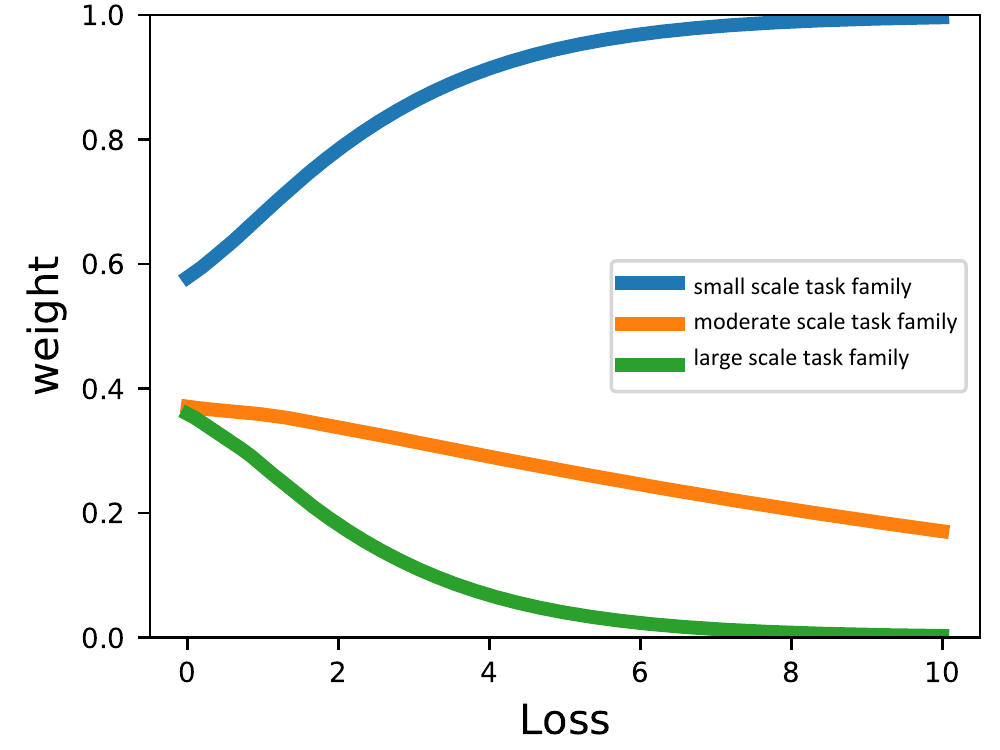}}
	\subfigure[ 50 for CIFAR-10]{
		\includegraphics[width=0.23\textwidth]{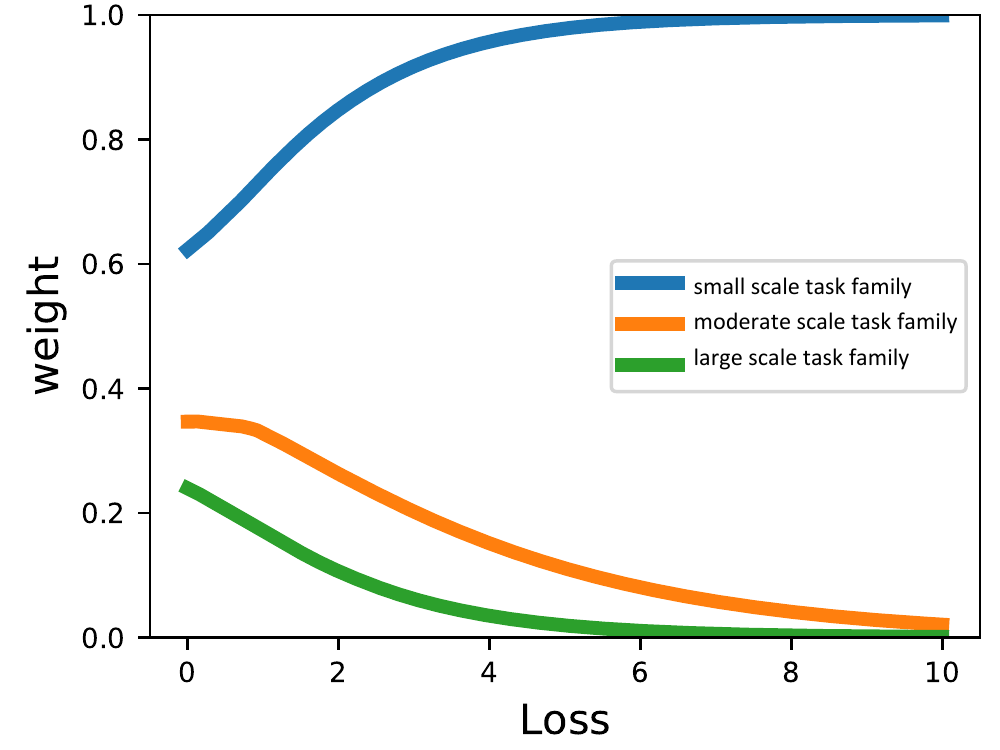}}
	\subfigure[ 100 for CIFAR-10]{
		\includegraphics[width=0.23\textwidth]{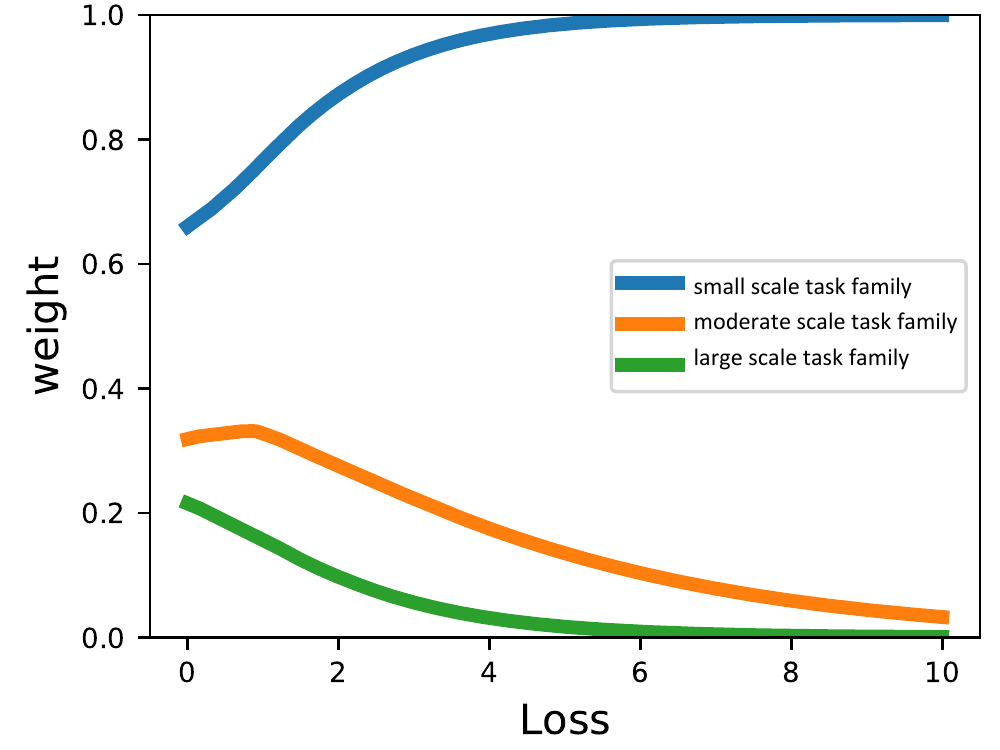}}
	\subfigure[ 200 for CIFAR-10]{
		\includegraphics[width=0.23\textwidth]{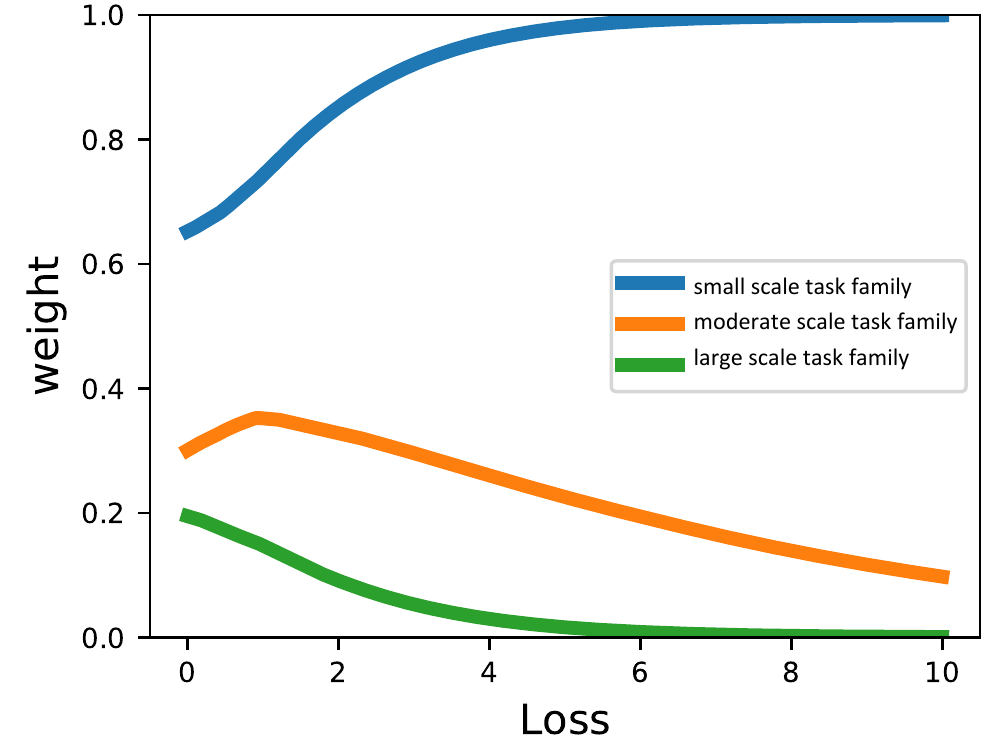}}\vspace{-3mm}
	\caption{Weighting schemes learned by CMW-Net on CIFAR-10-LT and CIFAR-100-LT with imbalance factors ranging from 1 to 200.}\label{SM1}
\end{figure*}
\begin{figure*}[t]
	\centering
	\subfigcapskip=-1mm
	\subfigure[ 1 for CIFAR-10]{
		\includegraphics[width=0.23\textwidth]{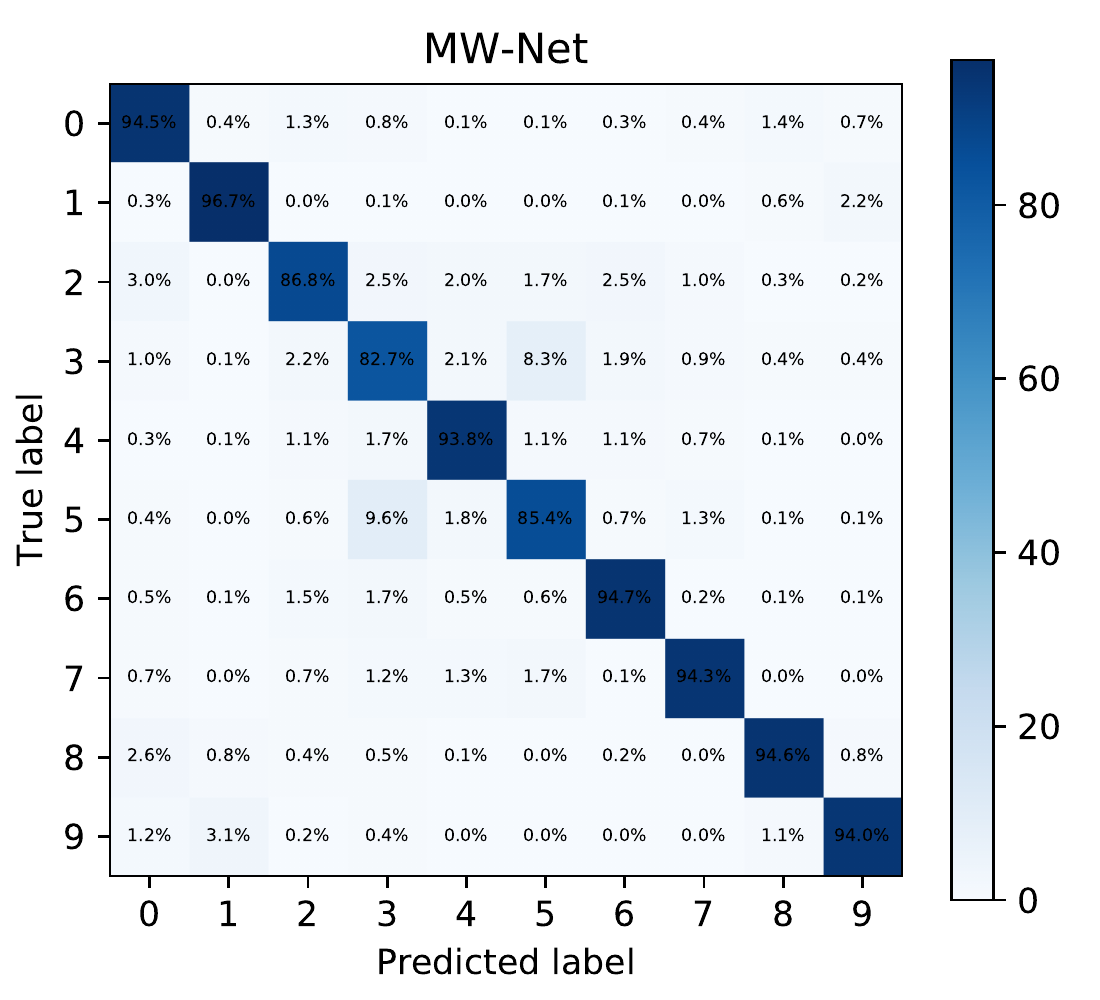}
		\includegraphics[width=0.23\textwidth]{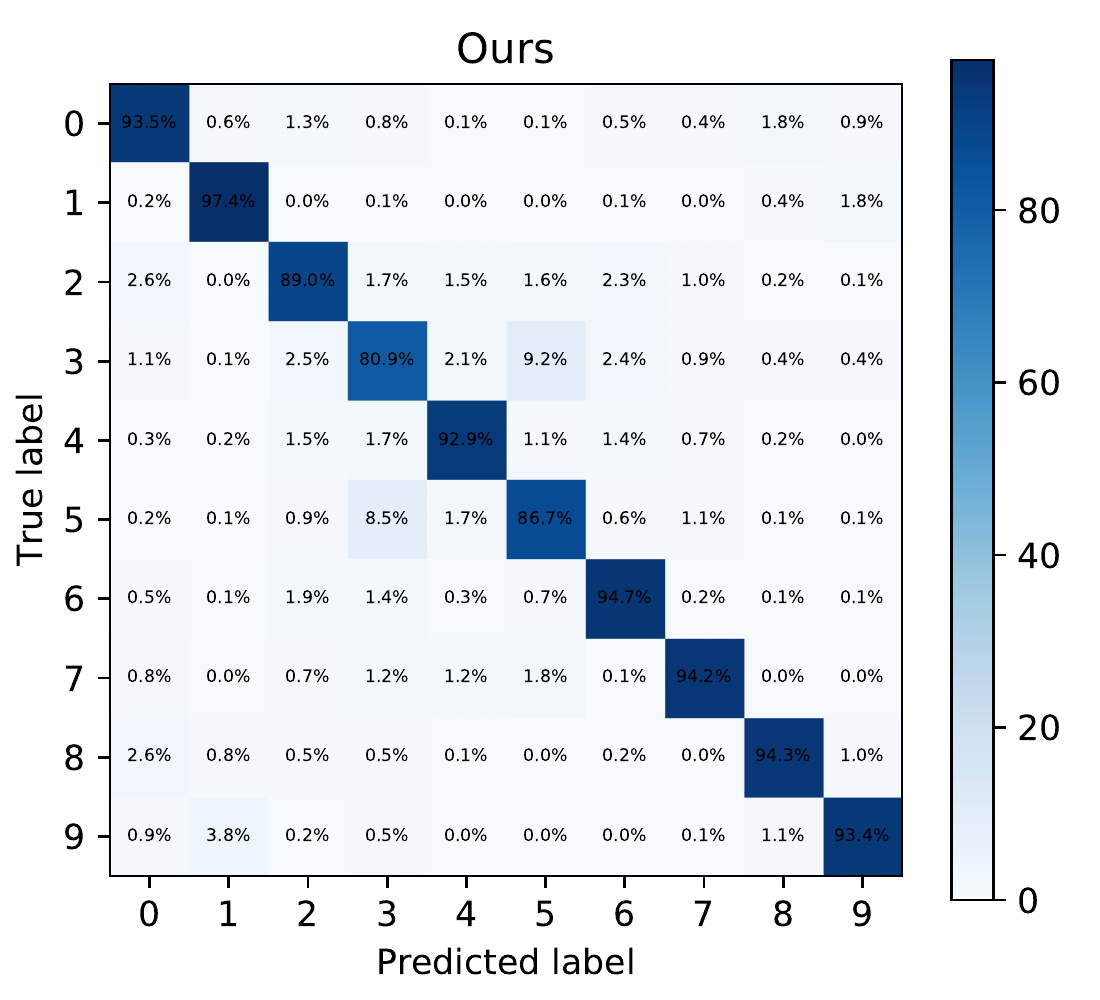}} \ \ \
	\subfigure[ 10 for CIFAR-10]{
		\includegraphics[width=0.23\textwidth]{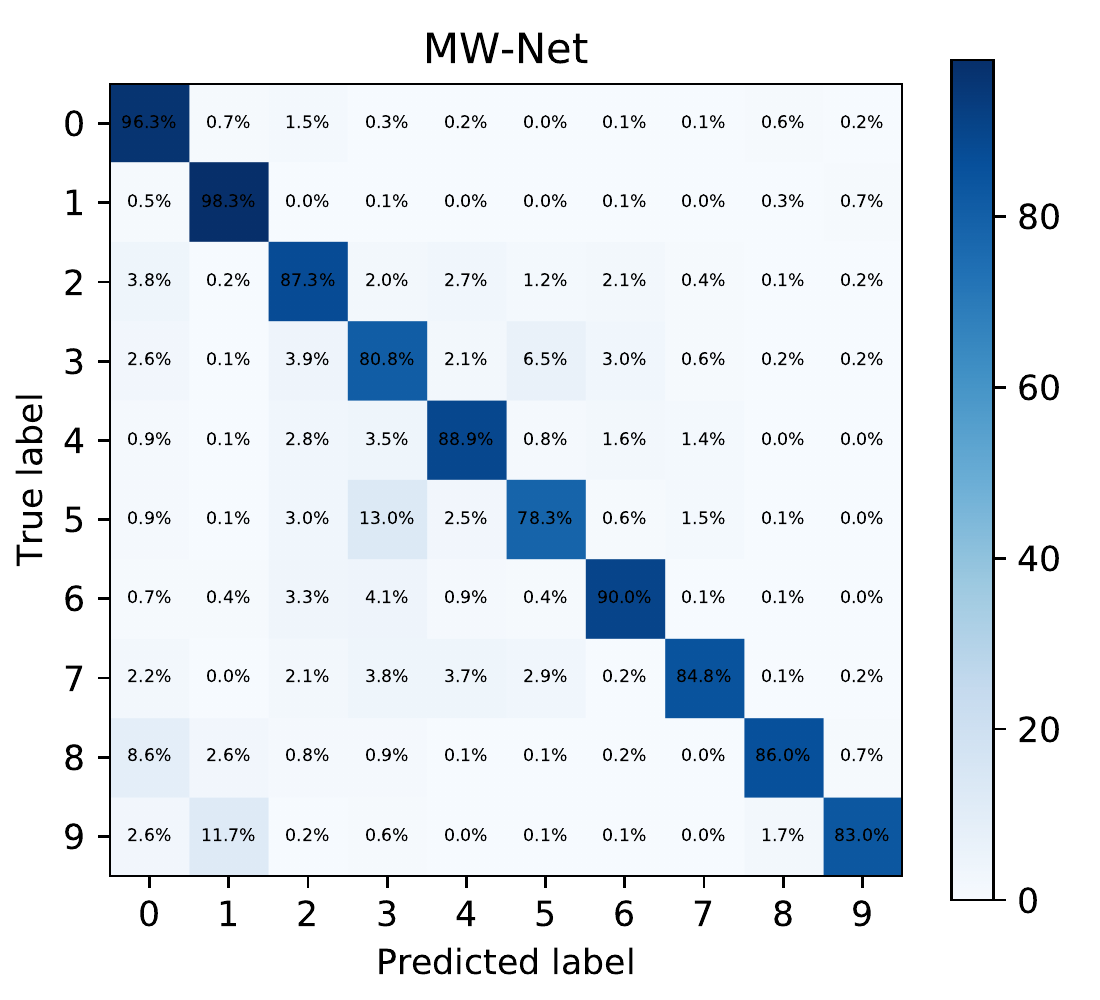}
		\includegraphics[width=0.23\textwidth]{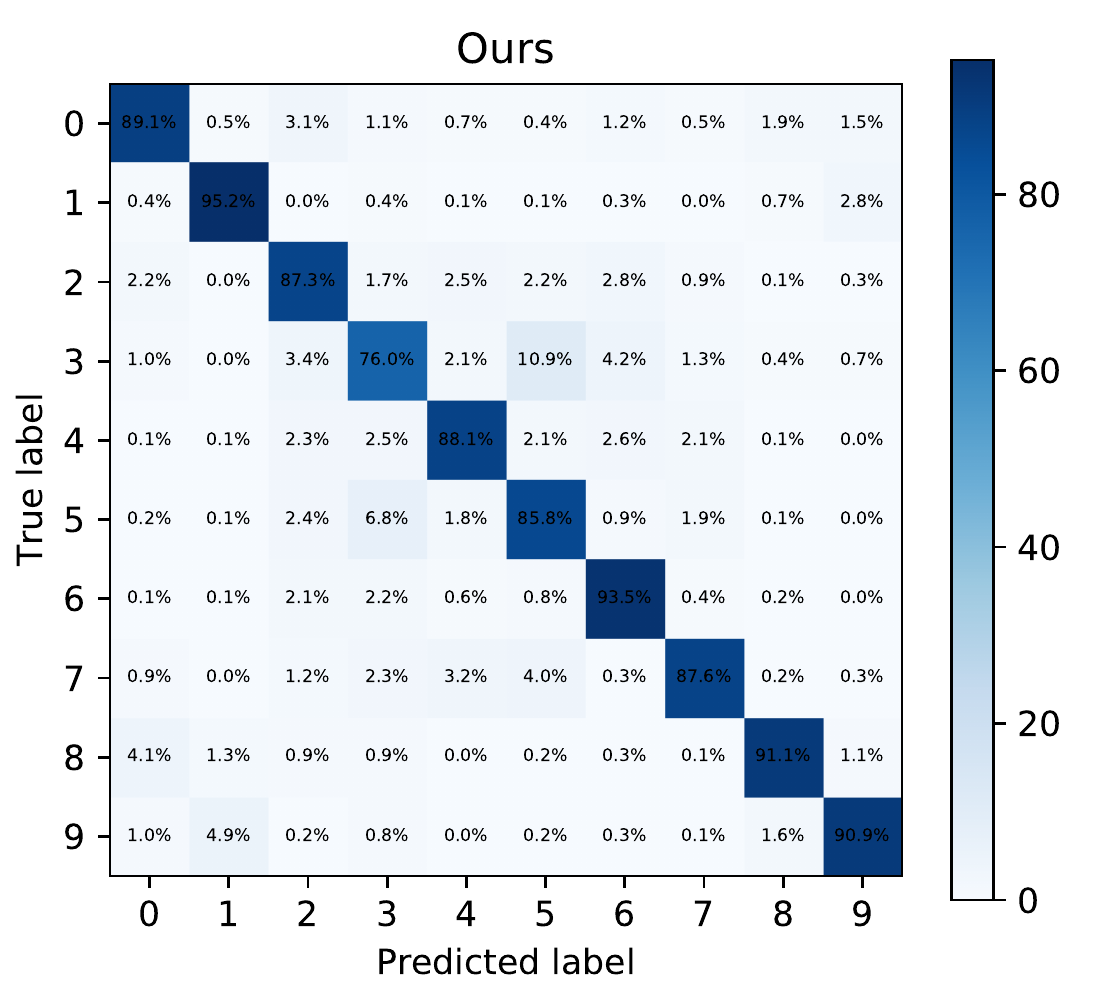}}  \\ \vspace{-3mm}
	\subfigure[ 20 for CIFAR-10]{
		\includegraphics[width=0.23\textwidth]{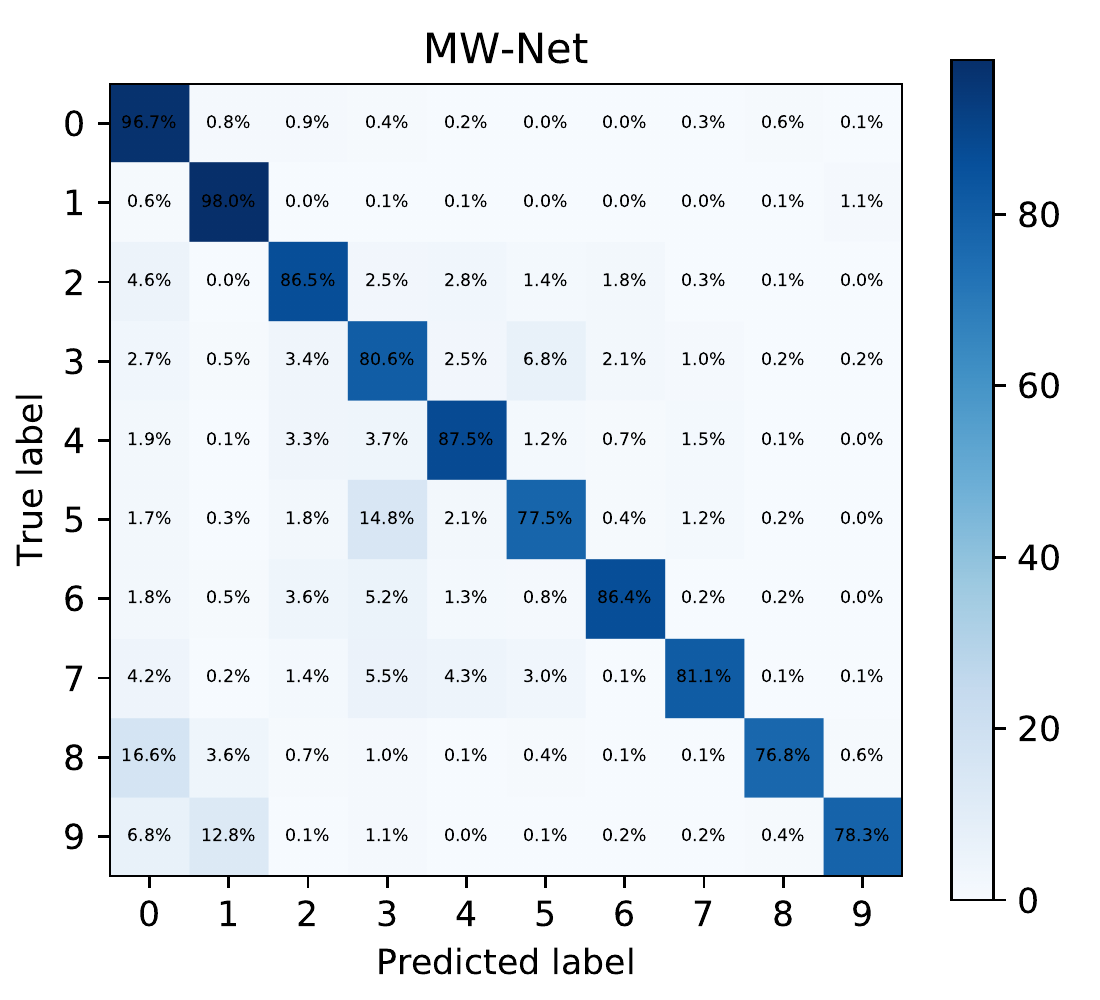}
		\includegraphics[width=0.23\textwidth]{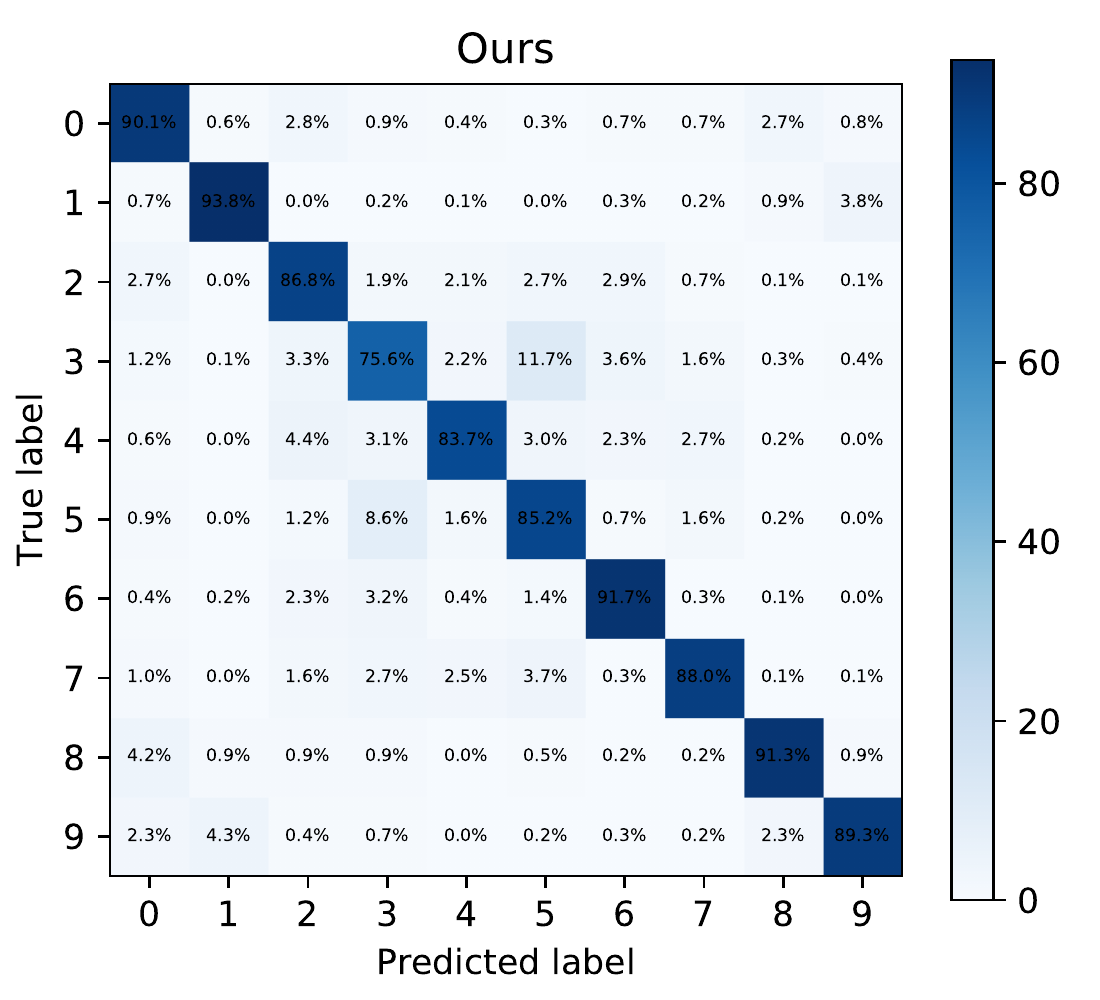}} \ \ \
	\subfigure[ 50 for CIFAR-10]{
		\includegraphics[width=0.23\textwidth]{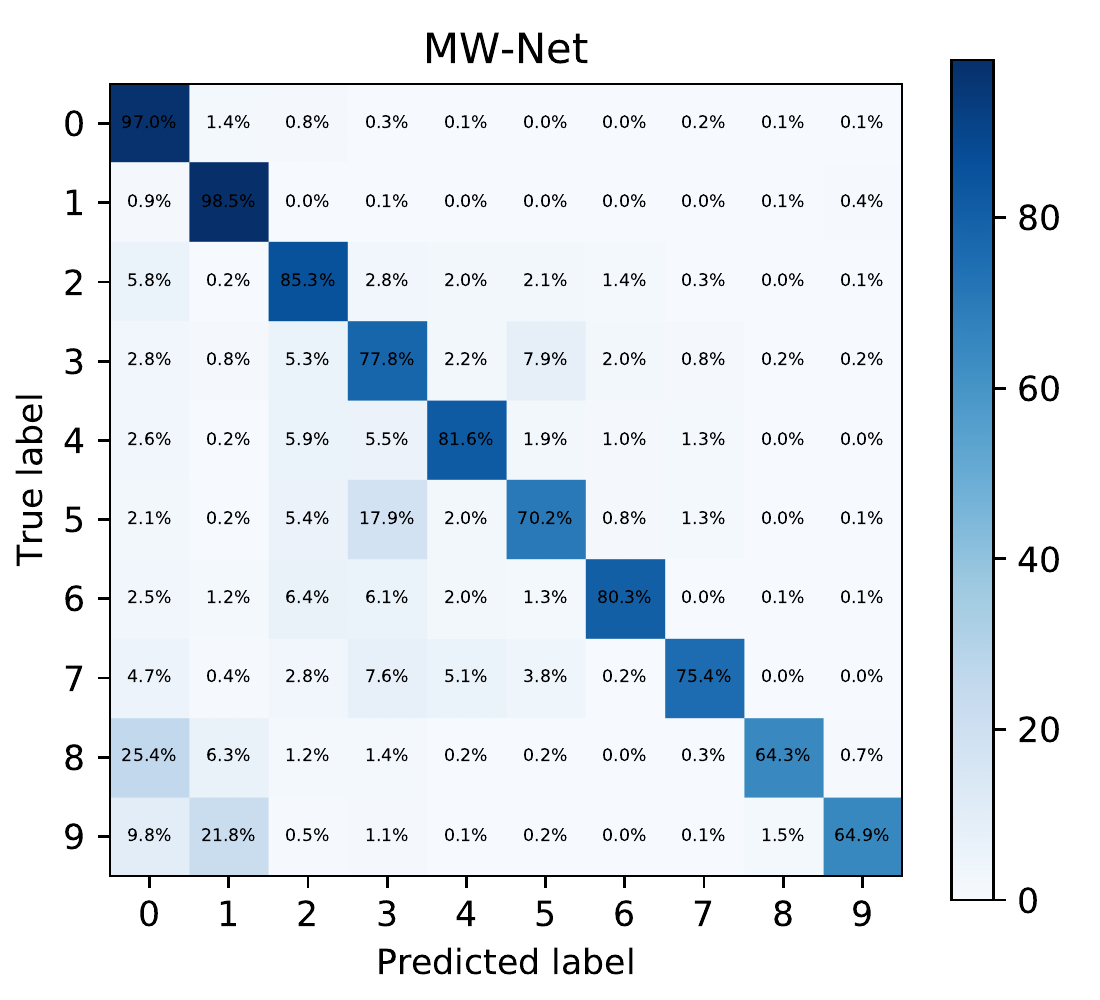}
		\includegraphics[width=0.23\textwidth]{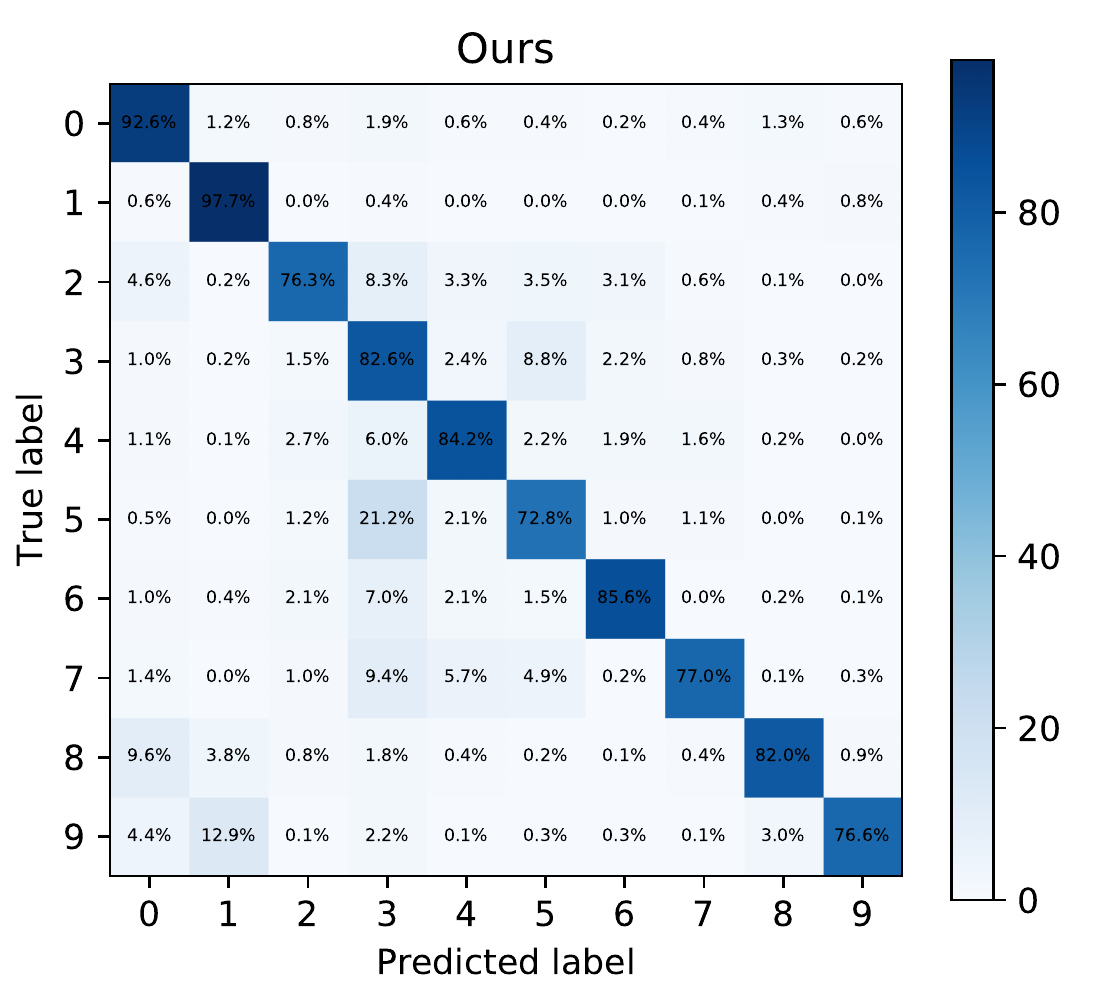}}\\ \vspace{-3mm}
	\subfigure[ 100 for CIFAR-10]{
		\includegraphics[width=0.23\textwidth]{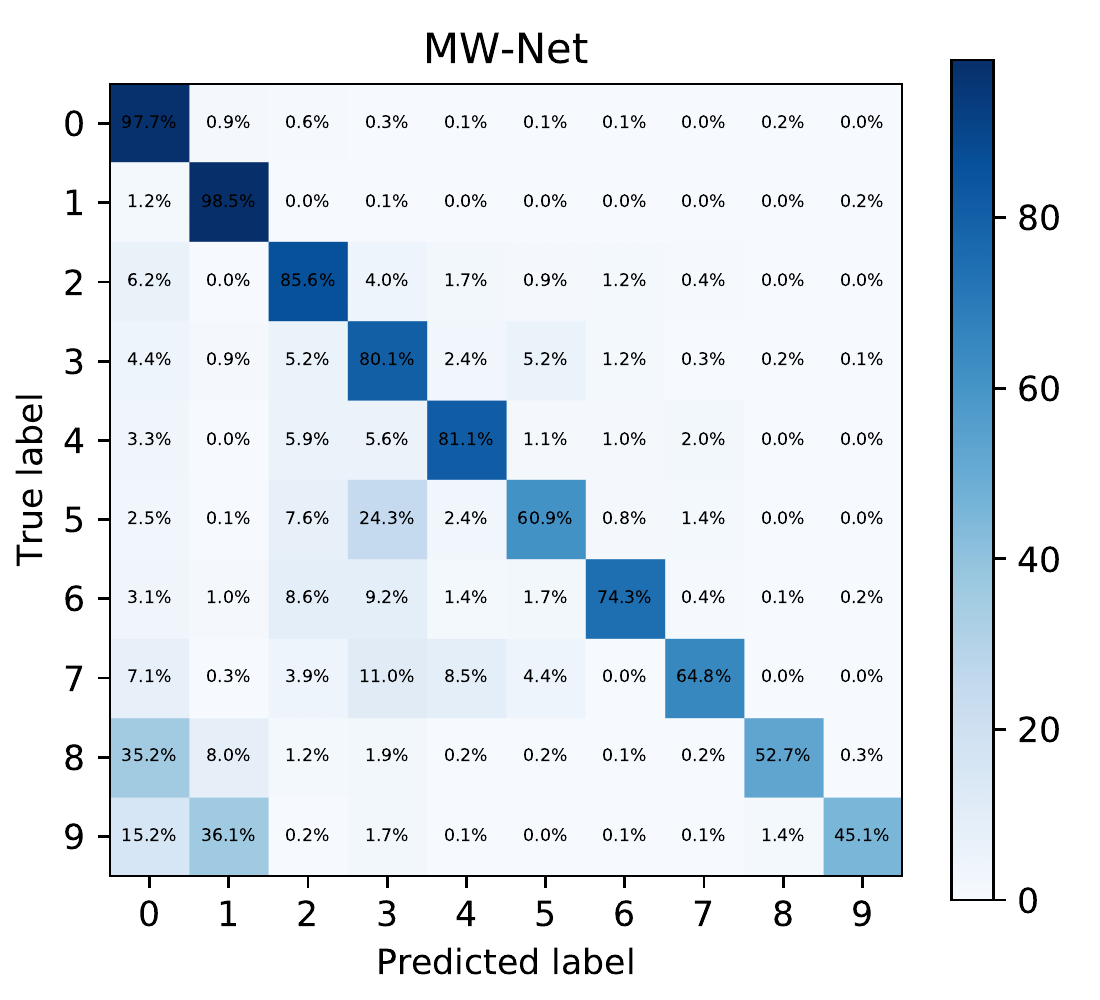}
		\includegraphics[width=0.23\textwidth]{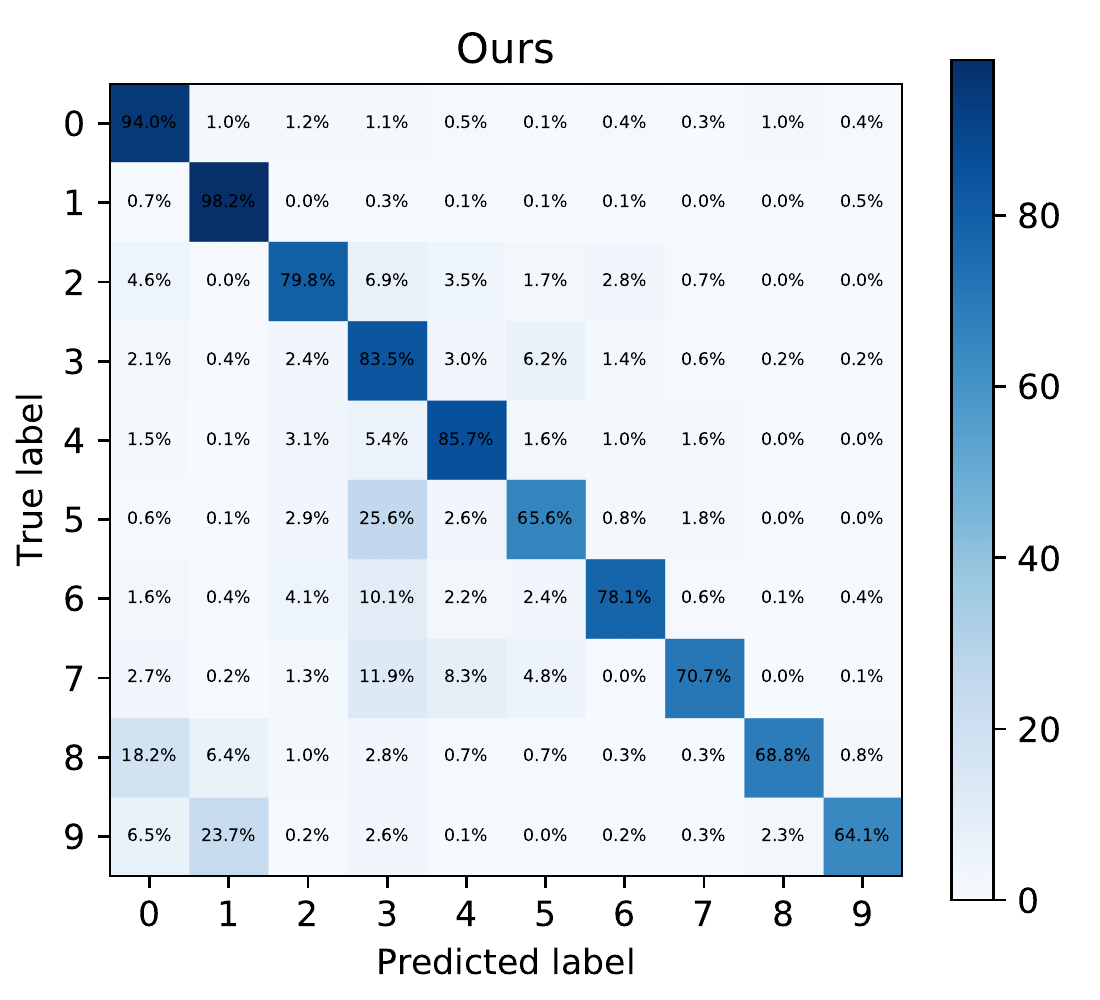}}   \ \ \
	\subfigure[ 200 for CIFAR-10]{
		\includegraphics[width=0.23\textwidth]{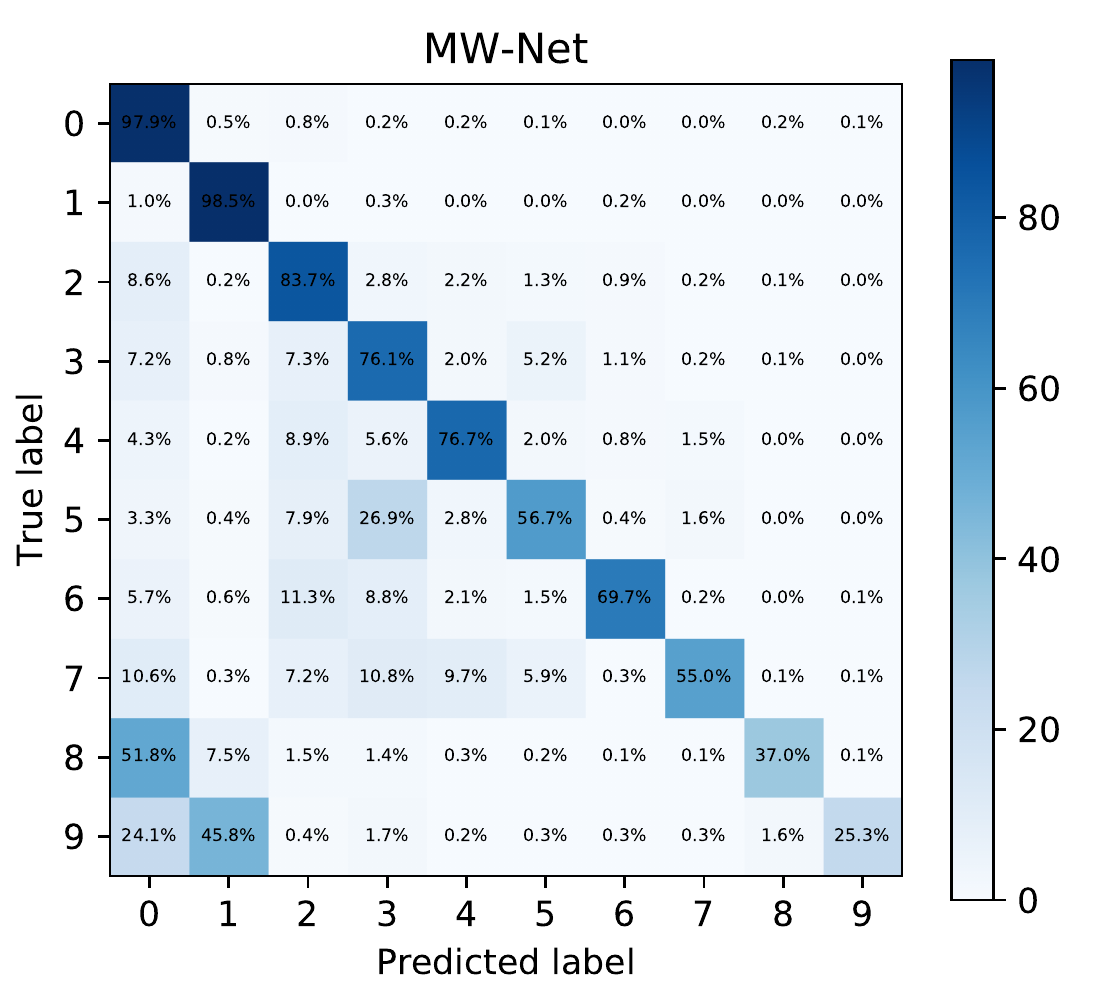}
		\includegraphics[width=0.23\textwidth]{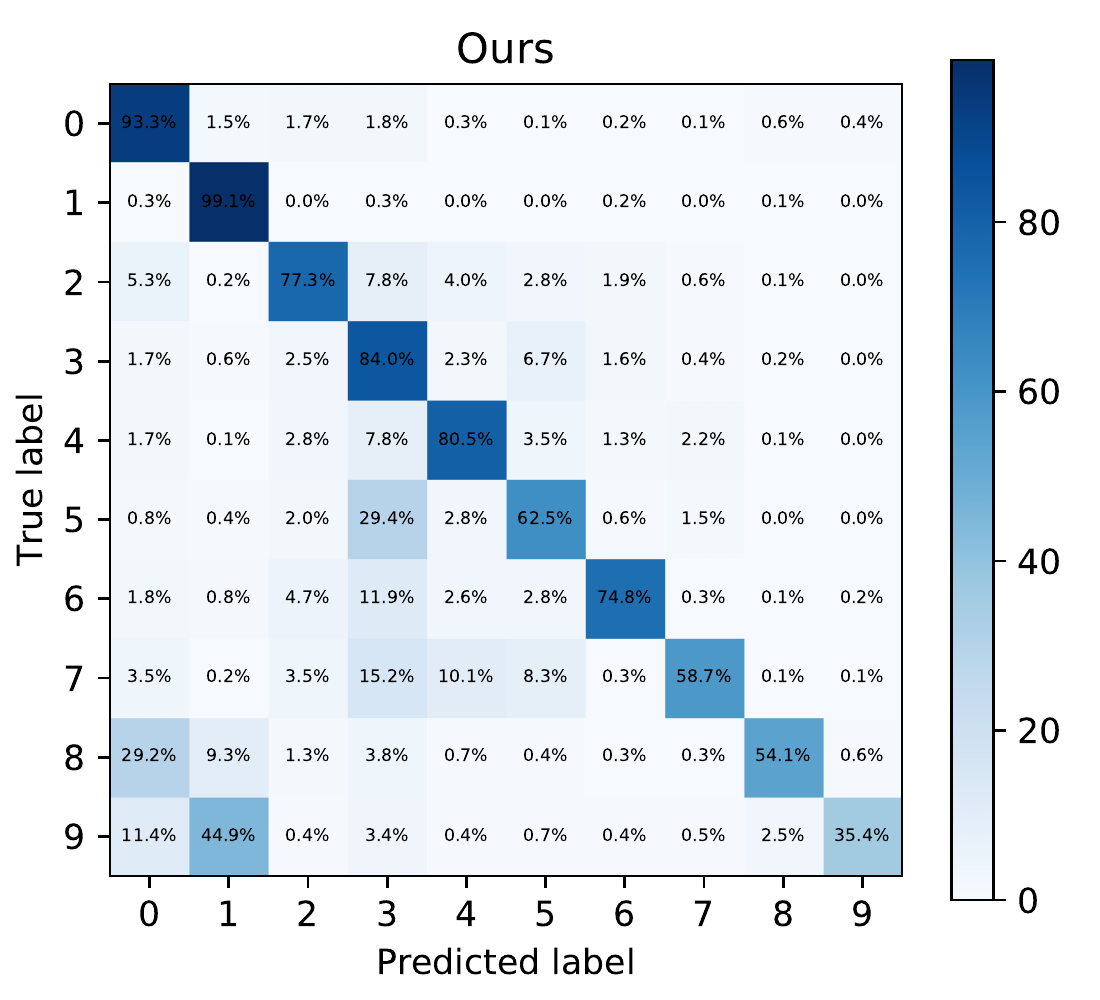}}\vspace{-3mm}
	\caption{Confusion matrices obtained by (left) MW-Net and (right) CMW-Net on CIFAR-10-LT with imbalance factors ranging from 1 to 200.}\label{SM2}
\end{figure*}

\subsection{Feature-independent Label Noise Experiment}
In this series of experiments, we adopted an 18-layer PreAct Resnet \cite{He2017Identity} as our classifier network, with softmax cross-entropy loss by SGD with a momentum 0.9, a weight decay $5 \times 10^{-4}$. For CMW-Net, we set the initial learning rate as 0.1 and the learning rate of classification network is divided by 10 after 80 and 100 epoch (for a total 120 epochs). For the CMW-Net-SL, we set the initial learning rate as 0.01 and the learning rate of classification network is divided by 10 after 150 epoch (for a total 300 epochs) following by Dividemix \cite{li2019dividemix}. The batch size is specified as 128 for all experiments. We adopt Adam optimizer to optimize CMW-Net and the learning rate of CMW-Net is fixed as $10^{-3}$, and the weight decay of CMW-Net is fixed as $10^{-4}$. We repeat the experiments with 3 random trials and report the mean value and standard deviation.

Motivated by M-correction \cite{arazo2019unsupervised} and Dividemix \cite{li2019dividemix}, we selected the meta data at each epoch according to the training loss. Specifically, we explore to create the meta dataset dynamically along iteration, based on the high-quality clean samples as well as its high-quality pseudo labels from the training set (with lowest losses) as an unbiased estimator of the clean data-label distribution in each iteration of our algorithm. To make the meta dataset balanced, we selected 10 images per class. In this case, the performance of meta dataset can be served as an indicator of whether CMW-Net is trained to filter noisy samples and generalize to clean test distribution.

Such meta dataset generation strategy may lack of diversity patterns to characterize the latent clean data-label distribution. To overcome this, we explore to utilize mixup technique \cite{zhang2018mixup} to enrich the variety of our proposed meta dataset distribution while maintaining the unbiasedness in terms of clean test distribution. The hyperparameter of convex combination is randomly sampled from a Beta distribution $Beta(1,1)$. Extensive experiments have verified the effectiveness of such created meta dataset from training dataset. Such property makes our meta-learning algorithm feasible to be applied to real-world biased datasets, where it is generally hard to collect an ideal high-quality extra clean meta dataset. We also use such meta dataset generation strategy in all our noisy labels experiments as well as all real-world biased datasets.

Figs. \ref{SM3} and \ref{SM4} show the empirical pdfs of cross-entropy loss for each class on CIFAR-10 dataset under symmetric and asymmetric noises with varying noise rates, respectively. The corresponding weighting functions and weight distributions over the training examples learned by MW-Net \cite{shu2019meta} and the proposed CMW-Net are also depicted. It can be easily observed that compared with MW-Net, CMW-Net has better flexibility to deal with both training data bias cases, even for inter-class heterogenous data biases. Specifically, the proposed CMW-Net can adaptively adjust its weighting schemes to adapt variations of noise rates, and behaves consistently with the underlying data biased patterns, naturally leading to its better performance on distinguishing clean and noisy images. Note that even for high noise rate (e.g., 80\%) scenarios, our method still shows fine capability of distinguishing clean and noisy images.

To further improve the learning effect of CMW-Net, we introduce additional soft label supervision to build the CMW-Net-SL strategy.
Figs. \ref{SM5} and \ref{SM6} show the confusion matrices learned by two methods, respectively. Specifically, the confusion matrices obtained by CMW-Net almost correspond to the noise transition matrices, and those calculated by CMW-Net-SL contain the refurbished labels by soft labels. Although the noise rate was in relatively high levels (e.g., 60\% and 80\%), most of the diagonal entries had probability larger than 0.95, implying the effectiveness of CMW-Net-SL on its fine label correction ability. This side information is thus validated to be able to compensate beneficially to the sample reweighting learning, and ameliorate both weighting scheme extracting and robust classifier learning in a stable way.

\subsection{Feature-dependent Label Noise Experiment}
In this series of experiments, we use ResNet-34 \cite{he2016deep} as the classifier network, with softmax cross-entropy loss by SGD with a momentum 0.9, a weight decay $5 \times 10^{-4}$ and an initial learning rate 0.1. The learning rate of ResNet-34 is set as CosineAnnealingWarmRestarts \cite{Loshchilov2017SGDR}. The learning rate of CWN-Net is fixed as $10^{-3}$, and the weight decay of CMW-Net is fixed as $10^{-4}$. The batch size is 128 for all experiments. We repeat the experiments with 3 random trials and report the mean value and standard deviation.

\begin{figure*}[t]
	\centering
	\subfigcapskip=-1mm
	\subfigure[Empirical pdf of cross-entropy loss for each class on CIFAR-10 dataset with varying noise rates under symmetric noise. ]{
		\label{fig1a} 
		\includegraphics[width=0.23\textwidth]{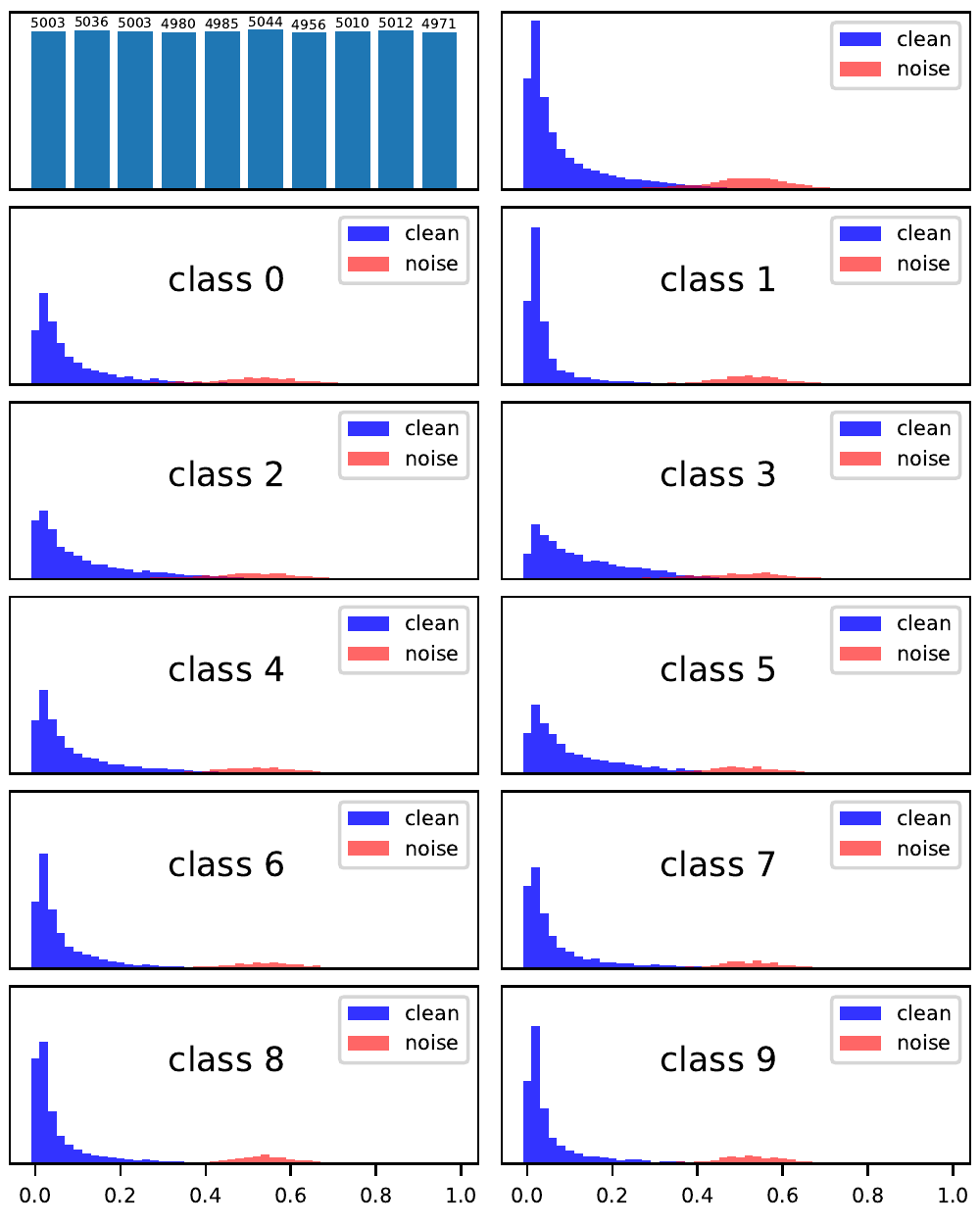} \ \ \
		\includegraphics[width=0.23\textwidth]{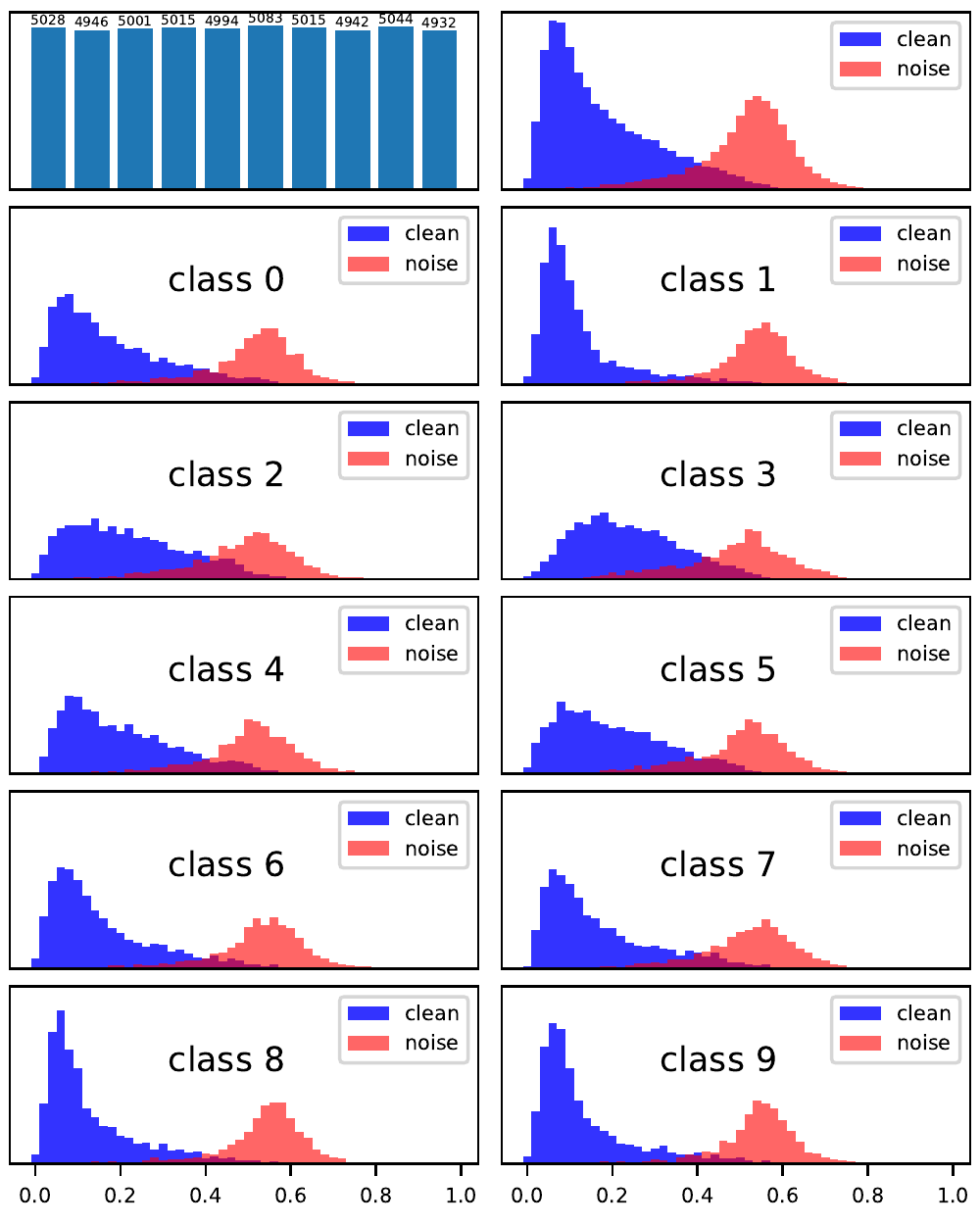}  \ \ \
		\includegraphics[width=0.23\textwidth]{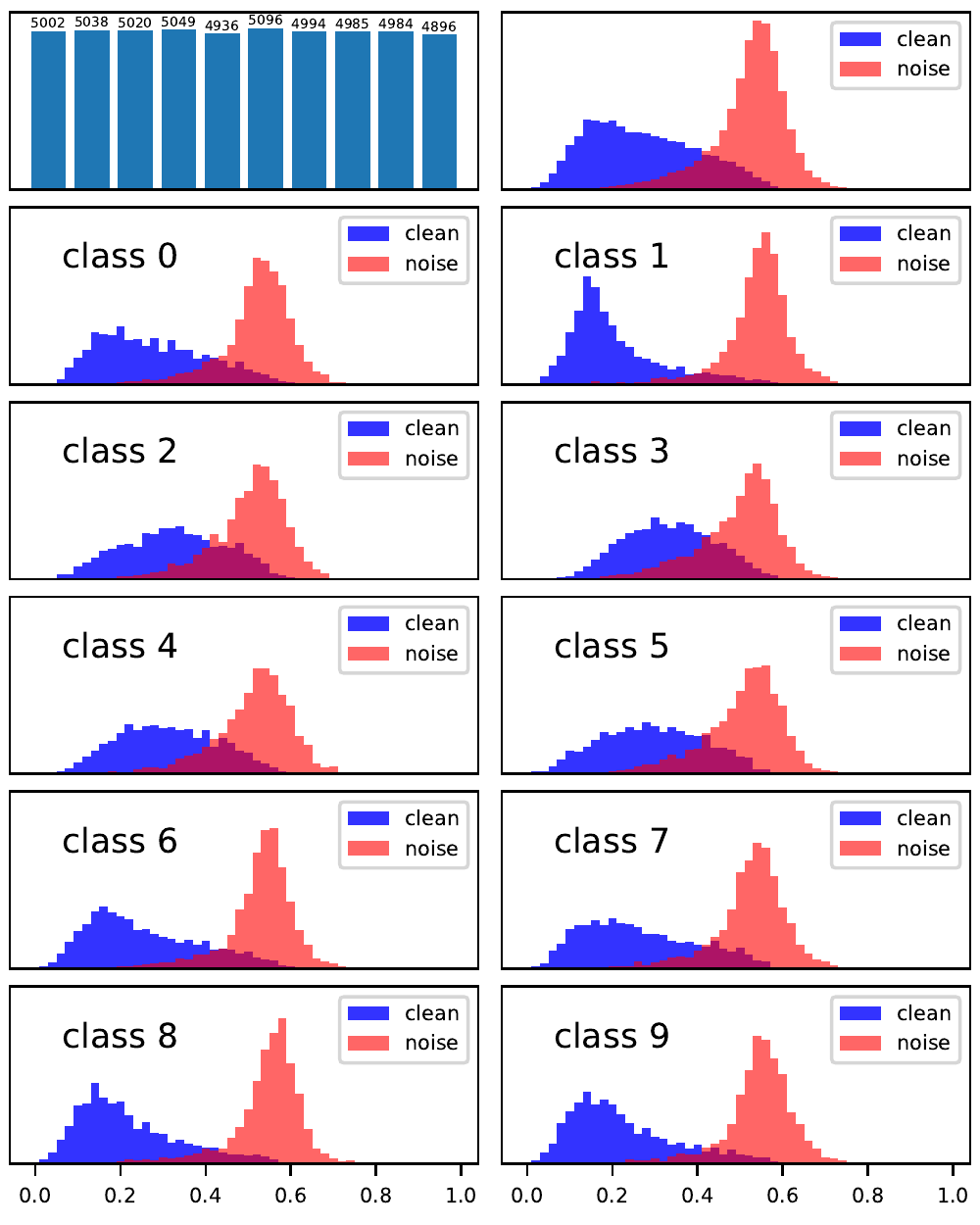} \ \ \
		\includegraphics[width=0.23\textwidth]{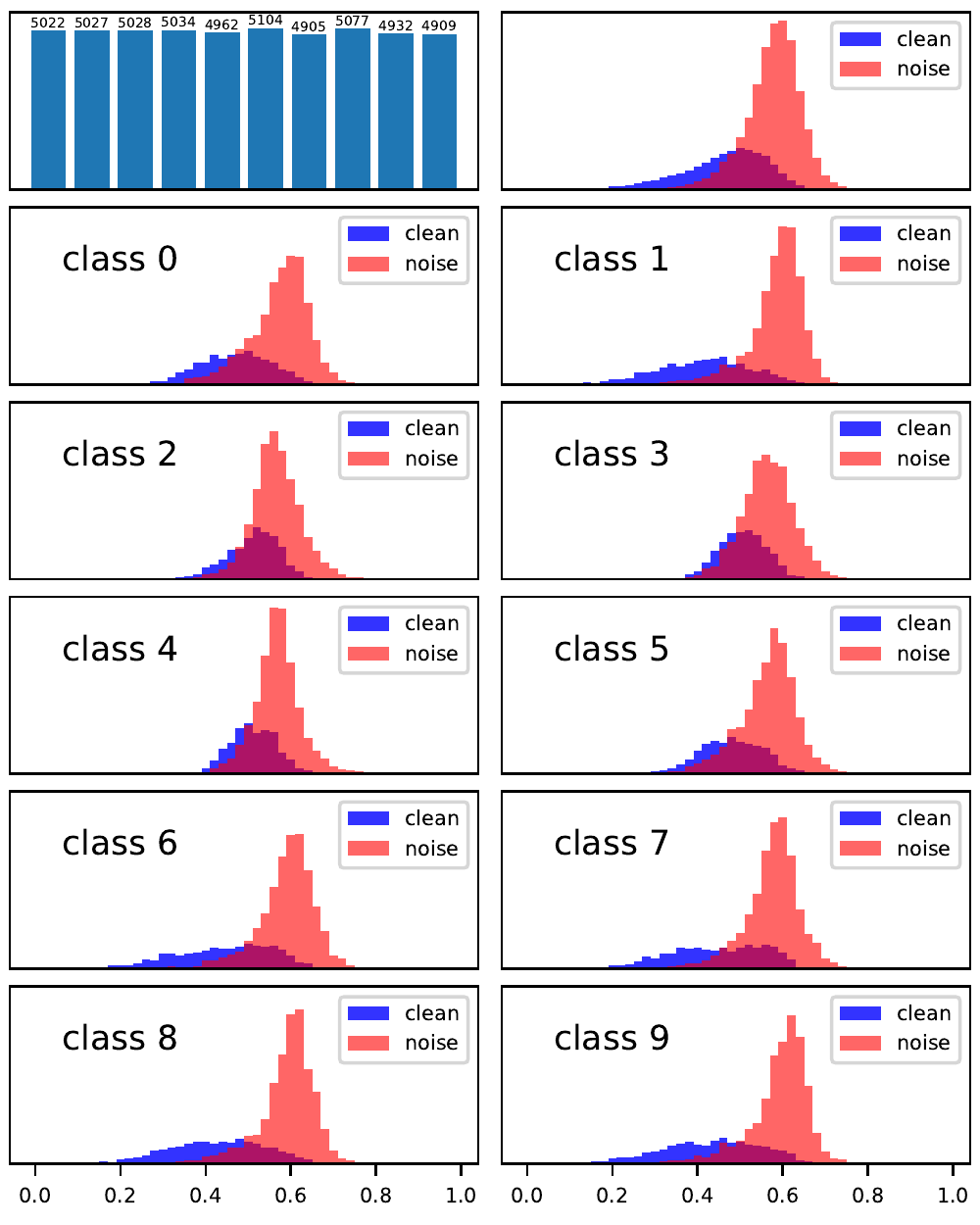}}
	\\ \vspace{-0.2cm}
	\subfigure[Weighting functions and histograms of all sample weights over all training examples learned by MW-Net under symmetric noise.]{
		\label{fig1b} 
		\includegraphics[width=0.23\textwidth]{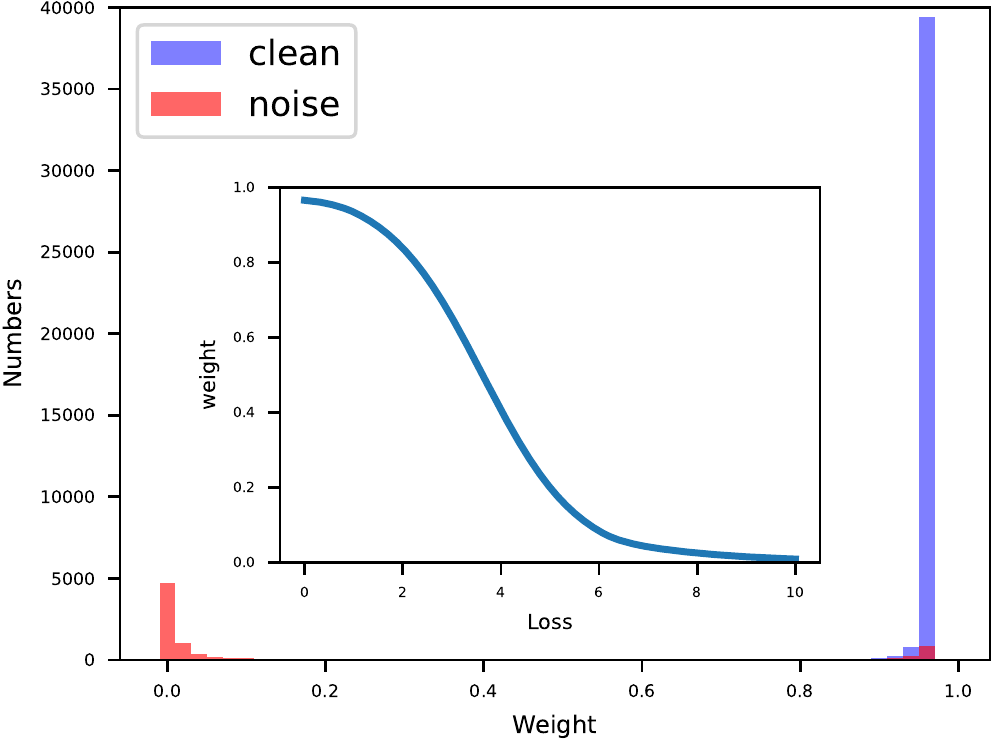}\ \ \
		\includegraphics[width=0.23\textwidth]{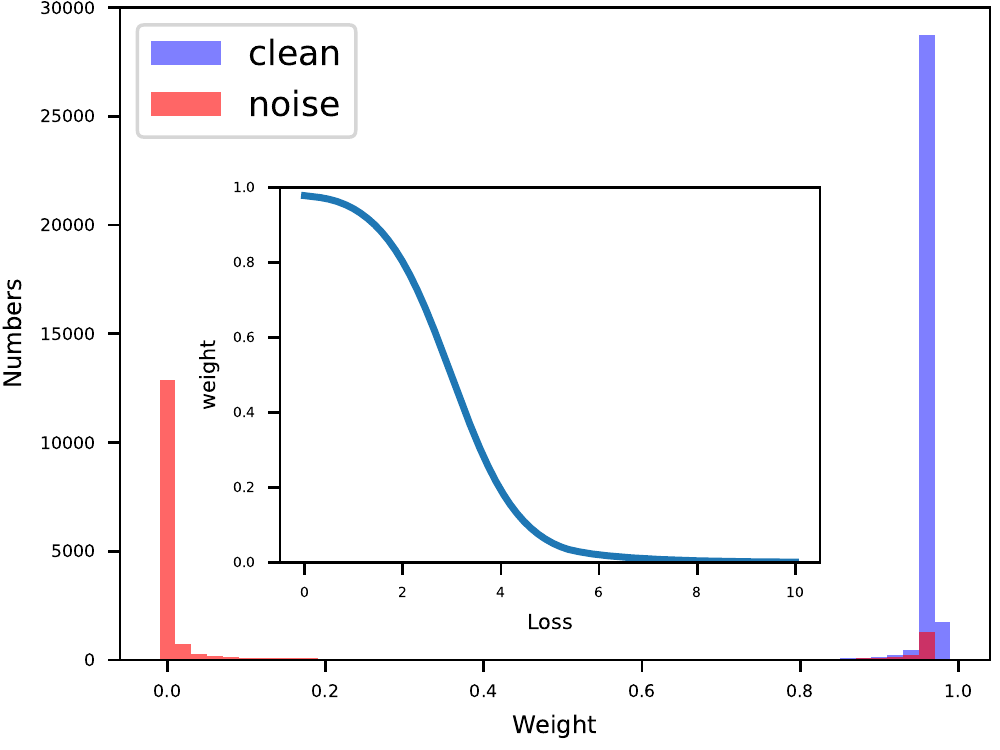}\ \ \
		\includegraphics[width=0.23\textwidth]{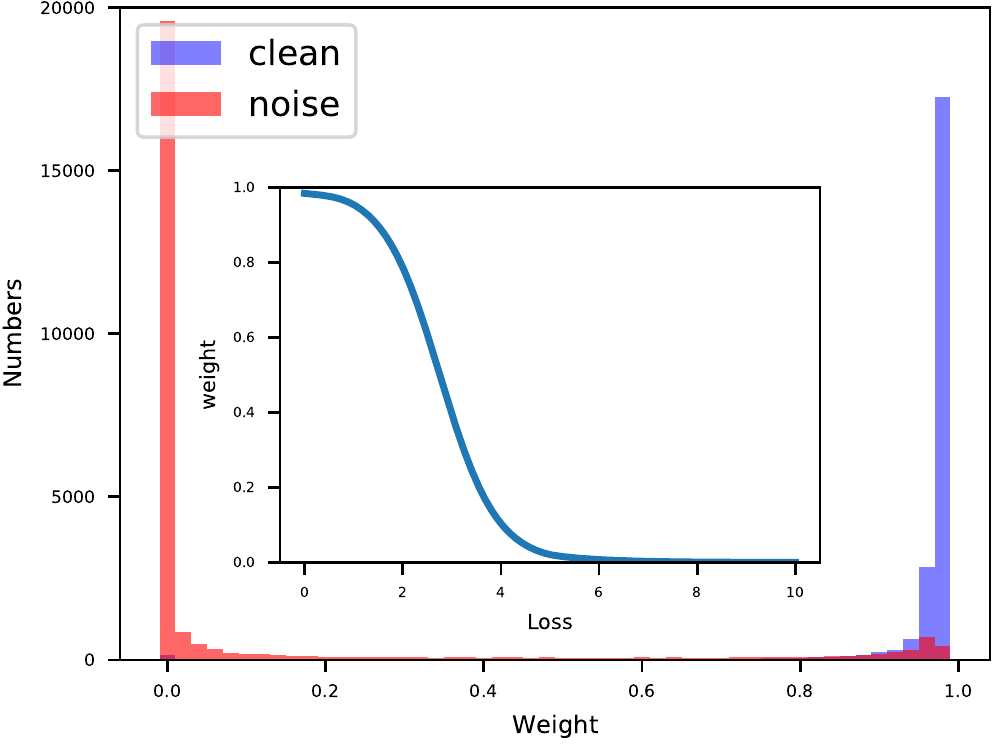}\ \ \
		\includegraphics[width=0.23\textwidth]{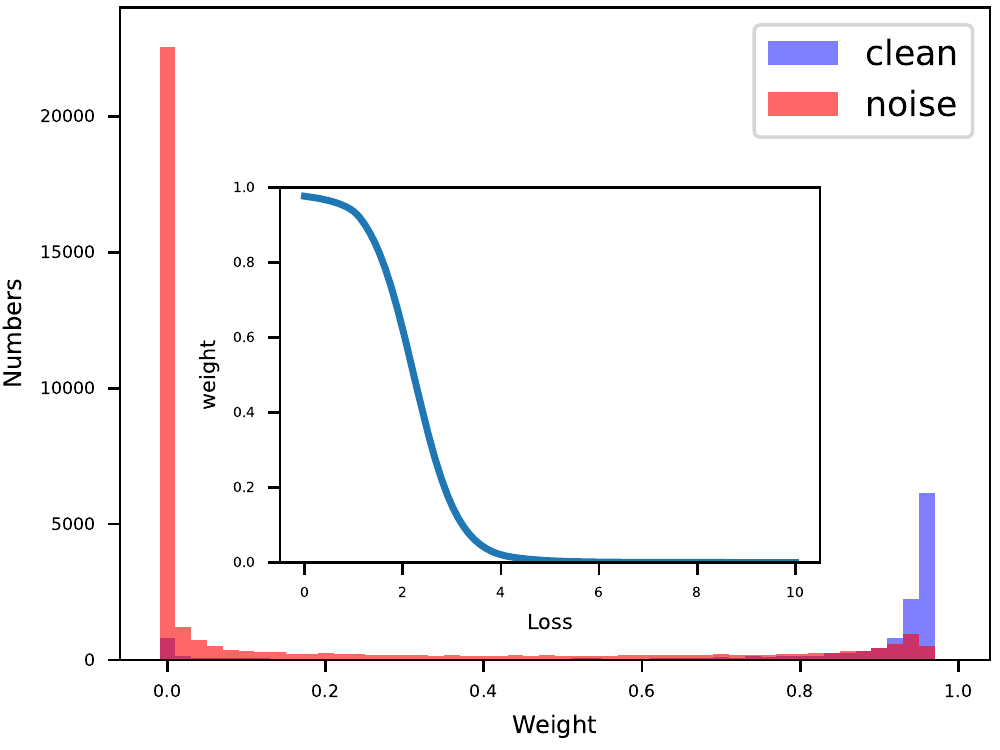}}
	\\ \vspace{-0.2cm}
	\subfigure[Weighting functions and histograms of all sample weights over all training examples learned by CMW-Net under symmetric noise.]{
		\label{fig1d} 
		\includegraphics[width=0.23\textwidth]{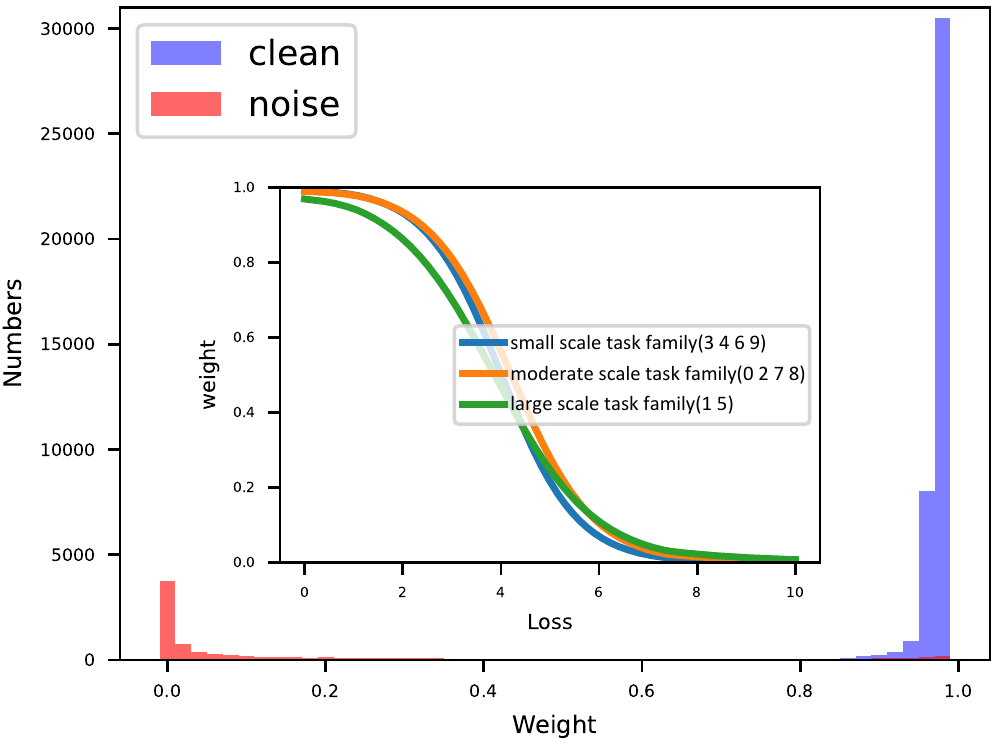}\ \ \
		\includegraphics[width=0.23\textwidth]{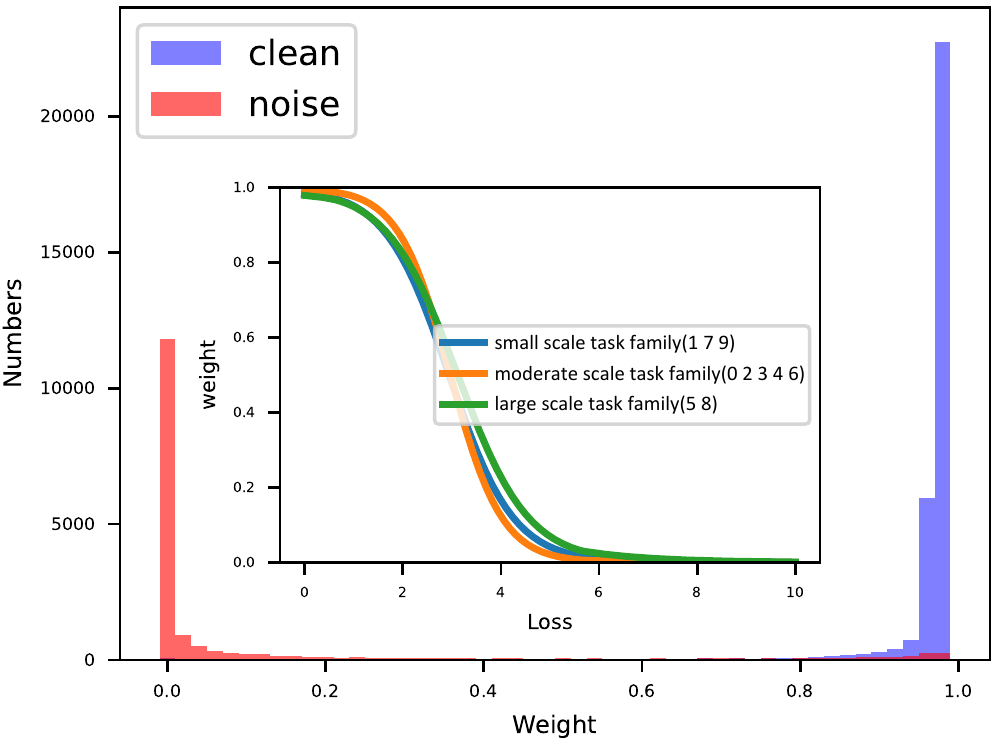}\ \ \
		\includegraphics[width=0.23\textwidth]{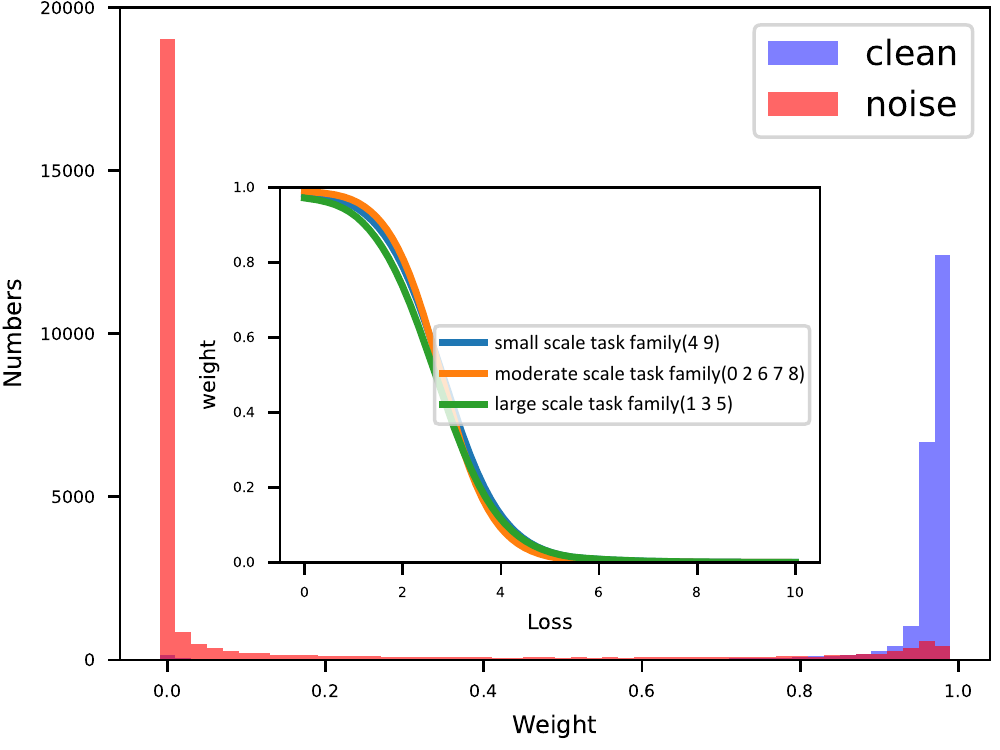}\ \ \
		\includegraphics[width=0.23\textwidth]{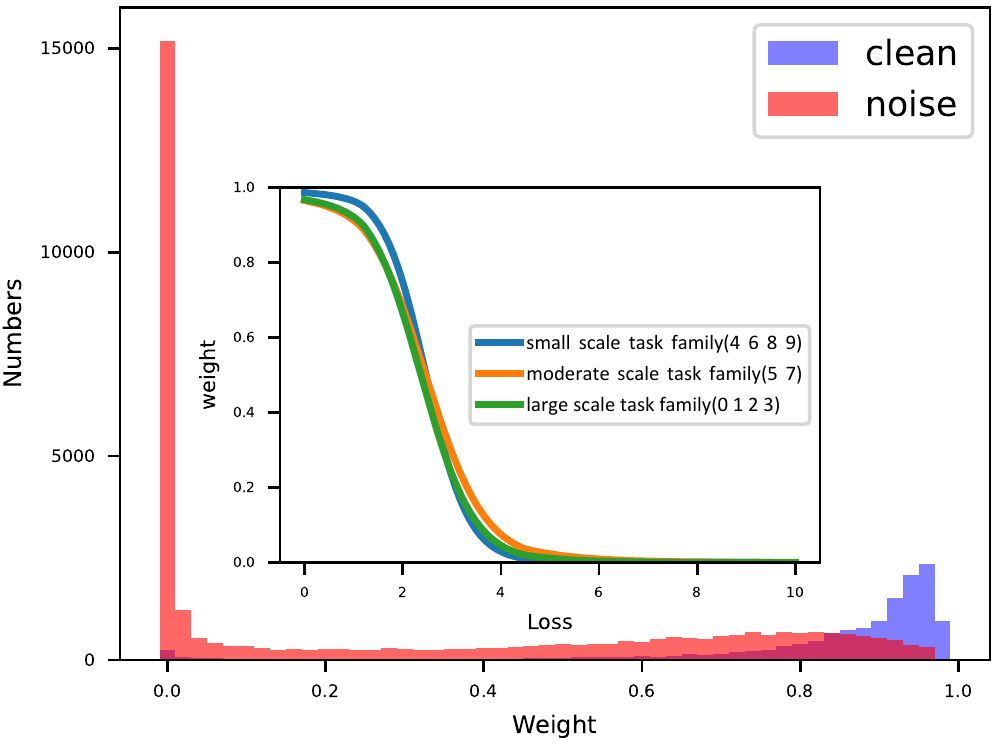}} \vspace{-0.3cm}
	\caption{(a) Empirical pdf of the cross-entropy loss calculated on all samples of each class on CIFAR-10 with varying noise rates (from left to right, the noise rates are 20\%, 40\%, 60\%, 80\%) under symmetric noise; (b)(c) The weighting functions and histograms of all sample weights over all training examples learned by MW-Net and CMW-Net.}\label{SM3}
\end{figure*} \vspace{-0.3cm}
\begin{figure*}[!h]
	\centering
	\subfigcapskip=-2mm
	\subfigure[Symmetry Noise (20\%)]{
		\includegraphics[width=0.23\textwidth]{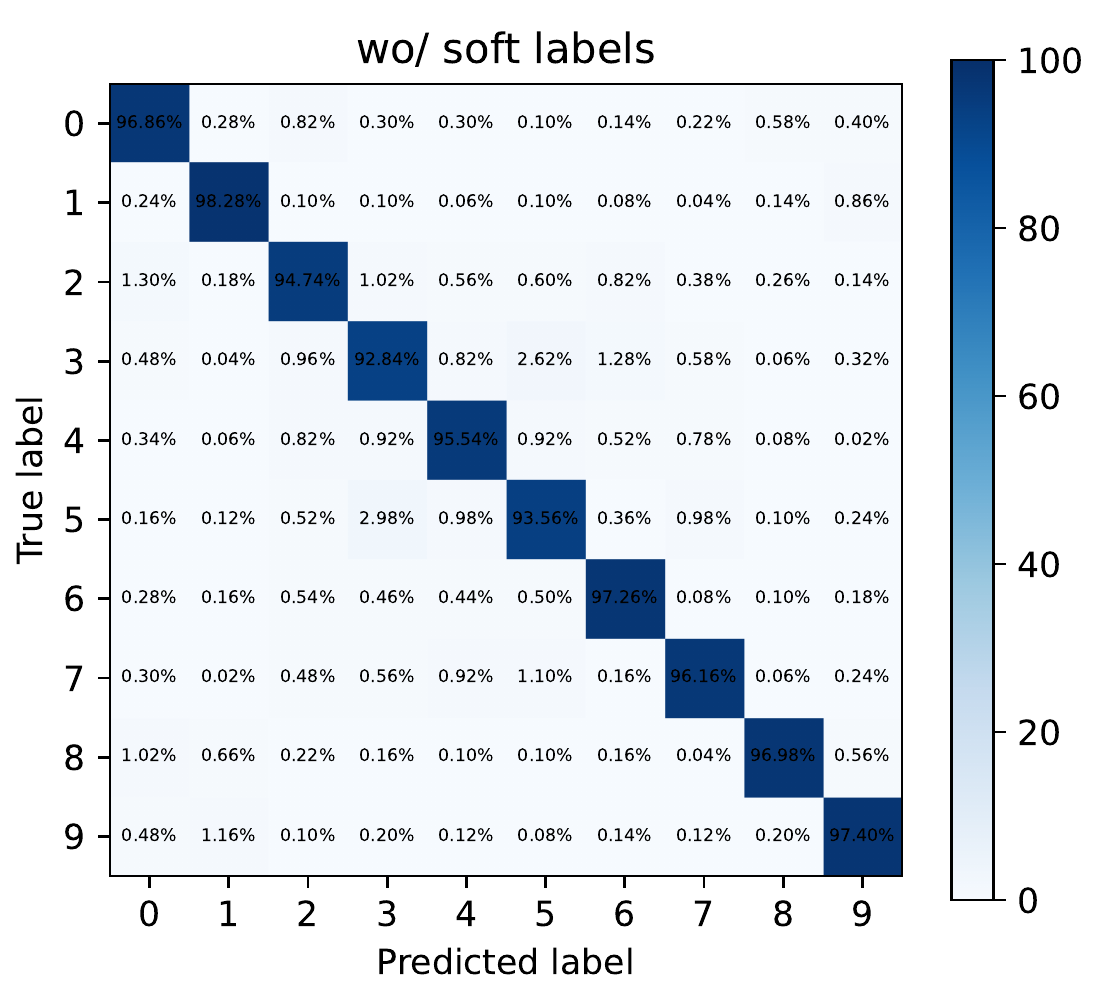}
		\includegraphics[width=0.23\textwidth]{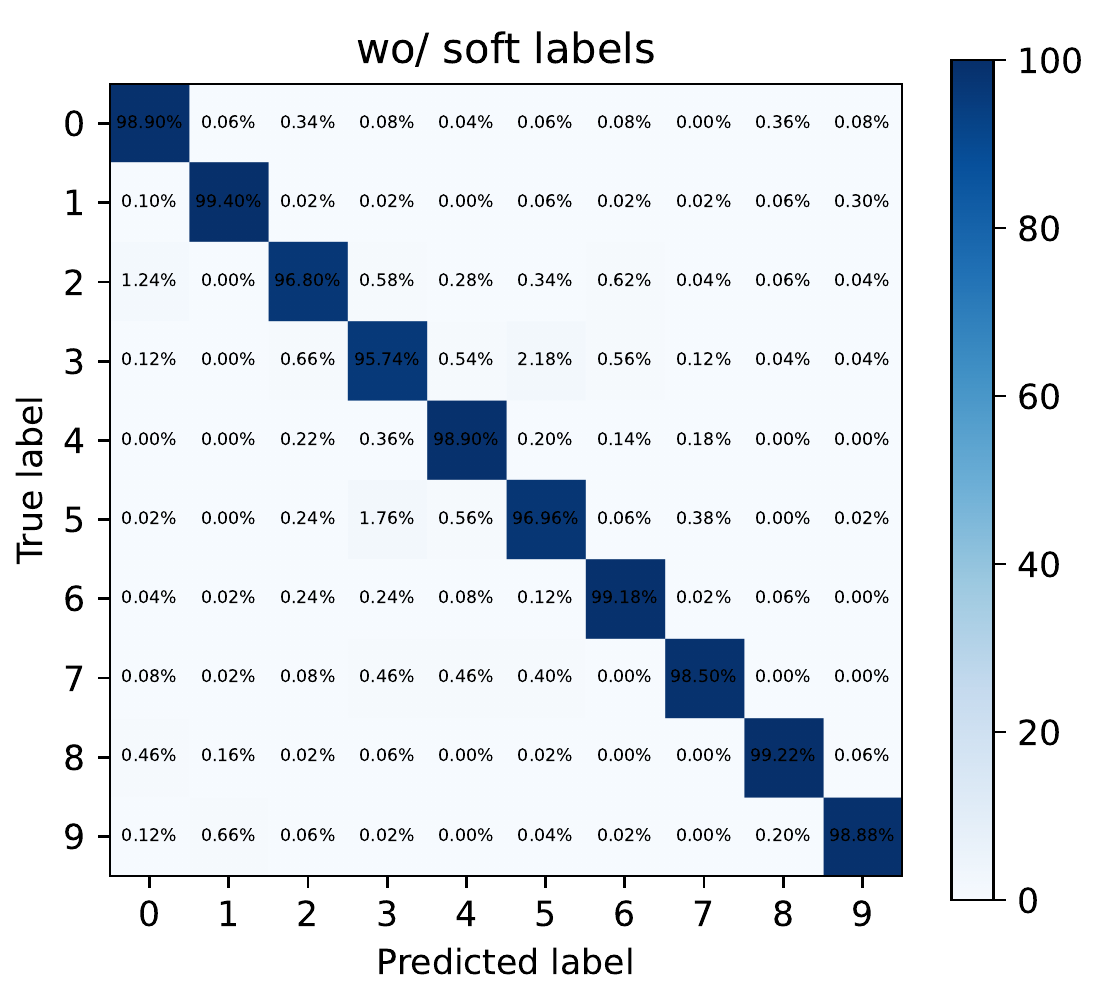}} \ \
	\subfigure[Symmetry Noise (40\%)]{
		\includegraphics[width=0.23\textwidth]{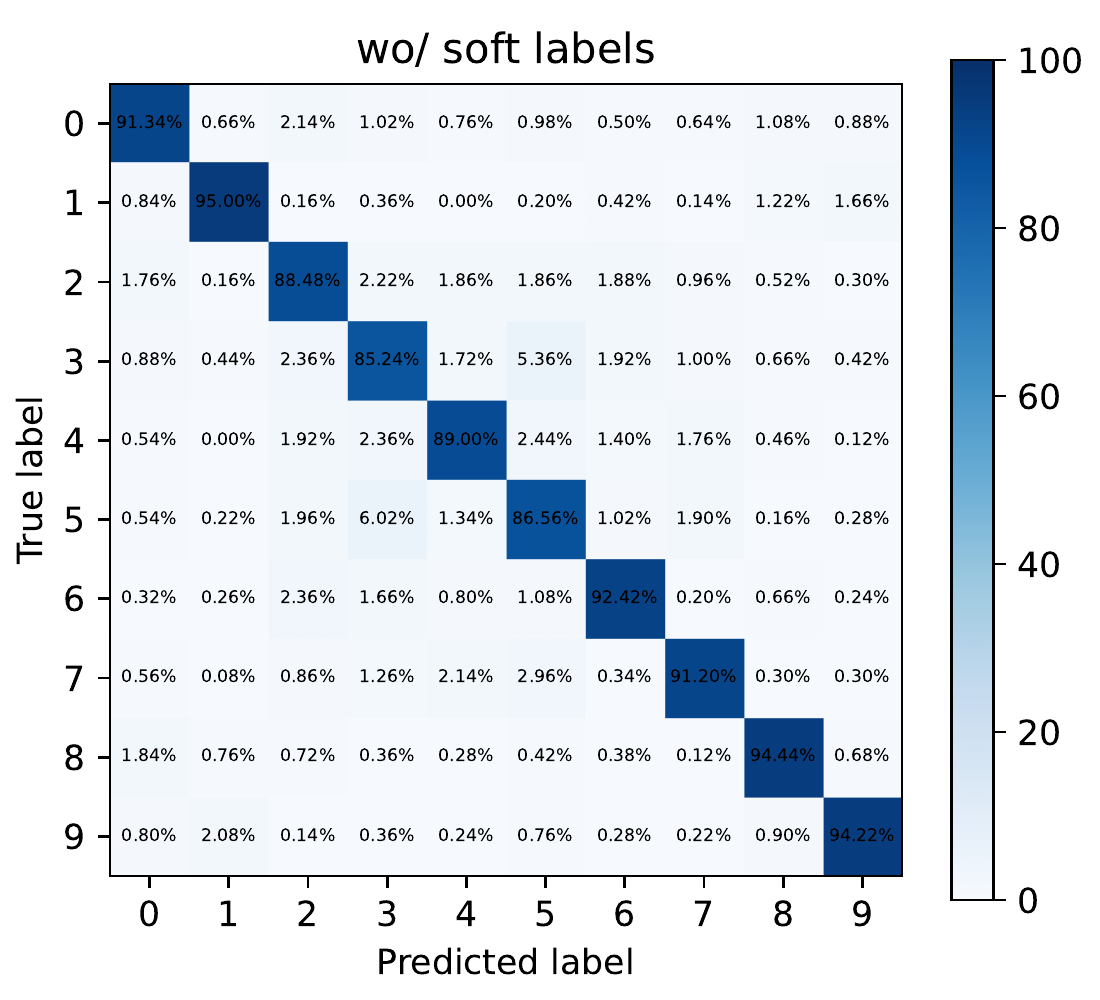}
		\includegraphics[width=0.23\textwidth]{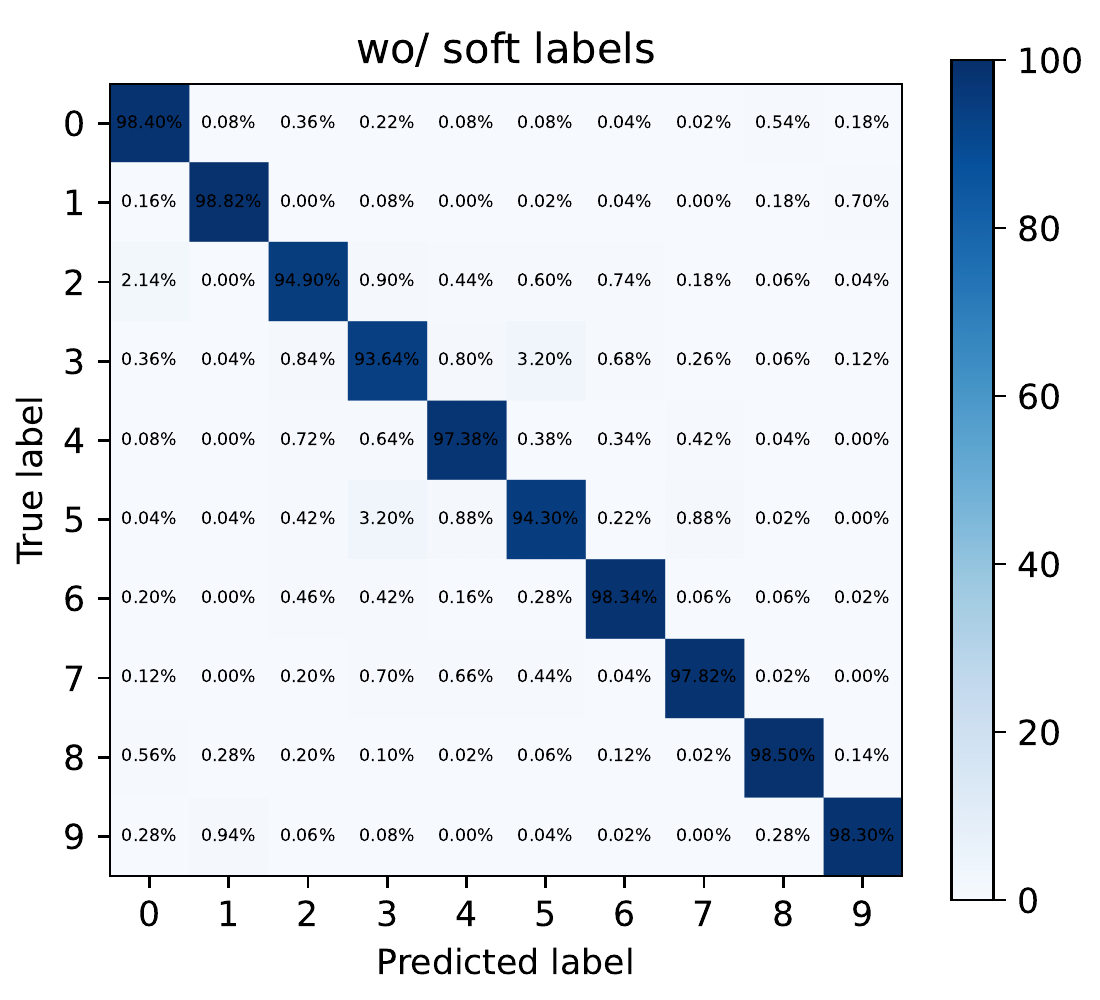}}  \\ \vspace{-4mm}
	\subfigure[Symmetry Noise (60\%)]{
		\includegraphics[width=0.23\textwidth]{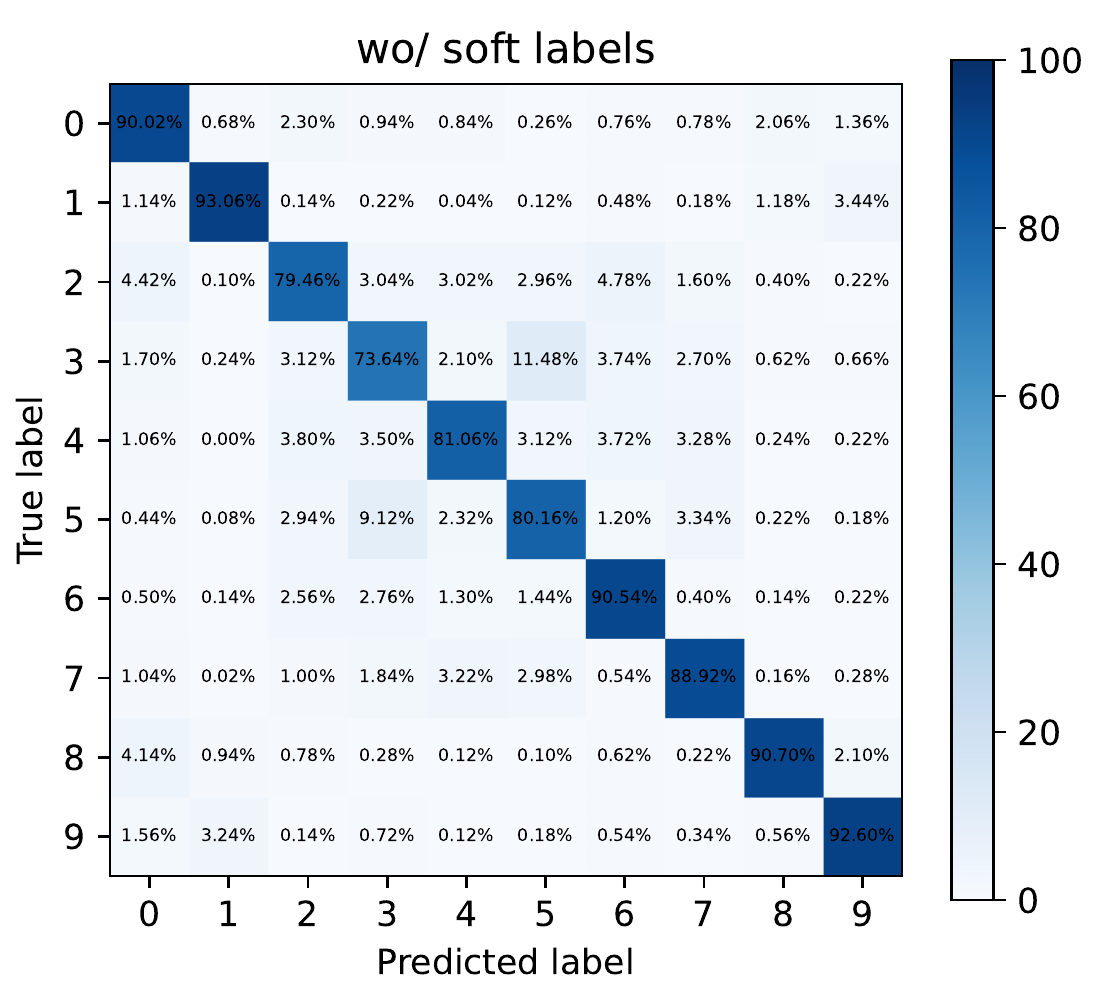}
		\includegraphics[width=0.24\textwidth]{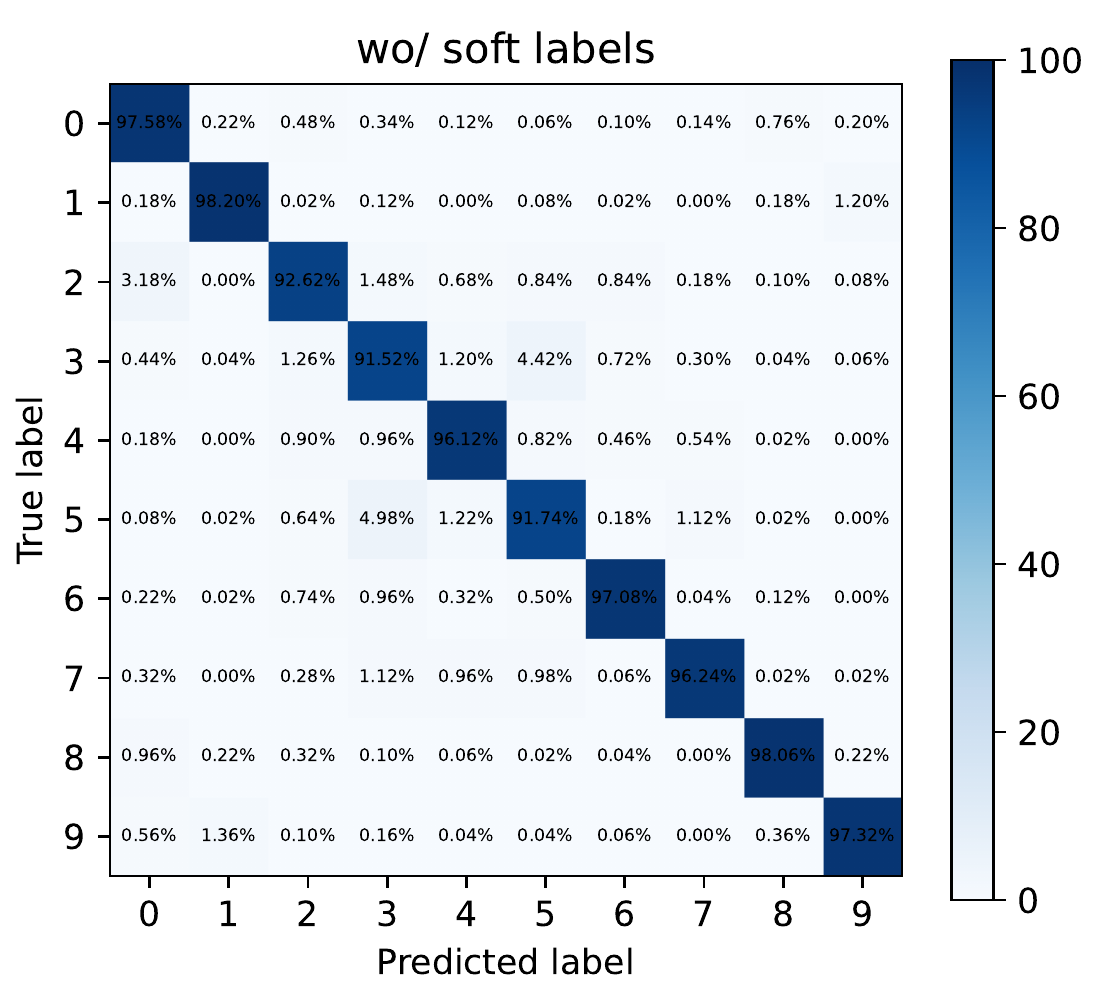}} \ \
	\subfigure[Symmetry Noise (80\%)]{
		\includegraphics[width=0.23\textwidth]{./fig/add/confusion_cmwn4_sym_80.pdf}
		\includegraphics[width=0.23\textwidth]{./fig/add/confusion_cmwn5_sym_80.pdf}} \vspace{-4mm}
	\caption{Confusion matrices obtained by CMW-Net without (left) or with (right) soft label amelioration on CIFAR-10 with symmetry noise with varying noise rates ranging from 20\% to 80\%.}\label{SM5}
\end{figure*}
\begin{figure*}[t]
	\centering
	\subfigcapskip=-1mm
	\subfigure[Empirical pdf of cross-entropy loss for each class on CIFAR-10 dataset with varying noise rates under asymmetric noise. ]{
		\label{fig1a} 
		\includegraphics[width=0.23\textwidth]{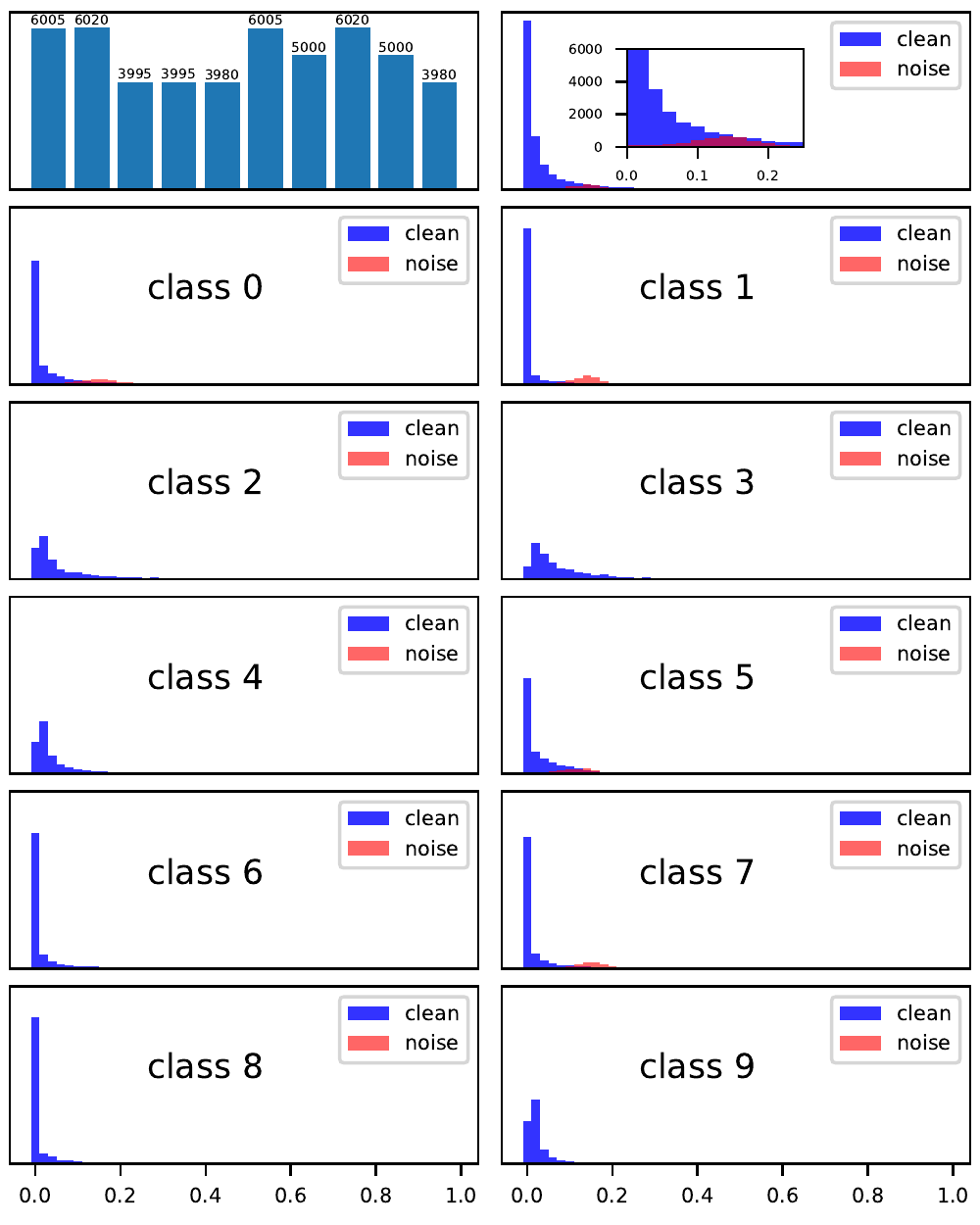} \ \ \
		\includegraphics[width=0.23\textwidth]{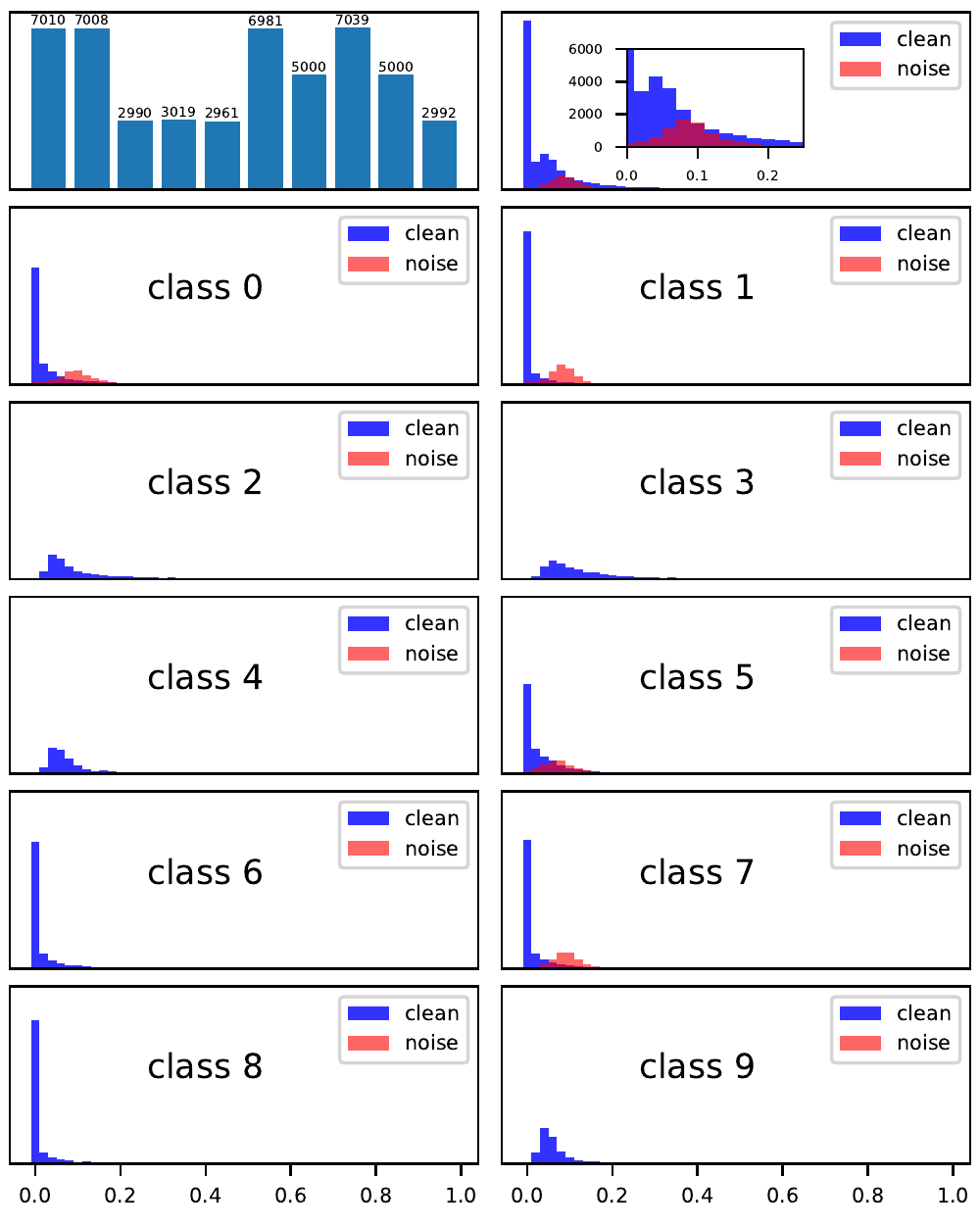}  \ \ \
		\includegraphics[width=0.23\textwidth]{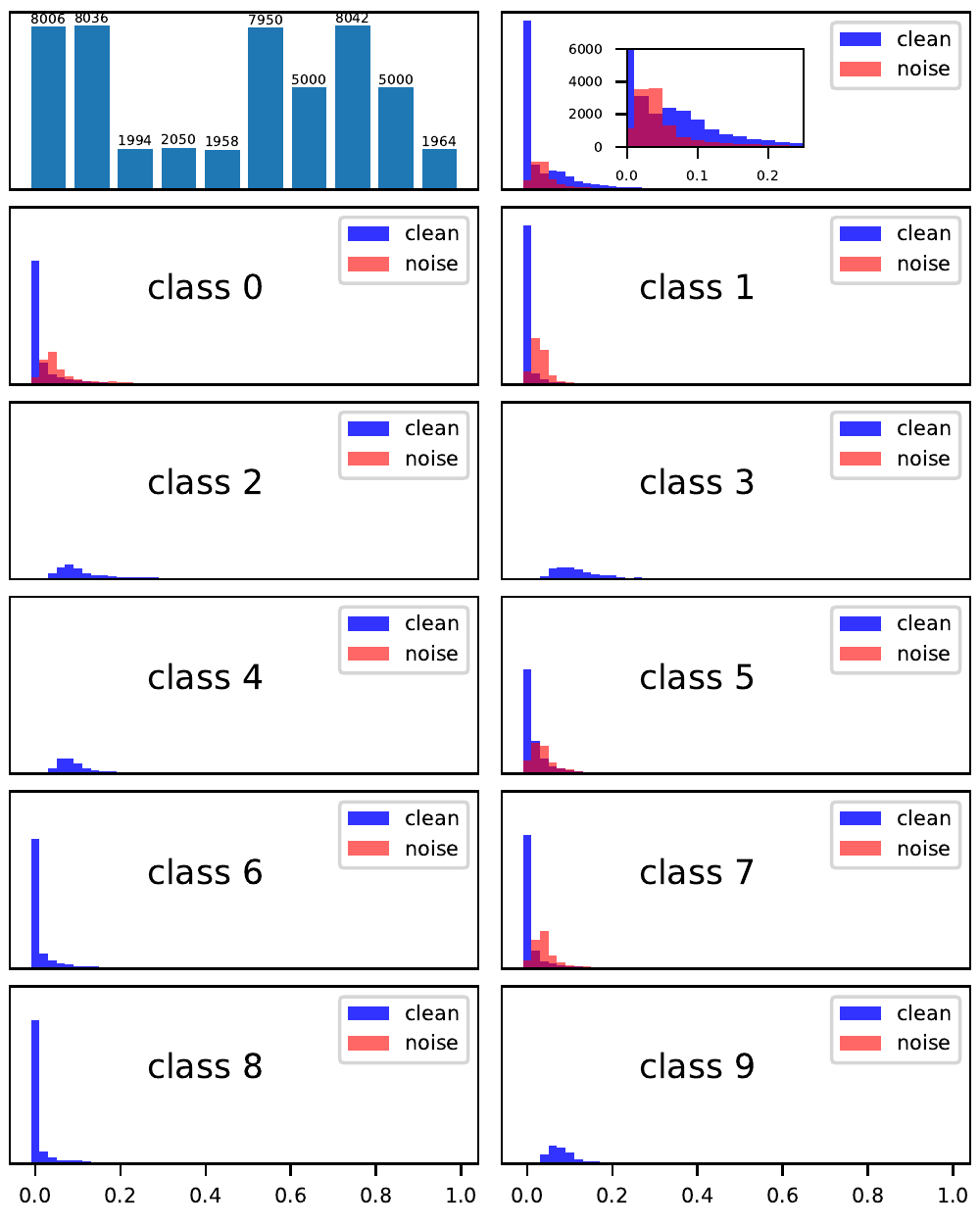} \ \ \
		\includegraphics[width=0.23\textwidth]{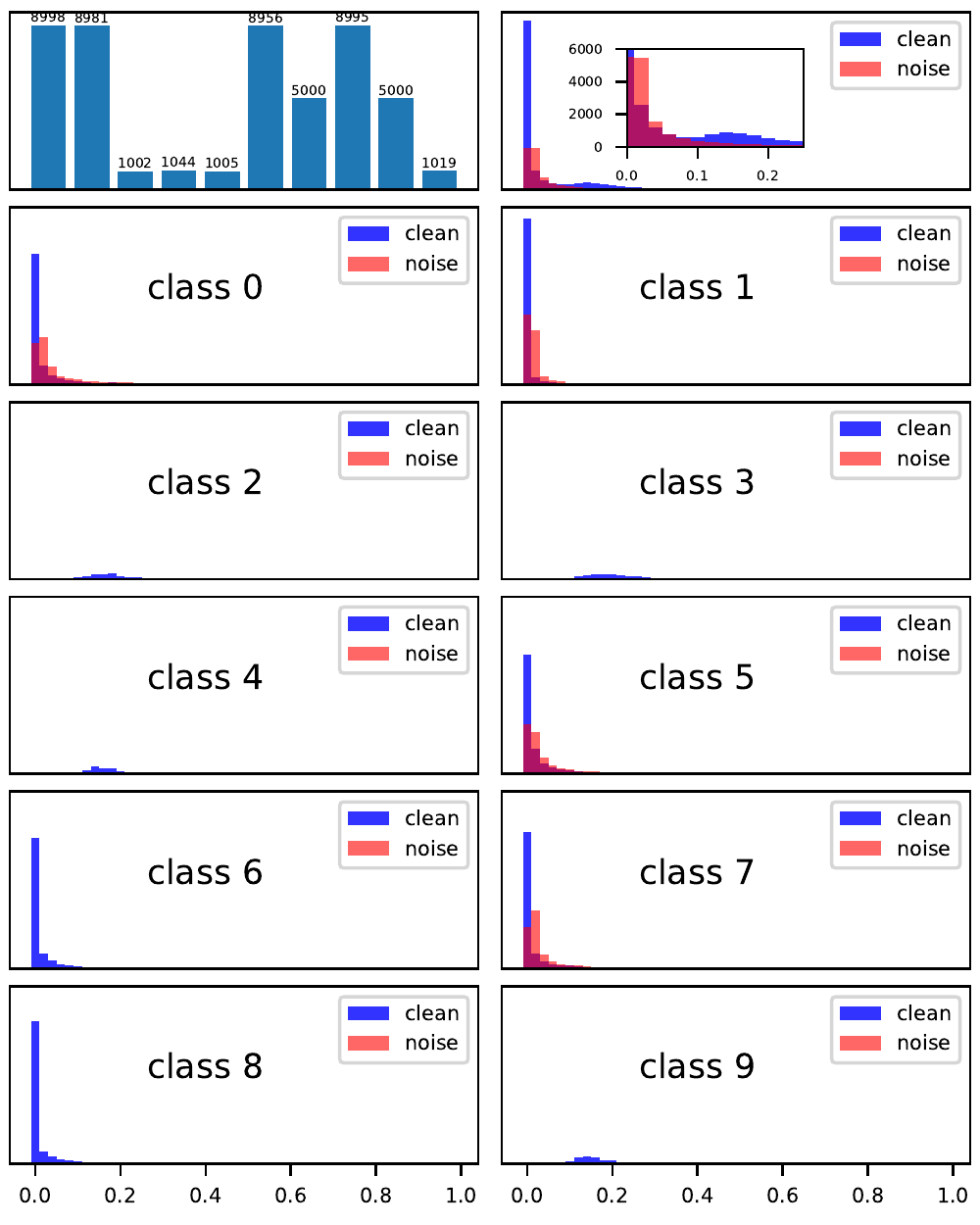}}
	\\ \vspace{-0.2cm}
	\subfigure[Weighting functions and histograms of all sample weights over all training examples learned by CMW-Net under asymmetric noise.]{
		\label{fig1b} 
		\includegraphics[width=0.23\textwidth]{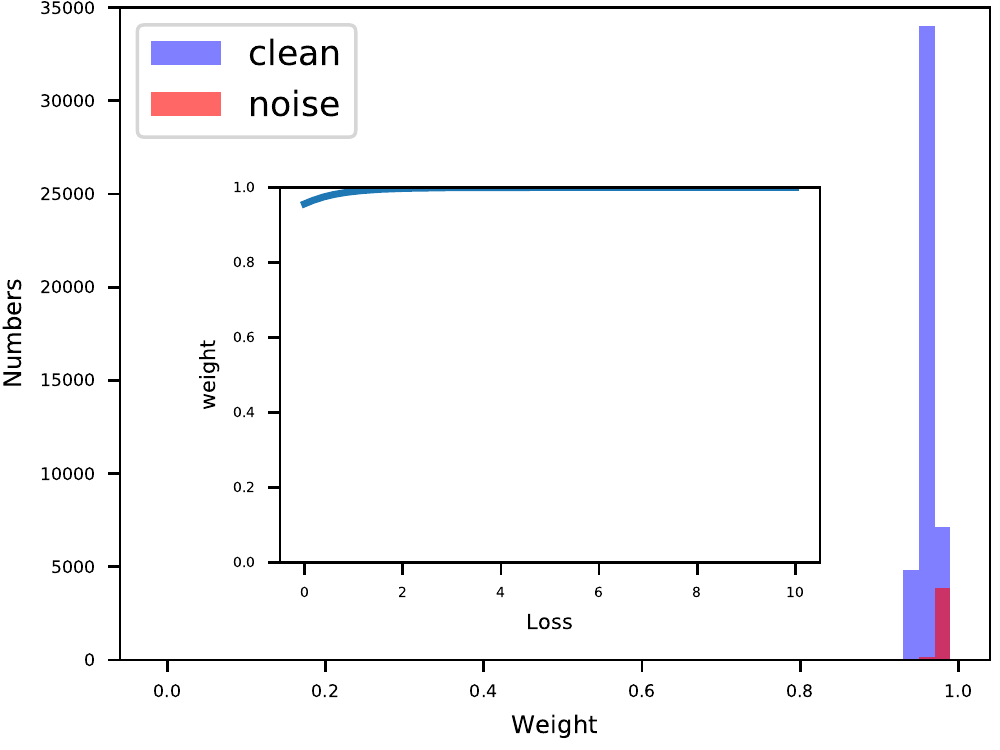}\ \ \
		\includegraphics[width=0.23\textwidth]{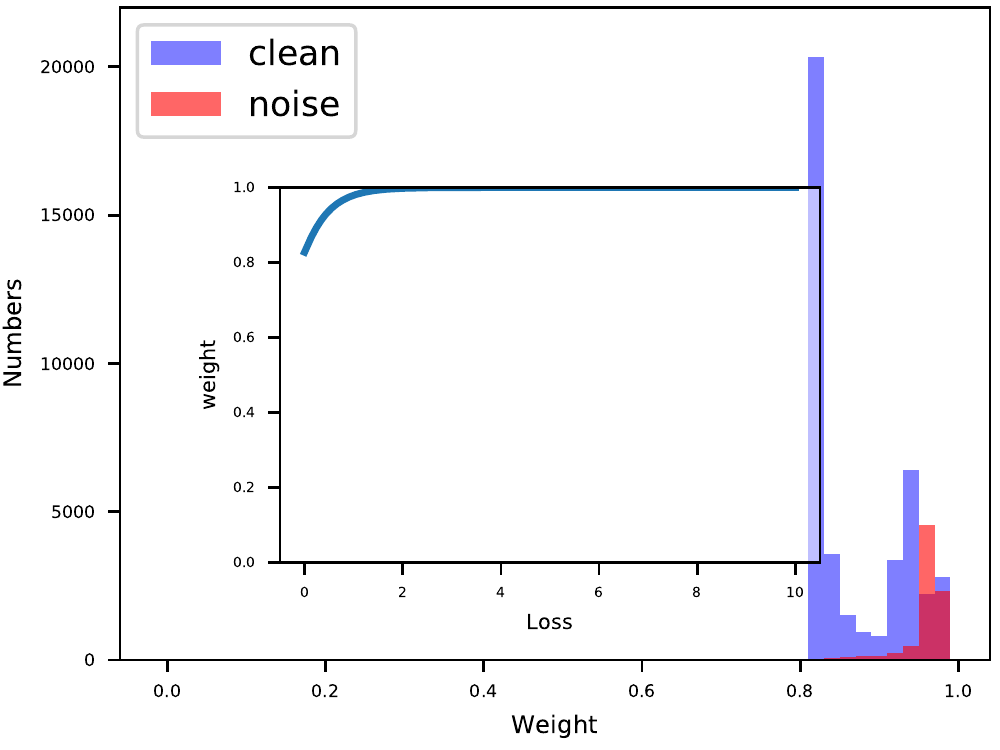}\ \ \
		\includegraphics[width=0.23\textwidth]{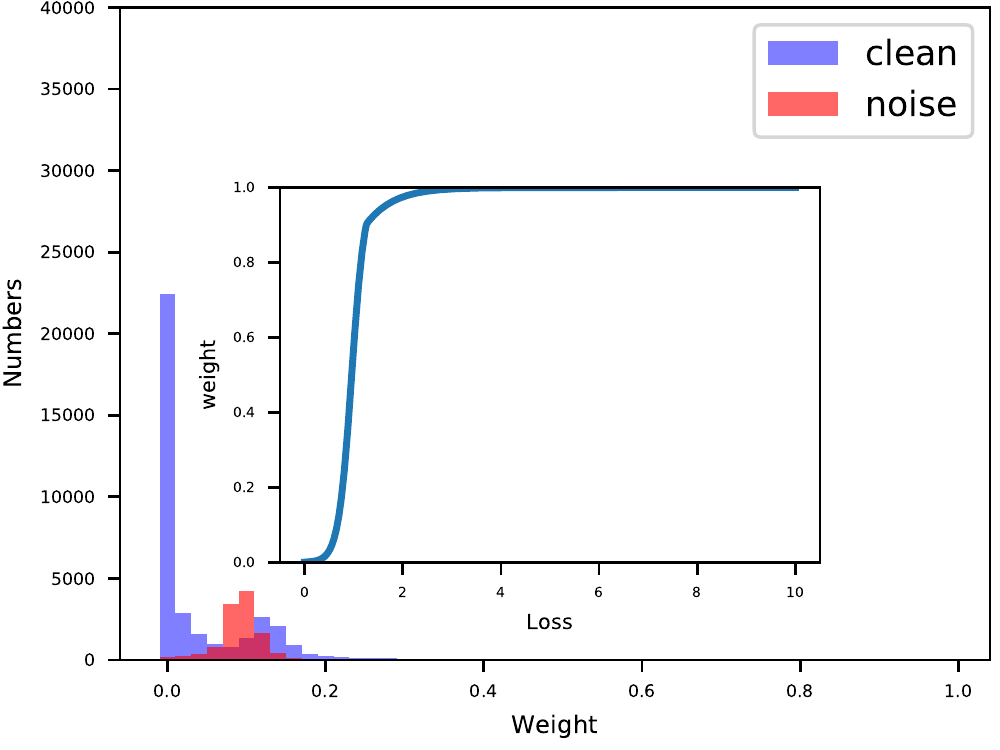}\ \ \
		\includegraphics[width=0.23\textwidth]{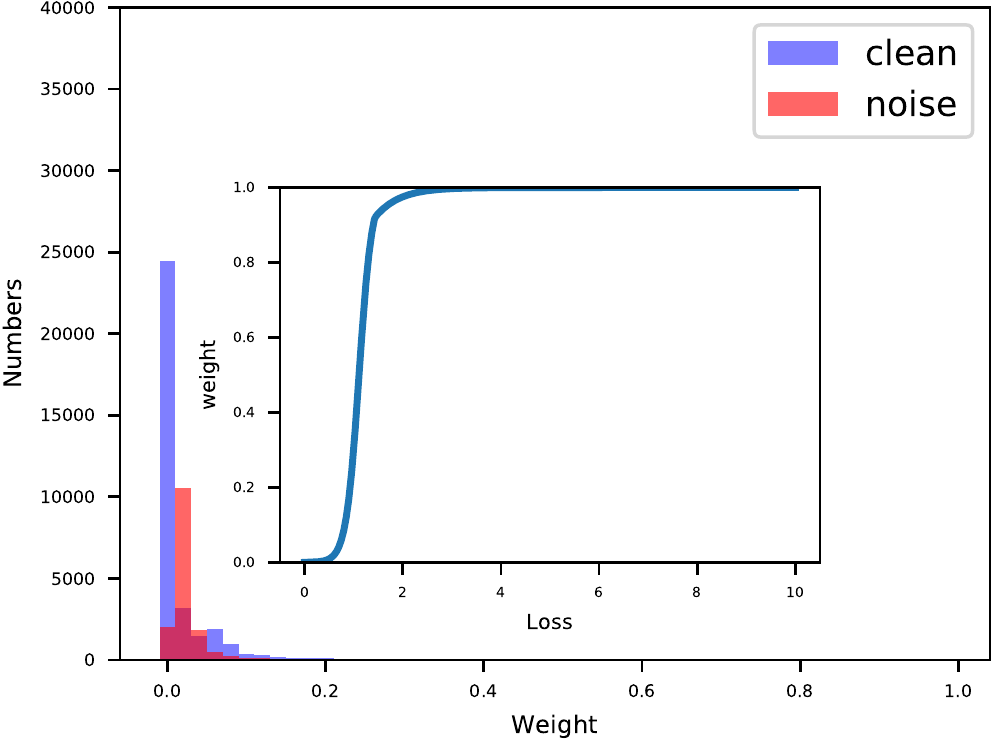}}
	\\ \vspace{-0.2cm}
	\subfigure[Weighting functions and weight distributions over the training examples learned by our CMW-Net under asymmetric noise.]{
		\label{fig1d} 
		\includegraphics[width=0.23\textwidth]{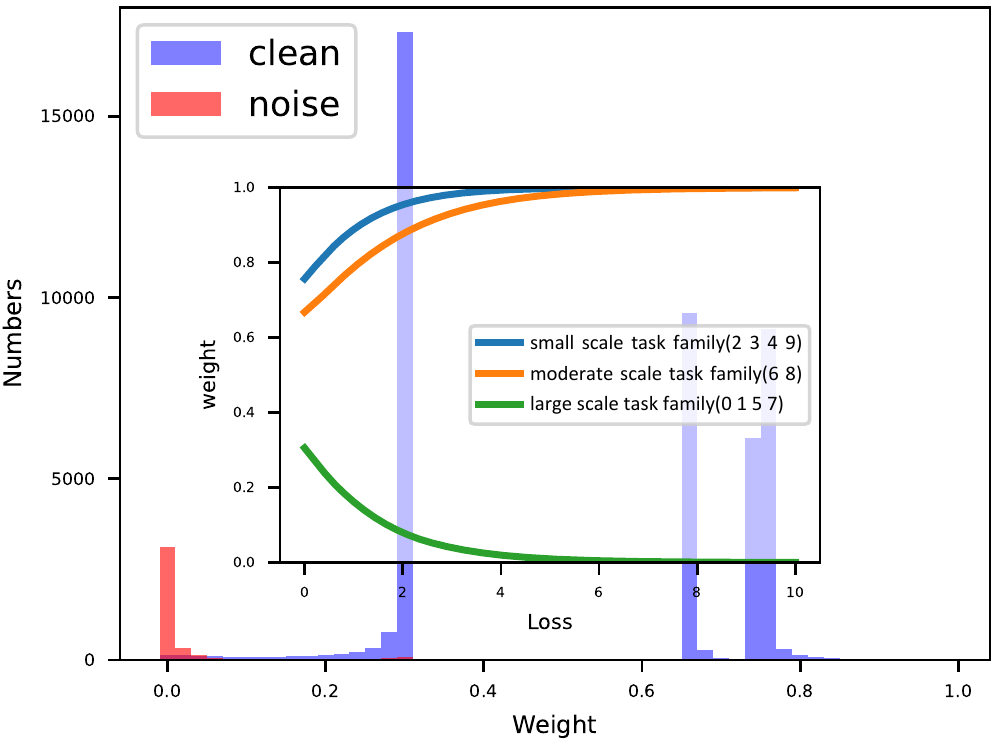}\ \ \
		\includegraphics[width=0.23\textwidth]{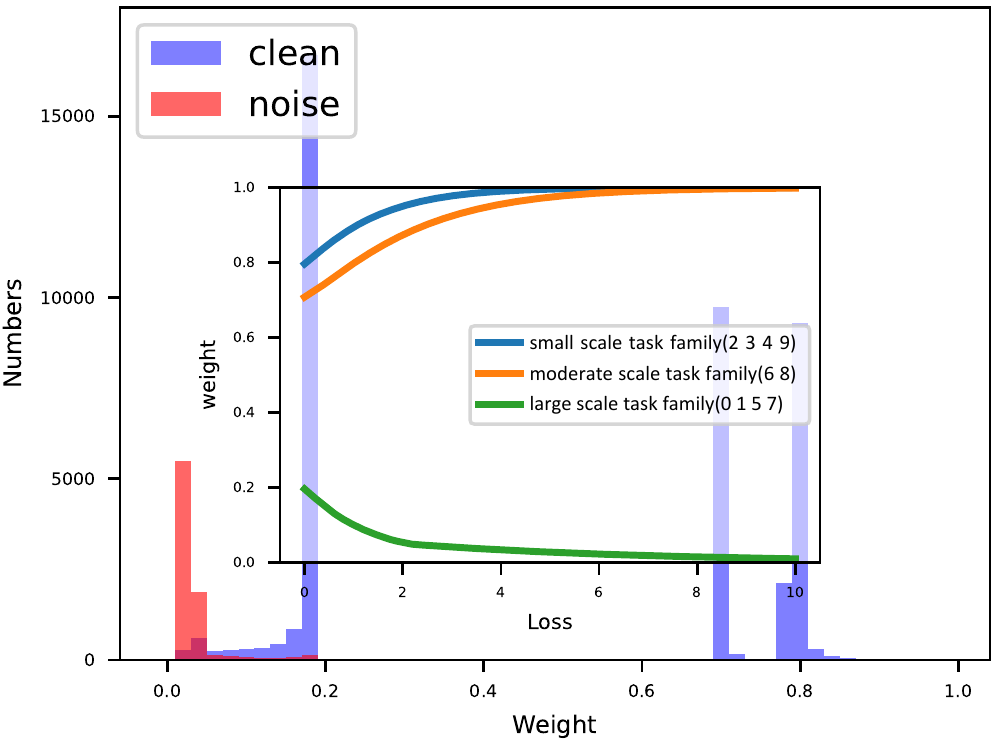}\ \ \
		\includegraphics[width=0.23\textwidth]{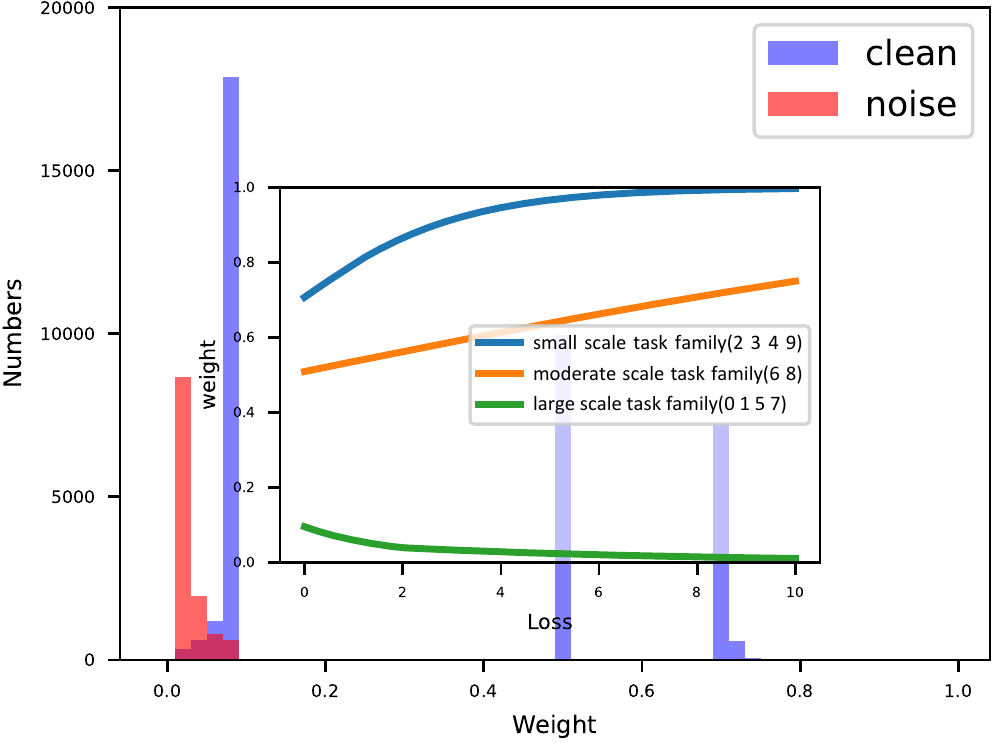}\ \ \
		\includegraphics[width=0.23\textwidth]{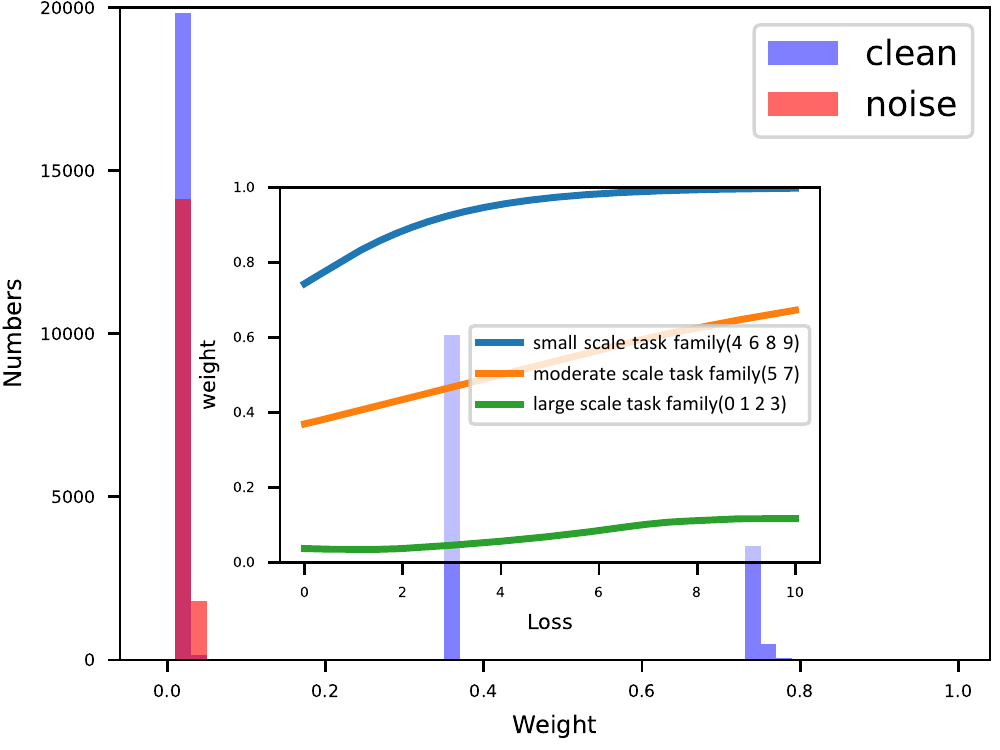}} \vspace{-0.3cm}
	\caption{(a) Empirical pdf of the cross-entropy loss calculated on all samples of each class on CIFAR-10 with varying noise rates (from left to right, the noise rates are 20\%, 40\%, 60\%, 80\%) under asymmetric noise; (b)(c) The weighting functions and histograms of all sample weights over all training examples learned by MW-Net and CMW-Net.}\label{SM4}
\end{figure*}
\begin{figure*}[!h]
	\centering
	\subfigcapskip=-2mm
	\subfigure[Asymmetry Noise (20\%)]{
		\includegraphics[width=0.23\textwidth]{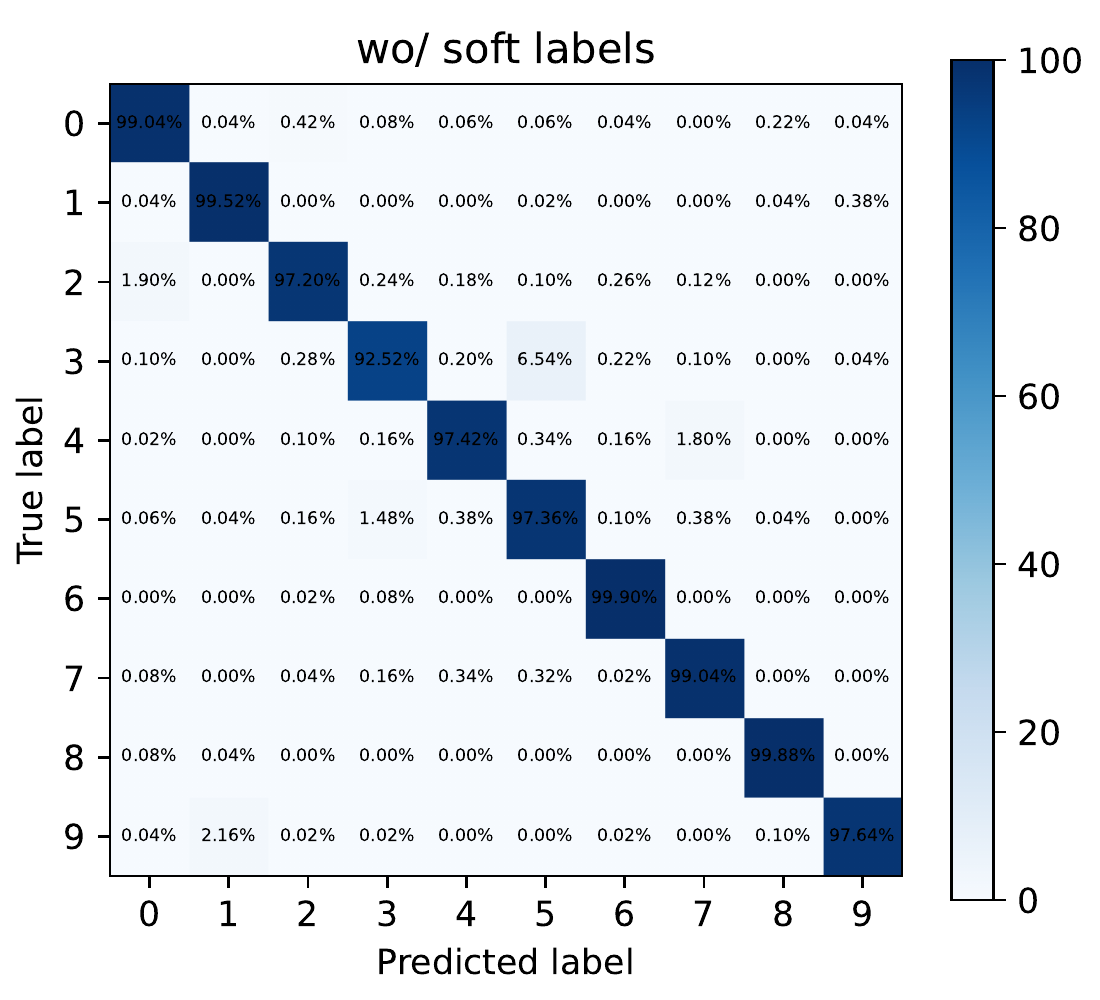}
		\includegraphics[width=0.23\textwidth]{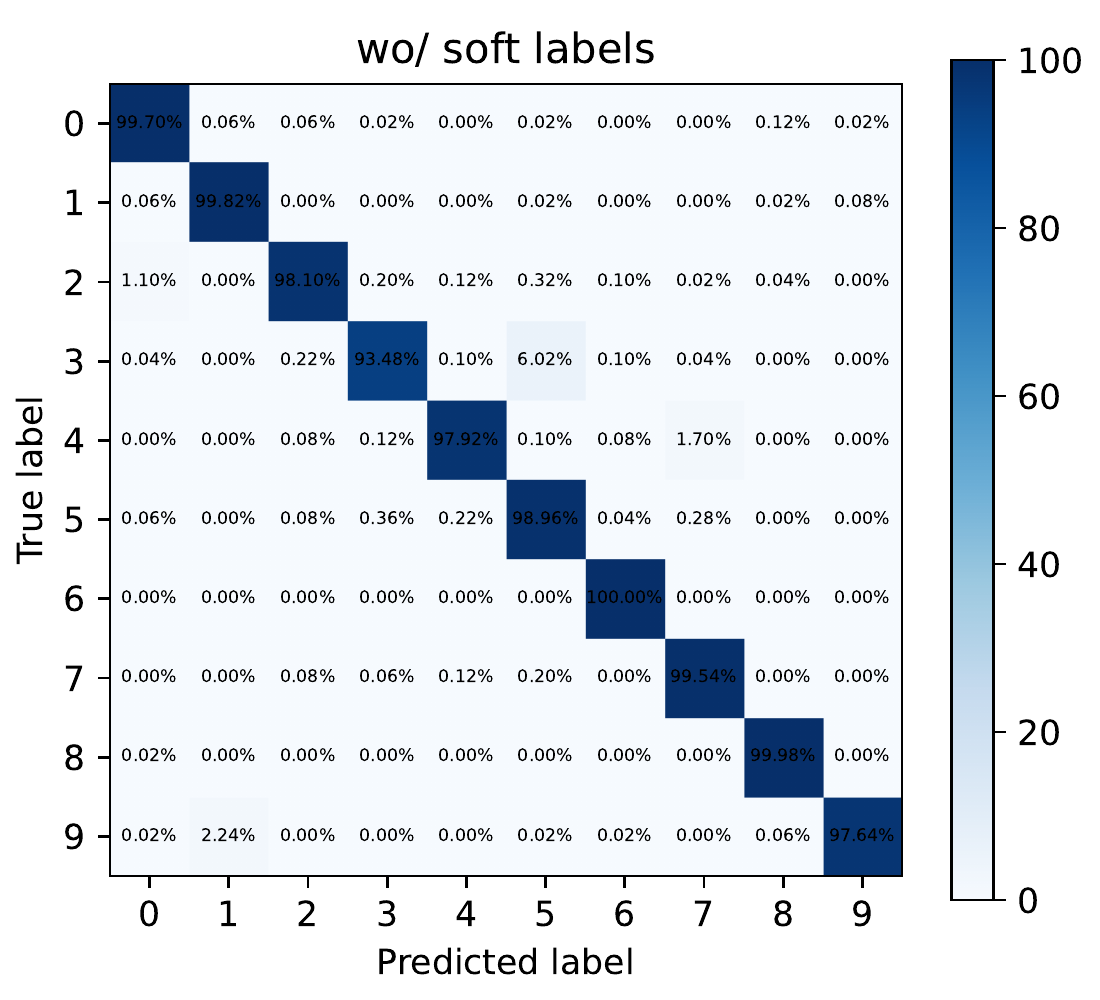}} \ \
	\subfigure[Asymmetry Noise (40\%)]{
		\includegraphics[width=0.23\textwidth]{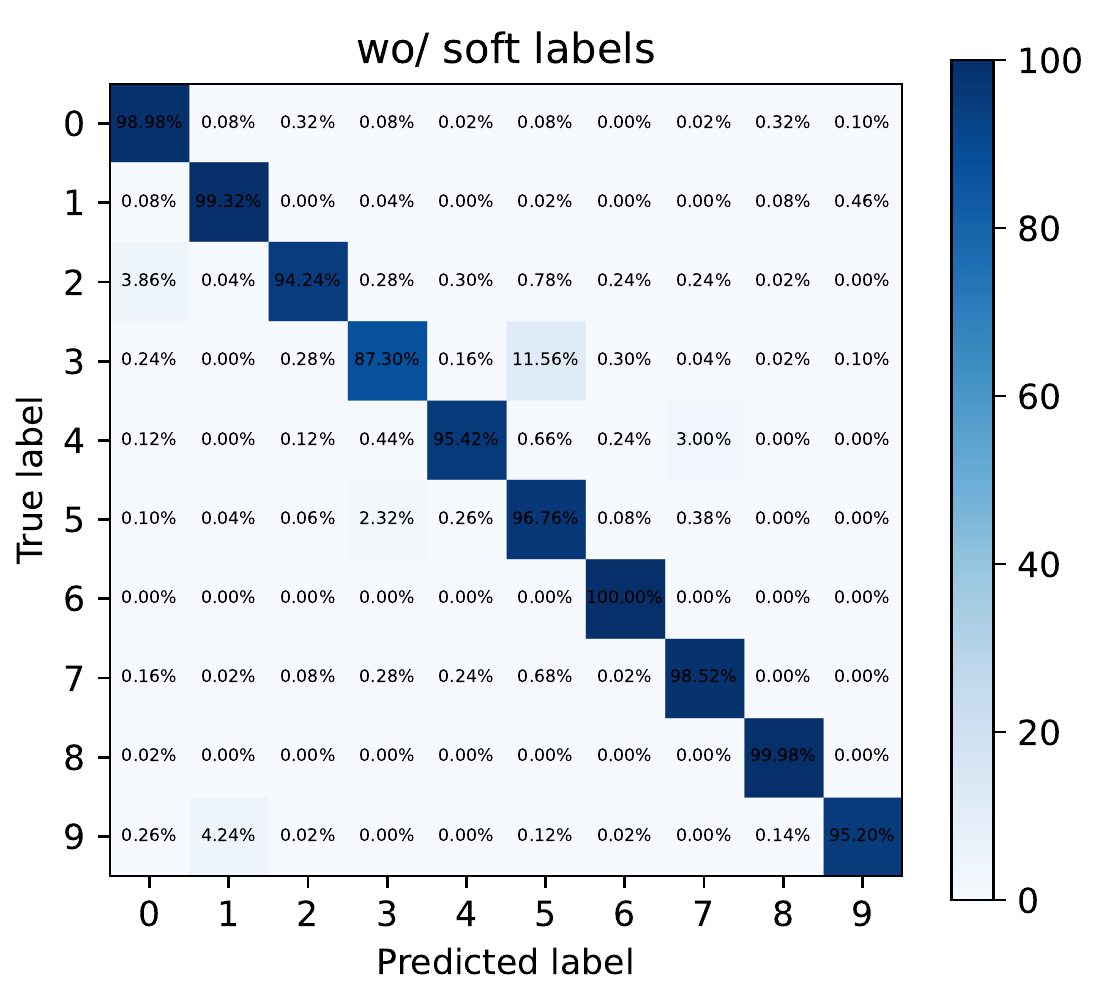}
		\includegraphics[width=0.23\textwidth]{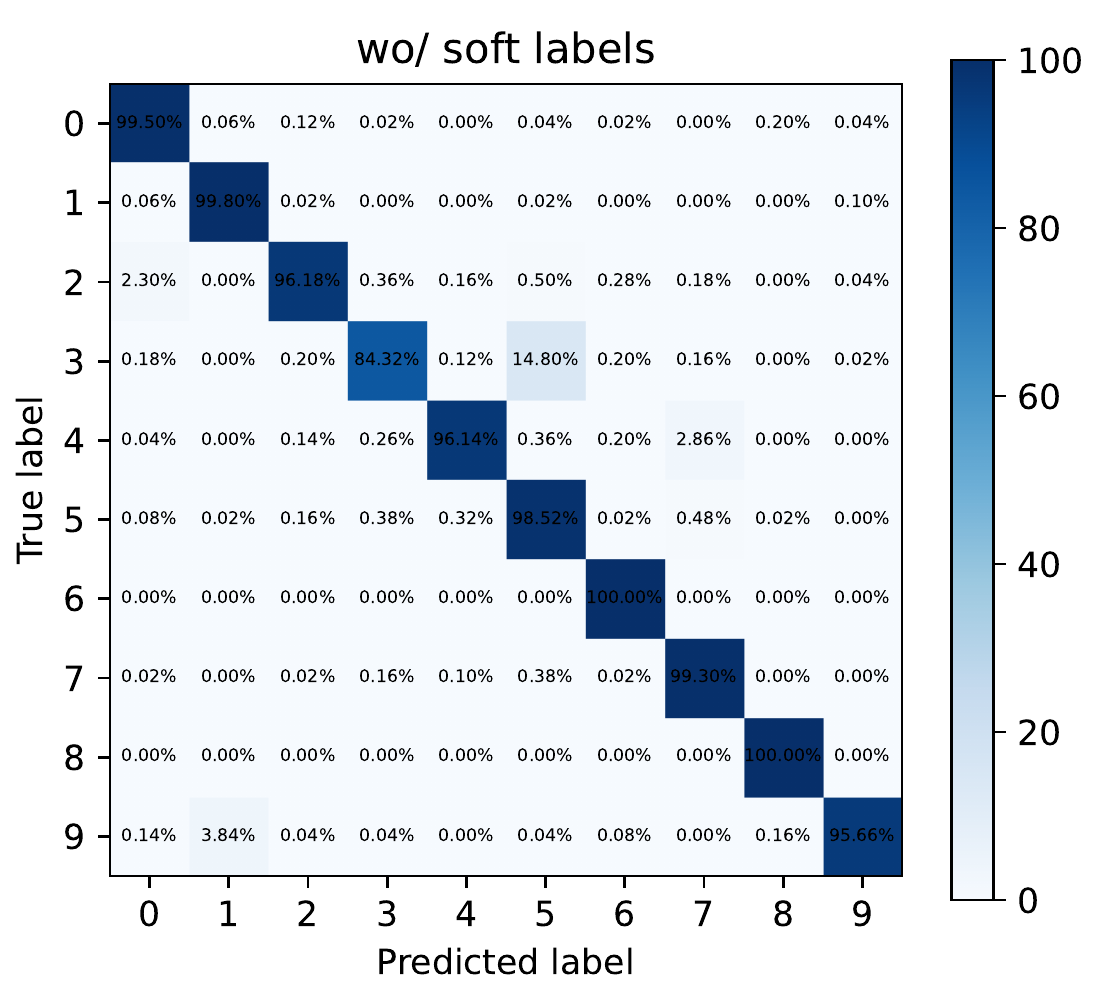}}  \\ \vspace{-3mm}
	\subfigure[Asymmetry Noise (60\%)]{
		\includegraphics[width=0.23\textwidth]{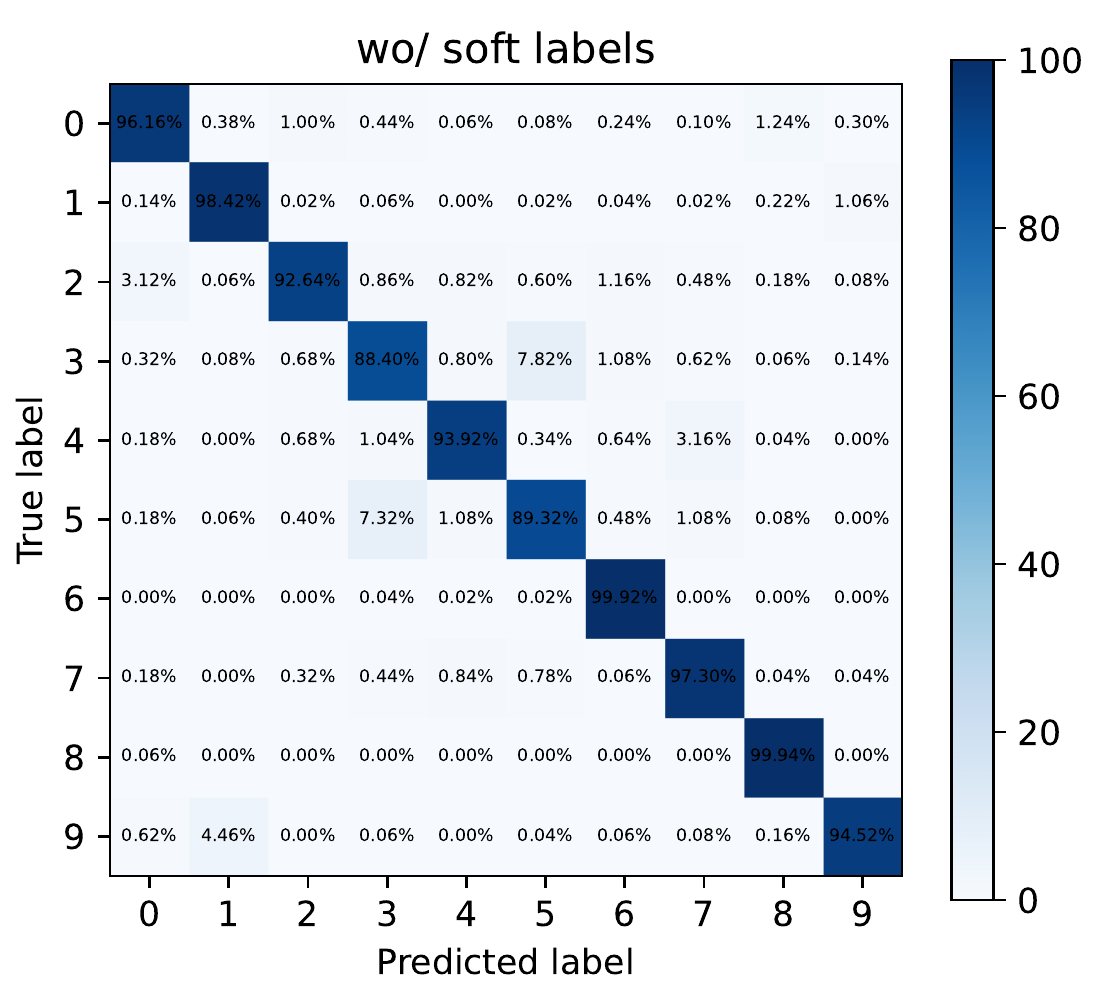}
		\includegraphics[width=0.23\textwidth]{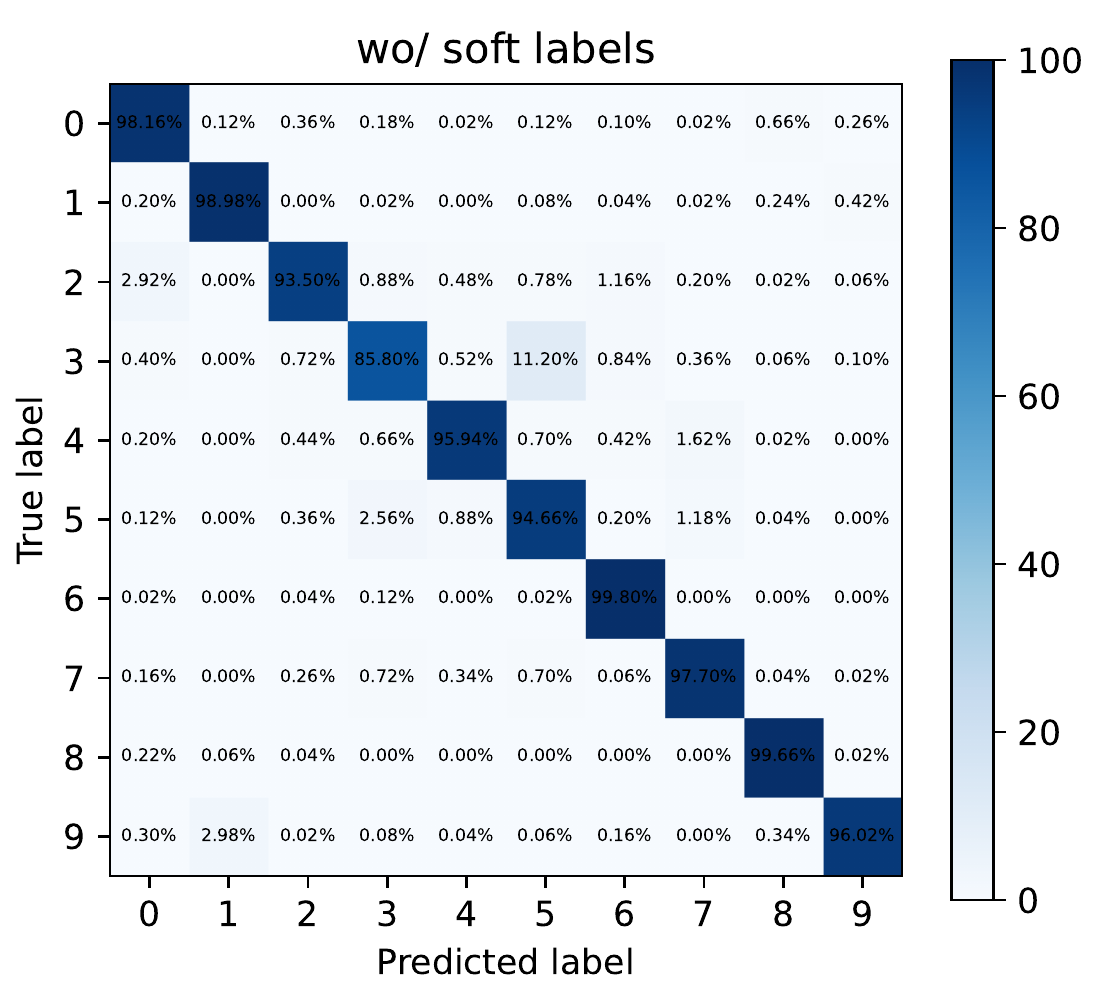}} \ \
	\subfigure[Asymmetry Noise (80\%)]{
		\includegraphics[width=0.23\textwidth]{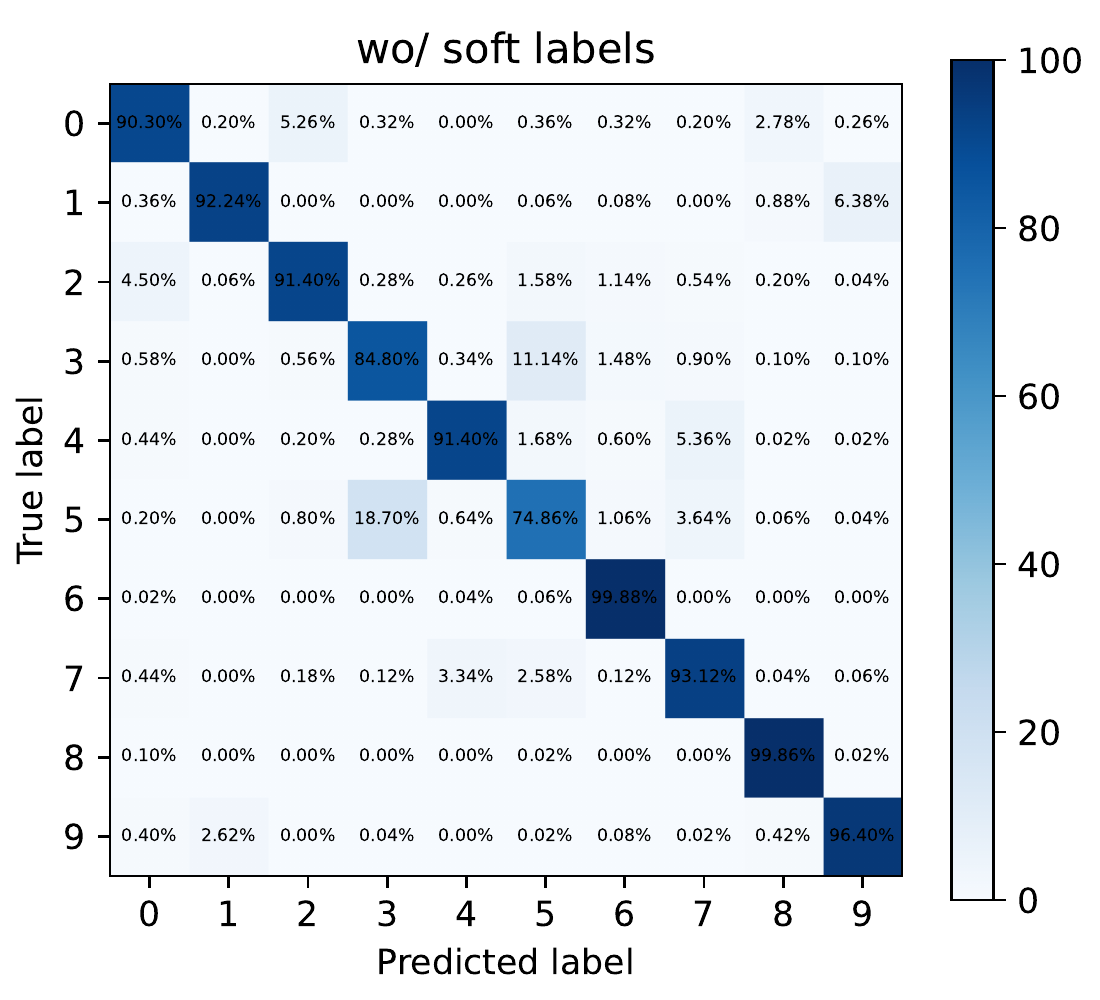}
		\includegraphics[width=0.23\textwidth]{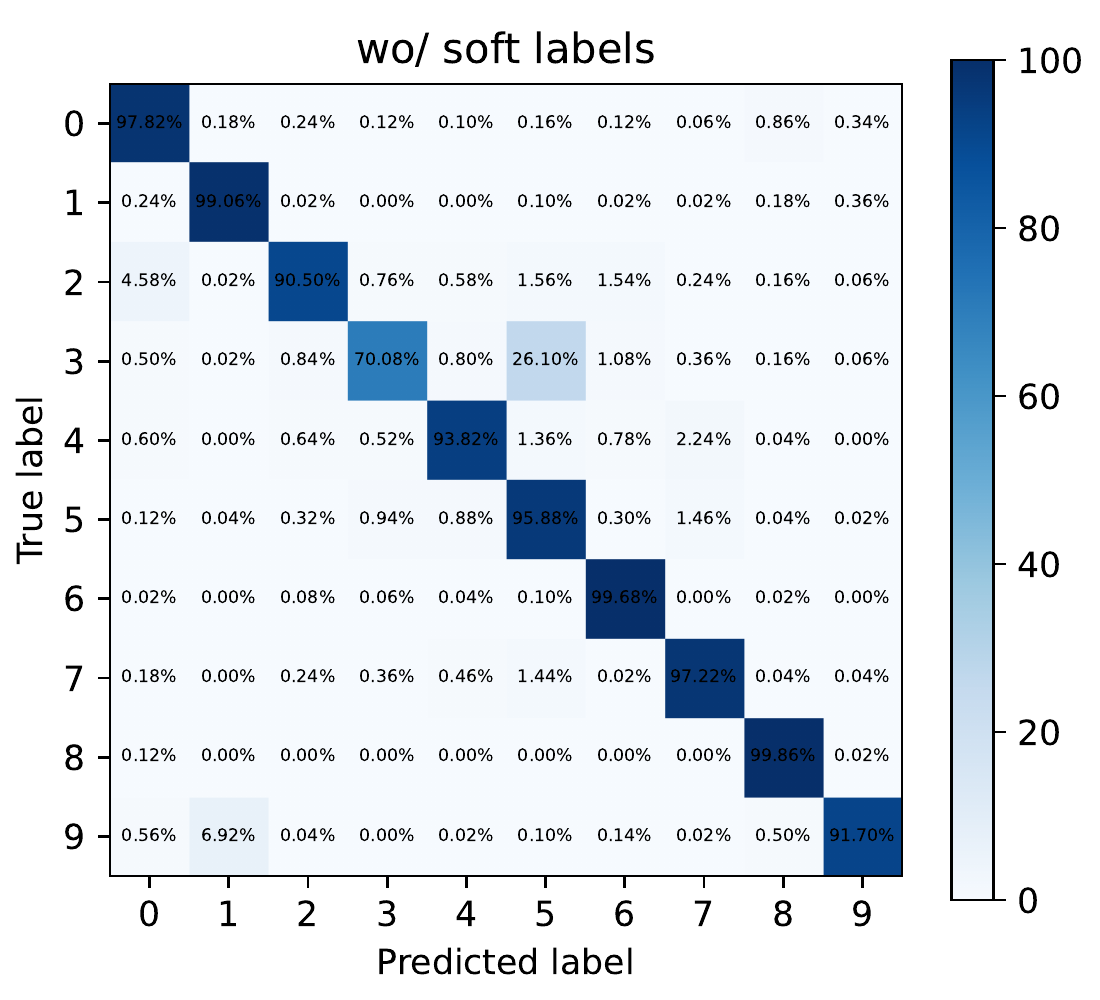}} \vspace{-4mm}
	\caption{Confusion matrices obtained by CMW-Net without (left) or with (right) soft label amelioration on CIFAR-10 with asymmetry noise with varying noise rates ranging from 20\% to 80\%.}\label{SM6}
\end{figure*}

\subsection{Additional Open-set Label Noise Experiment}

Open-set noise experiments use training samples that do not belong to any of the original classes in the dataset considered in the classification task.
Following \cite{wang2018iterative}, we yield CIFAR-10 with open-set noise by randomly replacing 40\% of its training images with images from CIFAR-100.
We used wide ResNet-28-2\cite{Zagoruyko2016Wide} as the base classifier network with softmax cross-entropy loss by SGD with a momentum 0.9, a weight decay $5\times10^{-4}$. We set the initial learning rate of classification network as 0.1 and the learning rate is divided by 10 after 80 and 100 epoch (for a total 120 epochs). The batch size is 128 for all experiments. We adopt Adam optimizer to learn CMW-Net, with a learning rate $10^{-3}$, and a weight decay $10^{-4}$. We repeat the experiments with 3 random trials and report the mean value and standard deviation.
We adopt the meta-data generation strategy as introduced in Sec. 4.2 of the main text, by randomly selecting 10 images per class at every epoch from the training set as the meta-data set. We train the network 20 epochs with cross-entropy loss for a warm-up to get the meta dataset stably.

The compared methods include: 1) ERM: use standard cross-entropy loss to train DNNs; 2) Forward \cite{patrini2017making}: correct the prediction by the label transition matrix;
3) GCE \cite{zhang2018generalized}: behave as a robust loss to handle the noisy labels; 4) M-correction \cite{arazo2019unsupervised} and 5) DivideMix \cite{li2019dividemix}: use different label correction methods; 6) L2RW \cite{ren2018learning} and 7) MW-Net \cite{shu2019meta}: represent the typical sample reweighting methods by meta-learning.

The classification accuracy on CIFAR-10 noisy datasets with 40\% open-set noise is reported in Table \ref{openset}. As can be seen, our method evidently outperforms all other competing methods, verifying that our model is capable of learning more accurate representation directly from datasets with open-set noisy labels. Such capability supports that our method can be applied to learning from web-search data possibly containing such type of open-set noisy labels, e.g., WebVision \cite{li2017webvision}.

\begin{table*}
	\setlength{\abovecaptionskip}{0.cm}
	\setlength{\belowcaptionskip}{-2cm}
	\caption{Test accuracy (\%) of all comparison methods under open-set noise on CIFAR-10.}\label{openset} \vspace{0mm}
	\centering
	\setlength{\tabcolsep}{0.5mm}{
	\begin{tabular}{l|c|c|c|c|c|c|c|c|c}
		\toprule
		Methods  &  ERM & Forward \cite{patrini2017making} &GCE \cite{zhang2018generalized}&M-correction \cite{arazo2019unsupervised}& DivideMix \cite{li2019dividemix} &  L2RW \cite{ren2018learning}& MW-Net \cite{shu2019meta}&  CMW-Net & CMW-Net-SL    \\  \hline
		Accuracy & 84.17$\pm$0.80 & 84.63$\pm$0.80 & 85.96 $\pm$ 0.72 & 89.71 $\pm$ 0.53 & 90.16$\pm$0.40 &  83.60$\pm$0.24 & 84.78 $\pm$ 0.51 & 84.81 $\pm$ 0.51 & \textbf{92.12 $\pm$ 0.18} \\
		\bottomrule
	\end{tabular}}
\end{table*}
\begin{figure}[t]
	\centering
	\subfigcapskip=-1mm
	\subfigure[On different numbers of task families]{\label{figablationa}
		\includegraphics[width=0.21\textwidth]{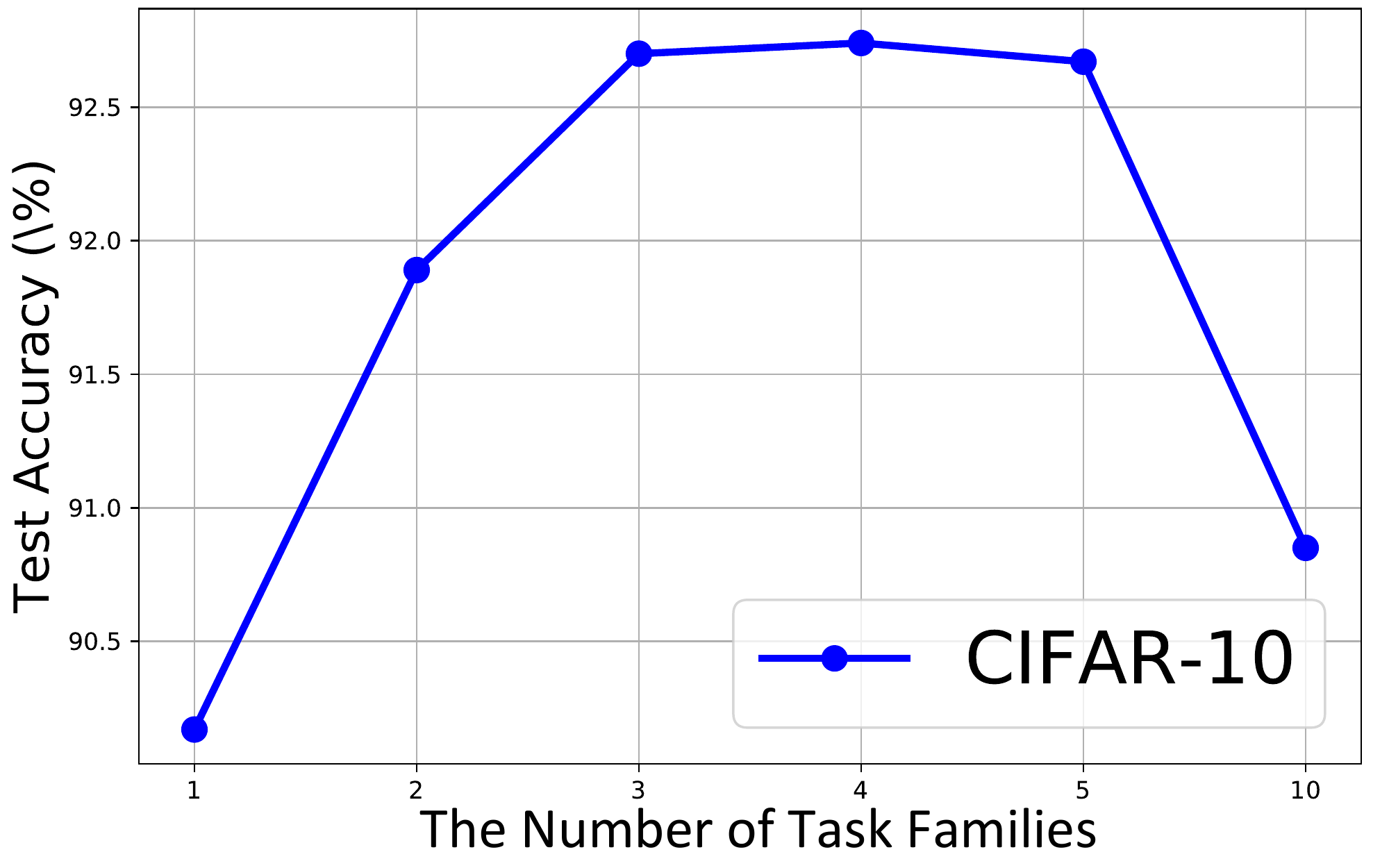}  \ \ \
		\includegraphics[width=0.21\textwidth]{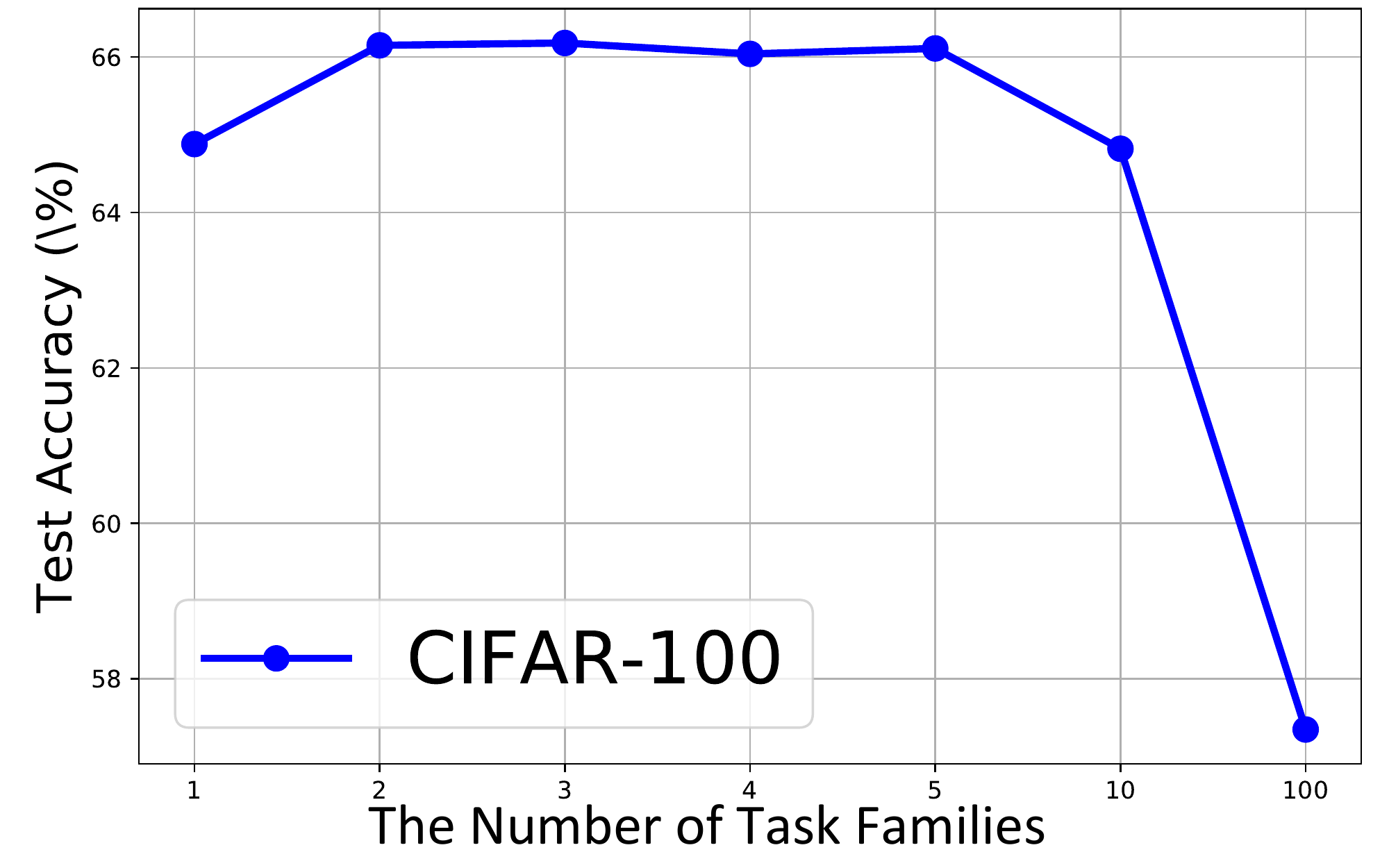}}  	\subfigure[On whether to use extra clean meta data]{\label{figablationb}
		\includegraphics[width=0.21\textwidth]{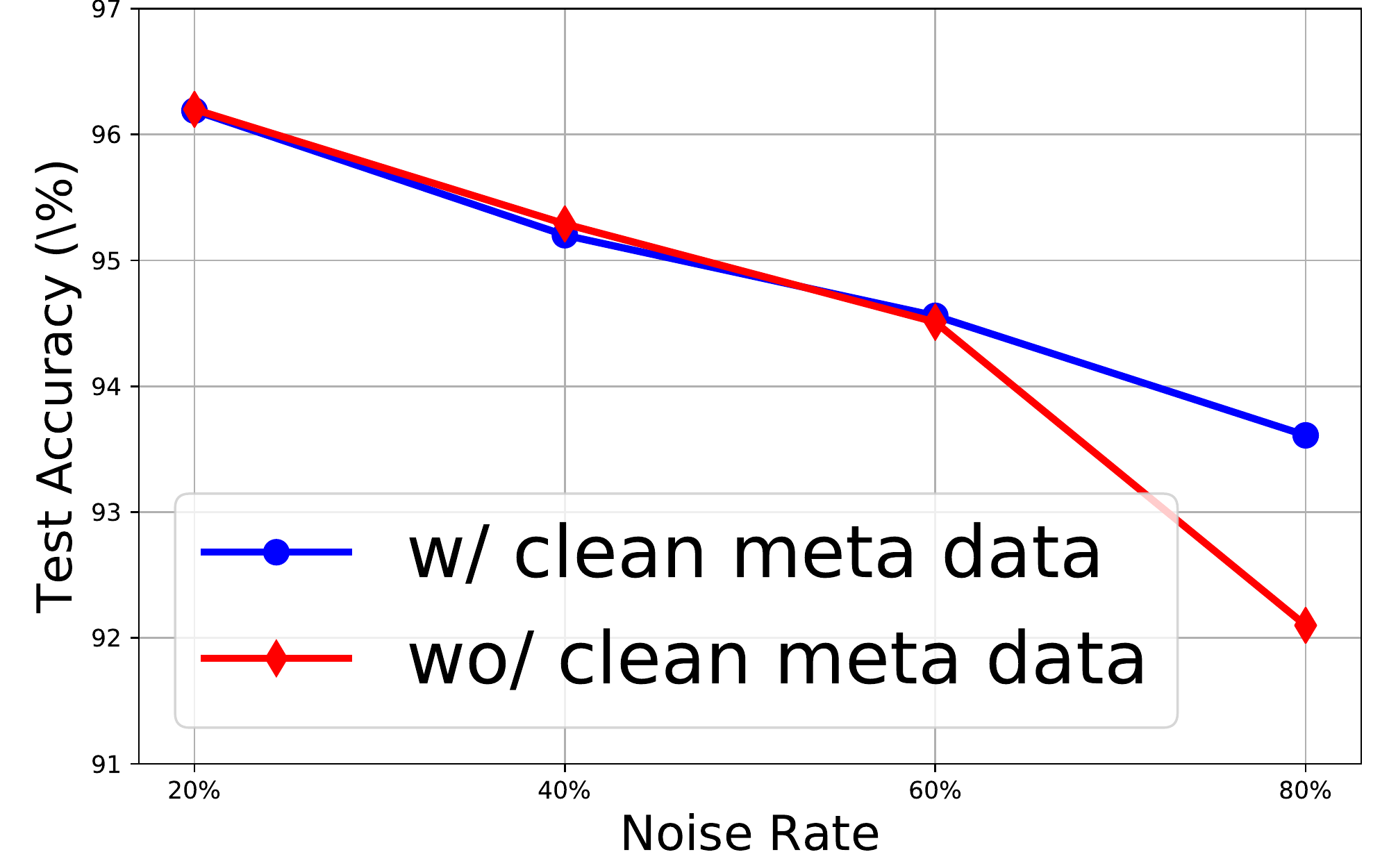} \ \ \
		\includegraphics[width=0.21\textwidth]{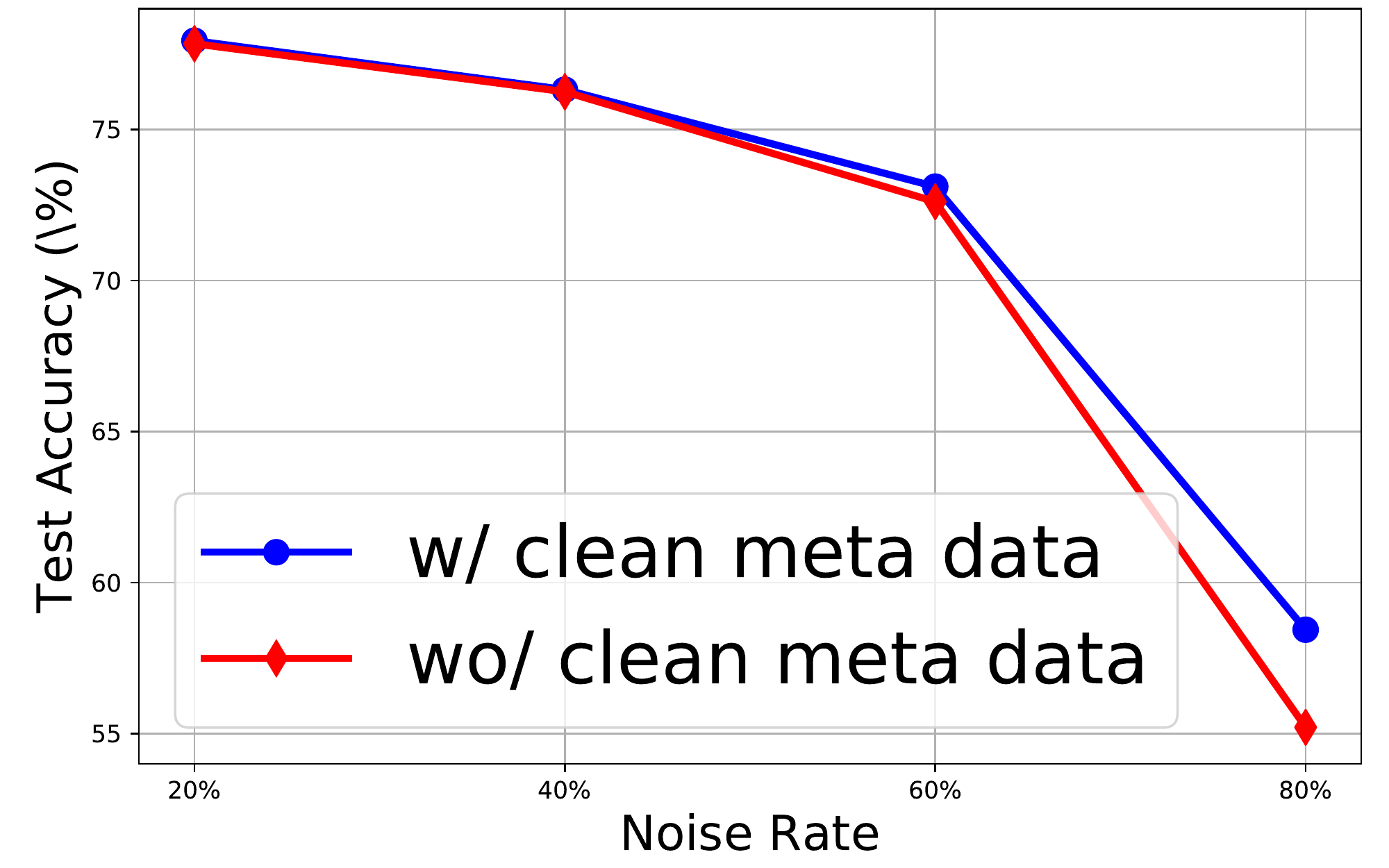}}
	\caption{Some ablation studies for parameter setting issues of our proposed method.}\label{figablation}  \end{figure}

\subsection{Ablation Study}

We perform ablation study to verify the effectiveness of two important components involved in our method:
1) the number of task families; 2) whether to use an extra clean meta-data or using the automatic meta-data-generation strategy as proposed in Sec. 4.2 of the main text. As shown in Fig. \ref{figablationa}, by setting the number of task families as three, our method can consistently adapt to inter-class heterogenous data bias. Actually, by setting $K=3$, where the training classes/tasks are separated as small, moderate, large-scales, correspondingly, all our experiments can achieve a stably fine performance. Furthermore, by observing Fig. \ref{figablationb}, we can see that the utilized  meta-data-generation strategy is practicable for dealing with real-world noisy datasets, in which an extra ideal clean meta data is always hard to be collected.

\section{More Experimental Results and Experimental Settings in Section 5}

\subsection{Learning with Real-world Noisy Datasets}

\textbf{Animal-10N.} \ ANIMAL-10N \cite{song2019selfie} contains 55,000 human-labeled online images for 10 animals with confusable appearances. The estimated label noise rate is 8\%. Following previous works \cite{song2019selfie,zhang2020learning}, 50,000 images are exploited as the training set and the left for testing.
Following SELFIE \cite{song2019selfie}, we use VGG-19 \cite{simonyan2015very} with batch normalization as the classifier network. The SGD optimizer is employed to train the network with a momentum 0.9, a weight decay $1\times10^{-3}$ for 100 epochs. We use an initial learning rate of 0.1, which is divided by 5 at 50\% and 75\% of the total number of epochs. The batch size is 128. We repeat the experiments with 3 random trials and report the mean value and standard deviation.

\textbf{Mini-WebVision.}\  As the full dataset of WebVision is very large, we follow \cite{jiang2018mentornet} to use a mini version, which contains the first 50 classes of the Google subset of the data for a total of about 61,000 images. Following the standard protocol \cite{jiang2018mentornet}, we test the trained model on the WebVision validation set and the ImageNet validation set. Following C2D \cite{zheltonozhskii2020contrast}, we used ResNet-50 architecture as the classifier network for training. For self-supervised pre-training, we directly use the pretrained self-supervised models released at \url{https://github.com/ContrastToDivide/C2D}, which is based on the SimCLR implementation.
We trained the network with softmax cross-entropy loss by SGD with a momentum 0.9, a weight decay $5\times10^{-4}$. We set the initial learning rate as 0.01 and the learning rate of classification network is divided by 10 after 50 epoch (for a total 90 epochs). The learning rate of CMW-Net is fixed as $10^{-4}$, and the weight decay of CMW-Net is fixed as $10^{-5}$. The batch size is 64.
We adopt the meta-data-generation strategy as introduced in Sec. 4.2 of the main text, to randomly select 10 images per class at every epoch from the training set as the meta-data set for the above two real-world biased datasets.

More typical noisy examples corrected by the proposed method on Animal-10N and Mini-WebVision are shown in Figs. \ref{figanifig} and \ref{figwebfig}, respectively. This further demonstrates our method's capability of recovering these easily confusable samples.

\begin{figure*}
	\centering 
	\subfigure[Samples selected from Animal-10N \cite{song2019selfie}. The original training label is {\color{red}{cat}}.]{
		\begin{minipage}[b]{0.12\textwidth} 
			\centering 
			\centerline{\includegraphics[width=0.9\textwidth, height=0.9\textwidth]{./fig/animal10/img_0/label_1_416.png}} 
			\centerline{\color{blue}{lynx}}
		\end{minipage}
		\begin{minipage}[b]{0.12\textwidth} 
			\centering 
			\centerline{\includegraphics[width=0.9\textwidth, height=0.9\textwidth]{./fig/animal10/img_0/label_1_3451.png}} 
			\centerline{\color{blue}{lynx}}
		\end{minipage}
		\begin{minipage}[b]{0.12\textwidth} 
			\centering 
			\centerline{\includegraphics[width=0.9\textwidth, height=0.9\textwidth]{./fig/animal10/img_0/label_1_2567.png}} 
			\centerline{\color{blue}{lynx}}
		\end{minipage}
		\begin{minipage}[b]{0.12\textwidth} 
			\centering 
			\centerline{\includegraphics[width=0.9\textwidth, height=0.9\textwidth]{./fig/animal10/img_0/label_1_963.png}} 
			\centerline{\color{blue}{lynx}}
		\end{minipage}
		\begin{minipage}[b]{0.12\textwidth} 
			\centering 
			\centerline{\includegraphics[width=0.9\textwidth, height=0.9\textwidth]{./fig/animal10/img_0/label_1_1298.png}} 
			\centerline{\color{blue}{lynx}}
		\end{minipage}
		\begin{minipage}[b]{0.12\textwidth} 
			\centering 
			\centerline{\includegraphics[width=0.9\textwidth, height=0.9\textwidth]{./fig/animal10/img_0/label_1_3291.png}} 
			\centerline{\color{blue}{lynx}}
		\end{minipage}
		\begin{minipage}[b]{0.12\textwidth} 
			\centering 
			\centerline{\includegraphics[width=0.9\textwidth, height=0.9\textwidth]{./fig/animal10/img_0/label_1_3330.png}} 
			\centerline{\color{blue}{lynx}}
		\end{minipage}
		\begin{minipage}[b]{0.12\textwidth} 
			\centering 
			\centerline{\includegraphics[width=0.9\textwidth, height=0.9\textwidth]{./fig/animal10/img_0/label_1_687.png}} 
			\centerline{\color{blue}{lynx}}
		\end{minipage}} \\ \vspace{-2mm}
		\subfigure[Samples selected from Animal-10N \cite{song2019selfie}. The original training label is {\color{red}{lynx}}.]{
		\begin{minipage}[b]{0.12\textwidth} 
			\centering 
			\centerline{\includegraphics[width=0.9\textwidth, height=0.9\textwidth]{./fig/animal10/img_1/label_0_5747.png}} 
			\centerline{\color{blue}{cat}}
		\end{minipage}
		\begin{minipage}[b]{0.12\textwidth} 
			\centering 
			\centerline{\includegraphics[width=0.9\textwidth, height=0.9\textwidth]{./fig/animal10/img_1/label_0_6030.png}} 
			\centerline{\color{blue}{cat}}
		\end{minipage}
		\begin{minipage}[b]{0.12\textwidth} 
			\centering 
			\centerline{\includegraphics[width=0.9\textwidth, height=0.9\textwidth]{./fig/animal10/img_1/label_0_6329.png}} 
			\centerline{\color{blue}{cat}}
		\end{minipage}
		\begin{minipage}[b]{0.12\textwidth} 
			\centering 
			\centerline{\includegraphics[width=0.9\textwidth, height=0.9\textwidth]{./fig/animal10/img_1/label_0_6342.png}} 
			\centerline{\color{blue}{cat}}
		\end{minipage}
		\begin{minipage}[b]{0.12\textwidth} 
			\centering 
			\centerline{\includegraphics[width=0.9\textwidth, height=0.9\textwidth]{./fig/animal10/img_1/label_0_6500.png}} 
			\centerline{\color{blue}{cat}}
		\end{minipage}
		\begin{minipage}[b]{0.12\textwidth} 
			\centering 
			\centerline{\includegraphics[width=0.9\textwidth, height=0.9\textwidth]{./fig/animal10/img_1/label_0_7750.png}} 
			\centerline{\color{blue}{cat}}
		\end{minipage}
		\begin{minipage}[b]{0.12\textwidth} 
			\centering 
			\centerline{\includegraphics[width=0.9\textwidth, height=0.9\textwidth]{./fig/animal10/img_1/label_0_8654.png}} 
			\centerline{\color{blue}{cat}}
		\end{minipage}
		\begin{minipage}[b]{0.12\textwidth} 
			\centering 
			\centerline{\includegraphics[width=0.9\textwidth, height=0.9\textwidth]{./fig/animal10/img_1/label_0_9049.png}} 
			\centerline{\color{blue}{cat}}
		\end{minipage}} \\ \vspace{-2mm}
		\subfigure[Samples selected from Animal-10N \cite{song2019selfie}. The original training label is {\color{red}{wolf}}.]{
		\begin{minipage}[b]{0.12\textwidth} 
			\centering 
			\centerline{\includegraphics[width=0.9\textwidth, height=0.9\textwidth]{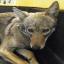}} 
			\centerline{\color{blue}{coyote}}
		\end{minipage}
		\begin{minipage}[b]{0.12\textwidth} 
			\centering 
			\centerline{\includegraphics[width=0.9\textwidth, height=0.9\textwidth]{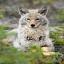}} 
			\centerline{\color{blue}{coyote}}
		\end{minipage}
		\begin{minipage}[b]{0.12\textwidth} 
			\centering 
			\centerline{\includegraphics[width=0.9\textwidth, height=0.9\textwidth]{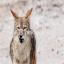}} 
			\centerline{\color{blue}{coyote}}
		\end{minipage}
		\begin{minipage}[b]{0.12\textwidth} 
			\centering 
			\centerline{\includegraphics[width=0.9\textwidth, height=0.9\textwidth]{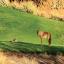}} 
			\centerline{\color{blue}{coyote}}
		\end{minipage}
		\begin{minipage}[b]{0.12\textwidth} 
			\centering 
			\centerline{\includegraphics[width=0.9\textwidth, height=0.9\textwidth]{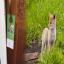}} 
			\centerline{\color{blue}{coyote}}
		\end{minipage}
		\begin{minipage}[b]{0.12\textwidth} 
			\centering 
			\centerline{\includegraphics[width=0.9\textwidth, height=0.9\textwidth]{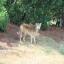}} 
			\centerline{\color{blue}{coyote}}
		\end{minipage}
		\begin{minipage}[b]{0.12\textwidth} 
			\centering 
			\centerline{\includegraphics[width=0.9\textwidth, height=0.9\textwidth]{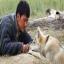}} 
			\centerline{\color{blue}{coyote}}
		\end{minipage}
		\begin{minipage}[b]{0.12\textwidth} 
			\centering 
			\centerline{\includegraphics[width=0.9\textwidth, height=0.9\textwidth]{./fig/animal10/img_1/label_0_9049.png}} 
			\centerline{\color{blue}{coyote}}
		\end{minipage}}\\  \vspace{-2mm}
		\subfigure[Samples selected from Animal-10N \cite{song2019selfie}. The original training label is {\color{red}{coyote}}.]{
		\begin{minipage}[b]{0.12\textwidth} 
			\centering 
			\centerline{\includegraphics[width=0.9\textwidth, height=0.9\textwidth]{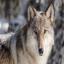}} 
			\centerline{\color{blue}{wolf}}
		\end{minipage}
		\begin{minipage}[b]{0.12\textwidth} 
			\centering 
			\centerline{\includegraphics[width=0.9\textwidth, height=0.9\textwidth]{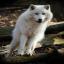}} 
			\centerline{\color{blue}{wolf}}
		\end{minipage}
		\begin{minipage}[b]{0.12\textwidth} 
			\centering 
			\centerline{\includegraphics[width=0.9\textwidth, height=0.9\textwidth]{./fig/animal10/img_3/label_2_17156.png}} 
			\centerline{\color{blue}{wolf}}
		\end{minipage}
		\begin{minipage}[b]{0.12\textwidth} 
			\centering 
			\centerline{\includegraphics[width=0.9\textwidth, height=0.9\textwidth]{./fig/animal10/img_3/label_2_17156.png}} 
			\centerline{\color{blue}{wolf}}
		\end{minipage}
		\begin{minipage}[b]{0.12\textwidth} 
			\centering 
			\centerline{\includegraphics[width=0.9\textwidth, height=0.9\textwidth]{./fig/animal10/img_3/label_2_17156.png}} 
			\centerline{\color{blue}{wolf}}
		\end{minipage}
		\begin{minipage}[b]{0.12\textwidth} 
			\centering 
			\centerline{\includegraphics[width=0.9\textwidth, height=0.9\textwidth]{./fig/animal10/img_3/label_2_17156.png}} 
			\centerline{\color{blue}{wolf}}
		\end{minipage}
		\begin{minipage}[b]{0.12\textwidth} 
			\centering 
			\centerline{\includegraphics[width=0.9\textwidth, height=0.9\textwidth]{./fig/animal10/img_3/label_2_17156.png}} 
			\centerline{\color{blue}{wolf}}
		\end{minipage}
		\begin{minipage}[b]{0.12\textwidth} 
			\centering 
			\centerline{\includegraphics[width=0.9\textwidth, height=0.9\textwidth]{./fig/animal10/img_3/label_2_17156.png}} 
			\centerline{\color{blue}{wolf}}
		\end{minipage}} \\ \vspace{-2mm}
		\subfigure[Samples selected from Animal-10N \cite{song2019selfie}. The original training label is {\color{red}{chimpanzee}}.]{
		\begin{minipage}[b]{0.12\textwidth} 
			\centering 
			\centerline{\includegraphics[width=0.9\textwidth, height=0.9\textwidth]{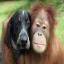}} 
			\centerline{\scriptsize{\color{blue}{cachimpanzeet}}}
		\end{minipage}
		\begin{minipage}[b]{0.12\textwidth} 
			\centering 
			\centerline{\includegraphics[width=0.9\textwidth, height=0.9\textwidth]{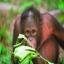}} 
			\centerline{\scriptsize{\color{blue}{cachimpanzeet}}}
		\end{minipage}
		\begin{minipage}[b]{0.12\textwidth} 
			\centering 
			\centerline{\includegraphics[width=0.9\textwidth, height=0.9\textwidth]{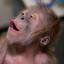}} 
			\centerline{\scriptsize{\color{blue}{cachimpanzeet}}}
		\end{minipage}
		\begin{minipage}[b]{0.12\textwidth} 
			\centering 
			\centerline{\includegraphics[width=0.9\textwidth, height=0.9\textwidth]{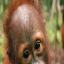}} 
			\centerline{\scriptsize{\color{blue}{cachimpanzeet}}}
		\end{minipage}
		\begin{minipage}[b]{0.12\textwidth} 
			\centering 
			\centerline{\includegraphics[width=0.9\textwidth, height=0.9\textwidth]{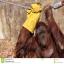}} 
			\centerline{\scriptsize{\color{blue}{cachimpanzeet}}}
		\end{minipage}
		\begin{minipage}[b]{0.12\textwidth} 
			\centering 
			\centerline{\includegraphics[width=0.9\textwidth, height=0.9\textwidth]{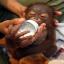}} 
			\centerline{\scriptsize{\color{blue}{cachimpanzeet}}}
		\end{minipage}
		\begin{minipage}[b]{0.12\textwidth} 
			\centering 
			\centerline{\includegraphics[width=0.9\textwidth, height=0.9\textwidth]{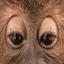}} 
			\centerline{\scriptsize{\color{blue}{cachimpanzeet}}}
		\end{minipage}
		\begin{minipage}[b]{0.12\textwidth} 
			\centering 
			\centerline{\includegraphics[width=0.9\textwidth, height=0.9\textwidth]{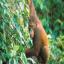}} 
			\centerline{\scriptsize{\color{blue}{cachimpanzeet}}}
		\end{minipage}}\\ \vspace{-2mm}
		\subfigure[Samples selected from Animal-10N \cite{song2019selfie}. The original training label is {\color{red}{cachimpanzeet}}.]{
		\begin{minipage}[b]{0.12\textwidth} 
			\centering 
			\centerline{\includegraphics[width=0.9\textwidth, height=0.9\textwidth]{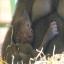}} 
			\centerline{\scriptsize{\color{blue}{chimpanzee}}}
		\end{minipage}
		\begin{minipage}[b]{0.12\textwidth} 
			\centering 
			\centerline{\includegraphics[width=0.9\textwidth, height=0.9\textwidth]{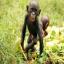}} 
			\centerline{\scriptsize{\color{blue}{chimpanzee}}}
		\end{minipage}
		\begin{minipage}[b]{0.12\textwidth} 
			\centering 
			\centerline{\includegraphics[width=0.9\textwidth, height=0.9\textwidth]{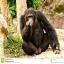}} 
			\centerline{\scriptsize{\color{blue}{chimpanzee}}}
		\end{minipage}
		\begin{minipage}[b]{0.12\textwidth} 
			\centering 
			\centerline{\includegraphics[width=0.9\textwidth, height=0.9\textwidth]{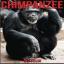}} 
			\centerline{\scriptsize{\color{blue}{chimpanzee}}}
		\end{minipage}
		\begin{minipage}[b]{0.12\textwidth} 
			\centering 
			\centerline{\includegraphics[width=0.9\textwidth, height=0.9\textwidth]{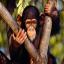}} 
			\centerline{\scriptsize{\color{blue}{chimpanzee}}}
		\end{minipage}
		\begin{minipage}[b]{0.12\textwidth} 
			\centering 
			\centerline{\includegraphics[width=0.9\textwidth, height=0.9\textwidth]{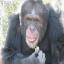}} 
			\centerline{\scriptsize{\color{blue}{chimpanzee}}}
		\end{minipage}
		\begin{minipage}[b]{0.12\textwidth} 
			\centering 
			\centerline{\includegraphics[width=0.9\textwidth, height=0.9\textwidth]{./fig/animal10/img_7/label_6_35391.png}} 
			\centerline{\scriptsize{\color{blue}{chimpanzee}}}
		\end{minipage}
		\begin{minipage}[b]{0.12\textwidth} 
			\centering 
			\centerline{\includegraphics[width=0.9\textwidth, height=0.9\textwidth]{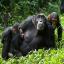}} 
			\centerline{\scriptsize{\color{blue}{chimpanzee}}}
		\end{minipage}} \\ \vspace{-2mm}
		\subfigure[Samples selected from Animal-10N \cite{song2019selfie}. The original training label is {\color{red}{hamster}}.]{
		\begin{minipage}[b]{0.12\textwidth} 
			\centering 
			\centerline{\includegraphics[width=0.9\textwidth, height=0.9\textwidth]{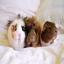}} 
			\centerline{\color{blue}{guinea pig}}
		\end{minipage}
		\begin{minipage}[b]{0.12\textwidth} 
			\centering 
			\centerline{\includegraphics[width=0.9\textwidth, height=0.9\textwidth]{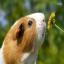}} 
			\centerline{\color{blue}{guinea pig}}
		\end{minipage}
		\begin{minipage}[b]{0.12\textwidth} 
			\centering 
			\centerline{\includegraphics[width=0.9\textwidth, height=0.9\textwidth]{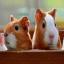}} 
			\centerline{\color{blue}{guinea pig}}
		\end{minipage}
		\begin{minipage}[b]{0.12\textwidth} 
			\centering 
			\centerline{\includegraphics[width=0.9\textwidth, height=0.9\textwidth]{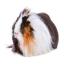}} 
			\centerline{\color{blue}{guinea pig}}
		\end{minipage}
		\begin{minipage}[b]{0.12\textwidth} 
			\centering 
			\centerline{\includegraphics[width=0.9\textwidth, height=0.9\textwidth]{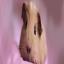}} 
			\centerline{\color{blue}{guinea pig}}
		\end{minipage}
		\begin{minipage}[b]{0.12\textwidth} 
			\centering 
			\centerline{\includegraphics[width=0.9\textwidth, height=0.9\textwidth]{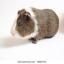}} 
			\centerline{\color{blue}{guinea pig}}
		\end{minipage}
		\begin{minipage}[b]{0.12\textwidth} 
			\centering 
			\centerline{\includegraphics[width=0.9\textwidth, height=0.9\textwidth]{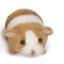}} 
			\centerline{\color{blue}{guinea pig}}
		\end{minipage}
		\begin{minipage}[b]{0.12\textwidth} 
			\centering 
			\centerline{\includegraphics[width=0.9\textwidth, height=0.9\textwidth]{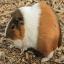}} 
			\centerline{\color{blue}{guinea pig}}
		\end{minipage}
	} \\ \vspace{-2mm}
		\subfigure[Samples selected from Animal-10N \cite{song2019selfie}. The original training label is {\color{red}{guinea pig}}.]{
		\begin{minipage}[b]{0.12\textwidth} 
			\centering 
			\centerline{\includegraphics[width=0.9\textwidth, height=0.9\textwidth]{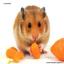}} 
			\centerline{\color{blue}{hamster}}
		\end{minipage}
		\begin{minipage}[b]{0.12\textwidth} 
			\centering 
			\centerline{\includegraphics[width=0.9\textwidth, height=0.9\textwidth]{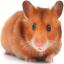}} 
			\centerline{\color{blue}{hamster}}
		\end{minipage}
		\begin{minipage}[b]{0.12\textwidth} 
			\centering 
			\centerline{\includegraphics[width=0.9\textwidth, height=0.9\textwidth]{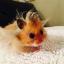}} 
			\centerline{\color{blue}{hamster}}
		\end{minipage}
		\begin{minipage}[b]{0.12\textwidth} 
			\centering 
			\centerline{\includegraphics[width=0.9\textwidth, height=0.9\textwidth]{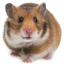}} 
			\centerline{\color{blue}{hamster}}
		\end{minipage}
		\begin{minipage}[b]{0.12\textwidth} 
			\centering 
			\centerline{\includegraphics[width=0.9\textwidth, height=0.9\textwidth]{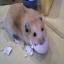}} 
			\centerline{\color{blue}{hamster}}
		\end{minipage}
		\begin{minipage}[b]{0.12\textwidth} 
			\centering 
			\centerline{\includegraphics[width=0.9\textwidth, height=0.9\textwidth]{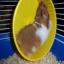}} 
			\centerline{\color{blue}{hamster}}
		\end{minipage}
		\begin{minipage}[b]{0.12\textwidth} 
			\centering 
			\centerline{\includegraphics[width=0.9\textwidth, height=0.9\textwidth]{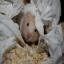}} 
			\centerline{\color{blue}{hamster}}
		\end{minipage}
		\begin{minipage}[b]{0.12\textwidth} 
			\centering 
			\centerline{\includegraphics[width=0.9\textwidth, height=0.9\textwidth]{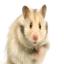}} 
			\centerline{\color{blue}{hamster}}
		\end{minipage}} \\ \vspace{-2mm}
		\caption{Examples of randomly selected samples with noisy labels corrected by our method on Animal-10N dataset \cite{song2019selfie}. The original training labels and generated pseudo-labels by our model are shown in {\color{red}{red}} and {\color{blue}{blue}}, respectively.}	\label{figanifig}
\end{figure*}

\begin{figure*}
	\centering 
	\subfigure[Samples selected from mini-WebVision \cite{li2017webvision}. The original training labels are {\color{red}{electric ray, crampfish, numbfish, torpedo}}.]{
		\begin{minipage}[b]{0.12\textwidth} 
			\centering 
			\centerline{\includegraphics[width=0.9\textwidth, height=0.9\textwidth]{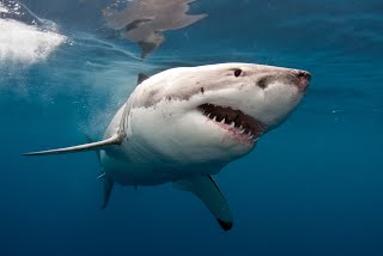}} 
			\centerline{\scriptsize\color{blue}{great white shark}}
		\end{minipage}
		\begin{minipage}[b]{0.12\textwidth} 
			\centering 
			\centerline{\includegraphics[width=0.9\textwidth, height=0.9\textwidth]{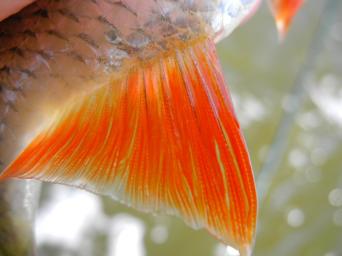}} 
			\centerline{\scriptsize\color{blue}{carassius auratus}}
		\end{minipage}
		\begin{minipage}[b]{0.12\textwidth} 
			\centering 
			\centerline{\includegraphics[width=0.9\textwidth, height=0.9\textwidth]{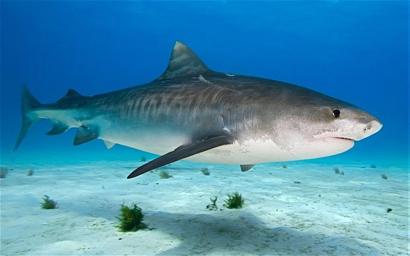}} 
			\centerline{\scriptsize\color{blue}{tiger shark}}
		\end{minipage}
		\begin{minipage}[b]{0.12\textwidth} 
			\centering 
			\centerline{\includegraphics[width=0.9\textwidth, height=0.9\textwidth]{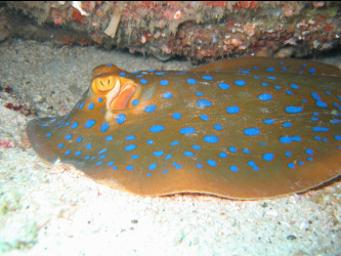}} 
			\centerline{\scriptsize\color{blue}{stingray}}
		\end{minipage}
		\begin{minipage}[b]{0.12\textwidth} 
			\centering 
			\centerline{\includegraphics[width=0.9\textwidth, height=0.9\textwidth]{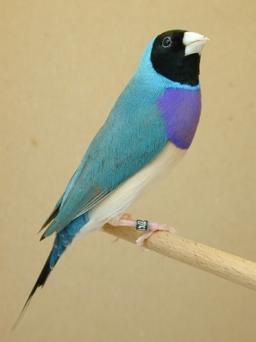}} 
			\centerline{\scriptsize\color{blue}{indigo bunting}}
		\end{minipage}
		\begin{minipage}[b]{0.12\textwidth} 
			\centering 
			\centerline{\includegraphics[width=0.9\textwidth, height=0.9\textwidth]{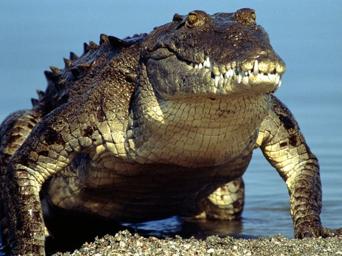}} 
			\centerline{\scriptsize\color{blue}{African crocodile}}
		\end{minipage}
		\begin{minipage}[b]{0.12\textwidth} 
			\centering 
			\centerline{\includegraphics[width=0.9\textwidth, height=0.9\textwidth]{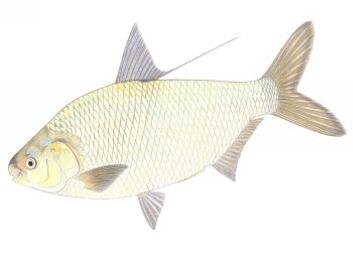}} 
			\centerline{\scriptsize\color{blue}{Tinca tinca}}
		\end{minipage}
		\begin{minipage}[b]{0.12\textwidth} 
			\centering 
			\centerline{\includegraphics[width=0.9\textwidth, height=0.9\textwidth]{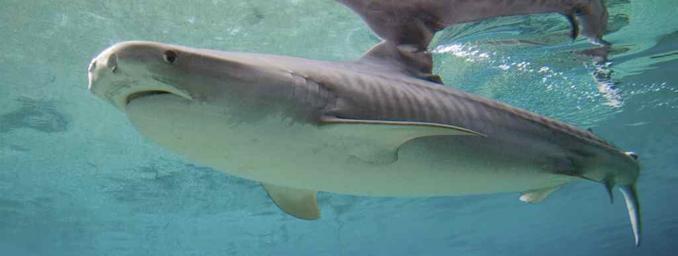}} 
			\centerline{\scriptsize\color{blue}{tiger shark}}
		\end{minipage}} \\ \vspace{-2mm}
		\subfigure[Samples selected from mini-WebVision \cite{li2017webvision}. The original training labels are {\color{red}{house finch, linnet, Carpodacus mexicanus}}.]{
			\begin{minipage}[b]{0.12\textwidth} 
				\centering 
				\centerline{\includegraphics[width=0.9\textwidth, height=0.9\textwidth]{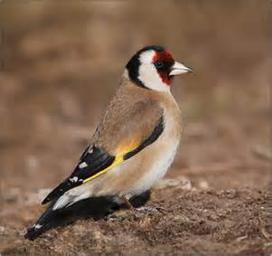}} 
				\centerline{\color{blue}{goldfinch}}
			\end{minipage}
			\begin{minipage}[b]{0.12\textwidth} 
				\centering 
				\centerline{\includegraphics[width=0.9\textwidth, height=0.9\textwidth]{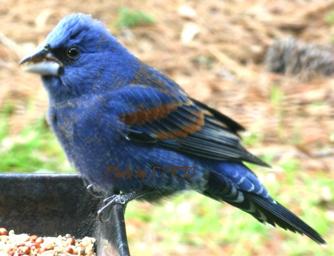}} 
				\centerline{\color{blue}{indigo bunting}}
			\end{minipage}
			\begin{minipage}[b]{0.12\textwidth} 
				\centering 
				\centerline{\includegraphics[width=0.9\textwidth, height=0.9\textwidth]{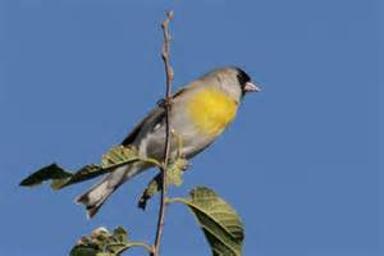}} 
				\centerline{\color{blue}{goldfinch}}
			\end{minipage}
			\begin{minipage}[b]{0.12\textwidth} 
				\centering 
				\centerline{\includegraphics[width=0.9\textwidth, height=0.9\textwidth]{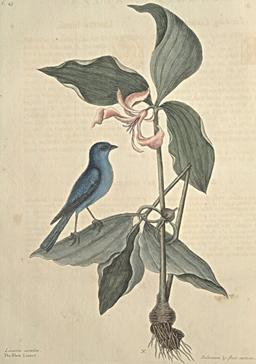}} 
				\centerline{\color{blue}{indigo bunting}}
			\end{minipage}
			\begin{minipage}[b]{0.12\textwidth} 
				\centering 
				\centerline{\includegraphics[width=0.9\textwidth, height=0.9\textwidth]{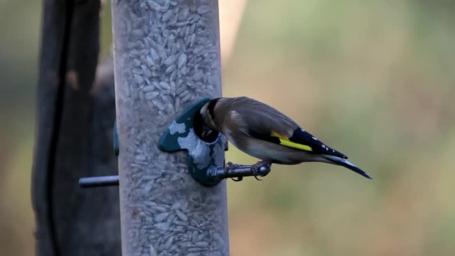}} 
				\centerline{\color{blue}{goldfinch}}
			\end{minipage}
			\begin{minipage}[b]{0.12\textwidth} 
				\centering 
				\centerline{\includegraphics[width=0.9\textwidth, height=0.9\textwidth]{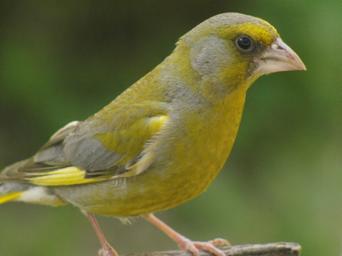}} 
				\centerline{\color{blue}{goldfinch}}
			\end{minipage}
			\begin{minipage}[b]{0.12\textwidth} 
				\centering 
				\centerline{\includegraphics[width=0.9\textwidth, height=0.9\textwidth]{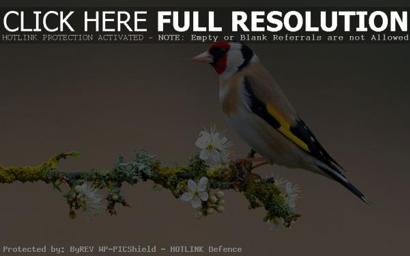}} 
				\centerline{\color{blue}{goldfinch}}
			\end{minipage}
			\begin{minipage}[b]{0.12\textwidth} 
				\centering 
				\centerline{\includegraphics[width=0.9\textwidth, height=0.9\textwidth]{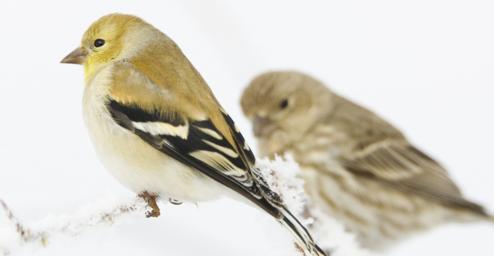}} 
				\centerline{\color{blue}{goldfinch}}
			\end{minipage}} \\ \vspace{-2mm}
			\subfigure[Samples selected from mini-WebVision \cite{li2017webvision}. The original training labels are {\color{red}{robin, American robin, Turdus migratorius}}.]{
				\begin{minipage}[b]{0.12\textwidth} 
					\centering 
					\centerline{\includegraphics[width=0.9\textwidth, height=0.9\textwidth]{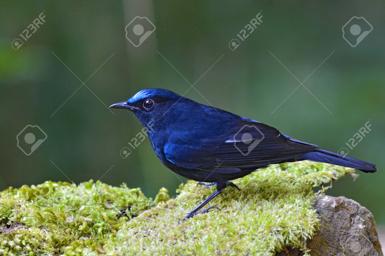}} 
					\centerline{\color{blue}{indigo bunting}}
				\end{minipage}
				\begin{minipage}[b]{0.12\textwidth} 
					\centering 
					\centerline{\includegraphics[width=0.9\textwidth, height=0.9\textwidth]{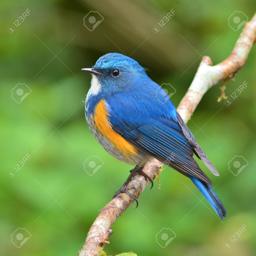}} 
					\centerline{\color{blue}{indigo bunting}}
				\end{minipage}
				\begin{minipage}[b]{0.12\textwidth} 
					\centering 
					\centerline{\includegraphics[width=0.9\textwidth, height=0.9\textwidth]{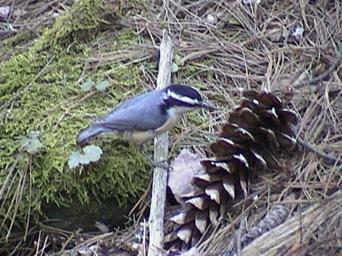}} 
					\centerline{\color{blue}{chickadee}}
				\end{minipage}
				\begin{minipage}[b]{0.12\textwidth} 
					\centering 
					\centerline{\includegraphics[width=0.9\textwidth, height=0.9\textwidth]{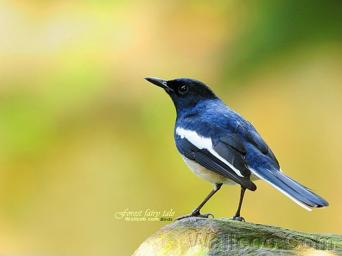}} 
					\centerline{\color{blue}{magpie}}
				\end{minipage}
				\begin{minipage}[b]{0.12\textwidth} 
					\centering 
					\centerline{\includegraphics[width=0.9\textwidth, height=0.9\textwidth]{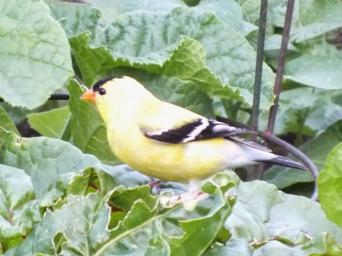}} 
					\centerline{\color{blue}{goldfinch}}
				\end{minipage}
				\begin{minipage}[b]{0.12\textwidth} 
					\centering 
					\centerline{\includegraphics[width=0.9\textwidth, height=0.9\textwidth]{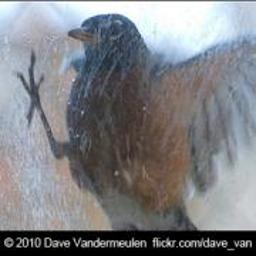}} 
					\centerline{\color{blue}{snowbird}}
				\end{minipage}
				\begin{minipage}[b]{0.12\textwidth} 
					\centering 
					\centerline{\includegraphics[width=0.9\textwidth, height=0.9\textwidth]{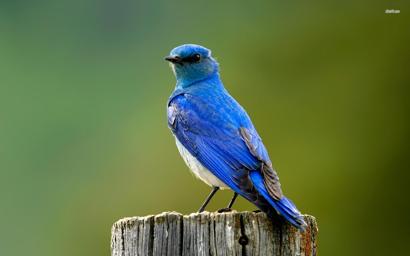}} 
					\centerline{\color{blue}{indigo bunting}}
				\end{minipage}
				\begin{minipage}[b]{0.12\textwidth} 
					\centering 
					\centerline{\includegraphics[width=0.9\textwidth, height=0.9\textwidth]{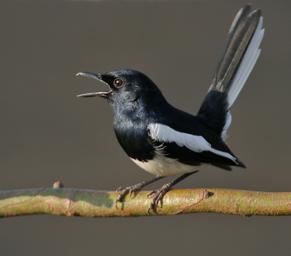}} 
					\centerline{\color{blue}{magpie}}
				\end{minipage}}\\  \vspace{-2mm}
				\subfigure[Samples selected from mini-WebVision \cite{li2017webvision}. The original training labels are {\color{red}{leatherback turtle, leatherback, leathery turtle, Dermochelys coriacea}}.]{
					\begin{minipage}[b]{0.12\textwidth} 
						\centering 
						\centerline{\includegraphics[width=0.9\textwidth, height=0.9\textwidth]{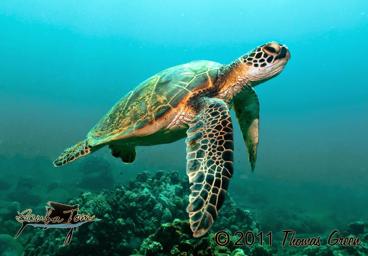}} 
						\centerline{\color{blue}{loggerhead turtle}}
					\end{minipage}
					\begin{minipage}[b]{0.12\textwidth} 
						\centering 
						\centerline{\includegraphics[width=0.9\textwidth, height=0.9\textwidth]{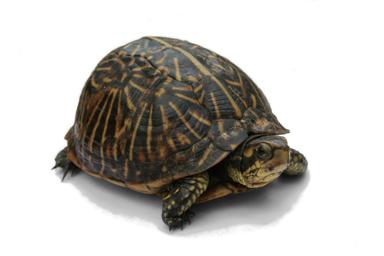}} 
						\centerline{\color{blue}{box turtle}}
					\end{minipage}
					\begin{minipage}[b]{0.12\textwidth} 
						\centering 
						\centerline{\includegraphics[width=0.9\textwidth, height=0.9\textwidth]{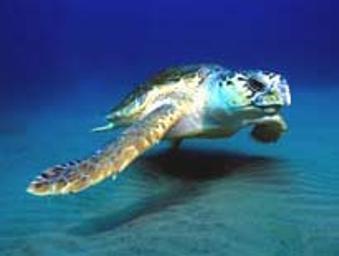}} 
						\centerline{\color{blue}{loggerhead turtle}}
					\end{minipage}
					\begin{minipage}[b]{0.12\textwidth} 
						\centering 
						\centerline{\includegraphics[width=0.9\textwidth, height=0.9\textwidth]{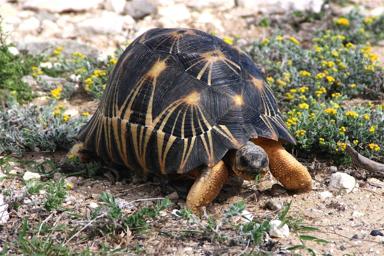}} 
						\centerline{\color{blue}{box turtle}}
					\end{minipage}
					\begin{minipage}[b]{0.12\textwidth} 
						\centering 
						\centerline{\includegraphics[width=0.9\textwidth, height=0.9\textwidth]{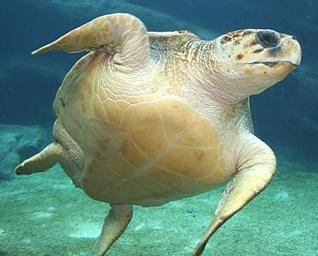}} 
						\centerline{\color{blue}{loggerhead turtle}}
					\end{minipage}
					\begin{minipage}[b]{0.12\textwidth} 
						\centering 
						\centerline{\includegraphics[width=0.9\textwidth, height=0.9\textwidth]{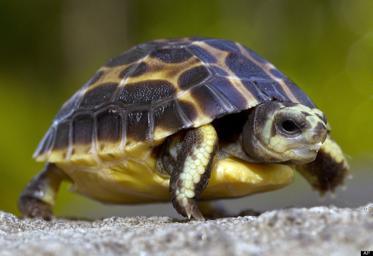}} 
						\centerline{\color{blue}{box turtle}}
					\end{minipage}
					\begin{minipage}[b]{0.12\textwidth} 
						\centering 
						\centerline{\includegraphics[width=0.9\textwidth, height=0.9\textwidth]{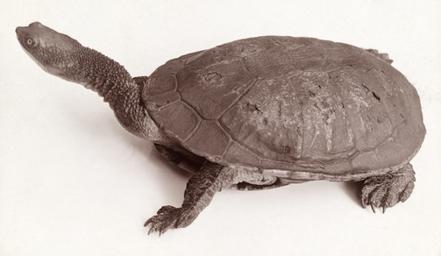}} 
						\centerline{\color{blue}{mud turtle}}
					\end{minipage}
					\begin{minipage}[b]{0.12\textwidth} 
						\centering 
						\centerline{\includegraphics[width=0.9\textwidth, height=0.9\textwidth]{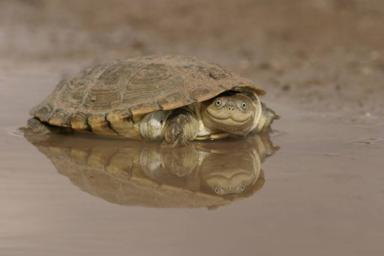}} 
						\centerline{\color{blue}{mud turtle}}
					\end{minipage}} \\ \vspace{-2mm}
					\subfigure[Samples selected from mini-WebVision \cite{li2017webvision}. The original training labels are {\color{red}{common iguana, iguana, Iguana iguana}}.]{
						\begin{minipage}[b]{0.12\textwidth} 
							\centering 
							\centerline{\includegraphics[width=0.9\textwidth, height=0.9\textwidth]{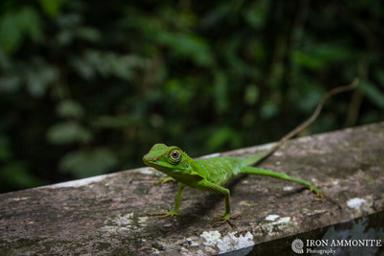}} 
							\centerline{\scriptsize{\color{blue}{American chameleon}}}
						\end{minipage}
						\begin{minipage}[b]{0.12\textwidth} 
							\centering 
							\centerline{\includegraphics[width=0.9\textwidth, height=0.9\textwidth]{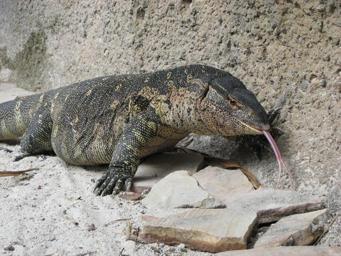}} 
							\centerline{\scriptsize{\color{blue}{Komodo dragon}}}
						\end{minipage}
						\begin{minipage}[b]{0.12\textwidth} 
							\centering 
							\centerline{\includegraphics[width=0.9\textwidth, height=0.9\textwidth]{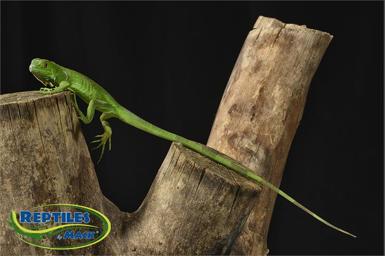}} 
							\centerline{\scriptsize{\color{blue}{Anolis carolinensis}}}
						\end{minipage}
						\begin{minipage}[b]{0.12\textwidth} 
							\centering 
							\centerline{\includegraphics[width=0.9\textwidth, height=0.9\textwidth]{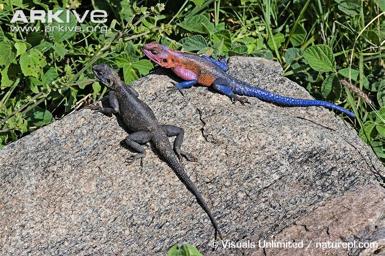}} 
							\centerline{\scriptsize{\color{blue}{agama}}}
						\end{minipage}
						\begin{minipage}[b]{0.12\textwidth} 
							\centering 
							\centerline{\includegraphics[width=0.9\textwidth, height=0.9\textwidth]{./fig/webvision/img_39/label_49247_40.png}} 
							\centerline{\scriptsize{\color{blue}{Anolis carolinensis}}}
						\end{minipage}
						\begin{minipage}[b]{0.12\textwidth} 
							\centering 
							\centerline{\includegraphics[width=0.9\textwidth, height=0.9\textwidth]{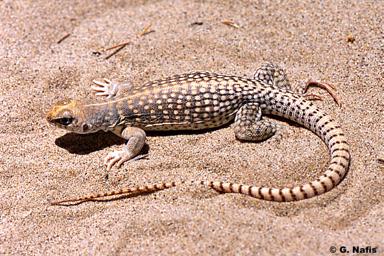}} 
							\centerline{\scriptsize{\color{blue}{whiptail lizard}}}
						\end{minipage}
						\begin{minipage}[b]{0.12\textwidth} 
							\centering 
							\centerline{\includegraphics[width=0.9\textwidth, height=0.9\textwidth]{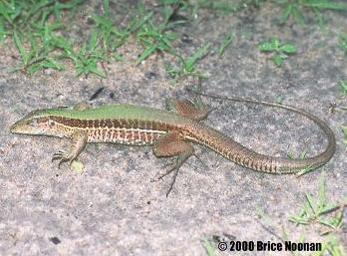}} 
							\centerline{\scriptsize{\color{blue}{whiptail lizard}}}
						\end{minipage}
						\begin{minipage}[b]{0.12\textwidth} 
							\centering 
							\centerline{\includegraphics[width=0.9\textwidth, height=0.9\textwidth]{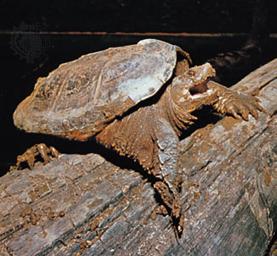}} 
							\centerline{\scriptsize{\color{blue}{mud turtle}}}
						\end{minipage}}\\ \vspace{-2mm}
						\subfigure[Samples selected from mini-WebVision \cite{li2017webvision}. The original training labels are {\color{red}{American chameleon, anole, Anolis carolinensis}}.]{
							\begin{minipage}[b]{0.12\textwidth} 
								\centering 
								\centerline{\includegraphics[width=0.9\textwidth, height=0.9\textwidth]{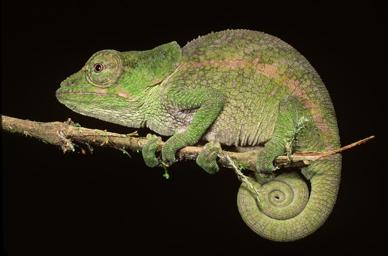}} 
								\centerline{\scriptsize{\color{blue}{African chameleon}}}
							\end{minipage}
							\begin{minipage}[b]{0.12\textwidth} 
								\centering 
								\centerline{\includegraphics[width=0.9\textwidth, height=0.9\textwidth]{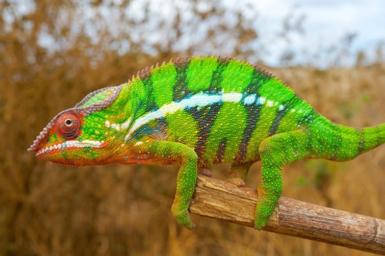}} 
								\centerline{\scriptsize{\color{blue}{African chameleon}}}
							\end{minipage}
							\begin{minipage}[b]{0.12\textwidth} 
								\centering 
								\centerline{\includegraphics[width=0.9\textwidth, height=0.9\textwidth]{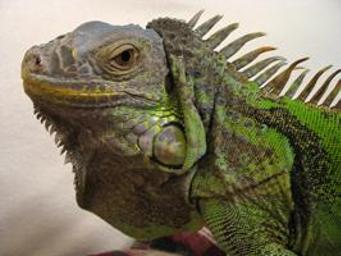}} 
								\centerline{\scriptsize{\color{blue}{common iguana}}}
							\end{minipage}
							\begin{minipage}[b]{0.12\textwidth} 
								\centering 
								\centerline{\includegraphics[width=0.9\textwidth, height=0.9\textwidth]{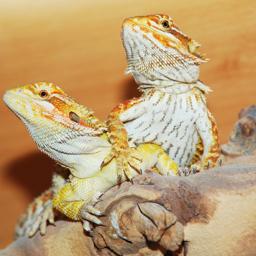}} 
								\centerline{\scriptsize{\color{blue}{agama}}}
							\end{minipage}
							\begin{minipage}[b]{0.12\textwidth} 
								\centering 
								\centerline{\includegraphics[width=0.9\textwidth, height=0.9\textwidth]{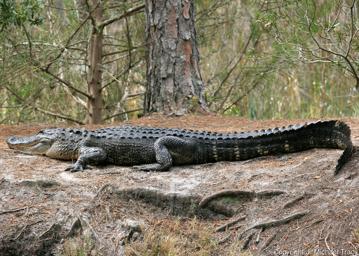}} 
								\centerline{\scriptsize{\color{blue}{African crocodile}}}
							\end{minipage}
							\begin{minipage}[b]{0.12\textwidth} 
								\centering 
								\centerline{\includegraphics[width=0.9\textwidth, height=0.9\textwidth]{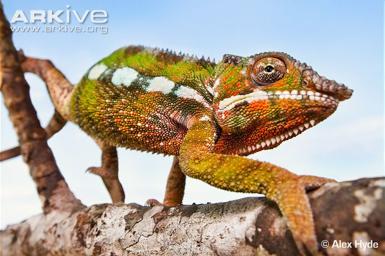}} 
								\centerline{\scriptsize{\color{blue}{African chameleon}}}
							\end{minipage}
							\begin{minipage}[b]{0.12\textwidth} 
								\centering 
								\centerline{\includegraphics[width=0.9\textwidth, height=0.9\textwidth]{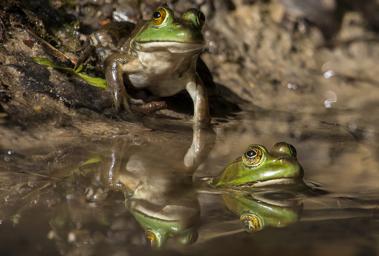}} 
								\centerline{\scriptsize{\color{blue}{bullfrog}}}
							\end{minipage}
							\begin{minipage}[b]{0.12\textwidth} 
								\centering 
								\centerline{\includegraphics[width=0.9\textwidth, height=0.9\textwidth]{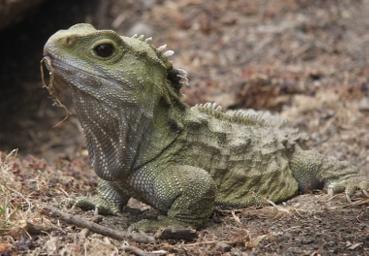}} 
								\centerline{\scriptsize{\color{blue}{common iguana}}}
							\end{minipage}} \\ \vspace{-2mm}
							\subfigure[Samples selected from mini-WebVision \cite{li2017webvision}. The original training labels are {\color{red}{Komodo dragon, Komodo lizard, dragon lizard, giant lizard, Varanus komodoensis}}.]{
								\begin{minipage}[b]{0.12\textwidth} 
									\centering 
									\centerline{\includegraphics[width=0.9\textwidth, height=0.9\textwidth]{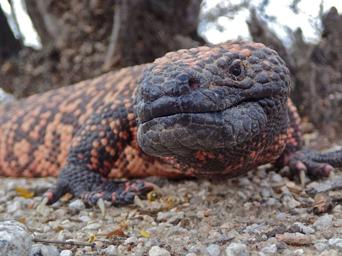}} 
									\centerline{\color{blue}{Gila monster}}
								\end{minipage}
								\begin{minipage}[b]{0.12\textwidth} 
									\centering 
									\centerline{\includegraphics[width=0.9\textwidth, height=0.9\textwidth]{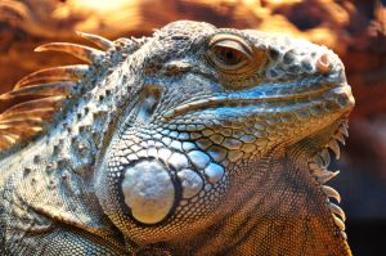}} 
									\centerline{\color{blue}{common iguana}}
								\end{minipage}
								\begin{minipage}[b]{0.12\textwidth} 
									\centering 
									\centerline{\includegraphics[width=0.9\textwidth, height=0.9\textwidth]{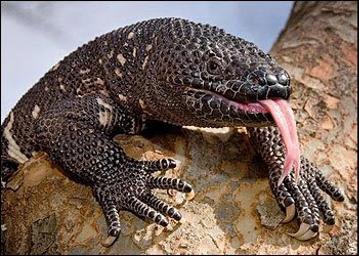}} 
									\centerline{\color{blue}{Gila monster}}
								\end{minipage}
								\begin{minipage}[b]{0.12\textwidth} 
									\centering 
									\centerline{\includegraphics[width=0.9\textwidth, height=0.9\textwidth]{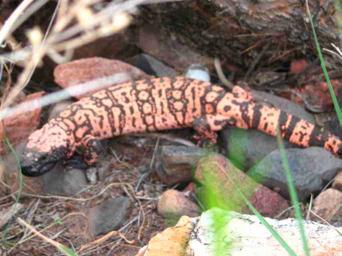}} 
									\centerline{\color{blue}{Gila monster}}
								\end{minipage}
								\begin{minipage}[b]{0.12\textwidth} 
									\centering 
									\centerline{\includegraphics[width=0.9\textwidth, height=0.9\textwidth]{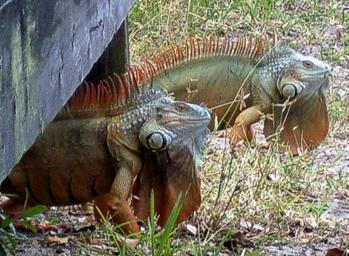}} 
									\centerline{\color{blue}{common iguana}}
								\end{minipage}
								\begin{minipage}[b]{0.12\textwidth} 
									\centering 
									\centerline{\includegraphics[width=0.9\textwidth, height=0.9\textwidth]{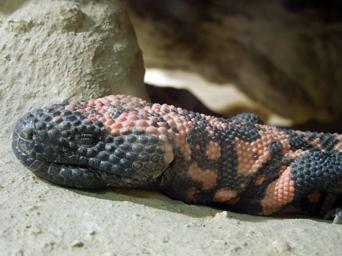}} 
									\centerline{\color{blue}{Gila monster}}
								\end{minipage}
								\begin{minipage}[b]{0.12\textwidth} 
									\centering 
									\centerline{\includegraphics[width=0.9\textwidth, height=0.9\textwidth]{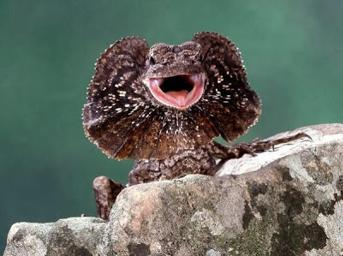}} 
									\centerline{\color{blue}{frilled lizard}}
								\end{minipage}
								\begin{minipage}[b]{0.12\textwidth} 
									\centering 
									\centerline{\includegraphics[width=0.9\textwidth, height=0.9\textwidth]{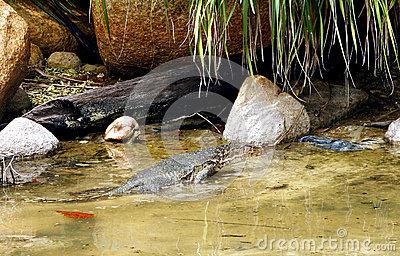}} 
									\centerline{\color{blue}{African crocodile}}
								\end{minipage}
							} \\ \vspace{-2mm}
							\subfigure[Samples selected from mini-WebVision \cite{li2017webvision}. The original training label is {\color{red}{tailed frog}}.]{ \label{figsm3}
								\begin{minipage}[b]{0.12\textwidth} 
									\centering 
									\centerline{\includegraphics[width=0.9\textwidth, height=0.9\textwidth]{./fig/img_32/label_40546_31.png}} 
									\centerline{\color{blue}{tree frog}}
								\end{minipage}
								\begin{minipage}[b]{0.12\textwidth} 
									\centering 
									\centerline{\includegraphics[width=0.9\textwidth, height=0.9\textwidth]{./fig/img_32/label_42460_31.png}} 
									\centerline{\color{blue}{tree frog}}
								\end{minipage}
								\begin{minipage}[b]{0.12\textwidth} 
									\centering 
									\centerline{\includegraphics[width=0.9\textwidth, height=0.9\textwidth]{./fig/img_32/label_40361_30.png}} 
									\centerline{\color{blue}{bullfrog}}
								\end{minipage}
								\begin{minipage}[b]{0.12\textwidth} 
									\centering 
									\centerline{\includegraphics[width=0.9\textwidth, height=0.9\textwidth]{./fig/img_32/label_41191_31.png}} 
									\centerline{\color{blue}{tree frog}}
								\end{minipage}
								\begin{minipage}[b]{0.12\textwidth} 
									\centering 
									\centerline{\includegraphics[width=0.9\textwidth, height=0.9\textwidth]{./fig/img_32/label_40362_30.png}} 
									\centerline{\color{blue}{bullfrog}}
								\end{minipage}
								\begin{minipage}[b]{0.12\textwidth} 
									\centering 
									\centerline{\includegraphics[width=0.9\textwidth, height=0.9\textwidth]{./fig/img_32/label_39820_30.png}} 
									\centerline{\color{blue}{bullfrog}}
								\end{minipage}
								\begin{minipage}[b]{0.12\textwidth} 
									\centering 
									\centerline{\includegraphics[width=0.9\textwidth, height=0.9\textwidth]{./fig/img_32/label_41184_35.png}} 
									\centerline{\color{blue}{mud turtle}}
								\end{minipage}
								\begin{minipage}[b]{0.12\textwidth} 
									\centering 
									\centerline{\includegraphics[width=0.9\textwidth, height=0.9\textwidth]{./fig/img_32/label_40763_29.png}} 
									\centerline{\color{blue}{axolotl}}
								\end{minipage}}  \\ \vspace{-4mm}
								\caption{Examples of randomly selected samples with noisy labels corrected by our method on mini-WebVision \cite{li2017webvision}. The original training labels and generated pseudo-labels by model are shown in {\color{red}{red}} and {\color{blue}{blue}}, respectively. }	\label{figwebfig}
							\end{figure*}

\subsection{Webly Supervised Fine-Grained Recognition}

\textbf{WebFG-496.}\ This dataset consists of three sub-datasets: Web-aircraft, Web-bird, and Web-car. WebFG-496 reuses the category labels of three famous manually labeled fine-grained datasets, FGVC-Aircraft, CUB200-2011, and Stanford Cars, which contain 100 types of airplanes, 200 species of birds, and 196 categories of cars, respectively, by collecting images from the Internet. It contains 53,339 training images with total 496 classes. The testing data take the testing sets in the original FGVC-Aircraft, CUB200-2011, and Stanford Cars.
We used Bilinear-CNN \cite{Lin2015Bilinear} as the classifier network. The network is pre-trained on ImageNet, and then fine-tuned on three sub-datasets of WebFG496. Following \cite{Lin2015Bilinear}, we adopt a two-stage training strategy. We firstly freeze the convolutional layer parameters and only optimize the last fully connected layers with the learning rate and batch size being $10^{-3}$ and 64 for total 200 epoch. Then we optimize the parameters of all layers in the fine-tuned model with learning rate and batch size being set as $10^{-4}$ and 32, respectively, for total 200 epoch. The learning rate of CMW-Net is fixed as $10^{-3}$, and the weight decay of CMW-Net is fixed as $10^{-4}$. We adopt the meta-data-generation strategy, as introduced in Sec. 4.2 of the main test, to randomly select 10 images per class at every epoch from the training set as the meta-data set.

\section{More Experimental Results and Experimental Settings in Section 6}

\textbf{ImageNet-LT.}\
The dataset is constructed as a long-tailed version of the original ImageNet-2012 \cite{deng2009imagenet} by sampling a subset following the Pareto distribution with the power value 6. It totally has 115.8K images from 1000 categories with maximally 1280 images per class and minimally 5 images per class. Following OLTR \cite{liu2019large}, besides the overall top-1 classification accuracy over all classes, we also calculate the accuracy of three disjoint subsets: many-shot classes (each with over training 100 samples), medium-shot classes (each with 20-100 training samples) and few-shot classes (each under 20 training samples). We adopt the two-stage training protocol following \cite{liu2019large}.  We use a Resnet-10 model initialized from scratch (i.e., random initialization) as the classifier model. We train the model with softmax cross-entropy loss by SGD with a momentum 0.9, a weight decay $5\times10^{-4}$, an initial learning rate 0.1 and a batch size of 128 for 30 epochs, and divide learning rate by 10 at 10 epoch.
The transferred CMW-Net is used at the first stage to produce proper sample weights for robust training. And it follows training protocol in \cite{liu2019large} at the second stage.

\textbf{WebVision.}\
WebVision \cite{li2017webvision} contains 2.4 million images crawled from Google and Flickr using 1,000 labels shared with the ImageNet dataset. Its training set is both heteroskedastic label noise and class imbalanced (more detailed statistics can be found in \cite{li2017webvision}), and it is considered as a popular benchmark for robust learning in the presence of heavy label noises.
We trained Inception-ResNet v2 \cite{szegedy2017inception} with softmax cross-entropy loss by SGD with a momentum 0.9, a weight decay $5\times 10^{-5}$, an initial learning rate 0.2 and a batch size of 256. The learning rate is divided by 10 after 30 and 60 epoch (for a total 90 epochs). The transferred CMW-Net is used at every iteration to produce proper sample weights for robust training.

\newpage
\section{More Experimental Results and Experimental Settings in Section 7}

\subsection{Partial-Label Learning}

We adopted two training stages to solve the problem. At the first stage, we train the network using the recent SOTA method, PRODEN \cite{lv2020progressive}, for 100 epochs and then we can get the training data with single noisy labels by one-hot encoding of the model predictions. Now the partial-label learning problem becomes  a conventional learning problem with all samples attached with single noisy labels. It is naturally to use the proposed method to further deal with such a problem. Thus at the second stage, we trained the network with obtained single noisy labeled data using the proposed CMW-Net method. We use a SGD optimizer with a momentum 0.9, a weight decay $5 \times 10^{-4}$, an initial learning rate 0.1, a batch size 128. The learning rate is divided by 10 after 80 and 100 epoch (for a total 120 epochs). We adopt Adam optimizer to learn CMW-Net. The learning rate of CMW-Net is fixed as $10^{-3}$, and the weight decay of CMW-Net is fixed as $10^{-4}$. We repeat the experiments with 3 random trials and report the mean value and standard deviation.
We adopt the meta-data-generation strategy as introduced in Sec. 4.2 of the main text, by randomly selecting 10 images per class at every epoch from the training set as the meta-data set.

Following PRODEN \cite{lv2020progressive}, we manually corrupt these datasets into partially labeled versions by a flipping probability $q$, where $q=P(\tilde{y}=1|y=0)$ gives the probability that a false positive label $\tilde{y}$ is flipped from a negative label $y$. We adopt a binomial flipping strategy: $c-1$ independent experiments are conducted on all training examples, each determining whether a negative label is flipped with probability $q$. Then for the examples that none of the negative labels are flipped, we additionally flip a random negative label to the candidate label set for ensuring all the training examples are partially labeled. We use five widely used benchmark datasets, including MNIST \cite{lecun1998gradient}, Fashion-MNIST
\cite{xiao2017fashion}, Kuzushiji-MNIST \cite{clanuwat2018deep}, CIFAR-10 and CIFAR-100 \cite{krizhevsky2009learning}.
For MNIST, Fashion-MNIST, and Kuzushiji-MNIST datasets, we use 5-layer perceptron (MLP), and for CIFAR-10 and CIFAR-100 dataset, we use ResNet-32 \cite{he2016deep} as the classifier network.

Table \ref{part} reports the mean test accuracies with standard deviation on five benchmark datasets. It can be seen that our method can consistently outperform the baseline PRODEN method under
both less-partial circumstances $q=0.1$ and stronger-partial circumstances $q=0.7$. Observing that PRODEN method behaves under strong-partial circumstances similarly as that under less-partial circumstances, which implies it tends to easily overfit to pseudo-labels estimated by model prediction. Considering that the obtained results are calculated on the basis of PRODEN method as single noisy labels dataset, our method can alleviate such pseudo-label issue and bring further performance improvement for such a partial label learning problem. Through introducing our method as a post-processing learning, we can obtain a more robust model based on the over-confident information. Particularly, our method can improve PRODEN method about 4-8 points on CIFAR-10 and 8-13 points on CIFAR-100 in classification accuracy. Applying our method to more partial-label learning method to obtain more robust results is thus potentially expected, and we leave this research for our future study.

\subsection{Semi-Supervised Learning}

Following Fixmatch \cite{sohn2020fixmatch}, we consider the settings by giving 4/25/400 labeled images for each class on CIFAR-10 and 4/25/100 labeled images for each class on CIFAR-100. We used WRN-28-2/WRN-28-8 for CIFAR-10/CIFAR-100 as the classifier network with an initial learning rate 0.03, a batch size 64 for label data and 448 for unlabeled data.
For ImageNet experiment, we use 10 \% of the training data for each class as labeled and treat the rest as unlabeled examples. We used ResNet-50 as the classifier and the batch size for labeled (unlabeled) images is 64 (320) with initial learning rate 0.03. We adopt RandAugment \cite{cubuk2020randaugment} as the strong augmentation for this experiment.
We adopt Adam optimizer to learn CMW-Net.
The learning rate of CMW-Net is fixed as $10^{-3}$, and the weight decay of CMW-Net is fixed as $10^{-4}$.

Table \ref{semi} shows the classification error rates on CIFAR-10/100 and ImageNet. From the table, one can observe that our method improves Fixmatch method and achieves the best performance under all label conditions on all datasets. Specifically, our method achieves around 2 points improvement on ImageNet as compared with Fixmatch, showing that our method is capable of finely handling such large-scale sample weight learning issue.

It should be noted that we have not used any extra meta-dataset in addition to the labeled images. Thus all comparison experiments on semi-supervised learning (SSL) have been implemented in a sufficiently fair manner for all comparison methods.

Specifically, in our meta-learning method, we directly take the provided labeled images in the implemented SSL task as meta-data, since they are relatively more delicately collected and with high label quality. Besides, we use the pseudo-labeled images automatically annotated in the training process from the unsupervised data as training data since they are with relatively lower label quality and inevitably contain label noises. This means that we have not used any extra labeled data to train CMW-Net, and the employed data source of our method is entirely similar to that used in the comparison FixMatch method.

To intrinsically explain why the proposed method can get such a performance gain as compared with other methods, we want to present the following explanations. We take the SOTA method FixMatch as example. The FixMatch is a self-training SSL method, which generates pseudo-labels of unlabeled images with the model’s predictions and then iteratively train the model with some selected reliable pseudo labeled samples, together with those pre-given labeled ones. In the method iteration, FixMatch uses a fixed confidence threshold to filter out unreliable pseudo labels (i.e., smaller than the pre-set hyper-parameter $\tau$). Albeit achieving good performance in some applications, the method mainly has two limitations. Firstly, it uses an essential hard-thresholding weighting manner by treating all selected pseudo-labeled samples equally (can be seen as imposing $1$-weight on these samples) and play similar role with those pre-given labeled images. The former, however, should evidently less reliable than the latter, and should more rationally be less weighted in a more elaborate soft-weight manner. Secondly, its involved hyper-parameter $\tau$ is pre-specified as a fixed constant. This is obviously not very appropriate, since this important parameter should be adaptably specified against different tasks, and even should be properly varied during iterations in handling one task to dynamically fit the reliability requirement in different training stages (e.g., less high-quality samples should be selected in the beginning but more in the end since the model is trained to be more mature in iteration).

Comparatively, our CMW-Net improves FixMatch by automatically learning a suitable weighting strategy from data substituting the original hard weighting scheme, to make sample weights capable of more sufficiently reflecting noise extents and  adaptable to training data/task. Two aforementioned limitations of FixMatch can thus be alleviated simultaneously. This explains why our method can get evident superior performance than FixMatch.

\begin{table*}
	\caption{Performance comparison of classification accuracy (\%) on partially labeled benchmark datasets.}\label{part} \vspace{0mm}
	\centering
	\setlength{\tabcolsep}{1mm}{
	\begin{tabular}{c|c|c|c|c|c|c}
		\toprule
		Dataset & Methods & Classifier & $q=0.1$ & $q=0.3$  &$q=0.5$ & $q=0.7$  \\ \hline
		\multirow{2}{*}{MNIST} & PRODEN & MLP & 98.59 $\pm$ 0.01 & 98.07 $\pm$ 0.03 & 98.42 $\pm$ 0.03 & 98.09 $\pm$ 0.05 \\
		& PRODEN+ Ours   & MLP & \textbf{98.99 $\pm$ 0.01} & \textbf{98.83 $\pm$ 0.02} & \textbf{98.57 $\pm$ 0.04} & \textbf{98.33 $\pm$ 0.02} \\ \hline\hline
		\multirow{2}{*}{Fashion-MNIST} & PRODEN & MLP & 89.51 $\pm$ 0.07 & 88.79 $\pm$ 0.06 & 88.32 $\pm$ 0.07 & 87.21 $\pm$ 0.13 \\
		& PRODEN+ Ours   & MLP & \textbf{90.47 $\pm$ 0.02} & \textbf{90.07 $\pm$ 0.05} & \textbf{89.38 $\pm$ 0.12} & \textbf{87.84 $\pm$ 0.13} \\\hline\hline
		\multirow{2}{*}{Kuzushiji-MNIST} & PRODEN & MLP & 91.07 $\pm$ 0.07 & 90.24 $\pm$ 0.12 & 88.31 $\pm$ 0.14 & 85.55 $\pm$ 0.58 \\
		& PRODEN+ Ours   &  MLP    & \textbf{93.07 $\pm$ 0.04} & \textbf{91.65 $\pm$ 0.03} & \textbf{88.86 $\pm$ 0.07} & \textbf{86.11 $\pm$ 0.17} \\\hline\hline
		\multirow{2}{*}{CIFAR-10} & PRODEN & ResNet-32 & 82.09 $\pm$ 0.05 & 81.70 $\pm$ 0.58 & 80.72 $\pm$ 1.08 & 76.24 $\pm$ 1.35 \\
		& PRODEN+ Ours   &  ResNet-32    & \textbf{89.77 $\pm$ 0.36} & \textbf{88.01 $\pm$ 0.27} & \textbf{86.04 $\pm$ 0.32} & \textbf{80.57 $\pm$ 1.33} \\\hline\hline
		--& -- & -- & $q=0.03$ & $q=0.05$  &$q=0.07$ & $q=0.10$  \\ \hline
		\multirow{2}{*}{CIFAR-100} & PRODEN & ResNet-32 & 48.06 $\pm$ 0.95 & 47.07 $\pm$ 1.32 & 46.49 $\pm$ 1.73 & 46.30 $\pm$ 1.98 \\
		& PRODEN+ Ours   & ResNet-32 &\textbf{ 61.22 $\pm$ 0.03} & \textbf{60.25 $\pm$ 0.17} & \textbf{59.17 $\pm$ 0.17} & \textbf{54.64 $\pm$ 0.15} \\
		\bottomrule
	\end{tabular}}
\end{table*}

\begin{table*}[t]
	\caption{Performance comparison of our method with SOTA methods trained on CIFAR-10, CIFAR-100 and ImageNet datasets in terms of test error over 3 trials. Results for all baselines are directly copied from \cite{sohn2020fixmatch}. } \label{semi}
	\centering
	\begin{tabular}{l|c|c|c|c|c|c|c}
		\toprule
		&   \multicolumn{3}{c|}{CIFAR-10}   &    \multicolumn{3}{c|}{CIFAR-100} & ImageNet\\
		\hline
		Method & 40 labels & 250 labels & 4000 labels &400 labels & 2500 labels & 10000 labels & 10\% labels\\ \hline
		$\Pi$-Model \cite{rasmus2015semi}& - & 54.26 $\pm$ 3.97 & 14.01 $\pm$ 0.38 & - & 57.25 $\pm$ 0.48 & 37.88 $\pm$ 0.11& - \\
		Pseudo-Labeling \cite{lee2013pseudo} & - & 49.78 $\pm$ 0.43 & 16.09 $\pm$ 0.28 & - & 57.38 $\pm$ 0.46 & 36.21 $\pm$ 0.19& - \\
		Mean Teacher \cite{tarvainen2017mean} & - & 32.32 $\pm$ 2.30 & 9.19 $\pm$ 0.19 & - & 53.91 $\pm$ 0.57 & 35.83 $\pm$ 0.24 & - \\
		MixMatch \cite{berthelot2019mixmatch} & 47.54 $\pm$ 11.50 & 11.05 $\pm$ 0.86 &  6.42 $\pm $0.10 & 67.61 $\pm$ 1.32 & 39.94 $\pm$ 0.37 & 28.31 $\pm$ 0.33 & - \\
		UDA \cite{xie2020unsupervised} & 29.05 $\pm$ 5.93 & 8.82 $\pm$ 1.08 & 4.88 $\pm$ 0.18 & 59.28 $\pm$ 0.88 & 33.13 $\pm$ 0.22 & 24.50 $\pm$ 0.25 & - \\
		FixMatch \cite{sohn2020fixmatch}& 13.81 $\pm$ 3.37 & 5.07 $\pm$ 0.65 & 4.26 $\pm$ 0.05 & 48.85 $\pm$ 1.75 & 28.29 $\pm$ 0.11 & 22.60 $\pm$ 0.12 & 32.9 (top1), 13.3 (top5) \\
		\hline
		FixMatch + CMW-Net & \textbf{9.6 $\pm$ 0.62} & \textbf{4.73 $\pm$ 0.15} & \textbf{4.25 $\pm$ 0.03}& \textbf{47.7 $\pm$ 1.14} & \textbf{27.43 $\pm$ 0.12} & \textbf{22.55 $\pm$ 0.09} & \textbf{30.8} (top1), \textbf{11.3} (top5) \\
		\bottomrule
	\end{tabular}
\end{table*}

\begin{table*}[t]
	\caption{Selective classification error (\%) on CIFAR-10, CIFAR-100 datasets for various coverage rates (\%) for SelectiveNet (left in each panel) and SelectiveNet+CMW-Net (right in each panel). The better result in each case is highlighted in bold.}\label{FG}
	\centering \tiny
	\setlength{\tabcolsep}{1mm}{
		\begin{tabular}{c|c|c|c|c|c|c|c|c|c|c|c|c|c }			\toprule
			\multirow{2}{*}{Dataset} & \multirow{2}{*}{Coverage} & \multicolumn{2}{c|}{200} & \multicolumn{2}{c|}{100} & \multicolumn{2}{c|}{50} & \multicolumn{2}{c|}{20} & \multicolumn{2}{c|}{10} & \multicolumn{2}{c}{0} \\ \cline{3-14}
			&      & SelectiveNet & Ours & SelectiveNet & Ours & SelectiveNet & Ours & SelectiveNet & Ours & SelectiveNet & Ours & SelectiveNet & Ours \\
			\hline
			\multirow{6}{*}{CIFAR-10} & 100 & 55.50 $\pm$ 0.44 & \textbf{55.44 $\pm$ 0.08} & 34.94 $\pm$ 0.94 & \textbf{33.31 $\pm$ 0.23} & 25.23 $\pm$ 0.48 & \textbf{22.49 $\pm$ 0.23} & 18.16 $\pm$ 0.02 & \textbf{15.15 $\pm$ 0.43} & 14.62 $\pm$ 0.57 & \textbf{12.23 $\pm$ 0.12} & 6.79 $\pm$ 0.03 & \textbf{6.02 $\pm$ 0.07} \\
			& 95 & 52.63 $\pm$ 1.48 & \textbf{52.10 $\pm$ 0.08} & 32.60 $\pm$ 0.98 & \textbf{30.95 $\pm$ 0.20} & 22.72 $\pm$ 0.49 & \textbf{21.50 $\pm$ 0.34} & 15.57 $\pm$ 0.04 & \textbf{13.21 $\pm$ 0.41} & 12.07 $\pm$ 0.53 & \textbf{9.94 $\pm$ 0.06} & 4.16 $\pm$ 0.09 & \textbf{4.00 $\pm$ 0.11} \\
			& 90 & 50.83 $\pm$ 1.34 & \textbf{50.63 $\pm$ 0.07} & 30.62 $\pm$ 0.11 & \textbf{29.49 $\pm$ 0.10} & 20.65 $\pm$ 0.49 &\textbf{19.65 $\pm$ 1.00} & 13.37 $\pm$ 0.13 & \textbf{11.41 $\pm$ 0.40} & 10.04 $\pm$ 0.41 & \textbf{8.01 $\pm$ 0.16} & 2.43 $\pm$ 0.08 & \textbf{2.29 $\pm$ 0.15} \\
			& 85 & 48.95 $\pm$ 0.99 & \textbf{47.91 $\pm$ 0.12} & 28.85 $\pm$ 0.33 & \textbf{28.21 $\pm$ 0.13} & 18.84 $\pm$ 0.43 & \textbf{17.89 $\pm$ 0.84} & 11.61 $\pm$ 0.03 & \textbf{9.62 $\pm$ 0.35} & 8.32 $\pm$ 0.21 & \textbf{6.34 $\pm$ 0.15} & 1.43 $\pm$ 0.08 & \textbf{1.17 $\pm$ 0.02} \\
			& 80 &  46.99 $\pm$ 0.76 & \textbf{45.11 $\pm$ 0.11} & 27.27 $\pm$ 0.46 & \textbf{26.32 $\pm$ 0.30} & 17.33 $\pm$ 0.32 & \textbf{16.31 $\pm$ 0.65}& 10.07 $\pm$ 0.01 & \textbf{8.03 $\pm$ 0.30} & 6.91 $\pm$ 0.12 & \textbf{4.80 $\pm$ 0.15} & 0.86 $\pm$ 0.06 &\textbf{0.80 $\pm$ 0.01} \\
			& 75 & 45.09 $\pm$ 0.62 & \textbf{42.19 $\pm$ 0.01} & 25.85 $\pm$ 0.65 & \textbf{24.33 $\pm$ 0.26} & 16.02 $\pm$ 0.30 & \textbf{14.64 $\pm$ 0.35} & 8.88 $\pm$ 0.12 & \textbf{6.52 $\pm$ 0.30} & 5.69 $\pm$ 0.11 & \textbf{3.59 $\pm$ 0.09} & 0.48 $\pm$ 0.02 & \textbf{0.55 $\pm$ 0.03} \\
			& 70 &  43.10 $\pm$ 0.47 & \textbf{39.21 $\pm$ 0.30} & 24.36 $\pm$ 0.76 & \textbf{22.36 $\pm$ 0.19} & 14.79 $\pm$ 0.29 & \textbf{13.13 $\pm$ 0.23} & 7.81 $\pm$ 0.22 & \textbf{5.29 $\pm$ 0.27} & 4.75 $\pm$ 0.16 & \textbf{2.57 $\pm$ 0.11} & 0.32 $\pm$ 0.01 & \textbf{0.35 $\pm$ 0.04} \\  \hline\hline
			\multirow{6}{*}{CIFAR-100} & 100 & 68.74 $\pm$ 0.42 & \textbf{65.62 $\pm$ 0.25} & 65.85 $\pm$ 0.15 & \textbf{60.77 $\pm$ 0.09} & 61.21 $\pm$ 0.19 & \textbf{55.70 $\pm$ 0.22} & 55.04 $\pm$ 0.39 & \textbf{46.68 $\pm$ 0.14} & 49.12 $\pm$ 0.38 & \textbf{40.46 $\pm$ 0.31} & 27.72 $\pm$ 0.35 & \textbf{25.75 $\pm$ 0.23} \\
			& 95 & 67.41 $\pm$ 0.43 & \textbf{64.13 $\pm$ 0.33 }& 64.41 $\pm$ 0.16 & \textbf{59.08 $\pm$ 0.06} & 59.55 $\pm$ 0.22 & \textbf{53.81 $\pm$ 0.20} & 53.07 $\pm$ 0.39 & \textbf{44.46 $\pm$ 0.13} & 46.99 $\pm$ 0.34 & \textbf{37.97 $\pm$ 0.33} & 24.99 $\pm$ 0.38 & \textbf{23.03 $\pm$ 0.24} \\
			& 90 &  66.09 $\pm$ 0.48 & \textbf{62.48 $\pm$ 0.28} & 63.01 $\pm$ 0.17 & \textbf{57.36 $\pm$ 0.04} & 57.81 $\pm$ 0.20 & \textbf{51.84 $\pm$ 0.20} & 51.12 $\pm$ 0.39 & \textbf{42.24 $\pm$ 0.18} & 44.89 $\pm$ 0.33 & \textbf{35.59 $\pm$ 0.27} & 22.59 $\pm$ 0.31 & \textbf{20.43 $\pm$ 0.18} \\
			& 85 & 64.65 $\pm$ 0.54 & \textbf{60.79 $\pm$ 0.22} & 61.51 $\pm$ 0.16 & \textbf{55.53 $\pm $0.04} & 56.21 $\pm$ 0.27 & \textbf{49.76 $\pm$ 0.21} & 49.24 $\pm$ 0.32 & \textbf{40.05 $\pm$ 0.13} & 42.81 $\pm$ 0.20 & \textbf{33.24 $\pm$ 0.27} & 20.31 $\pm$ 0.33 & \textbf{17.98 $\pm$ 0.21} \\
			& 80 & 63.09 $\pm$ 0.54 & \textbf{59.00 $\pm$ 0.26} & 59,85 $\pm$ 0.11 & \textbf{53.57 $\pm$ 0.09} & 54.25 $\pm$ 0.20 & \textbf{47.62 $\pm$ 0.15} & 47.15 $\pm$ 0.34 & \textbf{37.54 $\pm$ 0.26} & 40.67 $\pm$ 0.23 & \textbf{30.95 $\pm$ 0.21} & 18.17 $\pm$ 0.28 & \textbf{15.45 $\pm$ 0.01} \\
			& 75 & 61.50 $\pm$ 0.57 & \textbf{57.10 $\pm$ 0.13} & 58.11 $\pm$ 0.06 & \textbf{51.42 $\pm$ 0.18} & 52.19 $\pm$ 0.20 & \textbf{45.24 $\pm$ 0.19} & 45.03 $\pm$ 0.38 & \textbf{35.14 $\pm$ 0.28} & 38.65 $\pm$ 0.36 & \textbf{28.41 $\pm$ 0.11} & 16.32 $\pm$ 0.42 & \textbf{12.97 $\pm$ 0.03} \\
			& 70 & 59.72 $\pm$ 0.66 & \textbf{54.93 $\pm$ 0.11} & 56.21 $\pm$ 0.06 & \textbf{49.10 $\pm$ 0.16} & 50.04 $\pm$ 0.32 & \textbf{42.56 $\pm$ 0.29} & 42.77 $\pm$ 0.34 & \textbf{32.55 $\pm$ 0.33} & 36.42 $\pm$ 0.40 & \textbf{25.89 $\pm$ 0.20} & 14.63 $\pm$ 0.59 & \textbf{10.69 $\pm$ 0.05} \\
			\bottomrule
		\end{tabular}}
\end{table*}

\subsection{Selective Classification}
\subsubsection{Problem Formulation}		
We consider the selective classification problem in DNNs (supervised learning with a rejection option), which allows the learned classifier to abstain whenever they are not sufficiently confident in their prediction, so as to finely detect and control statistical uncertainties of training cases \cite{geifman2019selectivenet}. Specifically, let $P(X,Y)$ be the underlying joint distribution over $\mathcal{X}\times \mathcal{Y}$, where $\mathcal{X}, \mathcal{Y}$ denote the sample and label spaces, respectively, and $f: \mathcal{X} \rightarrow \mathcal{Y}$ be the prediction function (DNNs here). The expected risk is:
\begin{align*}
R(f) = \mathbb{E}_{P(X,Y)} [\ell(f(x),y)],
\end{align*}
where $\ell: Y\times Y\rightarrow \mathbb{R}^+$ is the loss function.
Given a dataset $D^{tr}=\{(x_i,y_i)\}_{i=1}^N$ where all $(x_i,y_i)$s are i.i.d. drawn from $X\times Y$, the empirical risk is then specified as
\begin{align*}
\hat{R}_{D^{tr}}(f) = \frac{1}{N} \sum_{i=1}^N\ell(f(x_i),y_i).
\end{align*}
The selective classifier is then defined as a pair of functions $(f,g)$, where $g: \mathcal{X} \rightarrow \mathbb{R}$ is a selection function that reveals the underlying uncertainty of inputs. Specifically, given input $x$, $(f,g)$ outputs:
\begin{align*}
(f,g)(x) = \begin{cases}
f(x),    & \text{if}\ g(x) \geq \tau\\
\text{Abstain}  &  \text{otherwise}
\end{cases},
\end{align*}
i.e., the model abstains from making a prediction when selection function $g(x)$ falls bellow a predetermined threshold $\tau$. We call $g(x)$ the uncertainty score of $x$, and different methods tend to use different $g$. The coverage is defined as the probability mass of the non-rejected region in $\mathcal{X}$, expressed as:
\begin{align*}
\phi(g) = \mathbb{E}_{P(X)} [g(x)],
\end{align*}
and its empirical coverage is
\begin{align*}
\hat{\phi}_{D^{tr}}(g) = \frac{1}{m} \sum_{i=1}^N g(x_i).
\end{align*}
The selective risk of $(f,g)$ is defined as
\begin{align*}
R(f,g) = \frac{\mathbb{E}_{P(X,Y)} [\ell(f(x),y)g(x)]}{\phi(g)},
\end{align*}
and empirical version is
\begin{align*}
\hat{R}_{D^{tr}}(f,g) = \frac{\frac{1}{N}\sum_{i=1}^N\ell(f(x_i),y_i)g(x_i)}{\hat{\phi}_{D^{tr}}(g)}.
\end{align*}

\begin{figure}[t]
	\centering
	\subfigcapskip=-1mm
	\includegraphics[width=0.4\textwidth]{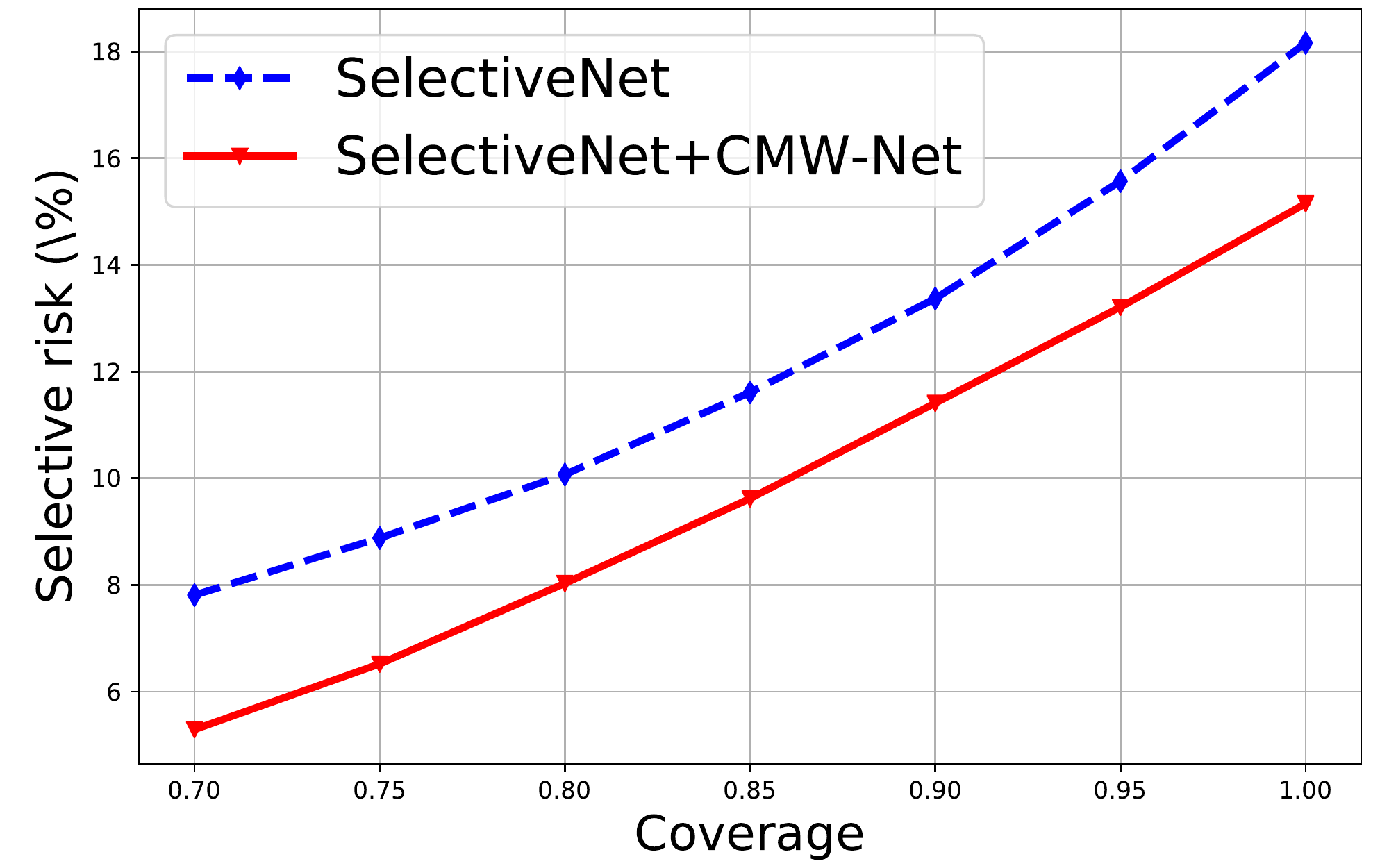} \ \ \
	\includegraphics[width=0.4\textwidth]{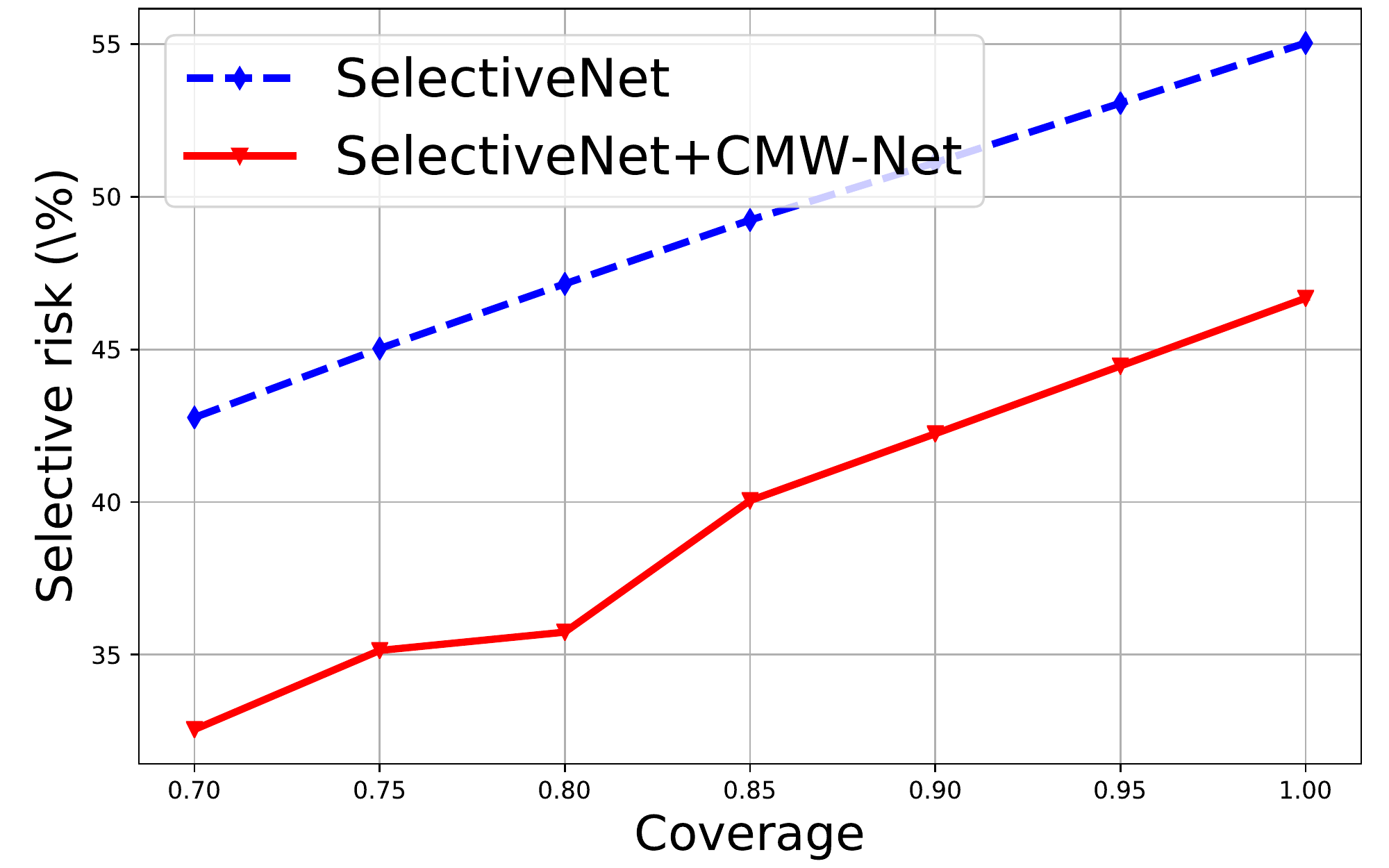}
	\caption{Risk coverage curves of SelectiveNet w/o CMW-Net strategy on (left) CIFAR-10 and (right) CIFAR100 under imbalance factor 20.}\label{fig4net}
\end{figure}

The SelectiveNet \cite{geifman2019selectivenet} tries to optimize the objective
\begin{align*}
\mathcal{L} =  \alpha \mathcal{L}_{D^{tr}}(f,g)  +(1-\alpha) \hat{R}_{D^{tr}}(h),
\end{align*}
where
\begin{align*}
\mathcal{L}_{D^{tr}}(f,g) = \hat{R}_{D^{tr}}(f,g)+ \lambda \max(0, c-\hat{\phi}_{D^{tr}}(g))^2,
\end{align*}
$c$ is the given coverage, and $\alpha,\lambda$ control the relative importance of each term. As stated in \cite{geifman2019selectivenet}, the auxiliary cross-entropy loss  $\hat{R}_{D^{tr}}(h)$ exposes the main body block to all training samples throughout the training process to avoid SelectiveNet overfitting to the wrong subset of the training set.

It can be seen that the rationality of the auxiliary cross-entropy loss $\hat{R}_{D^{tr}}(h)$ still inclines to be negatively affected by the data biased issues, like commonly existed class imbalance and noisy label cases in practical datasets. It is thus natural to employ the proposed CMW-Net on the term for making the learned SelectiveNet with better robustness to training samples.

\subsubsection{CMW-Net Amelioration and Experiments}
We can then readily ameliorate the selective classification by embedding CMW-Net weighting schemes into its optimization problem. Specifically, provided a meta dataset as
$D^{meta}=\{x_i^{meta},y_i^{meta}\}_{i=1}^M$, the objective of the problem is then reformulated as the following bi-level problem:
\begin{align*}
&{\Theta}^* = \mathop{\arg\min}_{{{\Theta} }}  \ \mathcal{L}_{D^{meta}}(f^*_{\Theta},g^*_{\Theta})\\
& \{f^*_{\Theta},g^*_{\Theta}\} = \mathop{\arg\min}_{f,g}\  \alpha \mathcal{L}_{D^{tr}}(f,g)  +(1-\alpha) \sum_{i=1}^N \mathcal{V}(\ell_i,N_i;\Theta)\ell(h(x_i),y_i).
\end{align*}
Through properly assigning sample weights to the loss terms for all training samples, it is expected to better eliminate the negative influence brought by complicated data biases.

We use long-tailed versions of CIFAR-10 and CIFAR-100 datasets under different imbalance factors for performance evaluation. The generation strategy is similar to that introduced in Sec. 4.1 of the main text. The baseline method is the recent SOTA method for this task: SelectiveNet \cite{geifman2019selectivenet}.  We use the VGG-16 network \cite{simonyan2015very} with batch normalization \cite{ioffe2015batch} and dropout \cite{srivastava2014dropout} as the classifier network in experiments. The network is optimized using SGD with initial learning rate of 0.1, momentum of 0.9, weight decay of $5\times 10^{-4}$, batch size of 128, and total training epoch of 300. The learning rate is decayed by 0.5 in every 25 epochs. As the meta-data-generation strategy as introduced in Sec. 4.2 of the main text, we randomly select 10 images per class at every epoch from the training set as the meta-data set.

The obtained experimental results are summarized in Table \ref{FG} and Fig. \ref{fig4net}. Specifically, the figure compares the risk-coverage curves of SelectiveNet equipped with and without CMW-Net for weighting its sample loss. It is easy to see the performance gain brought to the method by CMW-Net. From the figure, one can more comprehensively observe that selective classification errors of SelectiveNet consistently grow as we increase the degree of class imbalance. Comparatively, under the assistance of CMW-Net, the errors can be consistently reduced in all cases. These results validate the usefulness of CMW-Net for this specific learning task under biased data.

\newpage
\bibliographystyle{IEEEtran}
\bibliography{IEEEabrv,bayes_ref.bib}

\end{document}